%% file: main.tex
\documentclass[10pt,twocolumn,letterpaper]{article}

\usepackage[pagenumbers]{cvpr}

\input{preamble}
\input{tables/arxiv/numbers}

\definecolor{cvprblue}{rgb}{0.21,0.49,0.74}
\usepackage[pagebackref,breaklinks,colorlinks,allcolors=cvprblue]{hyperref}
\hypersetup{
  pdftitle={Teaching Vision-Language Models to Decline: Support Completion Against Mirage},
  pdfauthor={Alfaj Uddin Ahmed; Md. Shaown Miah; Md Arid Hasan; Syed Ishtiaque Ahmed; Sayeed Shafayet Chowdhury},
  pdfsubject={Vision-language models, mirage, support completion, abstention},
  pdfkeywords={vision-language models, mirage, hallucination, abstention, support completion, TIDE}
}

\def\paperID{}

\def\confName{CVPR}
\def\confYear{2027}

\title{Teaching Vision-Language Models to Decline:\\Support Completion Against Mirage}

\author{%
  Alfaj Uddin Ahmed$^{1}$ \qquad Md. Shaown Miah$^{2}$ \qquad Md Arid Hasan$^{4}$\\[2pt]
  Syed Ishtiaque Ahmed$^{4}$ \qquad Sayeed Shafayet Chowdhury$^{2,*}$\\[6pt]
  {\small $^{1}$Department of Biomedical Engineering, Purdue University, West Lafayette, IN, USA}\\
  {\small $^{2}$Department of Computer Science, Luddy School of Informatics, Computing, and Engineering,}\\
  {\small Indiana University Indianapolis, Indianapolis, IN, USA}\\
  {\small $^{4}$Department of Computer Science, University of Toronto, Toronto, ON, Canada}\\[3pt]
  {\small $^{*}$Corresponding author: \texttt{saychow@iu.edu}}%
}

\begin{document}
\maketitle
\input{sec/0_abstract}
\input{sec/1_intro}
\input{sec/2_related}
\input{sec/3_problem}
\input{sec/4_tide}
\input{sec/5_protocol}
\input{sec/6_results}

\input{sec/7_ablation}

\input{sec/8_mechanism}

\input{sec/9_limitations}
{
    \small
    \bibliographystyle{ieeenat_fullname}
    \bibliography{main}
}

\input{sec/X_suppl}

\end{document}

%% file: preamble.tex
\usepackage{graphicx}
\usepackage{booktabs}
\usepackage{multirow}
\usepackage{amsmath}
\usepackage{amssymb}
\usepackage{amsthm}
\usepackage{xcolor}
\theoremstyle{plain}\newtheorem{proposition}{Proposition}\newtheorem{corollary}{Corollary}
\theoremstyle{remark}

\newcommand{\Dc}{D_{4}}
\newcommand{\dec}[1]{\mathsf{decline}_{#1}}
\newcommand{\ans}{\mathsf{answer}}
\newcommand{\unmeasured}[1]{\textcolor{red}{\textbf{[UNMEASURED: #1]}}}

\newcommand{\rv}[1]{\ifcsname regval:#1\endcsname\csname regval:#1\endcsname\else\textcolor{red}{\textbf{[??#1]}}\fi}

\usepackage{microtype}

%% file: tables/arxiv/numbers.tex
\expandafter\def\csname regval:M-e0-U-imr\endcsname{0.759}
\expandafter\def\csname regval:M-e0-U-imr.n\endcsname{888}
\expandafter\def\csname regval:M-e0-U-imr.d\endcsname{1{,}170}
\expandafter\def\csname regval:M-e0-U-imr.ci\endcsname{[0.734, 0.783]}
\expandafter\def\csname regval:M-e0-U-tclia\endcsname{0.569}
\expandafter\def\csname regval:M-e0-U-tclia.n\endcsname{888}
\expandafter\def\csname regval:M-e0-U-tclia.d\endcsname{1{,}560}
\expandafter\def\csname regval:M-e0-U-blank\endcsname{0.700}
\expandafter\def\csname regval:M-e0-U-blank.n\endcsname{273}
\expandafter\def\csname regval:M-e0-U-blank.d\endcsname{390}
\expandafter\def\csname regval:M-e0-U-blank.ci\endcsname{[0.653, 0.743]}
\expandafter\def\csname regval:M-e0-U-id_unrelated\endcsname{0.813}
\expandafter\def\csname regval:M-e0-U-id_unrelated.n\endcsname{317}
\expandafter\def\csname regval:M-e0-U-id_unrelated.d\endcsname{390}
\expandafter\def\csname regval:M-e0-U-id_unrelated.ci\endcsname{[0.771, 0.848]}
\expandafter\def\csname regval:M-e0-U-od_unrelated\endcsname{0.764}
\expandafter\def\csname regval:M-e0-U-od_unrelated.n\endcsname{298}
\expandafter\def\csname regval:M-e0-U-od_unrelated.d\endcsname{390}
\expandafter\def\csname regval:M-e0-U-od_unrelated.ci\endcsname{[0.720, 0.803]}
\expandafter\def\csname regval:M-e0-U-imr-chest_vqa\endcsname{0.632}
\expandafter\def\csname regval:M-e0-U-imr-chest_vqa.n\endcsname{74}
\expandafter\def\csname regval:M-e0-U-imr-chest_vqa.d\endcsname{117}
\expandafter\def\csname regval:M-e0-U-imr-docvqa\endcsname{0.581}
\expandafter\def\csname regval:M-e0-U-imr-docvqa.n\endcsname{68}
\expandafter\def\csname regval:M-e0-U-imr-docvqa.d\endcsname{117}
\expandafter\def\csname regval:M-e0-U-imr-infovqa\endcsname{0.812}
\expandafter\def\csname regval:M-e0-U-imr-infovqa.n\endcsname{95}
\expandafter\def\csname regval:M-e0-U-imr-infovqa.d\endcsname{117}
\expandafter\def\csname regval:M-e0-U-imr-medxpertqa_mm\endcsname{0.889}
\expandafter\def\csname regval:M-e0-U-imr-medxpertqa_mm.n\endcsname{104}
\expandafter\def\csname regval:M-e0-U-imr-medxpertqa_mm.d\endcsname{117}
\expandafter\def\csname regval:M-e0-U-imr-microvqa\endcsname{0.923}
\expandafter\def\csname regval:M-e0-U-imr-microvqa.n\endcsname{108}
\expandafter\def\csname regval:M-e0-U-imr-microvqa.d\endcsname{117}
\expandafter\def\csname regval:M-e0-U-imr-mmmu_pro\endcsname{0.786}
\expandafter\def\csname regval:M-e0-U-imr-mmmu_pro.n\endcsname{92}
\expandafter\def\csname regval:M-e0-U-imr-mmmu_pro.d\endcsname{117}
\expandafter\def\csname regval:M-e0-U-imr-pathvqa\endcsname{0.632}
\expandafter\def\csname regval:M-e0-U-imr-pathvqa.n\endcsname{74}
\expandafter\def\csname regval:M-e0-U-imr-pathvqa.d\endcsname{117}
\expandafter\def\csname regval:M-e0-U-imr-rexvqa\endcsname{0.761}
\expandafter\def\csname regval:M-e0-U-imr-rexvqa.n\endcsname{89}
\expandafter\def\csname regval:M-e0-U-imr-rexvqa.d\endcsname{117}
\expandafter\def\csname regval:M-e0-U-imr-textvqa\endcsname{0.778}
\expandafter\def\csname regval:M-e0-U-imr-textvqa.n\endcsname{91}
\expandafter\def\csname regval:M-e0-U-imr-textvqa.d\endcsname{117}
\expandafter\def\csname regval:M-e0-U-imr-video_mmmu\endcsname{0.795}
\expandafter\def\csname regval:M-e0-U-imr-video_mmmu.n\endcsname{93}
\expandafter\def\csname regval:M-e0-U-imr-video_mmmu.d\endcsname{117}
\expandafter\def\csname regval:M-e0-B-iaimr\endcsname{0.431}
\expandafter\def\csname regval:M-e0-B-iaimr.n\endcsname{504}
\expandafter\def\csname regval:M-e0-B-iaimr.d\endcsname{1{,}170}
\expandafter\def\csname regval:M-e0-B-iaimr.ci\endcsname{[0.403, 0.459]}
\expandafter\def\csname regval:M-e0-B-iafmr\endcsname{0.323}
\expandafter\def\csname regval:M-e0-B-iafmr.n\endcsname{504}
\expandafter\def\csname regval:M-e0-B-iafmr.d\endcsname{1{,}560}
\expandafter\def\csname regval:M-e0-B-blank\endcsname{0.718}
\expandafter\def\csname regval:M-e0-B-blank.n\endcsname{280}
\expandafter\def\csname regval:M-e0-B-blank.d\endcsname{390}
\expandafter\def\csname regval:M-e0-B-blank.ci\endcsname{[0.671, 0.760]}
\expandafter\def\csname regval:M-e0-B-id_unrelated\endcsname{0.331}
\expandafter\def\csname regval:M-e0-B-id_unrelated.n\endcsname{129}
\expandafter\def\csname regval:M-e0-B-id_unrelated.d\endcsname{390}
\expandafter\def\csname regval:M-e0-B-id_unrelated.ci\endcsname{[0.286, 0.379]}
\expandafter\def\csname regval:M-e0-B-od_unrelated\endcsname{0.244}
\expandafter\def\csname regval:M-e0-B-od_unrelated.n\endcsname{95}
\expandafter\def\csname regval:M-e0-B-od_unrelated.d\endcsname{390}
\expandafter\def\csname regval:M-e0-B-od_unrelated.ci\endcsname{[0.204, 0.289]}
\expandafter\def\csname regval:M-e0-B-iaimr-chest_vqa\endcsname{0.256}
\expandafter\def\csname regval:M-e0-B-iaimr-chest_vqa.n\endcsname{30}
\expandafter\def\csname regval:M-e0-B-iaimr-chest_vqa.d\endcsname{117}
\expandafter\def\csname regval:M-e0-B-iaimr-docvqa\endcsname{0.137}
\expandafter\def\csname regval:M-e0-B-iaimr-docvqa.n\endcsname{16}
\expandafter\def\csname regval:M-e0-B-iaimr-docvqa.d\endcsname{117}
\expandafter\def\csname regval:M-e0-B-iaimr-infovqa\endcsname{0.282}
\expandafter\def\csname regval:M-e0-B-iaimr-infovqa.n\endcsname{33}
\expandafter\def\csname regval:M-e0-B-iaimr-infovqa.d\endcsname{117}
\expandafter\def\csname regval:M-e0-B-iaimr-medxpertqa_mm\endcsname{0.641}
\expandafter\def\csname regval:M-e0-B-iaimr-medxpertqa_mm.n\endcsname{75}
\expandafter\def\csname regval:M-e0-B-iaimr-medxpertqa_mm.d\endcsname{117}
\expandafter\def\csname regval:M-e0-B-iaimr-microvqa\endcsname{0.829}
\expandafter\def\csname regval:M-e0-B-iaimr-microvqa.n\endcsname{97}
\expandafter\def\csname regval:M-e0-B-iaimr-microvqa.d\endcsname{117}
\expandafter\def\csname regval:M-e0-B-iaimr-mmmu_pro\endcsname{0.538}
\expandafter\def\csname regval:M-e0-B-iaimr-mmmu_pro.n\endcsname{63}
\expandafter\def\csname regval:M-e0-B-iaimr-mmmu_pro.d\endcsname{117}
\expandafter\def\csname regval:M-e0-B-iaimr-pathvqa\endcsname{0.222}
\expandafter\def\csname regval:M-e0-B-iaimr-pathvqa.n\endcsname{26}
\expandafter\def\csname regval:M-e0-B-iaimr-pathvqa.d\endcsname{117}
\expandafter\def\csname regval:M-e0-B-iaimr-rexvqa\endcsname{0.632}
\expandafter\def\csname regval:M-e0-B-iaimr-rexvqa.n\endcsname{74}
\expandafter\def\csname regval:M-e0-B-iaimr-rexvqa.d\endcsname{117}
\expandafter\def\csname regval:M-e0-B-iaimr-textvqa\endcsname{0.231}
\expandafter\def\csname regval:M-e0-B-iaimr-textvqa.n\endcsname{27}
\expandafter\def\csname regval:M-e0-B-iaimr-textvqa.d\endcsname{117}
\expandafter\def\csname regval:M-e0-B-iaimr-video_mmmu\endcsname{0.538}
\expandafter\def\csname regval:M-e0-B-iaimr-video_mmmu.n\endcsname{63}
\expandafter\def\csname regval:M-e0-B-iaimr-video_mmmu.d\endcsname{117}
\expandafter\def\csname regval:M-e0-B-acc\endcsname{0.577}
\expandafter\def\csname regval:M-e0-B-acc.n\endcsname{225}
\expandafter\def\csname regval:M-e0-B-acc.d\endcsname{390}
\expandafter\def\csname regval:M-e0-B-acc.ci\endcsname{[0.527, 0.625]}
\expandafter\def\csname regval:M-e0-B-answered\endcsname{0.982}
\expandafter\def\csname regval:M-e0-B-answered.n\endcsname{383}
\expandafter\def\csname regval:M-e0-B-answered.d\endcsname{390}
\expandafter\def\csname regval:M-e0-B-fa\endcsname{0.018}
\expandafter\def\csname regval:M-e0-B-fa.n\endcsname{7}
\expandafter\def\csname regval:M-e0-B-fa.d\endcsname{390}
\expandafter\def\csname regval:M-e0-B-acc-cond\endcsname{0.588}
\expandafter\def\csname regval:M-e0-B-acc-cond.n\endcsname{225}
\expandafter\def\csname regval:M-e0-B-acc-cond.d\endcsname{383}
\expandafter\def\csname regval:M-e0-B-unparseable\endcsname{0.000}
\expandafter\def\csname regval:M-e0-B-unparseable.n\endcsname{0}
\expandafter\def\csname regval:M-e0-B-unparseable.d\endcsname{1{,}560}
\expandafter\def\csname regval:M-e0-B-probe-real-P1\endcsname{0.982}
\expandafter\def\csname regval:M-e0-B-probe-real-P1.d\endcsname{390}
\expandafter\def\csname regval:M-e0-B-probe-real-P2\endcsname{0.082}
\expandafter\def\csname regval:M-e0-B-probe-real-P2.d\endcsname{390}
\expandafter\def\csname regval:M-e0-B-probe-real-P3\endcsname{0.949}
\expandafter\def\csname regval:M-e0-B-probe-real-P3.d\endcsname{390}
\expandafter\def\csname regval:M-e0-B-probe-real-P4\endcsname{0.867}
\expandafter\def\csname regval:M-e0-B-probe-real-P4.d\endcsname{390}
\expandafter\def\csname regval:M-e0-B-probe-blank-P1\endcsname{0.756}
\expandafter\def\csname regval:M-e0-B-probe-blank-P1.d\endcsname{390}
\expandafter\def\csname regval:M-e0-B-probe-blank-P2\endcsname{0.349}
\expandafter\def\csname regval:M-e0-B-probe-blank-P2.d\endcsname{390}
\expandafter\def\csname regval:M-e0-B-probe-blank-P4\endcsname{0.564}
\expandafter\def\csname regval:M-e0-B-probe-blank-P4.d\endcsname{390}
\expandafter\def\csname regval:M-e0-B-probe-id_unrelated-P1\endcsname{0.695}
\expandafter\def\csname regval:M-e0-B-probe-id_unrelated-P1.d\endcsname{390}
\expandafter\def\csname regval:M-e0-B-probe-id_unrelated-P2\endcsname{0.487}
\expandafter\def\csname regval:M-e0-B-probe-id_unrelated-P2.d\endcsname{390}
\expandafter\def\csname regval:M-e0-B-probe-id_unrelated-P3\endcsname{0.749}
\expandafter\def\csname regval:M-e0-B-probe-id_unrelated-P3.d\endcsname{390}
\expandafter\def\csname regval:M-e0-B-probe-id_unrelated-P4\endcsname{0.272}
\expandafter\def\csname regval:M-e0-B-probe-id_unrelated-P4.d\endcsname{390}
\expandafter\def\csname regval:M-e0-B-probe-od_unrelated-P1\endcsname{0.595}
\expandafter\def\csname regval:M-e0-B-probe-od_unrelated-P1.d\endcsname{390}
\expandafter\def\csname regval:M-e0-B-probe-od_unrelated-P2\endcsname{0.659}
\expandafter\def\csname regval:M-e0-B-probe-od_unrelated-P2.d\endcsname{390}
\expandafter\def\csname regval:M-e0-B-probe-od_unrelated-P3\endcsname{0.741}
\expandafter\def\csname regval:M-e0-B-probe-od_unrelated-P3.d\endcsname{390}
\expandafter\def\csname regval:M-e0-B-probe-od_unrelated-P4\endcsname{0.164}
\expandafter\def\csname regval:M-e0-B-probe-od_unrelated-P4.d\endcsname{390}
\expandafter\def\csname regval:M-e0-B-opdec-F\endcsname{0.018}
\expandafter\def\csname regval:M-e0-B-opdec-F.d\endcsname{390}
\expandafter\def\csname regval:M-e0-B-opdec-M\endcsname{0.682}
\expandafter\def\csname regval:M-e0-B-opdec-M.d\endcsname{1{,}170}
\expandafter\def\csname regval:M-e0-B-opdec-sum\endcsname{0.516}
\expandafter\def\csname regval:M-e0-B-opdec-sum.d\endcsname{1{,}560}
\expandafter\def\csname regval:M-e0-B2-iaimr\endcsname{0.493}
\expandafter\def\csname regval:M-e0-B2-iaimr.n\endcsname{577}
\expandafter\def\csname regval:M-e0-B2-iaimr.d\endcsname{1{,}170}
\expandafter\def\csname regval:M-e0-B2-iaimr.ci\endcsname{[0.465, 0.522]}
\expandafter\def\csname regval:M-e0-B2-iafmr\endcsname{0.370}
\expandafter\def\csname regval:M-e0-B2-iafmr.n\endcsname{577}
\expandafter\def\csname regval:M-e0-B2-iafmr.d\endcsname{1{,}560}
\expandafter\def\csname regval:M-e0-B2-blank\endcsname{0.580}
\expandafter\def\csname regval:M-e0-B2-blank.n\endcsname{226}
\expandafter\def\csname regval:M-e0-B2-blank.d\endcsname{390}
\expandafter\def\csname regval:M-e0-B2-blank.ci\endcsname{[0.530, 0.627]}
\expandafter\def\csname regval:M-e0-B2-id_unrelated\endcsname{0.526}
\expandafter\def\csname regval:M-e0-B2-id_unrelated.n\endcsname{205}
\expandafter\def\csname regval:M-e0-B2-id_unrelated.d\endcsname{390}
\expandafter\def\csname regval:M-e0-B2-id_unrelated.ci\endcsname{[0.476, 0.575]}
\expandafter\def\csname regval:M-e0-B2-od_unrelated\endcsname{0.374}
\expandafter\def\csname regval:M-e0-B2-od_unrelated.n\endcsname{146}
\expandafter\def\csname regval:M-e0-B2-od_unrelated.d\endcsname{390}
\expandafter\def\csname regval:M-e0-B2-od_unrelated.ci\endcsname{[0.328, 0.423]}
\expandafter\def\csname regval:M-e0-B2-iaimr-chest_vqa\endcsname{0.462}
\expandafter\def\csname regval:M-e0-B2-iaimr-chest_vqa.n\endcsname{54}
\expandafter\def\csname regval:M-e0-B2-iaimr-chest_vqa.d\endcsname{117}
\expandafter\def\csname regval:M-e0-B2-iaimr-docvqa\endcsname{0.179}
\expandafter\def\csname regval:M-e0-B2-iaimr-docvqa.n\endcsname{21}
\expandafter\def\csname regval:M-e0-B2-iaimr-docvqa.d\endcsname{117}
\expandafter\def\csname regval:M-e0-B2-iaimr-infovqa\endcsname{0.256}
\expandafter\def\csname regval:M-e0-B2-iaimr-infovqa.n\endcsname{30}
\expandafter\def\csname regval:M-e0-B2-iaimr-infovqa.d\endcsname{117}
\expandafter\def\csname regval:M-e0-B2-iaimr-medxpertqa_mm\endcsname{0.692}
\expandafter\def\csname regval:M-e0-B2-iaimr-medxpertqa_mm.n\endcsname{81}
\expandafter\def\csname regval:M-e0-B2-iaimr-medxpertqa_mm.d\endcsname{117}
\expandafter\def\csname regval:M-e0-B2-iaimr-microvqa\endcsname{0.821}
\expandafter\def\csname regval:M-e0-B2-iaimr-microvqa.n\endcsname{96}
\expandafter\def\csname regval:M-e0-B2-iaimr-microvqa.d\endcsname{117}
\expandafter\def\csname regval:M-e0-B2-iaimr-mmmu_pro\endcsname{0.521}
\expandafter\def\csname regval:M-e0-B2-iaimr-mmmu_pro.n\endcsname{61}
\expandafter\def\csname regval:M-e0-B2-iaimr-mmmu_pro.d\endcsname{117}
\expandafter\def\csname regval:M-e0-B2-iaimr-pathvqa\endcsname{0.436}
\expandafter\def\csname regval:M-e0-B2-iaimr-pathvqa.n\endcsname{51}
\expandafter\def\csname regval:M-e0-B2-iaimr-pathvqa.d\endcsname{117}
\expandafter\def\csname regval:M-e0-B2-iaimr-rexvqa\endcsname{0.650}
\expandafter\def\csname regval:M-e0-B2-iaimr-rexvqa.n\endcsname{76}
\expandafter\def\csname regval:M-e0-B2-iaimr-rexvqa.d\endcsname{117}
\expandafter\def\csname regval:M-e0-B2-iaimr-textvqa\endcsname{0.325}
\expandafter\def\csname regval:M-e0-B2-iaimr-textvqa.n\endcsname{38}
\expandafter\def\csname regval:M-e0-B2-iaimr-textvqa.d\endcsname{117}
\expandafter\def\csname regval:M-e0-B2-iaimr-video_mmmu\endcsname{0.590}
\expandafter\def\csname regval:M-e0-B2-iaimr-video_mmmu.n\endcsname{69}
\expandafter\def\csname regval:M-e0-B2-iaimr-video_mmmu.d\endcsname{117}
\expandafter\def\csname regval:M-e0-B2-acc\endcsname{0.528}
\expandafter\def\csname regval:M-e0-B2-acc.n\endcsname{206}
\expandafter\def\csname regval:M-e0-B2-acc.d\endcsname{390}
\expandafter\def\csname regval:M-e0-B2-acc.ci\endcsname{[0.479, 0.577]}
\expandafter\def\csname regval:M-e0-B2-answered\endcsname{0.941}
\expandafter\def\csname regval:M-e0-B2-answered.n\endcsname{367}
\expandafter\def\csname regval:M-e0-B2-answered.d\endcsname{390}
\expandafter\def\csname regval:M-e0-B2-fa\endcsname{0.059}
\expandafter\def\csname regval:M-e0-B2-fa.n\endcsname{23}
\expandafter\def\csname regval:M-e0-B2-fa.d\endcsname{390}
\expandafter\def\csname regval:M-e0-B2-acc-cond\endcsname{0.561}
\expandafter\def\csname regval:M-e0-B2-acc-cond.n\endcsname{206}
\expandafter\def\csname regval:M-e0-B2-acc-cond.d\endcsname{367}
\expandafter\def\csname regval:M-e0-B2-unparseable\endcsname{0.000}
\expandafter\def\csname regval:M-e0-B2-unparseable.n\endcsname{0}
\expandafter\def\csname regval:M-e0-B2-unparseable.d\endcsname{1{,}560}
\expandafter\def\csname regval:M-e1-U-imr\endcsname{0.097}
\expandafter\def\csname regval:M-e1-U-imr.n\endcsname{113}
\expandafter\def\csname regval:M-e1-U-imr.d\endcsname{1{,}170}
\expandafter\def\csname regval:M-e1-U-imr.ci\endcsname{[0.081, 0.115]}
\expandafter\def\csname regval:M-e1-U-tclia\endcsname{0.072}
\expandafter\def\csname regval:M-e1-U-tclia.n\endcsname{113}
\expandafter\def\csname regval:M-e1-U-tclia.d\endcsname{1{,}560}
\expandafter\def\csname regval:M-e1-U-blank\endcsname{0.003}
\expandafter\def\csname regval:M-e1-U-blank.n\endcsname{1}
\expandafter\def\csname regval:M-e1-U-blank.d\endcsname{390}
\expandafter\def\csname regval:M-e1-U-blank.ci\endcsname{[0.001, 0.014]}
\expandafter\def\csname regval:M-e1-U-id_unrelated\endcsname{0.233}
\expandafter\def\csname regval:M-e1-U-id_unrelated.n\endcsname{91}
\expandafter\def\csname regval:M-e1-U-id_unrelated.d\endcsname{390}
\expandafter\def\csname regval:M-e1-U-id_unrelated.ci\endcsname{[0.194, 0.278]}
\expandafter\def\csname regval:M-e1-U-od_unrelated\endcsname{0.054}
\expandafter\def\csname regval:M-e1-U-od_unrelated.n\endcsname{21}
\expandafter\def\csname regval:M-e1-U-od_unrelated.d\endcsname{390}
\expandafter\def\csname regval:M-e1-U-od_unrelated.ci\endcsname{[0.035, 0.081]}
\expandafter\def\csname regval:M-e1-U-imr-chest_vqa\endcsname{0.231}
\expandafter\def\csname regval:M-e1-U-imr-chest_vqa.n\endcsname{27}
\expandafter\def\csname regval:M-e1-U-imr-chest_vqa.d\endcsname{117}
\expandafter\def\csname regval:M-e1-U-imr-docvqa\endcsname{0.043}
\expandafter\def\csname regval:M-e1-U-imr-docvqa.n\endcsname{5}
\expandafter\def\csname regval:M-e1-U-imr-docvqa.d\endcsname{117}
\expandafter\def\csname regval:M-e1-U-imr-infovqa\endcsname{0.000}
\expandafter\def\csname regval:M-e1-U-imr-infovqa.n\endcsname{0}
\expandafter\def\csname regval:M-e1-U-imr-infovqa.d\endcsname{117}
\expandafter\def\csname regval:M-e1-U-imr-medxpertqa_mm\endcsname{0.094}
\expandafter\def\csname regval:M-e1-U-imr-medxpertqa_mm.n\endcsname{11}
\expandafter\def\csname regval:M-e1-U-imr-medxpertqa_mm.d\endcsname{117}
\expandafter\def\csname regval:M-e1-U-imr-microvqa\endcsname{0.103}
\expandafter\def\csname regval:M-e1-U-imr-microvqa.n\endcsname{12}
\expandafter\def\csname regval:M-e1-U-imr-microvqa.d\endcsname{117}
\expandafter\def\csname regval:M-e1-U-imr-mmmu_pro\endcsname{0.043}
\expandafter\def\csname regval:M-e1-U-imr-mmmu_pro.n\endcsname{5}
\expandafter\def\csname regval:M-e1-U-imr-mmmu_pro.d\endcsname{117}
\expandafter\def\csname regval:M-e1-U-imr-pathvqa\endcsname{0.128}
\expandafter\def\csname regval:M-e1-U-imr-pathvqa.n\endcsname{15}
\expandafter\def\csname regval:M-e1-U-imr-pathvqa.d\endcsname{117}
\expandafter\def\csname regval:M-e1-U-imr-rexvqa\endcsname{0.205}
\expandafter\def\csname regval:M-e1-U-imr-rexvqa.n\endcsname{24}
\expandafter\def\csname regval:M-e1-U-imr-rexvqa.d\endcsname{117}
\expandafter\def\csname regval:M-e1-U-imr-textvqa\endcsname{0.026}
\expandafter\def\csname regval:M-e1-U-imr-textvqa.n\endcsname{3}
\expandafter\def\csname regval:M-e1-U-imr-textvqa.d\endcsname{117}
\expandafter\def\csname regval:M-e1-U-imr-video_mmmu\endcsname{0.094}
\expandafter\def\csname regval:M-e1-U-imr-video_mmmu.n\endcsname{11}
\expandafter\def\csname regval:M-e1-U-imr-video_mmmu.d\endcsname{117}
\expandafter\def\csname regval:M-e1-B-iaimr\endcsname{0.088}
\expandafter\def\csname regval:M-e1-B-iaimr.n\endcsname{103}
\expandafter\def\csname regval:M-e1-B-iaimr.d\endcsname{1{,}170}
\expandafter\def\csname regval:M-e1-B-iaimr.ci\endcsname{[0.073, 0.106]}
\expandafter\def\csname regval:M-e1-B-iafmr\endcsname{0.066}
\expandafter\def\csname regval:M-e1-B-iafmr.n\endcsname{103}
\expandafter\def\csname regval:M-e1-B-iafmr.d\endcsname{1{,}560}
\expandafter\def\csname regval:M-e1-B-blank\endcsname{0.003}
\expandafter\def\csname regval:M-e1-B-blank.n\endcsname{1}
\expandafter\def\csname regval:M-e1-B-blank.d\endcsname{390}
\expandafter\def\csname regval:M-e1-B-blank.ci\endcsname{[0.001, 0.014]}
\expandafter\def\csname regval:M-e1-B-id_unrelated\endcsname{0.208}
\expandafter\def\csname regval:M-e1-B-id_unrelated.n\endcsname{81}
\expandafter\def\csname regval:M-e1-B-id_unrelated.d\endcsname{390}
\expandafter\def\csname regval:M-e1-B-id_unrelated.ci\endcsname{[0.170, 0.251]}
\expandafter\def\csname regval:M-e1-B-od_unrelated\endcsname{0.054}
\expandafter\def\csname regval:M-e1-B-od_unrelated.n\endcsname{21}
\expandafter\def\csname regval:M-e1-B-od_unrelated.d\endcsname{390}
\expandafter\def\csname regval:M-e1-B-od_unrelated.ci\endcsname{[0.035, 0.081]}
\expandafter\def\csname regval:M-e1-B-iaimr-chest_vqa\endcsname{0.171}
\expandafter\def\csname regval:M-e1-B-iaimr-chest_vqa.n\endcsname{20}
\expandafter\def\csname regval:M-e1-B-iaimr-chest_vqa.d\endcsname{117}
\expandafter\def\csname regval:M-e1-B-iaimr-docvqa\endcsname{0.034}
\expandafter\def\csname regval:M-e1-B-iaimr-docvqa.n\endcsname{4}
\expandafter\def\csname regval:M-e1-B-iaimr-docvqa.d\endcsname{117}
\expandafter\def\csname regval:M-e1-B-iaimr-infovqa\endcsname{0.000}
\expandafter\def\csname regval:M-e1-B-iaimr-infovqa.n\endcsname{0}
\expandafter\def\csname regval:M-e1-B-iaimr-infovqa.d\endcsname{117}
\expandafter\def\csname regval:M-e1-B-iaimr-medxpertqa_mm\endcsname{0.094}
\expandafter\def\csname regval:M-e1-B-iaimr-medxpertqa_mm.n\endcsname{11}
\expandafter\def\csname regval:M-e1-B-iaimr-medxpertqa_mm.d\endcsname{117}
\expandafter\def\csname regval:M-e1-B-iaimr-microvqa\endcsname{0.103}
\expandafter\def\csname regval:M-e1-B-iaimr-microvqa.n\endcsname{12}
\expandafter\def\csname regval:M-e1-B-iaimr-microvqa.d\endcsname{117}
\expandafter\def\csname regval:M-e1-B-iaimr-mmmu_pro\endcsname{0.043}
\expandafter\def\csname regval:M-e1-B-iaimr-mmmu_pro.n\endcsname{5}
\expandafter\def\csname regval:M-e1-B-iaimr-mmmu_pro.d\endcsname{117}
\expandafter\def\csname regval:M-e1-B-iaimr-pathvqa\endcsname{0.128}
\expandafter\def\csname regval:M-e1-B-iaimr-pathvqa.n\endcsname{15}
\expandafter\def\csname regval:M-e1-B-iaimr-pathvqa.d\endcsname{117}
\expandafter\def\csname regval:M-e1-B-iaimr-rexvqa\endcsname{0.205}
\expandafter\def\csname regval:M-e1-B-iaimr-rexvqa.n\endcsname{24}
\expandafter\def\csname regval:M-e1-B-iaimr-rexvqa.d\endcsname{117}
\expandafter\def\csname regval:M-e1-B-iaimr-textvqa\endcsname{0.009}
\expandafter\def\csname regval:M-e1-B-iaimr-textvqa.n\endcsname{1}
\expandafter\def\csname regval:M-e1-B-iaimr-textvqa.d\endcsname{117}
\expandafter\def\csname regval:M-e1-B-iaimr-video_mmmu\endcsname{0.094}
\expandafter\def\csname regval:M-e1-B-iaimr-video_mmmu.n\endcsname{11}
\expandafter\def\csname regval:M-e1-B-iaimr-video_mmmu.d\endcsname{117}
\expandafter\def\csname regval:M-e1-B-acc\endcsname{0.559}
\expandafter\def\csname regval:M-e1-B-acc.n\endcsname{218}
\expandafter\def\csname regval:M-e1-B-acc.d\endcsname{390}
\expandafter\def\csname regval:M-e1-B-acc.ci\endcsname{[0.509, 0.607]}
\expandafter\def\csname regval:M-e1-B-answered\endcsname{0.874}
\expandafter\def\csname regval:M-e1-B-answered.n\endcsname{341}
\expandafter\def\csname regval:M-e1-B-answered.d\endcsname{390}
\expandafter\def\csname regval:M-e1-B-fa\endcsname{0.126}
\expandafter\def\csname regval:M-e1-B-fa.n\endcsname{49}
\expandafter\def\csname regval:M-e1-B-fa.d\endcsname{390}
\expandafter\def\csname regval:M-e1-B-acc-cond\endcsname{0.639}
\expandafter\def\csname regval:M-e1-B-acc-cond.n\endcsname{218}
\expandafter\def\csname regval:M-e1-B-acc-cond.d\endcsname{341}
\expandafter\def\csname regval:M-e1-B-unparseable\endcsname{0.000}
\expandafter\def\csname regval:M-e1-B-unparseable.n\endcsname{0}
\expandafter\def\csname regval:M-e1-B-unparseable.d\endcsname{1{,}560}
\expandafter\def\csname regval:M-e1-B-probe-real-P1\endcsname{0.874}
\expandafter\def\csname regval:M-e1-B-probe-real-P1.d\endcsname{390}
\expandafter\def\csname regval:M-e1-B-probe-real-P2\endcsname{0.126}
\expandafter\def\csname regval:M-e1-B-probe-real-P2.d\endcsname{390}
\expandafter\def\csname regval:M-e1-B-probe-real-P3\endcsname{0.726}
\expandafter\def\csname regval:M-e1-B-probe-real-P3.d\endcsname{390}
\expandafter\def\csname regval:M-e1-B-probe-real-P4\endcsname{0.451}
\expandafter\def\csname regval:M-e1-B-probe-real-P4.d\endcsname{390}
\expandafter\def\csname regval:M-e1-B-probe-blank-P1\endcsname{0.003}
\expandafter\def\csname regval:M-e1-B-probe-blank-P1.d\endcsname{390}
\expandafter\def\csname regval:M-e1-B-probe-blank-P2\endcsname{0.997}
\expandafter\def\csname regval:M-e1-B-probe-blank-P2.d\endcsname{390}
\expandafter\def\csname regval:M-e1-B-probe-blank-P4\endcsname{0.000}
\expandafter\def\csname regval:M-e1-B-probe-blank-P4.d\endcsname{390}
\expandafter\def\csname regval:M-e1-B-probe-id_unrelated-P1\endcsname{0.233}
\expandafter\def\csname regval:M-e1-B-probe-id_unrelated-P1.d\endcsname{390}
\expandafter\def\csname regval:M-e1-B-probe-id_unrelated-P2\endcsname{0.767}
\expandafter\def\csname regval:M-e1-B-probe-id_unrelated-P2.d\endcsname{390}
\expandafter\def\csname regval:M-e1-B-probe-id_unrelated-P3\endcsname{0.028}
\expandafter\def\csname regval:M-e1-B-probe-id_unrelated-P3.d\endcsname{390}
\expandafter\def\csname regval:M-e1-B-probe-id_unrelated-P4\endcsname{0.121}
\expandafter\def\csname regval:M-e1-B-probe-id_unrelated-P4.d\endcsname{390}
\expandafter\def\csname regval:M-e1-B-probe-od_unrelated-P1\endcsname{0.054}
\expandafter\def\csname regval:M-e1-B-probe-od_unrelated-P1.d\endcsname{390}
\expandafter\def\csname regval:M-e1-B-probe-od_unrelated-P2\endcsname{0.946}
\expandafter\def\csname regval:M-e1-B-probe-od_unrelated-P2.d\endcsname{390}
\expandafter\def\csname regval:M-e1-B-probe-od_unrelated-P3\endcsname{0.003}
\expandafter\def\csname regval:M-e1-B-probe-od_unrelated-P3.d\endcsname{390}
\expandafter\def\csname regval:M-e1-B-probe-od_unrelated-P4\endcsname{0.028}
\expandafter\def\csname regval:M-e1-B-probe-od_unrelated-P4.d\endcsname{390}
\expandafter\def\csname regval:M-e1-B-opdec-F\endcsname{0.126}
\expandafter\def\csname regval:M-e1-B-opdec-F.d\endcsname{390}
\expandafter\def\csname regval:M-e1-B-opdec-M\endcsname{0.097}
\expandafter\def\csname regval:M-e1-B-opdec-M.d\endcsname{1{,}170}
\expandafter\def\csname regval:M-e1-B-opdec-sum\endcsname{0.104}
\expandafter\def\csname regval:M-e1-B-opdec-sum.d\endcsname{1{,}560}
\expandafter\def\csname regval:M-e1-B2-iaimr\endcsname{0.089}
\expandafter\def\csname regval:M-e1-B2-iaimr.n\endcsname{104}
\expandafter\def\csname regval:M-e1-B2-iaimr.d\endcsname{1{,}170}
\expandafter\def\csname regval:M-e1-B2-iaimr.ci\endcsname{[0.074, 0.107]}
\expandafter\def\csname regval:M-e1-B2-iafmr\endcsname{0.067}
\expandafter\def\csname regval:M-e1-B2-iafmr.n\endcsname{104}
\expandafter\def\csname regval:M-e1-B2-iafmr.d\endcsname{1{,}560}
\expandafter\def\csname regval:M-e1-B2-blank\endcsname{0.003}
\expandafter\def\csname regval:M-e1-B2-blank.n\endcsname{1}
\expandafter\def\csname regval:M-e1-B2-blank.d\endcsname{390}
\expandafter\def\csname regval:M-e1-B2-blank.ci\endcsname{[0.001, 0.014]}
\expandafter\def\csname regval:M-e1-B2-id_unrelated\endcsname{0.213}
\expandafter\def\csname regval:M-e1-B2-id_unrelated.n\endcsname{83}
\expandafter\def\csname regval:M-e1-B2-id_unrelated.d\endcsname{390}
\expandafter\def\csname regval:M-e1-B2-id_unrelated.ci\endcsname{[0.175, 0.256]}
\expandafter\def\csname regval:M-e1-B2-od_unrelated\endcsname{0.051}
\expandafter\def\csname regval:M-e1-B2-od_unrelated.n\endcsname{20}
\expandafter\def\csname regval:M-e1-B2-od_unrelated.d\endcsname{390}
\expandafter\def\csname regval:M-e1-B2-od_unrelated.ci\endcsname{[0.033, 0.078]}
\expandafter\def\csname regval:M-e1-B2-iaimr-chest_vqa\endcsname{0.188}
\expandafter\def\csname regval:M-e1-B2-iaimr-chest_vqa.n\endcsname{22}
\expandafter\def\csname regval:M-e1-B2-iaimr-chest_vqa.d\endcsname{117}
\expandafter\def\csname regval:M-e1-B2-iaimr-docvqa\endcsname{0.043}
\expandafter\def\csname regval:M-e1-B2-iaimr-docvqa.n\endcsname{5}
\expandafter\def\csname regval:M-e1-B2-iaimr-docvqa.d\endcsname{117}
\expandafter\def\csname regval:M-e1-B2-iaimr-infovqa\endcsname{0.000}
\expandafter\def\csname regval:M-e1-B2-iaimr-infovqa.n\endcsname{0}
\expandafter\def\csname regval:M-e1-B2-iaimr-infovqa.d\endcsname{117}
\expandafter\def\csname regval:M-e1-B2-iaimr-medxpertqa_mm\endcsname{0.094}
\expandafter\def\csname regval:M-e1-B2-iaimr-medxpertqa_mm.n\endcsname{11}
\expandafter\def\csname regval:M-e1-B2-iaimr-medxpertqa_mm.d\endcsname{117}
\expandafter\def\csname regval:M-e1-B2-iaimr-microvqa\endcsname{0.103}
\expandafter\def\csname regval:M-e1-B2-iaimr-microvqa.n\endcsname{12}
\expandafter\def\csname regval:M-e1-B2-iaimr-microvqa.d\endcsname{117}
\expandafter\def\csname regval:M-e1-B2-iaimr-mmmu_pro\endcsname{0.043}
\expandafter\def\csname regval:M-e1-B2-iaimr-mmmu_pro.n\endcsname{5}
\expandafter\def\csname regval:M-e1-B2-iaimr-mmmu_pro.d\endcsname{117}
\expandafter\def\csname regval:M-e1-B2-iaimr-pathvqa\endcsname{0.111}
\expandafter\def\csname regval:M-e1-B2-iaimr-pathvqa.n\endcsname{13}
\expandafter\def\csname regval:M-e1-B2-iaimr-pathvqa.d\endcsname{117}
\expandafter\def\csname regval:M-e1-B2-iaimr-rexvqa\endcsname{0.188}
\expandafter\def\csname regval:M-e1-B2-iaimr-rexvqa.n\endcsname{22}
\expandafter\def\csname regval:M-e1-B2-iaimr-rexvqa.d\endcsname{117}
\expandafter\def\csname regval:M-e1-B2-iaimr-textvqa\endcsname{0.026}
\expandafter\def\csname regval:M-e1-B2-iaimr-textvqa.n\endcsname{3}
\expandafter\def\csname regval:M-e1-B2-iaimr-textvqa.d\endcsname{117}
\expandafter\def\csname regval:M-e1-B2-iaimr-video_mmmu\endcsname{0.094}
\expandafter\def\csname regval:M-e1-B2-iaimr-video_mmmu.n\endcsname{11}
\expandafter\def\csname regval:M-e1-B2-iaimr-video_mmmu.d\endcsname{117}
\expandafter\def\csname regval:M-e1-B2-acc\endcsname{0.564}
\expandafter\def\csname regval:M-e1-B2-acc.n\endcsname{220}
\expandafter\def\csname regval:M-e1-B2-acc.d\endcsname{390}
\expandafter\def\csname regval:M-e1-B2-acc.ci\endcsname{[0.514, 0.613]}
\expandafter\def\csname regval:M-e1-B2-answered\endcsname{0.867}
\expandafter\def\csname regval:M-e1-B2-answered.n\endcsname{338}
\expandafter\def\csname regval:M-e1-B2-answered.d\endcsname{390}
\expandafter\def\csname regval:M-e1-B2-fa\endcsname{0.133}
\expandafter\def\csname regval:M-e1-B2-fa.n\endcsname{52}
\expandafter\def\csname regval:M-e1-B2-fa.d\endcsname{390}
\expandafter\def\csname regval:M-e1-B2-acc-cond\endcsname{0.651}
\expandafter\def\csname regval:M-e1-B2-acc-cond.n\endcsname{220}
\expandafter\def\csname regval:M-e1-B2-acc-cond.d\endcsname{338}
\expandafter\def\csname regval:M-e1-B2-unparseable\endcsname{0.000}
\expandafter\def\csname regval:M-e1-B2-unparseable.n\endcsname{0}
\expandafter\def\csname regval:M-e1-B2-unparseable.d\endcsname{1{,}560}
\expandafter\def\csname regval:M-e2-U-imr\endcsname{0.859}
\expandafter\def\csname regval:M-e2-U-imr.n\endcsname{1{,}005}
\expandafter\def\csname regval:M-e2-U-imr.d\endcsname{1{,}170}
\expandafter\def\csname regval:M-e2-U-imr.ci\endcsname{[0.838, 0.878]}
\expandafter\def\csname regval:M-e2-U-tclia\endcsname{0.644}
\expandafter\def\csname regval:M-e2-U-tclia.n\endcsname{1{,}005}
\expandafter\def\csname regval:M-e2-U-tclia.d\endcsname{1{,}560}
\expandafter\def\csname regval:M-e2-U-blank\endcsname{0.646}
\expandafter\def\csname regval:M-e2-U-blank.n\endcsname{252}
\expandafter\def\csname regval:M-e2-U-blank.d\endcsname{390}
\expandafter\def\csname regval:M-e2-U-blank.ci\endcsname{[0.598, 0.692]}
\expandafter\def\csname regval:M-e2-U-id_unrelated\endcsname{0.962}
\expandafter\def\csname regval:M-e2-U-id_unrelated.n\endcsname{375}
\expandafter\def\csname regval:M-e2-U-id_unrelated.d\endcsname{390}
\expandafter\def\csname regval:M-e2-U-id_unrelated.ci\endcsname{[0.938, 0.977]}
\expandafter\def\csname regval:M-e2-U-od_unrelated\endcsname{0.969}
\expandafter\def\csname regval:M-e2-U-od_unrelated.n\endcsname{378}
\expandafter\def\csname regval:M-e2-U-od_unrelated.d\endcsname{390}
\expandafter\def\csname regval:M-e2-U-od_unrelated.ci\endcsname{[0.947, 0.982]}
\expandafter\def\csname regval:M-e2-U-imr-chest_vqa\endcsname{0.786}
\expandafter\def\csname regval:M-e2-U-imr-chest_vqa.n\endcsname{92}
\expandafter\def\csname regval:M-e2-U-imr-chest_vqa.d\endcsname{117}
\expandafter\def\csname regval:M-e2-U-imr-docvqa\endcsname{0.718}
\expandafter\def\csname regval:M-e2-U-imr-docvqa.n\endcsname{84}
\expandafter\def\csname regval:M-e2-U-imr-docvqa.d\endcsname{117}
\expandafter\def\csname regval:M-e2-U-imr-infovqa\endcsname{0.846}
\expandafter\def\csname regval:M-e2-U-imr-infovqa.n\endcsname{99}
\expandafter\def\csname regval:M-e2-U-imr-infovqa.d\endcsname{117}
\expandafter\def\csname regval:M-e2-U-imr-medxpertqa_mm\endcsname{0.957}
\expandafter\def\csname regval:M-e2-U-imr-medxpertqa_mm.n\endcsname{112}
\expandafter\def\csname regval:M-e2-U-imr-medxpertqa_mm.d\endcsname{117}
\expandafter\def\csname regval:M-e2-U-imr-microvqa\endcsname{0.974}
\expandafter\def\csname regval:M-e2-U-imr-microvqa.n\endcsname{114}
\expandafter\def\csname regval:M-e2-U-imr-microvqa.d\endcsname{117}
\expandafter\def\csname regval:M-e2-U-imr-mmmu_pro\endcsname{0.932}
\expandafter\def\csname regval:M-e2-U-imr-mmmu_pro.n\endcsname{109}
\expandafter\def\csname regval:M-e2-U-imr-mmmu_pro.d\endcsname{117}
\expandafter\def\csname regval:M-e2-U-imr-pathvqa\endcsname{0.752}
\expandafter\def\csname regval:M-e2-U-imr-pathvqa.n\endcsname{88}
\expandafter\def\csname regval:M-e2-U-imr-pathvqa.d\endcsname{117}
\expandafter\def\csname regval:M-e2-U-imr-rexvqa\endcsname{0.889}
\expandafter\def\csname regval:M-e2-U-imr-rexvqa.n\endcsname{104}
\expandafter\def\csname regval:M-e2-U-imr-rexvqa.d\endcsname{117}
\expandafter\def\csname regval:M-e2-U-imr-textvqa\endcsname{0.812}
\expandafter\def\csname regval:M-e2-U-imr-textvqa.n\endcsname{95}
\expandafter\def\csname regval:M-e2-U-imr-textvqa.d\endcsname{117}
\expandafter\def\csname regval:M-e2-U-imr-video_mmmu\endcsname{0.923}
\expandafter\def\csname regval:M-e2-U-imr-video_mmmu.n\endcsname{108}
\expandafter\def\csname regval:M-e2-U-imr-video_mmmu.d\endcsname{117}
\expandafter\def\csname regval:M-e2-B-iaimr\endcsname{0.799}
\expandafter\def\csname regval:M-e2-B-iaimr.n\endcsname{935}
\expandafter\def\csname regval:M-e2-B-iaimr.d\endcsname{1{,}170}
\expandafter\def\csname regval:M-e2-B-iaimr.ci\endcsname{[0.775, 0.821]}
\expandafter\def\csname regval:M-e2-B-iafmr\endcsname{0.599}
\expandafter\def\csname regval:M-e2-B-iafmr.n\endcsname{935}
\expandafter\def\csname regval:M-e2-B-iafmr.d\endcsname{1{,}560}
\expandafter\def\csname regval:M-e2-B-blank\endcsname{0.631}
\expandafter\def\csname regval:M-e2-B-blank.n\endcsname{246}
\expandafter\def\csname regval:M-e2-B-blank.d\endcsname{390}
\expandafter\def\csname regval:M-e2-B-blank.ci\endcsname{[0.582, 0.677]}
\expandafter\def\csname regval:M-e2-B-id_unrelated\endcsname{0.869}
\expandafter\def\csname regval:M-e2-B-id_unrelated.n\endcsname{339}
\expandafter\def\csname regval:M-e2-B-id_unrelated.d\endcsname{390}
\expandafter\def\csname regval:M-e2-B-id_unrelated.ci\endcsname{[0.832, 0.899]}
\expandafter\def\csname regval:M-e2-B-od_unrelated\endcsname{0.897}
\expandafter\def\csname regval:M-e2-B-od_unrelated.n\endcsname{350}
\expandafter\def\csname regval:M-e2-B-od_unrelated.d\endcsname{390}
\expandafter\def\csname regval:M-e2-B-od_unrelated.ci\endcsname{[0.863, 0.924]}
\expandafter\def\csname regval:M-e2-B-iaimr-chest_vqa\endcsname{0.607}
\expandafter\def\csname regval:M-e2-B-iaimr-chest_vqa.n\endcsname{71}
\expandafter\def\csname regval:M-e2-B-iaimr-chest_vqa.d\endcsname{117}
\expandafter\def\csname regval:M-e2-B-iaimr-docvqa\endcsname{0.641}
\expandafter\def\csname regval:M-e2-B-iaimr-docvqa.n\endcsname{75}
\expandafter\def\csname regval:M-e2-B-iaimr-docvqa.d\endcsname{117}
\expandafter\def\csname regval:M-e2-B-iaimr-infovqa\endcsname{0.786}
\expandafter\def\csname regval:M-e2-B-iaimr-infovqa.n\endcsname{92}
\expandafter\def\csname regval:M-e2-B-iaimr-infovqa.d\endcsname{117}
\expandafter\def\csname regval:M-e2-B-iaimr-medxpertqa_mm\endcsname{0.914}
\expandafter\def\csname regval:M-e2-B-iaimr-medxpertqa_mm.n\endcsname{107}
\expandafter\def\csname regval:M-e2-B-iaimr-medxpertqa_mm.d\endcsname{117}
\expandafter\def\csname regval:M-e2-B-iaimr-microvqa\endcsname{0.974}
\expandafter\def\csname regval:M-e2-B-iaimr-microvqa.n\endcsname{114}
\expandafter\def\csname regval:M-e2-B-iaimr-microvqa.d\endcsname{117}
\expandafter\def\csname regval:M-e2-B-iaimr-mmmu_pro\endcsname{0.863}
\expandafter\def\csname regval:M-e2-B-iaimr-mmmu_pro.n\endcsname{101}
\expandafter\def\csname regval:M-e2-B-iaimr-mmmu_pro.d\endcsname{117}
\expandafter\def\csname regval:M-e2-B-iaimr-pathvqa\endcsname{0.684}
\expandafter\def\csname regval:M-e2-B-iaimr-pathvqa.n\endcsname{80}
\expandafter\def\csname regval:M-e2-B-iaimr-pathvqa.d\endcsname{117}
\expandafter\def\csname regval:M-e2-B-iaimr-rexvqa\endcsname{0.923}
\expandafter\def\csname regval:M-e2-B-iaimr-rexvqa.n\endcsname{108}
\expandafter\def\csname regval:M-e2-B-iaimr-rexvqa.d\endcsname{117}
\expandafter\def\csname regval:M-e2-B-iaimr-textvqa\endcsname{0.735}
\expandafter\def\csname regval:M-e2-B-iaimr-textvqa.n\endcsname{86}
\expandafter\def\csname regval:M-e2-B-iaimr-textvqa.d\endcsname{117}
\expandafter\def\csname regval:M-e2-B-iaimr-video_mmmu\endcsname{0.863}
\expandafter\def\csname regval:M-e2-B-iaimr-video_mmmu.n\endcsname{101}
\expandafter\def\csname regval:M-e2-B-iaimr-video_mmmu.d\endcsname{117}
\expandafter\def\csname regval:M-e2-B-acc\endcsname{0.377}
\expandafter\def\csname regval:M-e2-B-acc.n\endcsname{147}
\expandafter\def\csname regval:M-e2-B-acc.d\endcsname{390}
\expandafter\def\csname regval:M-e2-B-acc.ci\endcsname{[0.330, 0.426]}
\expandafter\def\csname regval:M-e2-B-answered\endcsname{0.774}
\expandafter\def\csname regval:M-e2-B-answered.n\endcsname{302}
\expandafter\def\csname regval:M-e2-B-answered.d\endcsname{390}
\expandafter\def\csname regval:M-e2-B-fa\endcsname{0.226}
\expandafter\def\csname regval:M-e2-B-fa.n\endcsname{88}
\expandafter\def\csname regval:M-e2-B-fa.d\endcsname{390}
\expandafter\def\csname regval:M-e2-B-acc-cond\endcsname{0.487}
\expandafter\def\csname regval:M-e2-B-acc-cond.n\endcsname{147}
\expandafter\def\csname regval:M-e2-B-acc-cond.d\endcsname{302}
\expandafter\def\csname regval:M-e2-B-unparseable\endcsname{0.000}
\expandafter\def\csname regval:M-e2-B-unparseable.n\endcsname{0}
\expandafter\def\csname regval:M-e2-B-unparseable.d\endcsname{1{,}560}
\expandafter\def\csname regval:M-e2-B-probe-real-P1\endcsname{0.774}
\expandafter\def\csname regval:M-e2-B-probe-real-P1.d\endcsname{390}
\expandafter\def\csname regval:M-e2-B-probe-real-P2\endcsname{0.100}
\expandafter\def\csname regval:M-e2-B-probe-real-P2.d\endcsname{390}
\expandafter\def\csname regval:M-e2-B-probe-real-P3\endcsname{0.718}
\expandafter\def\csname regval:M-e2-B-probe-real-P3.d\endcsname{390}
\expandafter\def\csname regval:M-e2-B-probe-real-P4\endcsname{0.551}
\expandafter\def\csname regval:M-e2-B-probe-real-P4.d\endcsname{390}
\expandafter\def\csname regval:M-e2-B-probe-blank-P1\endcsname{0.656}
\expandafter\def\csname regval:M-e2-B-probe-blank-P1.d\endcsname{390}
\expandafter\def\csname regval:M-e2-B-probe-blank-P2\endcsname{0.426}
\expandafter\def\csname regval:M-e2-B-probe-blank-P2.d\endcsname{390}
\expandafter\def\csname regval:M-e2-B-probe-blank-P4\endcsname{0.482}
\expandafter\def\csname regval:M-e2-B-probe-blank-P4.d\endcsname{390}
\expandafter\def\csname regval:M-e2-B-probe-id_unrelated-P1\endcsname{0.426}
\expandafter\def\csname regval:M-e2-B-probe-id_unrelated-P1.d\endcsname{390}
\expandafter\def\csname regval:M-e2-B-probe-id_unrelated-P2\endcsname{0.087}
\expandafter\def\csname regval:M-e2-B-probe-id_unrelated-P2.d\endcsname{390}
\expandafter\def\csname regval:M-e2-B-probe-id_unrelated-P3\endcsname{0.097}
\expandafter\def\csname regval:M-e2-B-probe-id_unrelated-P3.d\endcsname{390}
\expandafter\def\csname regval:M-e2-B-probe-id_unrelated-P4\endcsname{0.210}
\expandafter\def\csname regval:M-e2-B-probe-id_unrelated-P4.d\endcsname{390}
\expandafter\def\csname regval:M-e2-B-probe-od_unrelated-P1\endcsname{0.246}
\expandafter\def\csname regval:M-e2-B-probe-od_unrelated-P1.d\endcsname{390}
\expandafter\def\csname regval:M-e2-B-probe-od_unrelated-P2\endcsname{0.085}
\expandafter\def\csname regval:M-e2-B-probe-od_unrelated-P2.d\endcsname{390}
\expandafter\def\csname regval:M-e2-B-probe-od_unrelated-P3\endcsname{0.046}
\expandafter\def\csname regval:M-e2-B-probe-od_unrelated-P3.d\endcsname{390}
\expandafter\def\csname regval:M-e2-B-probe-od_unrelated-P4\endcsname{0.095}
\expandafter\def\csname regval:M-e2-B-probe-od_unrelated-P4.d\endcsname{390}
\expandafter\def\csname regval:M-e2-B-opdec-F\endcsname{0.226}
\expandafter\def\csname regval:M-e2-B-opdec-F.d\endcsname{390}
\expandafter\def\csname regval:M-e2-B-opdec-M\endcsname{0.443}
\expandafter\def\csname regval:M-e2-B-opdec-M.d\endcsname{1{,}170}
\expandafter\def\csname regval:M-e2-B-opdec-sum\endcsname{0.389}
\expandafter\def\csname regval:M-e2-B-opdec-sum.d\endcsname{1{,}560}
\expandafter\def\csname regval:M-e2-B2-iaimr\endcsname{0.418}
\expandafter\def\csname regval:M-e2-B2-iaimr.n\endcsname{489}
\expandafter\def\csname regval:M-e2-B2-iaimr.d\endcsname{1{,}170}
\expandafter\def\csname regval:M-e2-B2-iaimr.ci\endcsname{[0.390, 0.446]}
\expandafter\def\csname regval:M-e2-B2-iafmr\endcsname{0.314}
\expandafter\def\csname regval:M-e2-B2-iafmr.n\endcsname{489}
\expandafter\def\csname regval:M-e2-B2-iafmr.d\endcsname{1{,}560}
\expandafter\def\csname regval:M-e2-B2-blank\endcsname{0.510}
\expandafter\def\csname regval:M-e2-B2-blank.n\endcsname{199}
\expandafter\def\csname regval:M-e2-B2-blank.d\endcsname{390}
\expandafter\def\csname regval:M-e2-B2-blank.ci\endcsname{[0.461, 0.559]}
\expandafter\def\csname regval:M-e2-B2-id_unrelated\endcsname{0.446}
\expandafter\def\csname regval:M-e2-B2-id_unrelated.n\endcsname{174}
\expandafter\def\csname regval:M-e2-B2-id_unrelated.d\endcsname{390}
\expandafter\def\csname regval:M-e2-B2-id_unrelated.ci\endcsname{[0.398, 0.496]}
\expandafter\def\csname regval:M-e2-B2-od_unrelated\endcsname{0.297}
\expandafter\def\csname regval:M-e2-B2-od_unrelated.n\endcsname{116}
\expandafter\def\csname regval:M-e2-B2-od_unrelated.d\endcsname{390}
\expandafter\def\csname regval:M-e2-B2-od_unrelated.ci\endcsname{[0.254, 0.345]}
\expandafter\def\csname regval:M-e2-B2-iaimr-chest_vqa\endcsname{0.299}
\expandafter\def\csname regval:M-e2-B2-iaimr-chest_vqa.n\endcsname{35}
\expandafter\def\csname regval:M-e2-B2-iaimr-chest_vqa.d\endcsname{117}
\expandafter\def\csname regval:M-e2-B2-iaimr-docvqa\endcsname{0.197}
\expandafter\def\csname regval:M-e2-B2-iaimr-docvqa.n\endcsname{23}
\expandafter\def\csname regval:M-e2-B2-iaimr-docvqa.d\endcsname{117}
\expandafter\def\csname regval:M-e2-B2-iaimr-infovqa\endcsname{0.239}
\expandafter\def\csname regval:M-e2-B2-iaimr-infovqa.n\endcsname{28}
\expandafter\def\csname regval:M-e2-B2-iaimr-infovqa.d\endcsname{117}
\expandafter\def\csname regval:M-e2-B2-iaimr-medxpertqa_mm\endcsname{0.393}
\expandafter\def\csname regval:M-e2-B2-iaimr-medxpertqa_mm.n\endcsname{46}
\expandafter\def\csname regval:M-e2-B2-iaimr-medxpertqa_mm.d\endcsname{117}
\expandafter\def\csname regval:M-e2-B2-iaimr-microvqa\endcsname{0.692}
\expandafter\def\csname regval:M-e2-B2-iaimr-microvqa.n\endcsname{81}
\expandafter\def\csname regval:M-e2-B2-iaimr-microvqa.d\endcsname{117}
\expandafter\def\csname regval:M-e2-B2-iaimr-mmmu_pro\endcsname{0.590}
\expandafter\def\csname regval:M-e2-B2-iaimr-mmmu_pro.n\endcsname{69}
\expandafter\def\csname regval:M-e2-B2-iaimr-mmmu_pro.d\endcsname{117}
\expandafter\def\csname regval:M-e2-B2-iaimr-pathvqa\endcsname{0.222}
\expandafter\def\csname regval:M-e2-B2-iaimr-pathvqa.n\endcsname{26}
\expandafter\def\csname regval:M-e2-B2-iaimr-pathvqa.d\endcsname{117}
\expandafter\def\csname regval:M-e2-B2-iaimr-rexvqa\endcsname{0.624}
\expandafter\def\csname regval:M-e2-B2-iaimr-rexvqa.n\endcsname{73}
\expandafter\def\csname regval:M-e2-B2-iaimr-rexvqa.d\endcsname{117}
\expandafter\def\csname regval:M-e2-B2-iaimr-textvqa\endcsname{0.367}
\expandafter\def\csname regval:M-e2-B2-iaimr-textvqa.n\endcsname{43}
\expandafter\def\csname regval:M-e2-B2-iaimr-textvqa.d\endcsname{117}
\expandafter\def\csname regval:M-e2-B2-iaimr-video_mmmu\endcsname{0.556}
\expandafter\def\csname regval:M-e2-B2-iaimr-video_mmmu.n\endcsname{65}
\expandafter\def\csname regval:M-e2-B2-iaimr-video_mmmu.d\endcsname{117}
\expandafter\def\csname regval:M-e2-B2-acc\endcsname{0.331}
\expandafter\def\csname regval:M-e2-B2-acc.n\endcsname{129}
\expandafter\def\csname regval:M-e2-B2-acc.d\endcsname{390}
\expandafter\def\csname regval:M-e2-B2-acc.ci\endcsname{[0.286, 0.379]}
\expandafter\def\csname regval:M-e2-B2-answered\endcsname{0.746}
\expandafter\def\csname regval:M-e2-B2-answered.n\endcsname{291}
\expandafter\def\csname regval:M-e2-B2-answered.d\endcsname{390}
\expandafter\def\csname regval:M-e2-B2-fa\endcsname{0.254}
\expandafter\def\csname regval:M-e2-B2-fa.n\endcsname{99}
\expandafter\def\csname regval:M-e2-B2-fa.d\endcsname{390}
\expandafter\def\csname regval:M-e2-B2-acc-cond\endcsname{0.443}
\expandafter\def\csname regval:M-e2-B2-acc-cond.n\endcsname{129}
\expandafter\def\csname regval:M-e2-B2-acc-cond.d\endcsname{291}
\expandafter\def\csname regval:M-e2-B2-unparseable\endcsname{0.000}
\expandafter\def\csname regval:M-e2-B2-unparseable.n\endcsname{0}
\expandafter\def\csname regval:M-e2-B2-unparseable.d\endcsname{1{,}560}
\expandafter\def\csname regval:M-e3-U-imr\endcsname{0.094}
\expandafter\def\csname regval:M-e3-U-imr.n\endcsname{110}
\expandafter\def\csname regval:M-e3-U-imr.d\endcsname{1{,}170}
\expandafter\def\csname regval:M-e3-U-imr.ci\endcsname{[0.079, 0.112]}
\expandafter\def\csname regval:M-e3-U-tclia\endcsname{0.070}
\expandafter\def\csname regval:M-e3-U-tclia.n\endcsname{110}
\expandafter\def\csname regval:M-e3-U-tclia.d\endcsname{1{,}560}
\expandafter\def\csname regval:M-e3-U-blank\endcsname{0.000}
\expandafter\def\csname regval:M-e3-U-blank.n\endcsname{0}
\expandafter\def\csname regval:M-e3-U-blank.d\endcsname{390}
\expandafter\def\csname regval:M-e3-U-blank.ci\endcsname{[0.000, 0.010]}
\expandafter\def\csname regval:M-e3-U-id_unrelated\endcsname{0.236}
\expandafter\def\csname regval:M-e3-U-id_unrelated.n\endcsname{92}
\expandafter\def\csname regval:M-e3-U-id_unrelated.d\endcsname{390}
\expandafter\def\csname regval:M-e3-U-id_unrelated.ci\endcsname{[0.197, 0.281]}
\expandafter\def\csname regval:M-e3-U-od_unrelated\endcsname{0.046}
\expandafter\def\csname regval:M-e3-U-od_unrelated.n\endcsname{18}
\expandafter\def\csname regval:M-e3-U-od_unrelated.d\endcsname{390}
\expandafter\def\csname regval:M-e3-U-od_unrelated.ci\endcsname{[0.029, 0.072]}
\expandafter\def\csname regval:M-e3-U-imr-chest_vqa\endcsname{0.231}
\expandafter\def\csname regval:M-e3-U-imr-chest_vqa.n\endcsname{27}
\expandafter\def\csname regval:M-e3-U-imr-chest_vqa.d\endcsname{117}
\expandafter\def\csname regval:M-e3-U-imr-docvqa\endcsname{0.026}
\expandafter\def\csname regval:M-e3-U-imr-docvqa.n\endcsname{3}
\expandafter\def\csname regval:M-e3-U-imr-docvqa.d\endcsname{117}
\expandafter\def\csname regval:M-e3-U-imr-infovqa\endcsname{0.009}
\expandafter\def\csname regval:M-e3-U-imr-infovqa.n\endcsname{1}
\expandafter\def\csname regval:M-e3-U-imr-infovqa.d\endcsname{117}
\expandafter\def\csname regval:M-e3-U-imr-medxpertqa_mm\endcsname{0.077}
\expandafter\def\csname regval:M-e3-U-imr-medxpertqa_mm.n\endcsname{9}
\expandafter\def\csname regval:M-e3-U-imr-medxpertqa_mm.d\endcsname{117}
\expandafter\def\csname regval:M-e3-U-imr-microvqa\endcsname{0.094}
\expandafter\def\csname regval:M-e3-U-imr-microvqa.n\endcsname{11}
\expandafter\def\csname regval:M-e3-U-imr-microvqa.d\endcsname{117}
\expandafter\def\csname regval:M-e3-U-imr-mmmu_pro\endcsname{0.051}
\expandafter\def\csname regval:M-e3-U-imr-mmmu_pro.n\endcsname{6}
\expandafter\def\csname regval:M-e3-U-imr-mmmu_pro.d\endcsname{117}
\expandafter\def\csname regval:M-e3-U-imr-pathvqa\endcsname{0.103}
\expandafter\def\csname regval:M-e3-U-imr-pathvqa.n\endcsname{12}
\expandafter\def\csname regval:M-e3-U-imr-pathvqa.d\endcsname{117}
\expandafter\def\csname regval:M-e3-U-imr-rexvqa\endcsname{0.214}
\expandafter\def\csname regval:M-e3-U-imr-rexvqa.n\endcsname{25}
\expandafter\def\csname regval:M-e3-U-imr-rexvqa.d\endcsname{117}
\expandafter\def\csname regval:M-e3-U-imr-textvqa\endcsname{0.000}
\expandafter\def\csname regval:M-e3-U-imr-textvqa.n\endcsname{0}
\expandafter\def\csname regval:M-e3-U-imr-textvqa.d\endcsname{117}
\expandafter\def\csname regval:M-e3-U-imr-video_mmmu\endcsname{0.137}
\expandafter\def\csname regval:M-e3-U-imr-video_mmmu.n\endcsname{16}
\expandafter\def\csname regval:M-e3-U-imr-video_mmmu.d\endcsname{117}
\expandafter\def\csname regval:M-e3-B-iaimr\endcsname{0.086}
\expandafter\def\csname regval:M-e3-B-iaimr.n\endcsname{101}
\expandafter\def\csname regval:M-e3-B-iaimr.d\endcsname{1{,}170}
\expandafter\def\csname regval:M-e3-B-iaimr.ci\endcsname{[0.072, 0.104]}
\expandafter\def\csname regval:M-e3-B-iafmr\endcsname{0.065}
\expandafter\def\csname regval:M-e3-B-iafmr.n\endcsname{101}
\expandafter\def\csname regval:M-e3-B-iafmr.d\endcsname{1{,}560}
\expandafter\def\csname regval:M-e3-B-blank\endcsname{0.000}
\expandafter\def\csname regval:M-e3-B-blank.n\endcsname{0}
\expandafter\def\csname regval:M-e3-B-blank.d\endcsname{390}
\expandafter\def\csname regval:M-e3-B-blank.ci\endcsname{[0.000, 0.010]}
\expandafter\def\csname regval:M-e3-B-id_unrelated\endcsname{0.215}
\expandafter\def\csname regval:M-e3-B-id_unrelated.n\endcsname{84}
\expandafter\def\csname regval:M-e3-B-id_unrelated.d\endcsname{390}
\expandafter\def\csname regval:M-e3-B-id_unrelated.ci\endcsname{[0.177, 0.259]}
\expandafter\def\csname regval:M-e3-B-od_unrelated\endcsname{0.044}
\expandafter\def\csname regval:M-e3-B-od_unrelated.n\endcsname{17}
\expandafter\def\csname regval:M-e3-B-od_unrelated.d\endcsname{390}
\expandafter\def\csname regval:M-e3-B-od_unrelated.ci\endcsname{[0.027, 0.069]}
\expandafter\def\csname regval:M-e3-B-iaimr-chest_vqa\endcsname{0.171}
\expandafter\def\csname regval:M-e3-B-iaimr-chest_vqa.n\endcsname{20}
\expandafter\def\csname regval:M-e3-B-iaimr-chest_vqa.d\endcsname{117}
\expandafter\def\csname regval:M-e3-B-iaimr-docvqa\endcsname{0.009}
\expandafter\def\csname regval:M-e3-B-iaimr-docvqa.n\endcsname{1}
\expandafter\def\csname regval:M-e3-B-iaimr-docvqa.d\endcsname{117}
\expandafter\def\csname regval:M-e3-B-iaimr-infovqa\endcsname{0.009}
\expandafter\def\csname regval:M-e3-B-iaimr-infovqa.n\endcsname{1}
\expandafter\def\csname regval:M-e3-B-iaimr-infovqa.d\endcsname{117}
\expandafter\def\csname regval:M-e3-B-iaimr-medxpertqa_mm\endcsname{0.077}
\expandafter\def\csname regval:M-e3-B-iaimr-medxpertqa_mm.n\endcsname{9}
\expandafter\def\csname regval:M-e3-B-iaimr-medxpertqa_mm.d\endcsname{117}
\expandafter\def\csname regval:M-e3-B-iaimr-microvqa\endcsname{0.094}
\expandafter\def\csname regval:M-e3-B-iaimr-microvqa.n\endcsname{11}
\expandafter\def\csname regval:M-e3-B-iaimr-microvqa.d\endcsname{117}
\expandafter\def\csname regval:M-e3-B-iaimr-mmmu_pro\endcsname{0.051}
\expandafter\def\csname regval:M-e3-B-iaimr-mmmu_pro.n\endcsname{6}
\expandafter\def\csname regval:M-e3-B-iaimr-mmmu_pro.d\endcsname{117}
\expandafter\def\csname regval:M-e3-B-iaimr-pathvqa\endcsname{0.103}
\expandafter\def\csname regval:M-e3-B-iaimr-pathvqa.n\endcsname{12}
\expandafter\def\csname regval:M-e3-B-iaimr-pathvqa.d\endcsname{117}
\expandafter\def\csname regval:M-e3-B-iaimr-rexvqa\endcsname{0.214}
\expandafter\def\csname regval:M-e3-B-iaimr-rexvqa.n\endcsname{25}
\expandafter\def\csname regval:M-e3-B-iaimr-rexvqa.d\endcsname{117}
\expandafter\def\csname regval:M-e3-B-iaimr-textvqa\endcsname{0.000}
\expandafter\def\csname regval:M-e3-B-iaimr-textvqa.n\endcsname{0}
\expandafter\def\csname regval:M-e3-B-iaimr-textvqa.d\endcsname{117}
\expandafter\def\csname regval:M-e3-B-iaimr-video_mmmu\endcsname{0.137}
\expandafter\def\csname regval:M-e3-B-iaimr-video_mmmu.n\endcsname{16}
\expandafter\def\csname regval:M-e3-B-iaimr-video_mmmu.d\endcsname{117}
\expandafter\def\csname regval:M-e3-B-acc\endcsname{0.526}
\expandafter\def\csname regval:M-e3-B-acc.n\endcsname{205}
\expandafter\def\csname regval:M-e3-B-acc.d\endcsname{390}
\expandafter\def\csname regval:M-e3-B-acc.ci\endcsname{[0.476, 0.575]}
\expandafter\def\csname regval:M-e3-B-answered\endcsname{0.849}
\expandafter\def\csname regval:M-e3-B-answered.n\endcsname{331}
\expandafter\def\csname regval:M-e3-B-answered.d\endcsname{390}
\expandafter\def\csname regval:M-e3-B-fa\endcsname{0.151}
\expandafter\def\csname regval:M-e3-B-fa.n\endcsname{59}
\expandafter\def\csname regval:M-e3-B-fa.d\endcsname{390}
\expandafter\def\csname regval:M-e3-B-acc-cond\endcsname{0.619}
\expandafter\def\csname regval:M-e3-B-acc-cond.n\endcsname{205}
\expandafter\def\csname regval:M-e3-B-acc-cond.d\endcsname{331}
\expandafter\def\csname regval:M-e3-B-unparseable\endcsname{0.000}
\expandafter\def\csname regval:M-e3-B-unparseable.n\endcsname{0}
\expandafter\def\csname regval:M-e3-B-unparseable.d\endcsname{1{,}560}
\expandafter\def\csname regval:M-e3-B-probe-real-P1\endcsname{0.849}
\expandafter\def\csname regval:M-e3-B-probe-real-P1.d\endcsname{390}
\expandafter\def\csname regval:M-e3-B-probe-real-P2\endcsname{0.151}
\expandafter\def\csname regval:M-e3-B-probe-real-P2.d\endcsname{390}
\expandafter\def\csname regval:M-e3-B-probe-real-P3\endcsname{0.695}
\expandafter\def\csname regval:M-e3-B-probe-real-P3.d\endcsname{390}
\expandafter\def\csname regval:M-e3-B-probe-real-P4\endcsname{0.431}
\expandafter\def\csname regval:M-e3-B-probe-real-P4.d\endcsname{390}
\expandafter\def\csname regval:M-e3-B-probe-blank-P1\endcsname{0.000}
\expandafter\def\csname regval:M-e3-B-probe-blank-P1.d\endcsname{390}
\expandafter\def\csname regval:M-e3-B-probe-blank-P2\endcsname{1.000}
\expandafter\def\csname regval:M-e3-B-probe-blank-P2.d\endcsname{390}
\expandafter\def\csname regval:M-e3-B-probe-blank-P4\endcsname{0.000}
\expandafter\def\csname regval:M-e3-B-probe-blank-P4.d\endcsname{390}
\expandafter\def\csname regval:M-e3-B-probe-id_unrelated-P1\endcsname{0.236}
\expandafter\def\csname regval:M-e3-B-probe-id_unrelated-P1.d\endcsname{390}
\expandafter\def\csname regval:M-e3-B-probe-id_unrelated-P2\endcsname{0.764}
\expandafter\def\csname regval:M-e3-B-probe-id_unrelated-P2.d\endcsname{390}
\expandafter\def\csname regval:M-e3-B-probe-id_unrelated-P3\endcsname{0.021}
\expandafter\def\csname regval:M-e3-B-probe-id_unrelated-P3.d\endcsname{390}
\expandafter\def\csname regval:M-e3-B-probe-id_unrelated-P4\endcsname{0.128}
\expandafter\def\csname regval:M-e3-B-probe-id_unrelated-P4.d\endcsname{390}
\expandafter\def\csname regval:M-e3-B-probe-od_unrelated-P1\endcsname{0.046}
\expandafter\def\csname regval:M-e3-B-probe-od_unrelated-P1.d\endcsname{390}
\expandafter\def\csname regval:M-e3-B-probe-od_unrelated-P2\endcsname{0.954}
\expandafter\def\csname regval:M-e3-B-probe-od_unrelated-P2.d\endcsname{390}
\expandafter\def\csname regval:M-e3-B-probe-od_unrelated-P3\endcsname{0.005}
\expandafter\def\csname regval:M-e3-B-probe-od_unrelated-P3.d\endcsname{390}
\expandafter\def\csname regval:M-e3-B-probe-od_unrelated-P4\endcsname{0.010}
\expandafter\def\csname regval:M-e3-B-probe-od_unrelated-P4.d\endcsname{390}
\expandafter\def\csname regval:M-e3-B-opdec-F\endcsname{0.151}
\expandafter\def\csname regval:M-e3-B-opdec-F.d\endcsname{390}
\expandafter\def\csname regval:M-e3-B-opdec-M\endcsname{0.094}
\expandafter\def\csname regval:M-e3-B-opdec-M.d\endcsname{1{,}170}
\expandafter\def\csname regval:M-e3-B-opdec-sum\endcsname{0.108}
\expandafter\def\csname regval:M-e3-B-opdec-sum.d\endcsname{1{,}560}
\expandafter\def\csname regval:M-e3-B2-iaimr\endcsname{0.088}
\expandafter\def\csname regval:M-e3-B2-iaimr.n\endcsname{103}
\expandafter\def\csname regval:M-e3-B2-iaimr.d\endcsname{1{,}170}
\expandafter\def\csname regval:M-e3-B2-iaimr.ci\endcsname{[0.073, 0.106]}
\expandafter\def\csname regval:M-e3-B2-iafmr\endcsname{0.066}
\expandafter\def\csname regval:M-e3-B2-iafmr.n\endcsname{103}
\expandafter\def\csname regval:M-e3-B2-iafmr.d\endcsname{1{,}560}
\expandafter\def\csname regval:M-e3-B2-blank\endcsname{0.000}
\expandafter\def\csname regval:M-e3-B2-blank.n\endcsname{0}
\expandafter\def\csname regval:M-e3-B2-blank.d\endcsname{390}
\expandafter\def\csname regval:M-e3-B2-blank.ci\endcsname{[0.000, 0.010]}
\expandafter\def\csname regval:M-e3-B2-id_unrelated\endcsname{0.218}
\expandafter\def\csname regval:M-e3-B2-id_unrelated.n\endcsname{85}
\expandafter\def\csname regval:M-e3-B2-id_unrelated.d\endcsname{390}
\expandafter\def\csname regval:M-e3-B2-id_unrelated.ci\endcsname{[0.180, 0.262]}
\expandafter\def\csname regval:M-e3-B2-od_unrelated\endcsname{0.046}
\expandafter\def\csname regval:M-e3-B2-od_unrelated.n\endcsname{18}
\expandafter\def\csname regval:M-e3-B2-od_unrelated.d\endcsname{390}
\expandafter\def\csname regval:M-e3-B2-od_unrelated.ci\endcsname{[0.029, 0.072]}
\expandafter\def\csname regval:M-e3-B2-iaimr-chest_vqa\endcsname{0.171}
\expandafter\def\csname regval:M-e3-B2-iaimr-chest_vqa.n\endcsname{20}
\expandafter\def\csname regval:M-e3-B2-iaimr-chest_vqa.d\endcsname{117}
\expandafter\def\csname regval:M-e3-B2-iaimr-docvqa\endcsname{0.026}
\expandafter\def\csname regval:M-e3-B2-iaimr-docvqa.n\endcsname{3}
\expandafter\def\csname regval:M-e3-B2-iaimr-docvqa.d\endcsname{117}
\expandafter\def\csname regval:M-e3-B2-iaimr-infovqa\endcsname{0.009}
\expandafter\def\csname regval:M-e3-B2-iaimr-infovqa.n\endcsname{1}
\expandafter\def\csname regval:M-e3-B2-iaimr-infovqa.d\endcsname{117}
\expandafter\def\csname regval:M-e3-B2-iaimr-medxpertqa_mm\endcsname{0.077}
\expandafter\def\csname regval:M-e3-B2-iaimr-medxpertqa_mm.n\endcsname{9}
\expandafter\def\csname regval:M-e3-B2-iaimr-medxpertqa_mm.d\endcsname{117}
\expandafter\def\csname regval:M-e3-B2-iaimr-microvqa\endcsname{0.094}
\expandafter\def\csname regval:M-e3-B2-iaimr-microvqa.n\endcsname{11}
\expandafter\def\csname regval:M-e3-B2-iaimr-microvqa.d\endcsname{117}
\expandafter\def\csname regval:M-e3-B2-iaimr-mmmu_pro\endcsname{0.051}
\expandafter\def\csname regval:M-e3-B2-iaimr-mmmu_pro.n\endcsname{6}
\expandafter\def\csname regval:M-e3-B2-iaimr-mmmu_pro.d\endcsname{117}
\expandafter\def\csname regval:M-e3-B2-iaimr-pathvqa\endcsname{0.103}
\expandafter\def\csname regval:M-e3-B2-iaimr-pathvqa.n\endcsname{12}
\expandafter\def\csname regval:M-e3-B2-iaimr-pathvqa.d\endcsname{117}
\expandafter\def\csname regval:M-e3-B2-iaimr-rexvqa\endcsname{0.214}
\expandafter\def\csname regval:M-e3-B2-iaimr-rexvqa.n\endcsname{25}
\expandafter\def\csname regval:M-e3-B2-iaimr-rexvqa.d\endcsname{117}
\expandafter\def\csname regval:M-e3-B2-iaimr-textvqa\endcsname{0.000}
\expandafter\def\csname regval:M-e3-B2-iaimr-textvqa.n\endcsname{0}
\expandafter\def\csname regval:M-e3-B2-iaimr-textvqa.d\endcsname{117}
\expandafter\def\csname regval:M-e3-B2-iaimr-video_mmmu\endcsname{0.137}
\expandafter\def\csname regval:M-e3-B2-iaimr-video_mmmu.n\endcsname{16}
\expandafter\def\csname regval:M-e3-B2-iaimr-video_mmmu.d\endcsname{117}
\expandafter\def\csname regval:M-e3-B2-acc\endcsname{0.526}
\expandafter\def\csname regval:M-e3-B2-acc.n\endcsname{205}
\expandafter\def\csname regval:M-e3-B2-acc.d\endcsname{390}
\expandafter\def\csname regval:M-e3-B2-acc.ci\endcsname{[0.476, 0.575]}
\expandafter\def\csname regval:M-e3-B2-answered\endcsname{0.841}
\expandafter\def\csname regval:M-e3-B2-answered.n\endcsname{328}
\expandafter\def\csname regval:M-e3-B2-answered.d\endcsname{390}
\expandafter\def\csname regval:M-e3-B2-fa\endcsname{0.159}
\expandafter\def\csname regval:M-e3-B2-fa.n\endcsname{62}
\expandafter\def\csname regval:M-e3-B2-fa.d\endcsname{390}
\expandafter\def\csname regval:M-e3-B2-acc-cond\endcsname{0.625}
\expandafter\def\csname regval:M-e3-B2-acc-cond.n\endcsname{205}
\expandafter\def\csname regval:M-e3-B2-acc-cond.d\endcsname{328}
\expandafter\def\csname regval:M-e3-B2-unparseable\endcsname{0.000}
\expandafter\def\csname regval:M-e3-B2-unparseable.n\endcsname{0}
\expandafter\def\csname regval:M-e3-B2-unparseable.d\endcsname{1{,}560}
\expandafter\def\csname regval:M-e4-U-imr\endcsname{0.615}
\expandafter\def\csname regval:M-e4-U-imr.n\endcsname{719}
\expandafter\def\csname regval:M-e4-U-imr.d\endcsname{1{,}170}
\expandafter\def\csname regval:M-e4-U-imr.ci\endcsname{[0.586, 0.642]}
\expandafter\def\csname regval:M-e4-U-tclia\endcsname{0.461}
\expandafter\def\csname regval:M-e4-U-tclia.n\endcsname{719}
\expandafter\def\csname regval:M-e4-U-tclia.d\endcsname{1{,}560}
\expandafter\def\csname regval:M-e4-U-blank\endcsname{0.000}
\expandafter\def\csname regval:M-e4-U-blank.n\endcsname{0}
\expandafter\def\csname regval:M-e4-U-blank.d\endcsname{390}
\expandafter\def\csname regval:M-e4-U-blank.ci\endcsname{[0.000, 0.010]}
\expandafter\def\csname regval:M-e4-U-id_unrelated\endcsname{0.936}
\expandafter\def\csname regval:M-e4-U-id_unrelated.n\endcsname{365}
\expandafter\def\csname regval:M-e4-U-id_unrelated.d\endcsname{390}
\expandafter\def\csname regval:M-e4-U-id_unrelated.ci\endcsname{[0.907, 0.956]}
\expandafter\def\csname regval:M-e4-U-od_unrelated\endcsname{0.908}
\expandafter\def\csname regval:M-e4-U-od_unrelated.n\endcsname{354}
\expandafter\def\csname regval:M-e4-U-od_unrelated.d\endcsname{390}
\expandafter\def\csname regval:M-e4-U-od_unrelated.ci\endcsname{[0.875, 0.933]}
\expandafter\def\csname regval:M-e4-U-imr-chest_vqa\endcsname{0.632}
\expandafter\def\csname regval:M-e4-U-imr-chest_vqa.n\endcsname{74}
\expandafter\def\csname regval:M-e4-U-imr-chest_vqa.d\endcsname{117}
\expandafter\def\csname regval:M-e4-U-imr-docvqa\endcsname{0.573}
\expandafter\def\csname regval:M-e4-U-imr-docvqa.n\endcsname{67}
\expandafter\def\csname regval:M-e4-U-imr-docvqa.d\endcsname{117}
\expandafter\def\csname regval:M-e4-U-imr-infovqa\endcsname{0.641}
\expandafter\def\csname regval:M-e4-U-imr-infovqa.n\endcsname{75}
\expandafter\def\csname regval:M-e4-U-imr-infovqa.d\endcsname{117}
\expandafter\def\csname regval:M-e4-U-imr-medxpertqa_mm\endcsname{0.607}
\expandafter\def\csname regval:M-e4-U-imr-medxpertqa_mm.n\endcsname{71}
\expandafter\def\csname regval:M-e4-U-imr-medxpertqa_mm.d\endcsname{117}
\expandafter\def\csname regval:M-e4-U-imr-microvqa\endcsname{0.564}
\expandafter\def\csname regval:M-e4-U-imr-microvqa.n\endcsname{66}
\expandafter\def\csname regval:M-e4-U-imr-microvqa.d\endcsname{117}
\expandafter\def\csname regval:M-e4-U-imr-mmmu_pro\endcsname{0.658}
\expandafter\def\csname regval:M-e4-U-imr-mmmu_pro.n\endcsname{77}
\expandafter\def\csname regval:M-e4-U-imr-mmmu_pro.d\endcsname{117}
\expandafter\def\csname regval:M-e4-U-imr-pathvqa\endcsname{0.632}
\expandafter\def\csname regval:M-e4-U-imr-pathvqa.n\endcsname{74}
\expandafter\def\csname regval:M-e4-U-imr-pathvqa.d\endcsname{117}
\expandafter\def\csname regval:M-e4-U-imr-rexvqa\endcsname{0.632}
\expandafter\def\csname regval:M-e4-U-imr-rexvqa.n\endcsname{74}
\expandafter\def\csname regval:M-e4-U-imr-rexvqa.d\endcsname{117}
\expandafter\def\csname regval:M-e4-U-imr-textvqa\endcsname{0.641}
\expandafter\def\csname regval:M-e4-U-imr-textvqa.n\endcsname{75}
\expandafter\def\csname regval:M-e4-U-imr-textvqa.d\endcsname{117}
\expandafter\def\csname regval:M-e4-U-imr-video_mmmu\endcsname{0.564}
\expandafter\def\csname regval:M-e4-U-imr-video_mmmu.n\endcsname{66}
\expandafter\def\csname regval:M-e4-U-imr-video_mmmu.d\endcsname{117}
\expandafter\def\csname regval:M-e4-B-iaimr\endcsname{0.490}
\expandafter\def\csname regval:M-e4-B-iaimr.n\endcsname{573}
\expandafter\def\csname regval:M-e4-B-iaimr.d\endcsname{1{,}170}
\expandafter\def\csname regval:M-e4-B-iaimr.ci\endcsname{[0.461, 0.518]}
\expandafter\def\csname regval:M-e4-B-iafmr\endcsname{0.367}
\expandafter\def\csname regval:M-e4-B-iafmr.n\endcsname{573}
\expandafter\def\csname regval:M-e4-B-iafmr.d\endcsname{1{,}560}
\expandafter\def\csname regval:M-e4-B-blank\endcsname{0.000}
\expandafter\def\csname regval:M-e4-B-blank.n\endcsname{0}
\expandafter\def\csname regval:M-e4-B-blank.d\endcsname{390}
\expandafter\def\csname regval:M-e4-B-blank.ci\endcsname{[0.000, 0.010]}
\expandafter\def\csname regval:M-e4-B-id_unrelated\endcsname{0.726}
\expandafter\def\csname regval:M-e4-B-id_unrelated.n\endcsname{283}
\expandafter\def\csname regval:M-e4-B-id_unrelated.d\endcsname{390}
\expandafter\def\csname regval:M-e4-B-id_unrelated.ci\endcsname{[0.679, 0.768]}
\expandafter\def\csname regval:M-e4-B-od_unrelated\endcsname{0.744}
\expandafter\def\csname regval:M-e4-B-od_unrelated.n\endcsname{290}
\expandafter\def\csname regval:M-e4-B-od_unrelated.d\endcsname{390}
\expandafter\def\csname regval:M-e4-B-od_unrelated.ci\endcsname{[0.698, 0.784]}
\expandafter\def\csname regval:M-e4-B-iaimr-chest_vqa\endcsname{0.487}
\expandafter\def\csname regval:M-e4-B-iaimr-chest_vqa.n\endcsname{57}
\expandafter\def\csname regval:M-e4-B-iaimr-chest_vqa.d\endcsname{117}
\expandafter\def\csname regval:M-e4-B-iaimr-docvqa\endcsname{0.325}
\expandafter\def\csname regval:M-e4-B-iaimr-docvqa.n\endcsname{38}
\expandafter\def\csname regval:M-e4-B-iaimr-docvqa.d\endcsname{117}
\expandafter\def\csname regval:M-e4-B-iaimr-infovqa\endcsname{0.496}
\expandafter\def\csname regval:M-e4-B-iaimr-infovqa.n\endcsname{58}
\expandafter\def\csname regval:M-e4-B-iaimr-infovqa.d\endcsname{117}
\expandafter\def\csname regval:M-e4-B-iaimr-medxpertqa_mm\endcsname{0.487}
\expandafter\def\csname regval:M-e4-B-iaimr-medxpertqa_mm.n\endcsname{57}
\expandafter\def\csname regval:M-e4-B-iaimr-medxpertqa_mm.d\endcsname{117}
\expandafter\def\csname regval:M-e4-B-iaimr-microvqa\endcsname{0.487}
\expandafter\def\csname regval:M-e4-B-iaimr-microvqa.n\endcsname{57}
\expandafter\def\csname regval:M-e4-B-iaimr-microvqa.d\endcsname{117}
\expandafter\def\csname regval:M-e4-B-iaimr-mmmu_pro\endcsname{0.479}
\expandafter\def\csname regval:M-e4-B-iaimr-mmmu_pro.n\endcsname{56}
\expandafter\def\csname regval:M-e4-B-iaimr-mmmu_pro.d\endcsname{117}
\expandafter\def\csname regval:M-e4-B-iaimr-pathvqa\endcsname{0.590}
\expandafter\def\csname regval:M-e4-B-iaimr-pathvqa.n\endcsname{69}
\expandafter\def\csname regval:M-e4-B-iaimr-pathvqa.d\endcsname{117}
\expandafter\def\csname regval:M-e4-B-iaimr-rexvqa\endcsname{0.590}
\expandafter\def\csname regval:M-e4-B-iaimr-rexvqa.n\endcsname{69}
\expandafter\def\csname regval:M-e4-B-iaimr-rexvqa.d\endcsname{117}
\expandafter\def\csname regval:M-e4-B-iaimr-textvqa\endcsname{0.479}
\expandafter\def\csname regval:M-e4-B-iaimr-textvqa.n\endcsname{56}
\expandafter\def\csname regval:M-e4-B-iaimr-textvqa.d\endcsname{117}
\expandafter\def\csname regval:M-e4-B-iaimr-video_mmmu\endcsname{0.479}
\expandafter\def\csname regval:M-e4-B-iaimr-video_mmmu.n\endcsname{56}
\expandafter\def\csname regval:M-e4-B-iaimr-video_mmmu.d\endcsname{117}
\expandafter\def\csname regval:M-e4-B-acc\endcsname{0.423}
\expandafter\def\csname regval:M-e4-B-acc.n\endcsname{165}
\expandafter\def\csname regval:M-e4-B-acc.d\endcsname{390}
\expandafter\def\csname regval:M-e4-B-acc.ci\endcsname{[0.375, 0.473]}
\expandafter\def\csname regval:M-e4-B-answered\endcsname{0.821}
\expandafter\def\csname regval:M-e4-B-answered.n\endcsname{320}
\expandafter\def\csname regval:M-e4-B-answered.d\endcsname{390}
\expandafter\def\csname regval:M-e4-B-fa\endcsname{0.179}
\expandafter\def\csname regval:M-e4-B-fa.n\endcsname{70}
\expandafter\def\csname regval:M-e4-B-fa.d\endcsname{390}
\expandafter\def\csname regval:M-e4-B-acc-cond\endcsname{0.516}
\expandafter\def\csname regval:M-e4-B-acc-cond.n\endcsname{165}
\expandafter\def\csname regval:M-e4-B-acc-cond.d\endcsname{320}
\expandafter\def\csname regval:M-e4-B-unparseable\endcsname{0.000}
\expandafter\def\csname regval:M-e4-B-unparseable.n\endcsname{0}
\expandafter\def\csname regval:M-e4-B-unparseable.d\endcsname{1{,}560}
\expandafter\def\csname regval:M-e4-B-probe-real-P1\endcsname{0.821}
\expandafter\def\csname regval:M-e4-B-probe-real-P1.d\endcsname{390}
\expandafter\def\csname regval:M-e4-B-probe-real-P2\endcsname{0.059}
\expandafter\def\csname regval:M-e4-B-probe-real-P2.d\endcsname{390}
\expandafter\def\csname regval:M-e4-B-probe-real-P3\endcsname{0.674}
\expandafter\def\csname regval:M-e4-B-probe-real-P3.d\endcsname{390}
\expandafter\def\csname regval:M-e4-B-probe-real-P4\endcsname{0.426}
\expandafter\def\csname regval:M-e4-B-probe-real-P4.d\endcsname{390}
\expandafter\def\csname regval:M-e4-B-probe-blank-P1\endcsname{0.000}
\expandafter\def\csname regval:M-e4-B-probe-blank-P1.d\endcsname{390}
\expandafter\def\csname regval:M-e4-B-probe-blank-P2\endcsname{1.000}
\expandafter\def\csname regval:M-e4-B-probe-blank-P2.d\endcsname{390}
\expandafter\def\csname regval:M-e4-B-probe-blank-P4\endcsname{0.000}
\expandafter\def\csname regval:M-e4-B-probe-blank-P4.d\endcsname{390}
\expandafter\def\csname regval:M-e4-B-probe-id_unrelated-P1\endcsname{0.328}
\expandafter\def\csname regval:M-e4-B-probe-id_unrelated-P1.d\endcsname{390}
\expandafter\def\csname regval:M-e4-B-probe-id_unrelated-P2\endcsname{0.254}
\expandafter\def\csname regval:M-e4-B-probe-id_unrelated-P2.d\endcsname{390}
\expandafter\def\csname regval:M-e4-B-probe-id_unrelated-P3\endcsname{0.026}
\expandafter\def\csname regval:M-e4-B-probe-id_unrelated-P3.d\endcsname{390}
\expandafter\def\csname regval:M-e4-B-probe-id_unrelated-P4\endcsname{0.156}
\expandafter\def\csname regval:M-e4-B-probe-id_unrelated-P4.d\endcsname{390}
\expandafter\def\csname regval:M-e4-B-probe-od_unrelated-P1\endcsname{0.203}
\expandafter\def\csname regval:M-e4-B-probe-od_unrelated-P1.d\endcsname{390}
\expandafter\def\csname regval:M-e4-B-probe-od_unrelated-P2\endcsname{0.249}
\expandafter\def\csname regval:M-e4-B-probe-od_unrelated-P2.d\endcsname{390}
\expandafter\def\csname regval:M-e4-B-probe-od_unrelated-P3\endcsname{0.013}
\expandafter\def\csname regval:M-e4-B-probe-od_unrelated-P3.d\endcsname{390}
\expandafter\def\csname regval:M-e4-B-probe-od_unrelated-P4\endcsname{0.056}
\expandafter\def\csname regval:M-e4-B-probe-od_unrelated-P4.d\endcsname{390}
\expandafter\def\csname regval:M-e4-B-opdec-F\endcsname{0.179}
\expandafter\def\csname regval:M-e4-B-opdec-F.d\endcsname{390}
\expandafter\def\csname regval:M-e4-B-opdec-M\endcsname{0.177}
\expandafter\def\csname regval:M-e4-B-opdec-M.d\endcsname{1{,}170}
\expandafter\def\csname regval:M-e4-B-opdec-sum\endcsname{0.177}
\expandafter\def\csname regval:M-e4-B-opdec-sum.d\endcsname{1{,}560}
\expandafter\def\csname regval:M-e4-B2-iaimr\endcsname{0.185}
\expandafter\def\csname regval:M-e4-B2-iaimr.n\endcsname{217}
\expandafter\def\csname regval:M-e4-B2-iaimr.d\endcsname{1{,}170}
\expandafter\def\csname regval:M-e4-B2-iaimr.ci\endcsname{[0.164, 0.209]}
\expandafter\def\csname regval:M-e4-B2-iafmr\endcsname{0.139}
\expandafter\def\csname regval:M-e4-B2-iafmr.n\endcsname{217}
\expandafter\def\csname regval:M-e4-B2-iafmr.d\endcsname{1{,}560}
\expandafter\def\csname regval:M-e4-B2-blank\endcsname{0.000}
\expandafter\def\csname regval:M-e4-B2-blank.n\endcsname{0}
\expandafter\def\csname regval:M-e4-B2-blank.d\endcsname{390}
\expandafter\def\csname regval:M-e4-B2-blank.ci\endcsname{[0.000, 0.010]}
\expandafter\def\csname regval:M-e4-B2-id_unrelated\endcsname{0.346}
\expandafter\def\csname regval:M-e4-B2-id_unrelated.n\endcsname{135}
\expandafter\def\csname regval:M-e4-B2-id_unrelated.d\endcsname{390}
\expandafter\def\csname regval:M-e4-B2-id_unrelated.ci\endcsname{[0.301, 0.395]}
\expandafter\def\csname regval:M-e4-B2-od_unrelated\endcsname{0.210}
\expandafter\def\csname regval:M-e4-B2-od_unrelated.n\endcsname{82}
\expandafter\def\csname regval:M-e4-B2-od_unrelated.d\endcsname{390}
\expandafter\def\csname regval:M-e4-B2-od_unrelated.ci\endcsname{[0.173, 0.253]}
\expandafter\def\csname regval:M-e4-B2-iaimr-chest_vqa\endcsname{0.205}
\expandafter\def\csname regval:M-e4-B2-iaimr-chest_vqa.n\endcsname{24}
\expandafter\def\csname regval:M-e4-B2-iaimr-chest_vqa.d\endcsname{117}
\expandafter\def\csname regval:M-e4-B2-iaimr-docvqa\endcsname{0.068}
\expandafter\def\csname regval:M-e4-B2-iaimr-docvqa.n\endcsname{8}
\expandafter\def\csname regval:M-e4-B2-iaimr-docvqa.d\endcsname{117}
\expandafter\def\csname regval:M-e4-B2-iaimr-infovqa\endcsname{0.068}
\expandafter\def\csname regval:M-e4-B2-iaimr-infovqa.n\endcsname{8}
\expandafter\def\csname regval:M-e4-B2-iaimr-infovqa.d\endcsname{117}
\expandafter\def\csname regval:M-e4-B2-iaimr-medxpertqa_mm\endcsname{0.137}
\expandafter\def\csname regval:M-e4-B2-iaimr-medxpertqa_mm.n\endcsname{16}
\expandafter\def\csname regval:M-e4-B2-iaimr-medxpertqa_mm.d\endcsname{117}
\expandafter\def\csname regval:M-e4-B2-iaimr-microvqa\endcsname{0.128}
\expandafter\def\csname regval:M-e4-B2-iaimr-microvqa.n\endcsname{15}
\expandafter\def\csname regval:M-e4-B2-iaimr-microvqa.d\endcsname{117}
\expandafter\def\csname regval:M-e4-B2-iaimr-mmmu_pro\endcsname{0.359}
\expandafter\def\csname regval:M-e4-B2-iaimr-mmmu_pro.n\endcsname{42}
\expandafter\def\csname regval:M-e4-B2-iaimr-mmmu_pro.d\endcsname{117}
\expandafter\def\csname regval:M-e4-B2-iaimr-pathvqa\endcsname{0.111}
\expandafter\def\csname regval:M-e4-B2-iaimr-pathvqa.n\endcsname{13}
\expandafter\def\csname regval:M-e4-B2-iaimr-pathvqa.d\endcsname{117}
\expandafter\def\csname regval:M-e4-B2-iaimr-rexvqa\endcsname{0.402}
\expandafter\def\csname regval:M-e4-B2-iaimr-rexvqa.n\endcsname{47}
\expandafter\def\csname regval:M-e4-B2-iaimr-rexvqa.d\endcsname{117}
\expandafter\def\csname regval:M-e4-B2-iaimr-textvqa\endcsname{0.137}
\expandafter\def\csname regval:M-e4-B2-iaimr-textvqa.n\endcsname{16}
\expandafter\def\csname regval:M-e4-B2-iaimr-textvqa.d\endcsname{117}
\expandafter\def\csname regval:M-e4-B2-iaimr-video_mmmu\endcsname{0.239}
\expandafter\def\csname regval:M-e4-B2-iaimr-video_mmmu.n\endcsname{28}
\expandafter\def\csname regval:M-e4-B2-iaimr-video_mmmu.d\endcsname{117}
\expandafter\def\csname regval:M-e4-B2-acc\endcsname{0.415}
\expandafter\def\csname regval:M-e4-B2-acc.n\endcsname{162}
\expandafter\def\csname regval:M-e4-B2-acc.d\endcsname{390}
\expandafter\def\csname regval:M-e4-B2-acc.ci\endcsname{[0.367, 0.465]}
\expandafter\def\csname regval:M-e4-B2-answered\endcsname{0.813}
\expandafter\def\csname regval:M-e4-B2-answered.n\endcsname{317}
\expandafter\def\csname regval:M-e4-B2-answered.d\endcsname{390}
\expandafter\def\csname regval:M-e4-B2-fa\endcsname{0.187}
\expandafter\def\csname regval:M-e4-B2-fa.n\endcsname{73}
\expandafter\def\csname regval:M-e4-B2-fa.d\endcsname{390}
\expandafter\def\csname regval:M-e4-B2-acc-cond\endcsname{0.511}
\expandafter\def\csname regval:M-e4-B2-acc-cond.n\endcsname{162}
\expandafter\def\csname regval:M-e4-B2-acc-cond.d\endcsname{317}
\expandafter\def\csname regval:M-e4-B2-unparseable\endcsname{0.000}
\expandafter\def\csname regval:M-e4-B2-unparseable.n\endcsname{0}
\expandafter\def\csname regval:M-e4-B2-unparseable.d\endcsname{1{,}560}
\expandafter\def\csname regval:M-e5-U-imr\endcsname{0.718}
\expandafter\def\csname regval:M-e5-U-imr.n\endcsname{840}
\expandafter\def\csname regval:M-e5-U-imr.d\endcsname{1{,}170}
\expandafter\def\csname regval:M-e5-U-imr.ci\endcsname{[0.692, 0.743]}
\expandafter\def\csname regval:M-e5-U-tclia\endcsname{0.538}
\expandafter\def\csname regval:M-e5-U-tclia.n\endcsname{840}
\expandafter\def\csname regval:M-e5-U-tclia.d\endcsname{1{,}560}
\expandafter\def\csname regval:M-e5-U-blank\endcsname{0.656}
\expandafter\def\csname regval:M-e5-U-blank.n\endcsname{256}
\expandafter\def\csname regval:M-e5-U-blank.d\endcsname{390}
\expandafter\def\csname regval:M-e5-U-blank.ci\endcsname{[0.608, 0.702]}
\expandafter\def\csname regval:M-e5-U-id_unrelated\endcsname{0.779}
\expandafter\def\csname regval:M-e5-U-id_unrelated.n\endcsname{304}
\expandafter\def\csname regval:M-e5-U-id_unrelated.d\endcsname{390}
\expandafter\def\csname regval:M-e5-U-id_unrelated.ci\endcsname{[0.736, 0.818]}
\expandafter\def\csname regval:M-e5-U-od_unrelated\endcsname{0.718}
\expandafter\def\csname regval:M-e5-U-od_unrelated.n\endcsname{280}
\expandafter\def\csname regval:M-e5-U-od_unrelated.d\endcsname{390}
\expandafter\def\csname regval:M-e5-U-od_unrelated.ci\endcsname{[0.671, 0.760]}
\expandafter\def\csname regval:M-e5-U-imr-chest_vqa\endcsname{0.667}
\expandafter\def\csname regval:M-e5-U-imr-chest_vqa.n\endcsname{78}
\expandafter\def\csname regval:M-e5-U-imr-chest_vqa.d\endcsname{117}
\expandafter\def\csname regval:M-e5-U-imr-docvqa\endcsname{0.564}
\expandafter\def\csname regval:M-e5-U-imr-docvqa.n\endcsname{66}
\expandafter\def\csname regval:M-e5-U-imr-docvqa.d\endcsname{117}
\expandafter\def\csname regval:M-e5-U-imr-infovqa\endcsname{0.761}
\expandafter\def\csname regval:M-e5-U-imr-infovqa.n\endcsname{89}
\expandafter\def\csname regval:M-e5-U-imr-infovqa.d\endcsname{117}
\expandafter\def\csname regval:M-e5-U-imr-medxpertqa_mm\endcsname{0.778}
\expandafter\def\csname regval:M-e5-U-imr-medxpertqa_mm.n\endcsname{91}
\expandafter\def\csname regval:M-e5-U-imr-medxpertqa_mm.d\endcsname{117}
\expandafter\def\csname regval:M-e5-U-imr-microvqa\endcsname{0.846}
\expandafter\def\csname regval:M-e5-U-imr-microvqa.n\endcsname{99}
\expandafter\def\csname regval:M-e5-U-imr-microvqa.d\endcsname{117}
\expandafter\def\csname regval:M-e5-U-imr-mmmu_pro\endcsname{0.795}
\expandafter\def\csname regval:M-e5-U-imr-mmmu_pro.n\endcsname{93}
\expandafter\def\csname regval:M-e5-U-imr-mmmu_pro.d\endcsname{117}
\expandafter\def\csname regval:M-e5-U-imr-pathvqa\endcsname{0.667}
\expandafter\def\csname regval:M-e5-U-imr-pathvqa.n\endcsname{78}
\expandafter\def\csname regval:M-e5-U-imr-pathvqa.d\endcsname{117}
\expandafter\def\csname regval:M-e5-U-imr-rexvqa\endcsname{0.624}
\expandafter\def\csname regval:M-e5-U-imr-rexvqa.n\endcsname{73}
\expandafter\def\csname regval:M-e5-U-imr-rexvqa.d\endcsname{117}
\expandafter\def\csname regval:M-e5-U-imr-textvqa\endcsname{0.727}
\expandafter\def\csname regval:M-e5-U-imr-textvqa.n\endcsname{85}
\expandafter\def\csname regval:M-e5-U-imr-textvqa.d\endcsname{117}
\expandafter\def\csname regval:M-e5-U-imr-video_mmmu\endcsname{0.752}
\expandafter\def\csname regval:M-e5-U-imr-video_mmmu.n\endcsname{88}
\expandafter\def\csname regval:M-e5-U-imr-video_mmmu.d\endcsname{117}
\expandafter\def\csname regval:M-e0-B-fa-chest_vqa\endcsname{0.026}
\expandafter\def\csname regval:M-e0-B-fa-chest_vqa.n\endcsname{1}
\expandafter\def\csname regval:M-e0-B-fa-chest_vqa.d\endcsname{39}
\expandafter\def\csname regval:M-e0-B-answered-chest_vqa\endcsname{0.974}
\expandafter\def\csname regval:M-e0-B-answered-chest_vqa.n\endcsname{38}
\expandafter\def\csname regval:M-e0-B-answered-chest_vqa.d\endcsname{39}
\expandafter\def\csname regval:M-e0-B-fa-docvqa\endcsname{0.000}
\expandafter\def\csname regval:M-e0-B-fa-docvqa.n\endcsname{0}
\expandafter\def\csname regval:M-e0-B-fa-docvqa.d\endcsname{39}
\expandafter\def\csname regval:M-e0-B-answered-docvqa\endcsname{1.000}
\expandafter\def\csname regval:M-e0-B-answered-docvqa.n\endcsname{39}
\expandafter\def\csname regval:M-e0-B-answered-docvqa.d\endcsname{39}
\expandafter\def\csname regval:M-e0-B-fa-infovqa\endcsname{0.051}
\expandafter\def\csname regval:M-e0-B-fa-infovqa.n\endcsname{2}
\expandafter\def\csname regval:M-e0-B-fa-infovqa.d\endcsname{39}
\expandafter\def\csname regval:M-e0-B-answered-infovqa\endcsname{0.949}
\expandafter\def\csname regval:M-e0-B-answered-infovqa.n\endcsname{37}
\expandafter\def\csname regval:M-e0-B-answered-infovqa.d\endcsname{39}
\expandafter\def\csname regval:M-e0-B-fa-medxpertqa_mm\endcsname{0.000}
\expandafter\def\csname regval:M-e0-B-fa-medxpertqa_mm.n\endcsname{0}
\expandafter\def\csname regval:M-e0-B-fa-medxpertqa_mm.d\endcsname{39}
\expandafter\def\csname regval:M-e0-B-answered-medxpertqa_mm\endcsname{1.000}
\expandafter\def\csname regval:M-e0-B-answered-medxpertqa_mm.n\endcsname{39}
\expandafter\def\csname regval:M-e0-B-answered-medxpertqa_mm.d\endcsname{39}
\expandafter\def\csname regval:M-e0-B-fa-microvqa\endcsname{0.000}
\expandafter\def\csname regval:M-e0-B-fa-microvqa.n\endcsname{0}
\expandafter\def\csname regval:M-e0-B-fa-microvqa.d\endcsname{39}
\expandafter\def\csname regval:M-e0-B-answered-microvqa\endcsname{1.000}
\expandafter\def\csname regval:M-e0-B-answered-microvqa.n\endcsname{39}
\expandafter\def\csname regval:M-e0-B-answered-microvqa.d\endcsname{39}
\expandafter\def\csname regval:M-e0-B-fa-mmmu_pro\endcsname{0.000}
\expandafter\def\csname regval:M-e0-B-fa-mmmu_pro.n\endcsname{0}
\expandafter\def\csname regval:M-e0-B-fa-mmmu_pro.d\endcsname{39}
\expandafter\def\csname regval:M-e0-B-answered-mmmu_pro\endcsname{1.000}
\expandafter\def\csname regval:M-e0-B-answered-mmmu_pro.n\endcsname{39}
\expandafter\def\csname regval:M-e0-B-answered-mmmu_pro.d\endcsname{39}
\expandafter\def\csname regval:M-e0-B-fa-pathvqa\endcsname{0.077}
\expandafter\def\csname regval:M-e0-B-fa-pathvqa.n\endcsname{3}
\expandafter\def\csname regval:M-e0-B-fa-pathvqa.d\endcsname{39}
\expandafter\def\csname regval:M-e0-B-answered-pathvqa\endcsname{0.923}
\expandafter\def\csname regval:M-e0-B-answered-pathvqa.n\endcsname{36}
\expandafter\def\csname regval:M-e0-B-answered-pathvqa.d\endcsname{39}
\expandafter\def\csname regval:M-e0-B-fa-rexvqa\endcsname{0.000}
\expandafter\def\csname regval:M-e0-B-fa-rexvqa.n\endcsname{0}
\expandafter\def\csname regval:M-e0-B-fa-rexvqa.d\endcsname{39}
\expandafter\def\csname regval:M-e0-B-answered-rexvqa\endcsname{1.000}
\expandafter\def\csname regval:M-e0-B-answered-rexvqa.n\endcsname{39}
\expandafter\def\csname regval:M-e0-B-answered-rexvqa.d\endcsname{39}
\expandafter\def\csname regval:M-e0-B-fa-textvqa\endcsname{0.000}
\expandafter\def\csname regval:M-e0-B-fa-textvqa.n\endcsname{0}
\expandafter\def\csname regval:M-e0-B-fa-textvqa.d\endcsname{39}
\expandafter\def\csname regval:M-e0-B-answered-textvqa\endcsname{1.000}
\expandafter\def\csname regval:M-e0-B-answered-textvqa.n\endcsname{39}
\expandafter\def\csname regval:M-e0-B-answered-textvqa.d\endcsname{39}
\expandafter\def\csname regval:M-e0-B-fa-video_mmmu\endcsname{0.026}
\expandafter\def\csname regval:M-e0-B-fa-video_mmmu.n\endcsname{1}
\expandafter\def\csname regval:M-e0-B-fa-video_mmmu.d\endcsname{39}
\expandafter\def\csname regval:M-e0-B-answered-video_mmmu\endcsname{0.974}
\expandafter\def\csname regval:M-e0-B-answered-video_mmmu.n\endcsname{38}
\expandafter\def\csname regval:M-e0-B-answered-video_mmmu.d\endcsname{39}
\expandafter\def\csname regval:M-e1-B-fa-chest_vqa\endcsname{0.154}
\expandafter\def\csname regval:M-e1-B-fa-chest_vqa.n\endcsname{6}
\expandafter\def\csname regval:M-e1-B-fa-chest_vqa.d\endcsname{39}
\expandafter\def\csname regval:M-e1-B-answered-chest_vqa\endcsname{0.846}
\expandafter\def\csname regval:M-e1-B-answered-chest_vqa.n\endcsname{33}
\expandafter\def\csname regval:M-e1-B-answered-chest_vqa.d\endcsname{39}
\expandafter\def\csname regval:M-e1-B-fa-docvqa\endcsname{0.051}
\expandafter\def\csname regval:M-e1-B-fa-docvqa.n\endcsname{2}
\expandafter\def\csname regval:M-e1-B-fa-docvqa.d\endcsname{39}
\expandafter\def\csname regval:M-e1-B-answered-docvqa\endcsname{0.949}
\expandafter\def\csname regval:M-e1-B-answered-docvqa.n\endcsname{37}
\expandafter\def\csname regval:M-e1-B-answered-docvqa.d\endcsname{39}
\expandafter\def\csname regval:M-e1-B-fa-infovqa\endcsname{0.026}
\expandafter\def\csname regval:M-e1-B-fa-infovqa.n\endcsname{1}
\expandafter\def\csname regval:M-e1-B-fa-infovqa.d\endcsname{39}
\expandafter\def\csname regval:M-e1-B-answered-infovqa\endcsname{0.974}
\expandafter\def\csname regval:M-e1-B-answered-infovqa.n\endcsname{38}
\expandafter\def\csname regval:M-e1-B-answered-infovqa.d\endcsname{39}
\expandafter\def\csname regval:M-e1-B-fa-medxpertqa_mm\endcsname{0.103}
\expandafter\def\csname regval:M-e1-B-fa-medxpertqa_mm.n\endcsname{4}
\expandafter\def\csname regval:M-e1-B-fa-medxpertqa_mm.d\endcsname{39}
\expandafter\def\csname regval:M-e1-B-answered-medxpertqa_mm\endcsname{0.897}
\expandafter\def\csname regval:M-e1-B-answered-medxpertqa_mm.n\endcsname{35}
\expandafter\def\csname regval:M-e1-B-answered-medxpertqa_mm.d\endcsname{39}
\expandafter\def\csname regval:M-e1-B-fa-microvqa\endcsname{0.077}
\expandafter\def\csname regval:M-e1-B-fa-microvqa.n\endcsname{3}
\expandafter\def\csname regval:M-e1-B-fa-microvqa.d\endcsname{39}
\expandafter\def\csname regval:M-e1-B-answered-microvqa\endcsname{0.923}
\expandafter\def\csname regval:M-e1-B-answered-microvqa.n\endcsname{36}
\expandafter\def\csname regval:M-e1-B-answered-microvqa.d\endcsname{39}
\expandafter\def\csname regval:M-e1-B-fa-mmmu_pro\endcsname{0.026}
\expandafter\def\csname regval:M-e1-B-fa-mmmu_pro.n\endcsname{1}
\expandafter\def\csname regval:M-e1-B-fa-mmmu_pro.d\endcsname{39}
\expandafter\def\csname regval:M-e1-B-answered-mmmu_pro\endcsname{0.974}
\expandafter\def\csname regval:M-e1-B-answered-mmmu_pro.n\endcsname{38}
\expandafter\def\csname regval:M-e1-B-answered-mmmu_pro.d\endcsname{39}
\expandafter\def\csname regval:M-e1-B-fa-pathvqa\endcsname{0.487}
\expandafter\def\csname regval:M-e1-B-fa-pathvqa.n\endcsname{19}
\expandafter\def\csname regval:M-e1-B-fa-pathvqa.d\endcsname{39}
\expandafter\def\csname regval:M-e1-B-answered-pathvqa\endcsname{0.513}
\expandafter\def\csname regval:M-e1-B-answered-pathvqa.n\endcsname{20}
\expandafter\def\csname regval:M-e1-B-answered-pathvqa.d\endcsname{39}
\expandafter\def\csname regval:M-e1-B-fa-rexvqa\endcsname{0.333}
\expandafter\def\csname regval:M-e1-B-fa-rexvqa.n\endcsname{13}
\expandafter\def\csname regval:M-e1-B-fa-rexvqa.d\endcsname{39}
\expandafter\def\csname regval:M-e1-B-answered-rexvqa\endcsname{0.667}
\expandafter\def\csname regval:M-e1-B-answered-rexvqa.n\endcsname{26}
\expandafter\def\csname regval:M-e1-B-answered-rexvqa.d\endcsname{39}
\expandafter\def\csname regval:M-e1-B-fa-textvqa\endcsname{0.000}
\expandafter\def\csname regval:M-e1-B-fa-textvqa.n\endcsname{0}
\expandafter\def\csname regval:M-e1-B-fa-textvqa.d\endcsname{39}
\expandafter\def\csname regval:M-e1-B-answered-textvqa\endcsname{1.000}
\expandafter\def\csname regval:M-e1-B-answered-textvqa.n\endcsname{39}
\expandafter\def\csname regval:M-e1-B-answered-textvqa.d\endcsname{39}
\expandafter\def\csname regval:M-e1-B-fa-video_mmmu\endcsname{0.000}
\expandafter\def\csname regval:M-e1-B-fa-video_mmmu.n\endcsname{0}
\expandafter\def\csname regval:M-e1-B-fa-video_mmmu.d\endcsname{39}
\expandafter\def\csname regval:M-e1-B-answered-video_mmmu\endcsname{1.000}
\expandafter\def\csname regval:M-e1-B-answered-video_mmmu.n\endcsname{39}
\expandafter\def\csname regval:M-e1-B-answered-video_mmmu.d\endcsname{39}
\expandafter\def\csname regval:M-e2-B-fa-chest_vqa\endcsname{0.641}
\expandafter\def\csname regval:M-e2-B-fa-chest_vqa.n\endcsname{25}
\expandafter\def\csname regval:M-e2-B-fa-chest_vqa.d\endcsname{39}
\expandafter\def\csname regval:M-e2-B-answered-chest_vqa\endcsname{0.359}
\expandafter\def\csname regval:M-e2-B-answered-chest_vqa.n\endcsname{14}
\expandafter\def\csname regval:M-e2-B-answered-chest_vqa.d\endcsname{39}
\expandafter\def\csname regval:M-e2-B-fa-docvqa\endcsname{0.077}
\expandafter\def\csname regval:M-e2-B-fa-docvqa.n\endcsname{3}
\expandafter\def\csname regval:M-e2-B-fa-docvqa.d\endcsname{39}
\expandafter\def\csname regval:M-e2-B-answered-docvqa\endcsname{0.923}
\expandafter\def\csname regval:M-e2-B-answered-docvqa.n\endcsname{36}
\expandafter\def\csname regval:M-e2-B-answered-docvqa.d\endcsname{39}
\expandafter\def\csname regval:M-e2-B-fa-infovqa\endcsname{0.154}
\expandafter\def\csname regval:M-e2-B-fa-infovqa.n\endcsname{6}
\expandafter\def\csname regval:M-e2-B-fa-infovqa.d\endcsname{39}
\expandafter\def\csname regval:M-e2-B-answered-infovqa\endcsname{0.846}
\expandafter\def\csname regval:M-e2-B-answered-infovqa.n\endcsname{33}
\expandafter\def\csname regval:M-e2-B-answered-infovqa.d\endcsname{39}
\expandafter\def\csname regval:M-e2-B-fa-medxpertqa_mm\endcsname{0.282}
\expandafter\def\csname regval:M-e2-B-fa-medxpertqa_mm.n\endcsname{11}
\expandafter\def\csname regval:M-e2-B-fa-medxpertqa_mm.d\endcsname{39}
\expandafter\def\csname regval:M-e2-B-answered-medxpertqa_mm\endcsname{0.718}
\expandafter\def\csname regval:M-e2-B-answered-medxpertqa_mm.n\endcsname{28}
\expandafter\def\csname regval:M-e2-B-answered-medxpertqa_mm.d\endcsname{39}
\expandafter\def\csname regval:M-e2-B-fa-microvqa\endcsname{0.000}
\expandafter\def\csname regval:M-e2-B-fa-microvqa.n\endcsname{0}
\expandafter\def\csname regval:M-e2-B-fa-microvqa.d\endcsname{39}
\expandafter\def\csname regval:M-e2-B-answered-microvqa\endcsname{1.000}
\expandafter\def\csname regval:M-e2-B-answered-microvqa.n\endcsname{39}
\expandafter\def\csname regval:M-e2-B-answered-microvqa.d\endcsname{39}
\expandafter\def\csname regval:M-e2-B-fa-mmmu_pro\endcsname{0.128}
\expandafter\def\csname regval:M-e2-B-fa-mmmu_pro.n\endcsname{5}
\expandafter\def\csname regval:M-e2-B-fa-mmmu_pro.d\endcsname{39}
\expandafter\def\csname regval:M-e2-B-answered-mmmu_pro\endcsname{0.872}
\expandafter\def\csname regval:M-e2-B-answered-mmmu_pro.n\endcsname{34}
\expandafter\def\csname regval:M-e2-B-answered-mmmu_pro.d\endcsname{39}
\expandafter\def\csname regval:M-e2-B-fa-pathvqa\endcsname{0.615}
\expandafter\def\csname regval:M-e2-B-fa-pathvqa.n\endcsname{24}
\expandafter\def\csname regval:M-e2-B-fa-pathvqa.d\endcsname{39}
\expandafter\def\csname regval:M-e2-B-answered-pathvqa\endcsname{0.385}
\expandafter\def\csname regval:M-e2-B-answered-pathvqa.n\endcsname{15}
\expandafter\def\csname regval:M-e2-B-answered-pathvqa.d\endcsname{39}
\expandafter\def\csname regval:M-e2-B-fa-rexvqa\endcsname{0.179}
\expandafter\def\csname regval:M-e2-B-fa-rexvqa.n\endcsname{7}
\expandafter\def\csname regval:M-e2-B-fa-rexvqa.d\endcsname{39}
\expandafter\def\csname regval:M-e2-B-answered-rexvqa\endcsname{0.821}
\expandafter\def\csname regval:M-e2-B-answered-rexvqa.n\endcsname{32}
\expandafter\def\csname regval:M-e2-B-answered-rexvqa.d\endcsname{39}
\expandafter\def\csname regval:M-e2-B-fa-textvqa\endcsname{0.128}
\expandafter\def\csname regval:M-e2-B-fa-textvqa.n\endcsname{5}
\expandafter\def\csname regval:M-e2-B-fa-textvqa.d\endcsname{39}
\expandafter\def\csname regval:M-e2-B-answered-textvqa\endcsname{0.872}
\expandafter\def\csname regval:M-e2-B-answered-textvqa.n\endcsname{34}
\expandafter\def\csname regval:M-e2-B-answered-textvqa.d\endcsname{39}
\expandafter\def\csname regval:M-e2-B-fa-video_mmmu\endcsname{0.051}
\expandafter\def\csname regval:M-e2-B-fa-video_mmmu.n\endcsname{2}
\expandafter\def\csname regval:M-e2-B-fa-video_mmmu.d\endcsname{39}
\expandafter\def\csname regval:M-e2-B-answered-video_mmmu\endcsname{0.949}
\expandafter\def\csname regval:M-e2-B-answered-video_mmmu.n\endcsname{37}
\expandafter\def\csname regval:M-e2-B-answered-video_mmmu.d\endcsname{39}
\expandafter\def\csname regval:M-e3-B-fa-chest_vqa\endcsname{0.205}
\expandafter\def\csname regval:M-e3-B-fa-chest_vqa.n\endcsname{8}
\expandafter\def\csname regval:M-e3-B-fa-chest_vqa.d\endcsname{39}
\expandafter\def\csname regval:M-e3-B-answered-chest_vqa\endcsname{0.795}
\expandafter\def\csname regval:M-e3-B-answered-chest_vqa.n\endcsname{31}
\expandafter\def\csname regval:M-e3-B-answered-chest_vqa.d\endcsname{39}
\expandafter\def\csname regval:M-e3-B-fa-docvqa\endcsname{0.026}
\expandafter\def\csname regval:M-e3-B-fa-docvqa.n\endcsname{1}
\expandafter\def\csname regval:M-e3-B-fa-docvqa.d\endcsname{39}
\expandafter\def\csname regval:M-e3-B-answered-docvqa\endcsname{0.974}
\expandafter\def\csname regval:M-e3-B-answered-docvqa.n\endcsname{38}
\expandafter\def\csname regval:M-e3-B-answered-docvqa.d\endcsname{39}
\expandafter\def\csname regval:M-e3-B-fa-infovqa\endcsname{0.077}
\expandafter\def\csname regval:M-e3-B-fa-infovqa.n\endcsname{3}
\expandafter\def\csname regval:M-e3-B-fa-infovqa.d\endcsname{39}
\expandafter\def\csname regval:M-e3-B-answered-infovqa\endcsname{0.923}
\expandafter\def\csname regval:M-e3-B-answered-infovqa.n\endcsname{36}
\expandafter\def\csname regval:M-e3-B-answered-infovqa.d\endcsname{39}
\expandafter\def\csname regval:M-e3-B-fa-medxpertqa_mm\endcsname{0.128}
\expandafter\def\csname regval:M-e3-B-fa-medxpertqa_mm.n\endcsname{5}
\expandafter\def\csname regval:M-e3-B-fa-medxpertqa_mm.d\endcsname{39}
\expandafter\def\csname regval:M-e3-B-answered-medxpertqa_mm\endcsname{0.872}
\expandafter\def\csname regval:M-e3-B-answered-medxpertqa_mm.n\endcsname{34}
\expandafter\def\csname regval:M-e3-B-answered-medxpertqa_mm.d\endcsname{39}
\expandafter\def\csname regval:M-e3-B-fa-microvqa\endcsname{0.077}
\expandafter\def\csname regval:M-e3-B-fa-microvqa.n\endcsname{3}
\expandafter\def\csname regval:M-e3-B-fa-microvqa.d\endcsname{39}
\expandafter\def\csname regval:M-e3-B-answered-microvqa\endcsname{0.923}
\expandafter\def\csname regval:M-e3-B-answered-microvqa.n\endcsname{36}
\expandafter\def\csname regval:M-e3-B-answered-microvqa.d\endcsname{39}
\expandafter\def\csname regval:M-e3-B-fa-mmmu_pro\endcsname{0.000}
\expandafter\def\csname regval:M-e3-B-fa-mmmu_pro.n\endcsname{0}
\expandafter\def\csname regval:M-e3-B-fa-mmmu_pro.d\endcsname{39}
\expandafter\def\csname regval:M-e3-B-answered-mmmu_pro\endcsname{1.000}
\expandafter\def\csname regval:M-e3-B-answered-mmmu_pro.n\endcsname{39}
\expandafter\def\csname regval:M-e3-B-answered-mmmu_pro.d\endcsname{39}
\expandafter\def\csname regval:M-e3-B-fa-pathvqa\endcsname{0.641}
\expandafter\def\csname regval:M-e3-B-fa-pathvqa.n\endcsname{25}
\expandafter\def\csname regval:M-e3-B-fa-pathvqa.d\endcsname{39}
\expandafter\def\csname regval:M-e3-B-answered-pathvqa\endcsname{0.359}
\expandafter\def\csname regval:M-e3-B-answered-pathvqa.n\endcsname{14}
\expandafter\def\csname regval:M-e3-B-answered-pathvqa.d\endcsname{39}
\expandafter\def\csname regval:M-e3-B-fa-rexvqa\endcsname{0.282}
\expandafter\def\csname regval:M-e3-B-fa-rexvqa.n\endcsname{11}
\expandafter\def\csname regval:M-e3-B-fa-rexvqa.d\endcsname{39}
\expandafter\def\csname regval:M-e3-B-answered-rexvqa\endcsname{0.718}
\expandafter\def\csname regval:M-e3-B-answered-rexvqa.n\endcsname{28}
\expandafter\def\csname regval:M-e3-B-answered-rexvqa.d\endcsname{39}
\expandafter\def\csname regval:M-e3-B-fa-textvqa\endcsname{0.077}
\expandafter\def\csname regval:M-e3-B-fa-textvqa.n\endcsname{3}
\expandafter\def\csname regval:M-e3-B-fa-textvqa.d\endcsname{39}
\expandafter\def\csname regval:M-e3-B-answered-textvqa\endcsname{0.923}
\expandafter\def\csname regval:M-e3-B-answered-textvqa.n\endcsname{36}
\expandafter\def\csname regval:M-e3-B-answered-textvqa.d\endcsname{39}
\expandafter\def\csname regval:M-e3-B-fa-video_mmmu\endcsname{0.000}
\expandafter\def\csname regval:M-e3-B-fa-video_mmmu.n\endcsname{0}
\expandafter\def\csname regval:M-e3-B-fa-video_mmmu.d\endcsname{39}
\expandafter\def\csname regval:M-e3-B-answered-video_mmmu\endcsname{1.000}
\expandafter\def\csname regval:M-e3-B-answered-video_mmmu.n\endcsname{39}
\expandafter\def\csname regval:M-e3-B-answered-video_mmmu.d\endcsname{39}
\expandafter\def\csname regval:M-e4-B-fa-chest_vqa\endcsname{0.359}
\expandafter\def\csname regval:M-e4-B-fa-chest_vqa.n\endcsname{14}
\expandafter\def\csname regval:M-e4-B-fa-chest_vqa.d\endcsname{39}
\expandafter\def\csname regval:M-e4-B-answered-chest_vqa\endcsname{0.641}
\expandafter\def\csname regval:M-e4-B-answered-chest_vqa.n\endcsname{25}
\expandafter\def\csname regval:M-e4-B-answered-chest_vqa.d\endcsname{39}
\expandafter\def\csname regval:M-e4-B-fa-docvqa\endcsname{0.231}
\expandafter\def\csname regval:M-e4-B-fa-docvqa.n\endcsname{9}
\expandafter\def\csname regval:M-e4-B-fa-docvqa.d\endcsname{39}
\expandafter\def\csname regval:M-e4-B-answered-docvqa\endcsname{0.769}
\expandafter\def\csname regval:M-e4-B-answered-docvqa.n\endcsname{30}
\expandafter\def\csname regval:M-e4-B-answered-docvqa.d\endcsname{39}
\expandafter\def\csname regval:M-e4-B-fa-infovqa\endcsname{0.179}
\expandafter\def\csname regval:M-e4-B-fa-infovqa.n\endcsname{7}
\expandafter\def\csname regval:M-e4-B-fa-infovqa.d\endcsname{39}
\expandafter\def\csname regval:M-e4-B-answered-infovqa\endcsname{0.821}
\expandafter\def\csname regval:M-e4-B-answered-infovqa.n\endcsname{32}
\expandafter\def\csname regval:M-e4-B-answered-infovqa.d\endcsname{39}
\expandafter\def\csname regval:M-e4-B-fa-medxpertqa_mm\endcsname{0.308}
\expandafter\def\csname regval:M-e4-B-fa-medxpertqa_mm.n\endcsname{12}
\expandafter\def\csname regval:M-e4-B-fa-medxpertqa_mm.d\endcsname{39}
\expandafter\def\csname regval:M-e4-B-answered-medxpertqa_mm\endcsname{0.692}
\expandafter\def\csname regval:M-e4-B-answered-medxpertqa_mm.n\endcsname{27}
\expandafter\def\csname regval:M-e4-B-answered-medxpertqa_mm.d\endcsname{39}
\expandafter\def\csname regval:M-e4-B-fa-microvqa\endcsname{0.128}
\expandafter\def\csname regval:M-e4-B-fa-microvqa.n\endcsname{5}
\expandafter\def\csname regval:M-e4-B-fa-microvqa.d\endcsname{39}
\expandafter\def\csname regval:M-e4-B-answered-microvqa\endcsname{0.872}
\expandafter\def\csname regval:M-e4-B-answered-microvqa.n\endcsname{34}
\expandafter\def\csname regval:M-e4-B-answered-microvqa.d\endcsname{39}
\expandafter\def\csname regval:M-e4-B-fa-mmmu_pro\endcsname{0.026}
\expandafter\def\csname regval:M-e4-B-fa-mmmu_pro.n\endcsname{1}
\expandafter\def\csname regval:M-e4-B-fa-mmmu_pro.d\endcsname{39}
\expandafter\def\csname regval:M-e4-B-answered-mmmu_pro\endcsname{0.974}
\expandafter\def\csname regval:M-e4-B-answered-mmmu_pro.n\endcsname{38}
\expandafter\def\csname regval:M-e4-B-answered-mmmu_pro.d\endcsname{39}
\expandafter\def\csname regval:M-e4-B-fa-pathvqa\endcsname{0.513}
\expandafter\def\csname regval:M-e4-B-fa-pathvqa.n\endcsname{20}
\expandafter\def\csname regval:M-e4-B-fa-pathvqa.d\endcsname{39}
\expandafter\def\csname regval:M-e4-B-answered-pathvqa\endcsname{0.487}
\expandafter\def\csname regval:M-e4-B-answered-pathvqa.n\endcsname{19}
\expandafter\def\csname regval:M-e4-B-answered-pathvqa.d\endcsname{39}
\expandafter\def\csname regval:M-e4-B-fa-rexvqa\endcsname{0.000}
\expandafter\def\csname regval:M-e4-B-fa-rexvqa.n\endcsname{0}
\expandafter\def\csname regval:M-e4-B-fa-rexvqa.d\endcsname{39}
\expandafter\def\csname regval:M-e4-B-answered-rexvqa\endcsname{1.000}
\expandafter\def\csname regval:M-e4-B-answered-rexvqa.n\endcsname{39}
\expandafter\def\csname regval:M-e4-B-answered-rexvqa.d\endcsname{39}
\expandafter\def\csname regval:M-e4-B-fa-textvqa\endcsname{0.026}
\expandafter\def\csname regval:M-e4-B-fa-textvqa.n\endcsname{1}
\expandafter\def\csname regval:M-e4-B-fa-textvqa.d\endcsname{39}
\expandafter\def\csname regval:M-e4-B-answered-textvqa\endcsname{0.974}
\expandafter\def\csname regval:M-e4-B-answered-textvqa.n\endcsname{38}
\expandafter\def\csname regval:M-e4-B-answered-textvqa.d\endcsname{39}
\expandafter\def\csname regval:M-e4-B-fa-video_mmmu\endcsname{0.026}
\expandafter\def\csname regval:M-e4-B-fa-video_mmmu.n\endcsname{1}
\expandafter\def\csname regval:M-e4-B-fa-video_mmmu.d\endcsname{39}
\expandafter\def\csname regval:M-e4-B-answered-video_mmmu\endcsname{0.974}
\expandafter\def\csname regval:M-e4-B-answered-video_mmmu.n\endcsname{38}
\expandafter\def\csname regval:M-e4-B-answered-video_mmmu.d\endcsname{39}
\expandafter\def\csname regval:M-e1-U-id-chest_vqa\endcsname{0.564}
\expandafter\def\csname regval:M-e1-U-id-chest_vqa.n\endcsname{22}
\expandafter\def\csname regval:M-e1-U-id-chest_vqa.d\endcsname{39}
\expandafter\def\csname regval:M-e1-U-id-docvqa\endcsname{0.077}
\expandafter\def\csname regval:M-e1-U-id-docvqa.n\endcsname{3}
\expandafter\def\csname regval:M-e1-U-id-docvqa.d\endcsname{39}
\expandafter\def\csname regval:M-e1-U-id-infovqa\endcsname{0.000}
\expandafter\def\csname regval:M-e1-U-id-infovqa.n\endcsname{0}
\expandafter\def\csname regval:M-e1-U-id-infovqa.d\endcsname{39}
\expandafter\def\csname regval:M-e1-U-id-medxpertqa_mm\endcsname{0.231}
\expandafter\def\csname regval:M-e1-U-id-medxpertqa_mm.n\endcsname{9}
\expandafter\def\csname regval:M-e1-U-id-medxpertqa_mm.d\endcsname{39}
\expandafter\def\csname regval:M-e1-U-id-microvqa\endcsname{0.308}
\expandafter\def\csname regval:M-e1-U-id-microvqa.n\endcsname{12}
\expandafter\def\csname regval:M-e1-U-id-microvqa.d\endcsname{39}
\expandafter\def\csname regval:M-e1-U-id-mmmu_pro\endcsname{0.051}
\expandafter\def\csname regval:M-e1-U-id-mmmu_pro.n\endcsname{2}
\expandafter\def\csname regval:M-e1-U-id-mmmu_pro.d\endcsname{39}
\expandafter\def\csname regval:M-e1-U-id-pathvqa\endcsname{0.333}
\expandafter\def\csname regval:M-e1-U-id-pathvqa.n\endcsname{13}
\expandafter\def\csname regval:M-e1-U-id-pathvqa.d\endcsname{39}
\expandafter\def\csname regval:M-e1-U-id-rexvqa\endcsname{0.590}
\expandafter\def\csname regval:M-e1-U-id-rexvqa.n\endcsname{23}
\expandafter\def\csname regval:M-e1-U-id-rexvqa.d\endcsname{39}
\expandafter\def\csname regval:M-e1-U-id-textvqa\endcsname{0.051}
\expandafter\def\csname regval:M-e1-U-id-textvqa.n\endcsname{2}
\expandafter\def\csname regval:M-e1-U-id-textvqa.d\endcsname{39}
\expandafter\def\csname regval:M-e1-U-id-video_mmmu\endcsname{0.128}
\expandafter\def\csname regval:M-e1-U-id-video_mmmu.n\endcsname{5}
\expandafter\def\csname regval:M-e1-U-id-video_mmmu.d\endcsname{39}
\expandafter\def\csname regval:M-e1-H6-spearman\endcsname{0.820}
\expandafter\def\csname regval:M-e1-H6-spearman.d\endcsname{10}
\expandafter\def\csname regval:M-e3-U-id-chest_vqa\endcsname{0.590}
\expandafter\def\csname regval:M-e3-U-id-chest_vqa.n\endcsname{23}
\expandafter\def\csname regval:M-e3-U-id-chest_vqa.d\endcsname{39}
\expandafter\def\csname regval:M-e3-U-id-docvqa\endcsname{0.077}
\expandafter\def\csname regval:M-e3-U-id-docvqa.n\endcsname{3}
\expandafter\def\csname regval:M-e3-U-id-docvqa.d\endcsname{39}
\expandafter\def\csname regval:M-e3-U-id-infovqa\endcsname{0.026}
\expandafter\def\csname regval:M-e3-U-id-infovqa.n\endcsname{1}
\expandafter\def\csname regval:M-e3-U-id-infovqa.d\endcsname{39}
\expandafter\def\csname regval:M-e3-U-id-medxpertqa_mm\endcsname{0.179}
\expandafter\def\csname regval:M-e3-U-id-medxpertqa_mm.n\endcsname{7}
\expandafter\def\csname regval:M-e3-U-id-medxpertqa_mm.d\endcsname{39}
\expandafter\def\csname regval:M-e3-U-id-microvqa\endcsname{0.282}
\expandafter\def\csname regval:M-e3-U-id-microvqa.n\endcsname{11}
\expandafter\def\csname regval:M-e3-U-id-microvqa.d\endcsname{39}
\expandafter\def\csname regval:M-e3-U-id-mmmu_pro\endcsname{0.077}
\expandafter\def\csname regval:M-e3-U-id-mmmu_pro.n\endcsname{3}
\expandafter\def\csname regval:M-e3-U-id-mmmu_pro.d\endcsname{39}
\expandafter\def\csname regval:M-e3-U-id-pathvqa\endcsname{0.282}
\expandafter\def\csname regval:M-e3-U-id-pathvqa.n\endcsname{11}
\expandafter\def\csname regval:M-e3-U-id-pathvqa.d\endcsname{39}
\expandafter\def\csname regval:M-e3-U-id-rexvqa\endcsname{0.615}
\expandafter\def\csname regval:M-e3-U-id-rexvqa.n\endcsname{24}
\expandafter\def\csname regval:M-e3-U-id-rexvqa.d\endcsname{39}
\expandafter\def\csname regval:M-e3-U-id-textvqa\endcsname{0.000}
\expandafter\def\csname regval:M-e3-U-id-textvqa.n\endcsname{0}
\expandafter\def\csname regval:M-e3-U-id-textvqa.d\endcsname{39}
\expandafter\def\csname regval:M-e3-U-id-video_mmmu\endcsname{0.231}
\expandafter\def\csname regval:M-e3-U-id-video_mmmu.n\endcsname{9}
\expandafter\def\csname regval:M-e3-U-id-video_mmmu.d\endcsname{39}
\expandafter\def\csname regval:M-e3-H6-spearman\endcsname{0.579}
\expandafter\def\csname regval:M-e3-H6-spearman.d\endcsname{10}
\expandafter\def\csname regval:M-contrast-U-e0-e1\endcsname{-0.662}
\expandafter\def\csname regval:M-contrast-U-e0-e1.d\endcsname{1{,}170}
\expandafter\def\csname regval:M-contrast-U-e0-e1.ci\endcsname{[-0.692, -0.632]}
\expandafter\def\csname regval:M-contrast-B-e0-e1\endcsname{-0.343}
\expandafter\def\csname regval:M-contrast-B-e0-e1.d\endcsname{1{,}170}
\expandafter\def\csname regval:M-contrast-B-e0-e1.ci\endcsname{[-0.380, -0.308]}
\expandafter\def\csname regval:M-contrast-U-e0-e3\endcsname{-0.665}
\expandafter\def\csname regval:M-contrast-U-e0-e3.d\endcsname{1{,}170}
\expandafter\def\csname regval:M-contrast-U-e0-e3.ci\endcsname{[-0.696, -0.636]}
\expandafter\def\csname regval:M-contrast-B-e0-e3\endcsname{-0.344}
\expandafter\def\csname regval:M-contrast-B-e0-e3.d\endcsname{1{,}170}
\expandafter\def\csname regval:M-contrast-B-e0-e3.ci\endcsname{[-0.381, -0.311]}
\expandafter\def\csname regval:M-contrast-U-e1-e3\endcsname{-0.003}
\expandafter\def\csname regval:M-contrast-U-e1-e3.d\endcsname{1{,}170}
\expandafter\def\csname regval:M-contrast-U-e1-e3.ci\endcsname{[-0.018, 0.014]}
\expandafter\def\csname regval:M-contrast-B-e1-e3\endcsname{-0.002}
\expandafter\def\csname regval:M-contrast-B-e1-e3.d\endcsname{1{,}170}
\expandafter\def\csname regval:M-contrast-B-e1-e3.ci\endcsname{[-0.016, 0.014]}
\expandafter\def\csname regval:M-contrast-U-e0-e2\endcsname{0.100}
\expandafter\def\csname regval:M-contrast-U-e0-e2.d\endcsname{1{,}170}
\expandafter\def\csname regval:M-contrast-U-e0-e2.ci\endcsname{[0.073, 0.127]}
\expandafter\def\csname regval:M-contrast-B-e0-e2\endcsname{0.368}
\expandafter\def\csname regval:M-contrast-B-e0-e2.d\endcsname{1{,}170}
\expandafter\def\csname regval:M-contrast-B-e0-e2.ci\endcsname{[0.338, 0.397]}
\expandafter\def\csname regval:M-contrast-U-e1-e4\endcsname{0.518}
\expandafter\def\csname regval:M-contrast-U-e1-e4.d\endcsname{1{,}170}
\expandafter\def\csname regval:M-contrast-U-e1-e4.ci\endcsname{[0.497, 0.538]}
\expandafter\def\csname regval:M-contrast-B-e1-e4\endcsname{0.402}
\expandafter\def\csname regval:M-contrast-B-e1-e4.d\endcsname{1{,}170}
\expandafter\def\csname regval:M-contrast-B-e1-e4.ci\endcsname{[0.376, 0.425]}
\expandafter\def\csname regval:M-contrast-U-e0-e5\endcsname{-0.041}
\expandafter\def\csname regval:M-contrast-U-e0-e5.d\endcsname{1{,}170}
\expandafter\def\csname regval:M-contrast-U-e0-e5.ci\endcsname{[-0.068, -0.014]}
\expandafter\def\csname regval:M-e0-dec-real-exact\endcsname{0}
\expandafter\def\csname regval:M-e0-dec-real-exact.n\endcsname{0}
\expandafter\def\csname regval:M-e0-dec-real-exact.d\endcsname{390}
\expandafter\def\csname regval:M-e0-dec-real-near\endcsname{0}
\expandafter\def\csname regval:M-e0-dec-real-near.n\endcsname{0}
\expandafter\def\csname regval:M-e0-dec-real-near.d\endcsname{390}
\expandafter\def\csname regval:M-e0-dec-real-refusalpattern\endcsname{58}
\expandafter\def\csname regval:M-e0-dec-real-refusalpattern.n\endcsname{58}
\expandafter\def\csname regval:M-e0-dec-real-refusalpattern.d\endcsname{390}
\expandafter\def\csname regval:M-e0-dec-real-empty\endcsname{0}
\expandafter\def\csname regval:M-e0-dec-real-empty.n\endcsname{0}
\expandafter\def\csname regval:M-e0-dec-real-empty.d\endcsname{390}
\expandafter\def\csname regval:M-e0-dec-real-answer\endcsname{332}
\expandafter\def\csname regval:M-e0-dec-real-answer.n\endcsname{332}
\expandafter\def\csname regval:M-e0-dec-real-answer.d\endcsname{390}
\expandafter\def\csname regval:M-e0-dec-blank-exact\endcsname{0}
\expandafter\def\csname regval:M-e0-dec-blank-exact.n\endcsname{0}
\expandafter\def\csname regval:M-e0-dec-blank-exact.d\endcsname{390}
\expandafter\def\csname regval:M-e0-dec-blank-near\endcsname{0}
\expandafter\def\csname regval:M-e0-dec-blank-near.n\endcsname{0}
\expandafter\def\csname regval:M-e0-dec-blank-near.d\endcsname{390}
\expandafter\def\csname regval:M-e0-dec-blank-refusalpattern\endcsname{141}
\expandafter\def\csname regval:M-e0-dec-blank-refusalpattern.n\endcsname{141}
\expandafter\def\csname regval:M-e0-dec-blank-refusalpattern.d\endcsname{390}
\expandafter\def\csname regval:M-e0-dec-blank-empty\endcsname{0}
\expandafter\def\csname regval:M-e0-dec-blank-empty.n\endcsname{0}
\expandafter\def\csname regval:M-e0-dec-blank-empty.d\endcsname{390}
\expandafter\def\csname regval:M-e0-dec-blank-answer\endcsname{249}
\expandafter\def\csname regval:M-e0-dec-blank-answer.n\endcsname{249}
\expandafter\def\csname regval:M-e0-dec-blank-answer.d\endcsname{390}
\expandafter\def\csname regval:M-e0-dec-id_unrelated-exact\endcsname{0}
\expandafter\def\csname regval:M-e0-dec-id_unrelated-exact.n\endcsname{0}
\expandafter\def\csname regval:M-e0-dec-id_unrelated-exact.d\endcsname{390}
\expandafter\def\csname regval:M-e0-dec-id_unrelated-near\endcsname{0}
\expandafter\def\csname regval:M-e0-dec-id_unrelated-near.n\endcsname{0}
\expandafter\def\csname regval:M-e0-dec-id_unrelated-near.d\endcsname{390}
\expandafter\def\csname regval:M-e0-dec-id_unrelated-refusalpattern\endcsname{159}
\expandafter\def\csname regval:M-e0-dec-id_unrelated-refusalpattern.n\endcsname{159}
\expandafter\def\csname regval:M-e0-dec-id_unrelated-refusalpattern.d\endcsname{390}
\expandafter\def\csname regval:M-e0-dec-id_unrelated-empty\endcsname{0}
\expandafter\def\csname regval:M-e0-dec-id_unrelated-empty.n\endcsname{0}
\expandafter\def\csname regval:M-e0-dec-id_unrelated-empty.d\endcsname{390}
\expandafter\def\csname regval:M-e0-dec-id_unrelated-answer\endcsname{231}
\expandafter\def\csname regval:M-e0-dec-id_unrelated-answer.n\endcsname{231}
\expandafter\def\csname regval:M-e0-dec-id_unrelated-answer.d\endcsname{390}
\expandafter\def\csname regval:M-e0-dec-od_unrelated-exact\endcsname{0}
\expandafter\def\csname regval:M-e0-dec-od_unrelated-exact.n\endcsname{0}
\expandafter\def\csname regval:M-e0-dec-od_unrelated-exact.d\endcsname{390}
\expandafter\def\csname regval:M-e0-dec-od_unrelated-near\endcsname{0}
\expandafter\def\csname regval:M-e0-dec-od_unrelated-near.n\endcsname{0}
\expandafter\def\csname regval:M-e0-dec-od_unrelated-near.d\endcsname{390}
\expandafter\def\csname regval:M-e0-dec-od_unrelated-refusalpattern\endcsname{191}
\expandafter\def\csname regval:M-e0-dec-od_unrelated-refusalpattern.n\endcsname{191}
\expandafter\def\csname regval:M-e0-dec-od_unrelated-refusalpattern.d\endcsname{390}
\expandafter\def\csname regval:M-e0-dec-od_unrelated-empty\endcsname{0}
\expandafter\def\csname regval:M-e0-dec-od_unrelated-empty.n\endcsname{0}
\expandafter\def\csname regval:M-e0-dec-od_unrelated-empty.d\endcsname{390}
\expandafter\def\csname regval:M-e0-dec-od_unrelated-answer\endcsname{199}
\expandafter\def\csname regval:M-e0-dec-od_unrelated-answer.n\endcsname{199}
\expandafter\def\csname regval:M-e0-dec-od_unrelated-answer.d\endcsname{390}
\expandafter\def\csname regval:M-e0-dec-F\endcsname{0.000}
\expandafter\def\csname regval:M-e0-dec-F.n\endcsname{0}
\expandafter\def\csname regval:M-e0-dec-F.d\endcsname{390}
\expandafter\def\csname regval:M-e0-dec-F.ci\endcsname{[0.000, 0.010]}
\expandafter\def\csname regval:M-e0-dec-M\endcsname{1.000}
\expandafter\def\csname regval:M-e0-dec-M.n\endcsname{1{,}170}
\expandafter\def\csname regval:M-e0-dec-M.d\endcsname{1{,}170}
\expandafter\def\csname regval:M-e0-dec-M.ci\endcsname{[0.997, 1.000]}
\expandafter\def\csname regval:M-e0-dec-sum\endcsname{0.750}
\expandafter\def\csname regval:M-e0-dec-sum.d\endcsname{1{,}560}
\expandafter\def\csname regval:M-e0-dec-sum.ci\endcsname{[0.750, 0.750]}
\expandafter\def\csname regval:M-e0-dec-se\endcsname{0.011}
\expandafter\def\csname regval:M-e0-dec-se.d\endcsname{1{,}560}
\expandafter\def\csname regval:M-e0-dec-boot-sd\endcsname{0.000}
\expandafter\def\csname regval:M-e0-dec-boot-sd.d\endcsname{390}
\expandafter\def\csname regval:M-e0-dec-exact-sum\endcsname{0.750}
\expandafter\def\csname regval:M-e0-dec-exact-sum.d\endcsname{1{,}560}
\expandafter\def\csname regval:M-e0-dec-exact-sum.ci\endcsname{[0.750, 0.750]}
\expandafter\def\csname regval:M-e0-dec-exact-F\endcsname{0.000}
\expandafter\def\csname regval:M-e0-dec-exact-F.n\endcsname{0}
\expandafter\def\csname regval:M-e0-dec-exact-F.d\endcsname{390}
\expandafter\def\csname regval:M-e0-dec-exact-M\endcsname{1.000}
\expandafter\def\csname regval:M-e0-dec-exact-M.n\endcsname{1{,}170}
\expandafter\def\csname regval:M-e0-dec-exact-M.d\endcsname{1{,}170}
\expandafter\def\csname regval:M-e0-dec-blank-answerrate\endcsname{1.000}
\expandafter\def\csname regval:M-e0-dec-blank-answerrate.n\endcsname{390}
\expandafter\def\csname regval:M-e0-dec-blank-answerrate.d\endcsname{390}
\expandafter\def\csname regval:M-e0-dec-id_unrelated-answerrate\endcsname{1.000}
\expandafter\def\csname regval:M-e0-dec-id_unrelated-answerrate.n\endcsname{390}
\expandafter\def\csname regval:M-e0-dec-id_unrelated-answerrate.d\endcsname{390}
\expandafter\def\csname regval:M-e0-dec-od_unrelated-answerrate\endcsname{1.000}
\expandafter\def\csname regval:M-e0-dec-od_unrelated-answerrate.n\endcsname{390}
\expandafter\def\csname regval:M-e0-dec-od_unrelated-answerrate.d\endcsname{390}
\expandafter\def\csname regval:M-e0-dec-agree-real\endcsname{383}
\expandafter\def\csname regval:M-e0-dec-agree-real.n\endcsname{383}
\expandafter\def\csname regval:M-e0-dec-agree-real.d\endcsname{390}
\expandafter\def\csname regval:M-e0-dec-agree-inv\endcsname{0}
\expandafter\def\csname regval:M-e0-dec-agree-inv.n\endcsname{0}
\expandafter\def\csname regval:M-e0-dec-agree-inv.d\endcsname{1}
\expandafter\def\csname regval:M-e0-dec-inv-declined\endcsname{0}
\expandafter\def\csname regval:M-e0-dec-inv-declined.n\endcsname{0}
\expandafter\def\csname regval:M-e0-dec-inv-declined.d\endcsname{1{,}170}
\expandafter\def\csname regval:M-e1-dec-real-exact\endcsname{49}
\expandafter\def\csname regval:M-e1-dec-real-exact.n\endcsname{49}
\expandafter\def\csname regval:M-e1-dec-real-exact.d\endcsname{390}
\expandafter\def\csname regval:M-e1-dec-real-near\endcsname{0}
\expandafter\def\csname regval:M-e1-dec-real-near.n\endcsname{0}
\expandafter\def\csname regval:M-e1-dec-real-near.d\endcsname{390}
\expandafter\def\csname regval:M-e1-dec-real-refusalpattern\endcsname{0}
\expandafter\def\csname regval:M-e1-dec-real-refusalpattern.n\endcsname{0}
\expandafter\def\csname regval:M-e1-dec-real-refusalpattern.d\endcsname{390}
\expandafter\def\csname regval:M-e1-dec-real-empty\endcsname{0}
\expandafter\def\csname regval:M-e1-dec-real-empty.n\endcsname{0}
\expandafter\def\csname regval:M-e1-dec-real-empty.d\endcsname{390}
\expandafter\def\csname regval:M-e1-dec-real-answer\endcsname{341}
\expandafter\def\csname regval:M-e1-dec-real-answer.n\endcsname{341}
\expandafter\def\csname regval:M-e1-dec-real-answer.d\endcsname{390}
\expandafter\def\csname regval:M-e1-dec-blank-exact\endcsname{389}
\expandafter\def\csname regval:M-e1-dec-blank-exact.n\endcsname{389}
\expandafter\def\csname regval:M-e1-dec-blank-exact.d\endcsname{390}
\expandafter\def\csname regval:M-e1-dec-blank-near\endcsname{0}
\expandafter\def\csname regval:M-e1-dec-blank-near.n\endcsname{0}
\expandafter\def\csname regval:M-e1-dec-blank-near.d\endcsname{390}
\expandafter\def\csname regval:M-e1-dec-blank-refusalpattern\endcsname{0}
\expandafter\def\csname regval:M-e1-dec-blank-refusalpattern.n\endcsname{0}
\expandafter\def\csname regval:M-e1-dec-blank-refusalpattern.d\endcsname{390}
\expandafter\def\csname regval:M-e1-dec-blank-empty\endcsname{0}
\expandafter\def\csname regval:M-e1-dec-blank-empty.n\endcsname{0}
\expandafter\def\csname regval:M-e1-dec-blank-empty.d\endcsname{390}
\expandafter\def\csname regval:M-e1-dec-blank-answer\endcsname{1}
\expandafter\def\csname regval:M-e1-dec-blank-answer.n\endcsname{1}
\expandafter\def\csname regval:M-e1-dec-blank-answer.d\endcsname{390}
\expandafter\def\csname regval:M-e1-dec-id_unrelated-exact\endcsname{299}
\expandafter\def\csname regval:M-e1-dec-id_unrelated-exact.n\endcsname{299}
\expandafter\def\csname regval:M-e1-dec-id_unrelated-exact.d\endcsname{390}
\expandafter\def\csname regval:M-e1-dec-id_unrelated-near\endcsname{0}
\expandafter\def\csname regval:M-e1-dec-id_unrelated-near.n\endcsname{0}
\expandafter\def\csname regval:M-e1-dec-id_unrelated-near.d\endcsname{390}
\expandafter\def\csname regval:M-e1-dec-id_unrelated-refusalpattern\endcsname{0}
\expandafter\def\csname regval:M-e1-dec-id_unrelated-refusalpattern.n\endcsname{0}
\expandafter\def\csname regval:M-e1-dec-id_unrelated-refusalpattern.d\endcsname{390}
\expandafter\def\csname regval:M-e1-dec-id_unrelated-empty\endcsname{0}
\expandafter\def\csname regval:M-e1-dec-id_unrelated-empty.n\endcsname{0}
\expandafter\def\csname regval:M-e1-dec-id_unrelated-empty.d\endcsname{390}
\expandafter\def\csname regval:M-e1-dec-id_unrelated-answer\endcsname{91}
\expandafter\def\csname regval:M-e1-dec-id_unrelated-answer.n\endcsname{91}
\expandafter\def\csname regval:M-e1-dec-id_unrelated-answer.d\endcsname{390}
\expandafter\def\csname regval:M-e1-dec-od_unrelated-exact\endcsname{368}
\expandafter\def\csname regval:M-e1-dec-od_unrelated-exact.n\endcsname{368}
\expandafter\def\csname regval:M-e1-dec-od_unrelated-exact.d\endcsname{390}
\expandafter\def\csname regval:M-e1-dec-od_unrelated-near\endcsname{1}
\expandafter\def\csname regval:M-e1-dec-od_unrelated-near.n\endcsname{1}
\expandafter\def\csname regval:M-e1-dec-od_unrelated-near.d\endcsname{390}
\expandafter\def\csname regval:M-e1-dec-od_unrelated-refusalpattern\endcsname{0}
\expandafter\def\csname regval:M-e1-dec-od_unrelated-refusalpattern.n\endcsname{0}
\expandafter\def\csname regval:M-e1-dec-od_unrelated-refusalpattern.d\endcsname{390}
\expandafter\def\csname regval:M-e1-dec-od_unrelated-empty\endcsname{0}
\expandafter\def\csname regval:M-e1-dec-od_unrelated-empty.n\endcsname{0}
\expandafter\def\csname regval:M-e1-dec-od_unrelated-empty.d\endcsname{390}
\expandafter\def\csname regval:M-e1-dec-od_unrelated-answer\endcsname{21}
\expandafter\def\csname regval:M-e1-dec-od_unrelated-answer.n\endcsname{21}
\expandafter\def\csname regval:M-e1-dec-od_unrelated-answer.d\endcsname{390}
\expandafter\def\csname regval:M-e1-dec-F\endcsname{0.126}
\expandafter\def\csname regval:M-e1-dec-F.n\endcsname{49}
\expandafter\def\csname regval:M-e1-dec-F.d\endcsname{390}
\expandafter\def\csname regval:M-e1-dec-F.ci\endcsname{[0.096, 0.162]}
\expandafter\def\csname regval:M-e1-dec-M\endcsname{0.097}
\expandafter\def\csname regval:M-e1-dec-M.n\endcsname{113}
\expandafter\def\csname regval:M-e1-dec-M.d\endcsname{1{,}170}
\expandafter\def\csname regval:M-e1-dec-M.ci\endcsname{[0.081, 0.115]}
\expandafter\def\csname regval:M-e1-dec-sum\endcsname{0.104}
\expandafter\def\csname regval:M-e1-dec-sum.d\endcsname{1{,}560}
\expandafter\def\csname regval:M-e1-dec-sum.ci\endcsname{[0.090, 0.118]}
\expandafter\def\csname regval:M-e1-dec-se\endcsname{0.008}
\expandafter\def\csname regval:M-e1-dec-se.d\endcsname{1{,}560}
\expandafter\def\csname regval:M-e1-dec-boot-sd\endcsname{0.007}
\expandafter\def\csname regval:M-e1-dec-boot-sd.d\endcsname{390}
\expandafter\def\csname regval:M-e1-dec-exact-sum\endcsname{0.104}
\expandafter\def\csname regval:M-e1-dec-exact-sum.d\endcsname{1{,}560}
\expandafter\def\csname regval:M-e1-dec-exact-sum.ci\endcsname{[0.091, 0.119]}
\expandafter\def\csname regval:M-e1-dec-exact-F\endcsname{0.126}
\expandafter\def\csname regval:M-e1-dec-exact-F.n\endcsname{49}
\expandafter\def\csname regval:M-e1-dec-exact-F.d\endcsname{390}
\expandafter\def\csname regval:M-e1-dec-exact-M\endcsname{0.097}
\expandafter\def\csname regval:M-e1-dec-exact-M.n\endcsname{114}
\expandafter\def\csname regval:M-e1-dec-exact-M.d\endcsname{1{,}170}
\expandafter\def\csname regval:M-e1-dec-blank-answerrate\endcsname{0.003}
\expandafter\def\csname regval:M-e1-dec-blank-answerrate.n\endcsname{1}
\expandafter\def\csname regval:M-e1-dec-blank-answerrate.d\endcsname{390}
\expandafter\def\csname regval:M-e1-dec-id_unrelated-answerrate\endcsname{0.233}
\expandafter\def\csname regval:M-e1-dec-id_unrelated-answerrate.n\endcsname{91}
\expandafter\def\csname regval:M-e1-dec-id_unrelated-answerrate.d\endcsname{390}
\expandafter\def\csname regval:M-e1-dec-od_unrelated-answerrate\endcsname{0.054}
\expandafter\def\csname regval:M-e1-dec-od_unrelated-answerrate.n\endcsname{21}
\expandafter\def\csname regval:M-e1-dec-od_unrelated-answerrate.d\endcsname{390}
\expandafter\def\csname regval:M-e1-dec-agree-real\endcsname{390}
\expandafter\def\csname regval:M-e1-dec-agree-real.n\endcsname{390}
\expandafter\def\csname regval:M-e1-dec-agree-real.d\endcsname{390}
\expandafter\def\csname regval:M-e1-dec-agree-inv\endcsname{1{,}057}
\expandafter\def\csname regval:M-e1-dec-agree-inv.n\endcsname{1{,}057}
\expandafter\def\csname regval:M-e1-dec-agree-inv.d\endcsname{1{,}057}
\expandafter\def\csname regval:M-e1-dec-inv-declined\endcsname{1{,}057}
\expandafter\def\csname regval:M-e1-dec-inv-declined.n\endcsname{1{,}057}
\expandafter\def\csname regval:M-e1-dec-inv-declined.d\endcsname{1{,}170}
\expandafter\def\csname regval:M-e2-dec-real-exact\endcsname{0}
\expandafter\def\csname regval:M-e2-dec-real-exact.n\endcsname{0}
\expandafter\def\csname regval:M-e2-dec-real-exact.d\endcsname{390}
\expandafter\def\csname regval:M-e2-dec-real-near\endcsname{0}
\expandafter\def\csname regval:M-e2-dec-real-near.n\endcsname{0}
\expandafter\def\csname regval:M-e2-dec-real-near.d\endcsname{390}
\expandafter\def\csname regval:M-e2-dec-real-refusalpattern\endcsname{23}
\expandafter\def\csname regval:M-e2-dec-real-refusalpattern.n\endcsname{23}
\expandafter\def\csname regval:M-e2-dec-real-refusalpattern.d\endcsname{390}
\expandafter\def\csname regval:M-e2-dec-real-empty\endcsname{0}
\expandafter\def\csname regval:M-e2-dec-real-empty.n\endcsname{0}
\expandafter\def\csname regval:M-e2-dec-real-empty.d\endcsname{390}
\expandafter\def\csname regval:M-e2-dec-real-answer\endcsname{367}
\expandafter\def\csname regval:M-e2-dec-real-answer.n\endcsname{367}
\expandafter\def\csname regval:M-e2-dec-real-answer.d\endcsname{390}
\expandafter\def\csname regval:M-e2-dec-blank-exact\endcsname{0}
\expandafter\def\csname regval:M-e2-dec-blank-exact.n\endcsname{0}
\expandafter\def\csname regval:M-e2-dec-blank-exact.d\endcsname{390}
\expandafter\def\csname regval:M-e2-dec-blank-near\endcsname{0}
\expandafter\def\csname regval:M-e2-dec-blank-near.n\endcsname{0}
\expandafter\def\csname regval:M-e2-dec-blank-near.d\endcsname{390}
\expandafter\def\csname regval:M-e2-dec-blank-refusalpattern\endcsname{169}
\expandafter\def\csname regval:M-e2-dec-blank-refusalpattern.n\endcsname{169}
\expandafter\def\csname regval:M-e2-dec-blank-refusalpattern.d\endcsname{390}
\expandafter\def\csname regval:M-e2-dec-blank-empty\endcsname{0}
\expandafter\def\csname regval:M-e2-dec-blank-empty.n\endcsname{0}
\expandafter\def\csname regval:M-e2-dec-blank-empty.d\endcsname{390}
\expandafter\def\csname regval:M-e2-dec-blank-answer\endcsname{221}
\expandafter\def\csname regval:M-e2-dec-blank-answer.n\endcsname{221}
\expandafter\def\csname regval:M-e2-dec-blank-answer.d\endcsname{390}
\expandafter\def\csname regval:M-e2-dec-id_unrelated-exact\endcsname{0}
\expandafter\def\csname regval:M-e2-dec-id_unrelated-exact.n\endcsname{0}
\expandafter\def\csname regval:M-e2-dec-id_unrelated-exact.d\endcsname{390}
\expandafter\def\csname regval:M-e2-dec-id_unrelated-near\endcsname{0}
\expandafter\def\csname regval:M-e2-dec-id_unrelated-near.n\endcsname{0}
\expandafter\def\csname regval:M-e2-dec-id_unrelated-near.d\endcsname{390}
\expandafter\def\csname regval:M-e2-dec-id_unrelated-refusalpattern\endcsname{15}
\expandafter\def\csname regval:M-e2-dec-id_unrelated-refusalpattern.n\endcsname{15}
\expandafter\def\csname regval:M-e2-dec-id_unrelated-refusalpattern.d\endcsname{390}
\expandafter\def\csname regval:M-e2-dec-id_unrelated-empty\endcsname{1}
\expandafter\def\csname regval:M-e2-dec-id_unrelated-empty.n\endcsname{1}
\expandafter\def\csname regval:M-e2-dec-id_unrelated-empty.d\endcsname{390}
\expandafter\def\csname regval:M-e2-dec-id_unrelated-answer\endcsname{374}
\expandafter\def\csname regval:M-e2-dec-id_unrelated-answer.n\endcsname{374}
\expandafter\def\csname regval:M-e2-dec-id_unrelated-answer.d\endcsname{390}
\expandafter\def\csname regval:M-e2-dec-od_unrelated-exact\endcsname{0}
\expandafter\def\csname regval:M-e2-dec-od_unrelated-exact.n\endcsname{0}
\expandafter\def\csname regval:M-e2-dec-od_unrelated-exact.d\endcsname{390}
\expandafter\def\csname regval:M-e2-dec-od_unrelated-near\endcsname{0}
\expandafter\def\csname regval:M-e2-dec-od_unrelated-near.n\endcsname{0}
\expandafter\def\csname regval:M-e2-dec-od_unrelated-near.d\endcsname{390}
\expandafter\def\csname regval:M-e2-dec-od_unrelated-refusalpattern\endcsname{10}
\expandafter\def\csname regval:M-e2-dec-od_unrelated-refusalpattern.n\endcsname{10}
\expandafter\def\csname regval:M-e2-dec-od_unrelated-refusalpattern.d\endcsname{390}
\expandafter\def\csname regval:M-e2-dec-od_unrelated-empty\endcsname{0}
\expandafter\def\csname regval:M-e2-dec-od_unrelated-empty.n\endcsname{0}
\expandafter\def\csname regval:M-e2-dec-od_unrelated-empty.d\endcsname{390}
\expandafter\def\csname regval:M-e2-dec-od_unrelated-answer\endcsname{380}
\expandafter\def\csname regval:M-e2-dec-od_unrelated-answer.n\endcsname{380}
\expandafter\def\csname regval:M-e2-dec-od_unrelated-answer.d\endcsname{390}
\expandafter\def\csname regval:M-e2-dec-F\endcsname{0.000}
\expandafter\def\csname regval:M-e2-dec-F.n\endcsname{0}
\expandafter\def\csname regval:M-e2-dec-F.d\endcsname{390}
\expandafter\def\csname regval:M-e2-dec-F.ci\endcsname{[0.000, 0.010]}
\expandafter\def\csname regval:M-e2-dec-M\endcsname{1.000}
\expandafter\def\csname regval:M-e2-dec-M.n\endcsname{1{,}170}
\expandafter\def\csname regval:M-e2-dec-M.d\endcsname{1{,}170}
\expandafter\def\csname regval:M-e2-dec-M.ci\endcsname{[0.997, 1.000]}
\expandafter\def\csname regval:M-e2-dec-sum\endcsname{0.750}
\expandafter\def\csname regval:M-e2-dec-sum.d\endcsname{1{,}560}
\expandafter\def\csname regval:M-e2-dec-sum.ci\endcsname{[0.750, 0.750]}
\expandafter\def\csname regval:M-e2-dec-se\endcsname{0.011}
\expandafter\def\csname regval:M-e2-dec-se.d\endcsname{1{,}560}
\expandafter\def\csname regval:M-e2-dec-boot-sd\endcsname{0.000}
\expandafter\def\csname regval:M-e2-dec-boot-sd.d\endcsname{390}
\expandafter\def\csname regval:M-e2-dec-exact-sum\endcsname{0.750}
\expandafter\def\csname regval:M-e2-dec-exact-sum.d\endcsname{1{,}560}
\expandafter\def\csname regval:M-e2-dec-exact-sum.ci\endcsname{[0.750, 0.750]}
\expandafter\def\csname regval:M-e2-dec-exact-F\endcsname{0.000}
\expandafter\def\csname regval:M-e2-dec-exact-F.n\endcsname{0}
\expandafter\def\csname regval:M-e2-dec-exact-F.d\endcsname{390}
\expandafter\def\csname regval:M-e2-dec-exact-M\endcsname{1.000}
\expandafter\def\csname regval:M-e2-dec-exact-M.n\endcsname{1{,}170}
\expandafter\def\csname regval:M-e2-dec-exact-M.d\endcsname{1{,}170}
\expandafter\def\csname regval:M-e2-dec-blank-answerrate\endcsname{1.000}
\expandafter\def\csname regval:M-e2-dec-blank-answerrate.n\endcsname{390}
\expandafter\def\csname regval:M-e2-dec-blank-answerrate.d\endcsname{390}
\expandafter\def\csname regval:M-e2-dec-id_unrelated-answerrate\endcsname{1.000}
\expandafter\def\csname regval:M-e2-dec-id_unrelated-answerrate.n\endcsname{390}
\expandafter\def\csname regval:M-e2-dec-id_unrelated-answerrate.d\endcsname{390}
\expandafter\def\csname regval:M-e2-dec-od_unrelated-answerrate\endcsname{1.000}
\expandafter\def\csname regval:M-e2-dec-od_unrelated-answerrate.n\endcsname{390}
\expandafter\def\csname regval:M-e2-dec-od_unrelated-answerrate.d\endcsname{390}
\expandafter\def\csname regval:M-e2-dec-agree-real\endcsname{302}
\expandafter\def\csname regval:M-e2-dec-agree-real.n\endcsname{302}
\expandafter\def\csname regval:M-e2-dec-agree-real.d\endcsname{390}
\expandafter\def\csname regval:M-e2-dec-agree-inv\endcsname{0}
\expandafter\def\csname regval:M-e2-dec-agree-inv.n\endcsname{0}
\expandafter\def\csname regval:M-e2-dec-agree-inv.d\endcsname{1}
\expandafter\def\csname regval:M-e2-dec-inv-declined\endcsname{0}
\expandafter\def\csname regval:M-e2-dec-inv-declined.n\endcsname{0}
\expandafter\def\csname regval:M-e2-dec-inv-declined.d\endcsname{1{,}170}
\expandafter\def\csname regval:M-e3-dec-real-exact\endcsname{59}
\expandafter\def\csname regval:M-e3-dec-real-exact.n\endcsname{59}
\expandafter\def\csname regval:M-e3-dec-real-exact.d\endcsname{390}
\expandafter\def\csname regval:M-e3-dec-real-near\endcsname{0}
\expandafter\def\csname regval:M-e3-dec-real-near.n\endcsname{0}
\expandafter\def\csname regval:M-e3-dec-real-near.d\endcsname{390}
\expandafter\def\csname regval:M-e3-dec-real-refusalpattern\endcsname{0}
\expandafter\def\csname regval:M-e3-dec-real-refusalpattern.n\endcsname{0}
\expandafter\def\csname regval:M-e3-dec-real-refusalpattern.d\endcsname{390}
\expandafter\def\csname regval:M-e3-dec-real-empty\endcsname{0}
\expandafter\def\csname regval:M-e3-dec-real-empty.n\endcsname{0}
\expandafter\def\csname regval:M-e3-dec-real-empty.d\endcsname{390}
\expandafter\def\csname regval:M-e3-dec-real-answer\endcsname{331}
\expandafter\def\csname regval:M-e3-dec-real-answer.n\endcsname{331}
\expandafter\def\csname regval:M-e3-dec-real-answer.d\endcsname{390}
\expandafter\def\csname regval:M-e3-dec-blank-exact\endcsname{390}
\expandafter\def\csname regval:M-e3-dec-blank-exact.n\endcsname{390}
\expandafter\def\csname regval:M-e3-dec-blank-exact.d\endcsname{390}
\expandafter\def\csname regval:M-e3-dec-blank-near\endcsname{0}
\expandafter\def\csname regval:M-e3-dec-blank-near.n\endcsname{0}
\expandafter\def\csname regval:M-e3-dec-blank-near.d\endcsname{390}
\expandafter\def\csname regval:M-e3-dec-blank-refusalpattern\endcsname{0}
\expandafter\def\csname regval:M-e3-dec-blank-refusalpattern.n\endcsname{0}
\expandafter\def\csname regval:M-e3-dec-blank-refusalpattern.d\endcsname{390}
\expandafter\def\csname regval:M-e3-dec-blank-empty\endcsname{0}
\expandafter\def\csname regval:M-e3-dec-blank-empty.n\endcsname{0}
\expandafter\def\csname regval:M-e3-dec-blank-empty.d\endcsname{390}
\expandafter\def\csname regval:M-e3-dec-blank-answer\endcsname{0}
\expandafter\def\csname regval:M-e3-dec-blank-answer.n\endcsname{0}
\expandafter\def\csname regval:M-e3-dec-blank-answer.d\endcsname{390}
\expandafter\def\csname regval:M-e3-dec-id_unrelated-exact\endcsname{298}
\expandafter\def\csname regval:M-e3-dec-id_unrelated-exact.n\endcsname{298}
\expandafter\def\csname regval:M-e3-dec-id_unrelated-exact.d\endcsname{390}
\expandafter\def\csname regval:M-e3-dec-id_unrelated-near\endcsname{0}
\expandafter\def\csname regval:M-e3-dec-id_unrelated-near.n\endcsname{0}
\expandafter\def\csname regval:M-e3-dec-id_unrelated-near.d\endcsname{390}
\expandafter\def\csname regval:M-e3-dec-id_unrelated-refusalpattern\endcsname{0}
\expandafter\def\csname regval:M-e3-dec-id_unrelated-refusalpattern.n\endcsname{0}
\expandafter\def\csname regval:M-e3-dec-id_unrelated-refusalpattern.d\endcsname{390}
\expandafter\def\csname regval:M-e3-dec-id_unrelated-empty\endcsname{0}
\expandafter\def\csname regval:M-e3-dec-id_unrelated-empty.n\endcsname{0}
\expandafter\def\csname regval:M-e3-dec-id_unrelated-empty.d\endcsname{390}
\expandafter\def\csname regval:M-e3-dec-id_unrelated-answer\endcsname{92}
\expandafter\def\csname regval:M-e3-dec-id_unrelated-answer.n\endcsname{92}
\expandafter\def\csname regval:M-e3-dec-id_unrelated-answer.d\endcsname{390}
\expandafter\def\csname regval:M-e3-dec-od_unrelated-exact\endcsname{372}
\expandafter\def\csname regval:M-e3-dec-od_unrelated-exact.n\endcsname{372}
\expandafter\def\csname regval:M-e3-dec-od_unrelated-exact.d\endcsname{390}
\expandafter\def\csname regval:M-e3-dec-od_unrelated-near\endcsname{0}
\expandafter\def\csname regval:M-e3-dec-od_unrelated-near.n\endcsname{0}
\expandafter\def\csname regval:M-e3-dec-od_unrelated-near.d\endcsname{390}
\expandafter\def\csname regval:M-e3-dec-od_unrelated-refusalpattern\endcsname{0}
\expandafter\def\csname regval:M-e3-dec-od_unrelated-refusalpattern.n\endcsname{0}
\expandafter\def\csname regval:M-e3-dec-od_unrelated-refusalpattern.d\endcsname{390}
\expandafter\def\csname regval:M-e3-dec-od_unrelated-empty\endcsname{0}
\expandafter\def\csname regval:M-e3-dec-od_unrelated-empty.n\endcsname{0}
\expandafter\def\csname regval:M-e3-dec-od_unrelated-empty.d\endcsname{390}
\expandafter\def\csname regval:M-e3-dec-od_unrelated-answer\endcsname{18}
\expandafter\def\csname regval:M-e3-dec-od_unrelated-answer.n\endcsname{18}
\expandafter\def\csname regval:M-e3-dec-od_unrelated-answer.d\endcsname{390}
\expandafter\def\csname regval:M-e3-dec-F\endcsname{0.151}
\expandafter\def\csname regval:M-e3-dec-F.n\endcsname{59}
\expandafter\def\csname regval:M-e3-dec-F.d\endcsname{390}
\expandafter\def\csname regval:M-e3-dec-F.ci\endcsname{[0.119, 0.190]}
\expandafter\def\csname regval:M-e3-dec-M\endcsname{0.094}
\expandafter\def\csname regval:M-e3-dec-M.n\endcsname{110}
\expandafter\def\csname regval:M-e3-dec-M.d\endcsname{1{,}170}
\expandafter\def\csname regval:M-e3-dec-M.ci\endcsname{[0.079, 0.112]}
\expandafter\def\csname regval:M-e3-dec-sum\endcsname{0.108}
\expandafter\def\csname regval:M-e3-dec-sum.d\endcsname{1{,}560}
\expandafter\def\csname regval:M-e3-dec-sum.ci\endcsname{[0.094, 0.122]}
\expandafter\def\csname regval:M-e3-dec-se\endcsname{0.008}
\expandafter\def\csname regval:M-e3-dec-se.d\endcsname{1{,}560}
\expandafter\def\csname regval:M-e3-dec-boot-sd\endcsname{0.007}
\expandafter\def\csname regval:M-e3-dec-boot-sd.d\endcsname{390}
\expandafter\def\csname regval:M-e3-dec-exact-sum\endcsname{0.108}
\expandafter\def\csname regval:M-e3-dec-exact-sum.d\endcsname{1{,}560}
\expandafter\def\csname regval:M-e3-dec-exact-sum.ci\endcsname{[0.094, 0.122]}
\expandafter\def\csname regval:M-e3-dec-exact-F\endcsname{0.151}
\expandafter\def\csname regval:M-e3-dec-exact-F.n\endcsname{59}
\expandafter\def\csname regval:M-e3-dec-exact-F.d\endcsname{390}
\expandafter\def\csname regval:M-e3-dec-exact-M\endcsname{0.094}
\expandafter\def\csname regval:M-e3-dec-exact-M.n\endcsname{110}
\expandafter\def\csname regval:M-e3-dec-exact-M.d\endcsname{1{,}170}
\expandafter\def\csname regval:M-e3-dec-blank-answerrate\endcsname{0.000}
\expandafter\def\csname regval:M-e3-dec-blank-answerrate.n\endcsname{0}
\expandafter\def\csname regval:M-e3-dec-blank-answerrate.d\endcsname{390}
\expandafter\def\csname regval:M-e3-dec-id_unrelated-answerrate\endcsname{0.236}
\expandafter\def\csname regval:M-e3-dec-id_unrelated-answerrate.n\endcsname{92}
\expandafter\def\csname regval:M-e3-dec-id_unrelated-answerrate.d\endcsname{390}
\expandafter\def\csname regval:M-e3-dec-od_unrelated-answerrate\endcsname{0.046}
\expandafter\def\csname regval:M-e3-dec-od_unrelated-answerrate.n\endcsname{18}
\expandafter\def\csname regval:M-e3-dec-od_unrelated-answerrate.d\endcsname{390}
\expandafter\def\csname regval:M-e3-dec-agree-real\endcsname{390}
\expandafter\def\csname regval:M-e3-dec-agree-real.n\endcsname{390}
\expandafter\def\csname regval:M-e3-dec-agree-real.d\endcsname{390}
\expandafter\def\csname regval:M-e3-dec-agree-inv\endcsname{1{,}060}
\expandafter\def\csname regval:M-e3-dec-agree-inv.n\endcsname{1{,}060}
\expandafter\def\csname regval:M-e3-dec-agree-inv.d\endcsname{1{,}060}
\expandafter\def\csname regval:M-e3-dec-inv-declined\endcsname{1{,}060}
\expandafter\def\csname regval:M-e3-dec-inv-declined.n\endcsname{1{,}060}
\expandafter\def\csname regval:M-e3-dec-inv-declined.d\endcsname{1{,}170}
\expandafter\def\csname regval:M-e4-dec-real-exact\endcsname{3}
\expandafter\def\csname regval:M-e4-dec-real-exact.n\endcsname{3}
\expandafter\def\csname regval:M-e4-dec-real-exact.d\endcsname{390}
\expandafter\def\csname regval:M-e4-dec-real-near\endcsname{2}
\expandafter\def\csname regval:M-e4-dec-real-near.n\endcsname{2}
\expandafter\def\csname regval:M-e4-dec-real-near.d\endcsname{390}
\expandafter\def\csname regval:M-e4-dec-real-refusalpattern\endcsname{1}
\expandafter\def\csname regval:M-e4-dec-real-refusalpattern.n\endcsname{1}
\expandafter\def\csname regval:M-e4-dec-real-refusalpattern.d\endcsname{390}
\expandafter\def\csname regval:M-e4-dec-real-empty\endcsname{0}
\expandafter\def\csname regval:M-e4-dec-real-empty.n\endcsname{0}
\expandafter\def\csname regval:M-e4-dec-real-empty.d\endcsname{390}
\expandafter\def\csname regval:M-e4-dec-real-answer\endcsname{384}
\expandafter\def\csname regval:M-e4-dec-real-answer.n\endcsname{384}
\expandafter\def\csname regval:M-e4-dec-real-answer.d\endcsname{390}
\expandafter\def\csname regval:M-e4-dec-blank-exact\endcsname{390}
\expandafter\def\csname regval:M-e4-dec-blank-exact.n\endcsname{390}
\expandafter\def\csname regval:M-e4-dec-blank-exact.d\endcsname{390}
\expandafter\def\csname regval:M-e4-dec-blank-near\endcsname{0}
\expandafter\def\csname regval:M-e4-dec-blank-near.n\endcsname{0}
\expandafter\def\csname regval:M-e4-dec-blank-near.d\endcsname{390}
\expandafter\def\csname regval:M-e4-dec-blank-refusalpattern\endcsname{0}
\expandafter\def\csname regval:M-e4-dec-blank-refusalpattern.n\endcsname{0}
\expandafter\def\csname regval:M-e4-dec-blank-refusalpattern.d\endcsname{390}
\expandafter\def\csname regval:M-e4-dec-blank-empty\endcsname{0}
\expandafter\def\csname regval:M-e4-dec-blank-empty.n\endcsname{0}
\expandafter\def\csname regval:M-e4-dec-blank-empty.d\endcsname{390}
\expandafter\def\csname regval:M-e4-dec-blank-answer\endcsname{0}
\expandafter\def\csname regval:M-e4-dec-blank-answer.n\endcsname{0}
\expandafter\def\csname regval:M-e4-dec-blank-answer.d\endcsname{390}
\expandafter\def\csname regval:M-e4-dec-id_unrelated-exact\endcsname{3}
\expandafter\def\csname regval:M-e4-dec-id_unrelated-exact.n\endcsname{3}
\expandafter\def\csname regval:M-e4-dec-id_unrelated-exact.d\endcsname{390}
\expandafter\def\csname regval:M-e4-dec-id_unrelated-near\endcsname{2}
\expandafter\def\csname regval:M-e4-dec-id_unrelated-near.n\endcsname{2}
\expandafter\def\csname regval:M-e4-dec-id_unrelated-near.d\endcsname{390}
\expandafter\def\csname regval:M-e4-dec-id_unrelated-refusalpattern\endcsname{1}
\expandafter\def\csname regval:M-e4-dec-id_unrelated-refusalpattern.n\endcsname{1}
\expandafter\def\csname regval:M-e4-dec-id_unrelated-refusalpattern.d\endcsname{390}
\expandafter\def\csname regval:M-e4-dec-id_unrelated-empty\endcsname{0}
\expandafter\def\csname regval:M-e4-dec-id_unrelated-empty.n\endcsname{0}
\expandafter\def\csname regval:M-e4-dec-id_unrelated-empty.d\endcsname{390}
\expandafter\def\csname regval:M-e4-dec-id_unrelated-answer\endcsname{384}
\expandafter\def\csname regval:M-e4-dec-id_unrelated-answer.n\endcsname{384}
\expandafter\def\csname regval:M-e4-dec-id_unrelated-answer.d\endcsname{390}
\expandafter\def\csname regval:M-e4-dec-od_unrelated-exact\endcsname{5}
\expandafter\def\csname regval:M-e4-dec-od_unrelated-exact.n\endcsname{5}
\expandafter\def\csname regval:M-e4-dec-od_unrelated-exact.d\endcsname{390}
\expandafter\def\csname regval:M-e4-dec-od_unrelated-near\endcsname{2}
\expandafter\def\csname regval:M-e4-dec-od_unrelated-near.n\endcsname{2}
\expandafter\def\csname regval:M-e4-dec-od_unrelated-near.d\endcsname{390}
\expandafter\def\csname regval:M-e4-dec-od_unrelated-refusalpattern\endcsname{0}
\expandafter\def\csname regval:M-e4-dec-od_unrelated-refusalpattern.n\endcsname{0}
\expandafter\def\csname regval:M-e4-dec-od_unrelated-refusalpattern.d\endcsname{390}
\expandafter\def\csname regval:M-e4-dec-od_unrelated-empty\endcsname{0}
\expandafter\def\csname regval:M-e4-dec-od_unrelated-empty.n\endcsname{0}
\expandafter\def\csname regval:M-e4-dec-od_unrelated-empty.d\endcsname{390}
\expandafter\def\csname regval:M-e4-dec-od_unrelated-answer\endcsname{383}
\expandafter\def\csname regval:M-e4-dec-od_unrelated-answer.n\endcsname{383}
\expandafter\def\csname regval:M-e4-dec-od_unrelated-answer.d\endcsname{390}
\expandafter\def\csname regval:M-e4-dec-F\endcsname{0.013}
\expandafter\def\csname regval:M-e4-dec-F.n\endcsname{5}
\expandafter\def\csname regval:M-e4-dec-F.d\endcsname{390}
\expandafter\def\csname regval:M-e4-dec-F.ci\endcsname{[0.005, 0.030]}
\expandafter\def\csname regval:M-e4-dec-M\endcsname{0.656}
\expandafter\def\csname regval:M-e4-dec-M.n\endcsname{768}
\expandafter\def\csname regval:M-e4-dec-M.d\endcsname{1{,}170}
\expandafter\def\csname regval:M-e4-dec-M.ci\endcsname{[0.629, 0.683]}
\expandafter\def\csname regval:M-e4-dec-sum\endcsname{0.495}
\expandafter\def\csname regval:M-e4-dec-sum.d\endcsname{1{,}560}
\expandafter\def\csname regval:M-e4-dec-sum.ci\endcsname{[0.490, 0.501]}
\expandafter\def\csname regval:M-e4-dec-se\endcsname{0.013}
\expandafter\def\csname regval:M-e4-dec-se.d\endcsname{1{,}560}
\expandafter\def\csname regval:M-e4-dec-boot-sd\endcsname{0.003}
\expandafter\def\csname regval:M-e4-dec-boot-sd.d\endcsname{390}
\expandafter\def\csname regval:M-e4-dec-exact-sum\endcsname{0.497}
\expandafter\def\csname regval:M-e4-dec-exact-sum.d\endcsname{1{,}560}
\expandafter\def\csname regval:M-e4-dec-exact-sum.ci\endcsname{[0.492, 0.501]}
\expandafter\def\csname regval:M-e4-dec-exact-F\endcsname{0.008}
\expandafter\def\csname regval:M-e4-dec-exact-F.n\endcsname{3}
\expandafter\def\csname regval:M-e4-dec-exact-F.d\endcsname{390}
\expandafter\def\csname regval:M-e4-dec-exact-M\endcsname{0.660}
\expandafter\def\csname regval:M-e4-dec-exact-M.n\endcsname{772}
\expandafter\def\csname regval:M-e4-dec-exact-M.d\endcsname{1{,}170}
\expandafter\def\csname regval:M-e4-dec-blank-answerrate\endcsname{0.000}
\expandafter\def\csname regval:M-e4-dec-blank-answerrate.n\endcsname{0}
\expandafter\def\csname regval:M-e4-dec-blank-answerrate.d\endcsname{390}
\expandafter\def\csname regval:M-e4-dec-id_unrelated-answerrate\endcsname{0.987}
\expandafter\def\csname regval:M-e4-dec-id_unrelated-answerrate.n\endcsname{385}
\expandafter\def\csname regval:M-e4-dec-id_unrelated-answerrate.d\endcsname{390}
\expandafter\def\csname regval:M-e4-dec-od_unrelated-answerrate\endcsname{0.982}
\expandafter\def\csname regval:M-e4-dec-od_unrelated-answerrate.n\endcsname{383}
\expandafter\def\csname regval:M-e4-dec-od_unrelated-answerrate.d\endcsname{390}
\expandafter\def\csname regval:M-e4-dec-agree-real\endcsname{325}
\expandafter\def\csname regval:M-e4-dec-agree-real.n\endcsname{325}
\expandafter\def\csname regval:M-e4-dec-agree-real.d\endcsname{390}
\expandafter\def\csname regval:M-e4-dec-agree-inv\endcsname{402}
\expandafter\def\csname regval:M-e4-dec-agree-inv.n\endcsname{402}
\expandafter\def\csname regval:M-e4-dec-agree-inv.d\endcsname{402}
\expandafter\def\csname regval:M-e4-dec-inv-declined\endcsname{402}
\expandafter\def\csname regval:M-e4-dec-inv-declined.n\endcsname{402}
\expandafter\def\csname regval:M-e4-dec-inv-declined.d\endcsname{1{,}170}
\expandafter\def\csname regval:M-e5-dec-real-exact\endcsname{0}
\expandafter\def\csname regval:M-e5-dec-real-exact.n\endcsname{0}
\expandafter\def\csname regval:M-e5-dec-real-exact.d\endcsname{390}
\expandafter\def\csname regval:M-e5-dec-real-near\endcsname{0}
\expandafter\def\csname regval:M-e5-dec-real-near.n\endcsname{0}
\expandafter\def\csname regval:M-e5-dec-real-near.d\endcsname{390}
\expandafter\def\csname regval:M-e5-dec-real-refusalpattern\endcsname{71}
\expandafter\def\csname regval:M-e5-dec-real-refusalpattern.n\endcsname{71}
\expandafter\def\csname regval:M-e5-dec-real-refusalpattern.d\endcsname{390}
\expandafter\def\csname regval:M-e5-dec-real-empty\endcsname{0}
\expandafter\def\csname regval:M-e5-dec-real-empty.n\endcsname{0}
\expandafter\def\csname regval:M-e5-dec-real-empty.d\endcsname{390}
\expandafter\def\csname regval:M-e5-dec-real-answer\endcsname{319}
\expandafter\def\csname regval:M-e5-dec-real-answer.n\endcsname{319}
\expandafter\def\csname regval:M-e5-dec-real-answer.d\endcsname{390}
\expandafter\def\csname regval:M-e5-dec-blank-exact\endcsname{0}
\expandafter\def\csname regval:M-e5-dec-blank-exact.n\endcsname{0}
\expandafter\def\csname regval:M-e5-dec-blank-exact.d\endcsname{390}
\expandafter\def\csname regval:M-e5-dec-blank-near\endcsname{0}
\expandafter\def\csname regval:M-e5-dec-blank-near.n\endcsname{0}
\expandafter\def\csname regval:M-e5-dec-blank-near.d\endcsname{390}
\expandafter\def\csname regval:M-e5-dec-blank-refusalpattern\endcsname{137}
\expandafter\def\csname regval:M-e5-dec-blank-refusalpattern.n\endcsname{137}
\expandafter\def\csname regval:M-e5-dec-blank-refusalpattern.d\endcsname{390}
\expandafter\def\csname regval:M-e5-dec-blank-empty\endcsname{0}
\expandafter\def\csname regval:M-e5-dec-blank-empty.n\endcsname{0}
\expandafter\def\csname regval:M-e5-dec-blank-empty.d\endcsname{390}
\expandafter\def\csname regval:M-e5-dec-blank-answer\endcsname{253}
\expandafter\def\csname regval:M-e5-dec-blank-answer.n\endcsname{253}
\expandafter\def\csname regval:M-e5-dec-blank-answer.d\endcsname{390}
\expandafter\def\csname regval:M-e5-dec-id_unrelated-exact\endcsname{0}
\expandafter\def\csname regval:M-e5-dec-id_unrelated-exact.n\endcsname{0}
\expandafter\def\csname regval:M-e5-dec-id_unrelated-exact.d\endcsname{390}
\expandafter\def\csname regval:M-e5-dec-id_unrelated-near\endcsname{0}
\expandafter\def\csname regval:M-e5-dec-id_unrelated-near.n\endcsname{0}
\expandafter\def\csname regval:M-e5-dec-id_unrelated-near.d\endcsname{390}
\expandafter\def\csname regval:M-e5-dec-id_unrelated-refusalpattern\endcsname{181}
\expandafter\def\csname regval:M-e5-dec-id_unrelated-refusalpattern.n\endcsname{181}
\expandafter\def\csname regval:M-e5-dec-id_unrelated-refusalpattern.d\endcsname{390}
\expandafter\def\csname regval:M-e5-dec-id_unrelated-empty\endcsname{0}
\expandafter\def\csname regval:M-e5-dec-id_unrelated-empty.n\endcsname{0}
\expandafter\def\csname regval:M-e5-dec-id_unrelated-empty.d\endcsname{390}
\expandafter\def\csname regval:M-e5-dec-id_unrelated-answer\endcsname{209}
\expandafter\def\csname regval:M-e5-dec-id_unrelated-answer.n\endcsname{209}
\expandafter\def\csname regval:M-e5-dec-id_unrelated-answer.d\endcsname{390}
\expandafter\def\csname regval:M-e5-dec-od_unrelated-exact\endcsname{0}
\expandafter\def\csname regval:M-e5-dec-od_unrelated-exact.n\endcsname{0}
\expandafter\def\csname regval:M-e5-dec-od_unrelated-exact.d\endcsname{390}
\expandafter\def\csname regval:M-e5-dec-od_unrelated-near\endcsname{0}
\expandafter\def\csname regval:M-e5-dec-od_unrelated-near.n\endcsname{0}
\expandafter\def\csname regval:M-e5-dec-od_unrelated-near.d\endcsname{390}
\expandafter\def\csname regval:M-e5-dec-od_unrelated-refusalpattern\endcsname{213}
\expandafter\def\csname regval:M-e5-dec-od_unrelated-refusalpattern.n\endcsname{213}
\expandafter\def\csname regval:M-e5-dec-od_unrelated-refusalpattern.d\endcsname{390}
\expandafter\def\csname regval:M-e5-dec-od_unrelated-empty\endcsname{0}
\expandafter\def\csname regval:M-e5-dec-od_unrelated-empty.n\endcsname{0}
\expandafter\def\csname regval:M-e5-dec-od_unrelated-empty.d\endcsname{390}
\expandafter\def\csname regval:M-e5-dec-od_unrelated-answer\endcsname{177}
\expandafter\def\csname regval:M-e5-dec-od_unrelated-answer.n\endcsname{177}
\expandafter\def\csname regval:M-e5-dec-od_unrelated-answer.d\endcsname{390}
\expandafter\def\csname regval:M-e5-dec-F\endcsname{0.000}
\expandafter\def\csname regval:M-e5-dec-F.n\endcsname{0}
\expandafter\def\csname regval:M-e5-dec-F.d\endcsname{390}
\expandafter\def\csname regval:M-e5-dec-F.ci\endcsname{[0.000, 0.010]}
\expandafter\def\csname regval:M-e5-dec-M\endcsname{1.000}
\expandafter\def\csname regval:M-e5-dec-M.n\endcsname{1{,}170}
\expandafter\def\csname regval:M-e5-dec-M.d\endcsname{1{,}170}
\expandafter\def\csname regval:M-e5-dec-M.ci\endcsname{[0.997, 1.000]}
\expandafter\def\csname regval:M-e5-dec-sum\endcsname{0.750}
\expandafter\def\csname regval:M-e5-dec-sum.d\endcsname{1{,}560}
\expandafter\def\csname regval:M-e5-dec-sum.ci\endcsname{[0.750, 0.750]}
\expandafter\def\csname regval:M-e5-dec-se\endcsname{0.011}
\expandafter\def\csname regval:M-e5-dec-se.d\endcsname{1{,}560}
\expandafter\def\csname regval:M-e5-dec-boot-sd\endcsname{0.000}
\expandafter\def\csname regval:M-e5-dec-boot-sd.d\endcsname{390}
\expandafter\def\csname regval:M-e5-dec-exact-sum\endcsname{0.750}
\expandafter\def\csname regval:M-e5-dec-exact-sum.d\endcsname{1{,}560}
\expandafter\def\csname regval:M-e5-dec-exact-sum.ci\endcsname{[0.750, 0.750]}
\expandafter\def\csname regval:M-e5-dec-exact-F\endcsname{0.000}
\expandafter\def\csname regval:M-e5-dec-exact-F.n\endcsname{0}
\expandafter\def\csname regval:M-e5-dec-exact-F.d\endcsname{390}
\expandafter\def\csname regval:M-e5-dec-exact-M\endcsname{1.000}
\expandafter\def\csname regval:M-e5-dec-exact-M.n\endcsname{1{,}170}
\expandafter\def\csname regval:M-e5-dec-exact-M.d\endcsname{1{,}170}
\expandafter\def\csname regval:M-e5-dec-blank-answerrate\endcsname{1.000}
\expandafter\def\csname regval:M-e5-dec-blank-answerrate.n\endcsname{390}
\expandafter\def\csname regval:M-e5-dec-blank-answerrate.d\endcsname{390}
\expandafter\def\csname regval:M-e5-dec-id_unrelated-answerrate\endcsname{1.000}
\expandafter\def\csname regval:M-e5-dec-id_unrelated-answerrate.n\endcsname{390}
\expandafter\def\csname regval:M-e5-dec-id_unrelated-answerrate.d\endcsname{390}
\expandafter\def\csname regval:M-e5-dec-od_unrelated-answerrate\endcsname{1.000}
\expandafter\def\csname regval:M-e5-dec-od_unrelated-answerrate.n\endcsname{390}
\expandafter\def\csname regval:M-e5-dec-od_unrelated-answerrate.d\endcsname{390}
\expandafter\def\csname regval:M-stress-e0-P0-gpt-5.5-iaimr\endcsname{0.457}
\expandafter\def\csname regval:M-stress-e0-P0-gpt-5.5-iaimr.n\endcsname{137}
\expandafter\def\csname regval:M-stress-e0-P0-gpt-5.5-iaimr.d\endcsname{300}
\expandafter\def\csname regval:M-stress-e0-P0-gpt-5.5-iaimr.ci\endcsname{[0.401, 0.513]}
\expandafter\def\csname regval:M-stress-e0-P0-gpt-5.5-blank\endcsname{0.720}
\expandafter\def\csname regval:M-stress-e0-P0-gpt-5.5-blank.n\endcsname{72}
\expandafter\def\csname regval:M-stress-e0-P0-gpt-5.5-blank.d\endcsname{100}
\expandafter\def\csname regval:M-stress-e0-P0-gpt-5.5-id\endcsname{0.360}
\expandafter\def\csname regval:M-stress-e0-P0-gpt-5.5-id.n\endcsname{36}
\expandafter\def\csname regval:M-stress-e0-P0-gpt-5.5-id.d\endcsname{100}
\expandafter\def\csname regval:M-stress-e0-P0-gpt-5.5-od\endcsname{0.290}
\expandafter\def\csname regval:M-stress-e0-P0-gpt-5.5-od.n\endcsname{29}
\expandafter\def\csname regval:M-stress-e0-P0-gpt-5.5-od.d\endcsname{100}
\expandafter\def\csname regval:M-stress-e0-P0-gpt-5.5-acc\endcsname{0.550}
\expandafter\def\csname regval:M-stress-e0-P0-gpt-5.5-acc.n\endcsname{55}
\expandafter\def\csname regval:M-stress-e0-P0-gpt-5.5-acc.d\endcsname{100}
\expandafter\def\csname regval:M-stress-e0-P0-gpt-5.5-answered\endcsname{0.970}
\expandafter\def\csname regval:M-stress-e0-P0-gpt-5.5-answered.n\endcsname{97}
\expandafter\def\csname regval:M-stress-e0-P0-gpt-5.5-answered.d\endcsname{100}
\expandafter\def\csname regval:M-stress-e0-P0-gpt-4o-mini-iaimr\endcsname{0.503}
\expandafter\def\csname regval:M-stress-e0-P0-gpt-4o-mini-iaimr.n\endcsname{151}
\expandafter\def\csname regval:M-stress-e0-P0-gpt-4o-mini-iaimr.d\endcsname{300}
\expandafter\def\csname regval:M-stress-e0-P0-gpt-4o-mini-iaimr.ci\endcsname{[0.447, 0.559]}
\expandafter\def\csname regval:M-stress-e0-P0-gpt-4o-mini-blank\endcsname{0.560}
\expandafter\def\csname regval:M-stress-e0-P0-gpt-4o-mini-blank.n\endcsname{56}
\expandafter\def\csname regval:M-stress-e0-P0-gpt-4o-mini-blank.d\endcsname{100}
\expandafter\def\csname regval:M-stress-e0-P0-gpt-4o-mini-id\endcsname{0.580}
\expandafter\def\csname regval:M-stress-e0-P0-gpt-4o-mini-id.n\endcsname{58}
\expandafter\def\csname regval:M-stress-e0-P0-gpt-4o-mini-id.d\endcsname{100}
\expandafter\def\csname regval:M-stress-e0-P0-gpt-4o-mini-od\endcsname{0.370}
\expandafter\def\csname regval:M-stress-e0-P0-gpt-4o-mini-od.n\endcsname{37}
\expandafter\def\csname regval:M-stress-e0-P0-gpt-4o-mini-od.d\endcsname{100}
\expandafter\def\csname regval:M-stress-e0-P0-gpt-4o-mini-acc\endcsname{0.460}
\expandafter\def\csname regval:M-stress-e0-P0-gpt-4o-mini-acc.n\endcsname{46}
\expandafter\def\csname regval:M-stress-e0-P0-gpt-4o-mini-acc.d\endcsname{100}
\expandafter\def\csname regval:M-stress-e0-P0-gpt-4o-mini-answered\endcsname{0.950}
\expandafter\def\csname regval:M-stress-e0-P0-gpt-4o-mini-answered.n\endcsname{95}
\expandafter\def\csname regval:M-stress-e0-P0-gpt-4o-mini-answered.d\endcsname{100}
\expandafter\def\csname regval:M-stress-e0-P1-gpt-5.5-iaimr\endcsname{0.920}
\expandafter\def\csname regval:M-stress-e0-P1-gpt-5.5-iaimr.n\endcsname{276}
\expandafter\def\csname regval:M-stress-e0-P1-gpt-5.5-iaimr.d\endcsname{300}
\expandafter\def\csname regval:M-stress-e0-P1-gpt-5.5-iaimr.ci\endcsname{[0.884, 0.946]}
\expandafter\def\csname regval:M-stress-e0-P1-gpt-5.5-blank\endcsname{1.000}
\expandafter\def\csname regval:M-stress-e0-P1-gpt-5.5-blank.n\endcsname{100}
\expandafter\def\csname regval:M-stress-e0-P1-gpt-5.5-blank.d\endcsname{100}
\expandafter\def\csname regval:M-stress-e0-P1-gpt-5.5-id\endcsname{0.910}
\expandafter\def\csname regval:M-stress-e0-P1-gpt-5.5-id.n\endcsname{91}
\expandafter\def\csname regval:M-stress-e0-P1-gpt-5.5-id.d\endcsname{100}
\expandafter\def\csname regval:M-stress-e0-P1-gpt-5.5-od\endcsname{0.850}
\expandafter\def\csname regval:M-stress-e0-P1-gpt-5.5-od.n\endcsname{85}
\expandafter\def\csname regval:M-stress-e0-P1-gpt-5.5-od.d\endcsname{100}
\expandafter\def\csname regval:M-stress-e0-P1-gpt-5.5-acc\endcsname{0.500}
\expandafter\def\csname regval:M-stress-e0-P1-gpt-5.5-acc.n\endcsname{50}
\expandafter\def\csname regval:M-stress-e0-P1-gpt-5.5-acc.d\endcsname{100}
\expandafter\def\csname regval:M-stress-e0-P1-gpt-5.5-answered\endcsname{1.000}
\expandafter\def\csname regval:M-stress-e0-P1-gpt-5.5-answered.n\endcsname{100}
\expandafter\def\csname regval:M-stress-e0-P1-gpt-5.5-answered.d\endcsname{100}
\expandafter\def\csname regval:M-stress-e0-P1-gpt-4o-mini-iaimr\endcsname{0.710}
\expandafter\def\csname regval:M-stress-e0-P1-gpt-4o-mini-iaimr.n\endcsname{213}
\expandafter\def\csname regval:M-stress-e0-P1-gpt-4o-mini-iaimr.d\endcsname{300}
\expandafter\def\csname regval:M-stress-e0-P1-gpt-4o-mini-iaimr.ci\endcsname{[0.656, 0.758]}
\expandafter\def\csname regval:M-stress-e0-P1-gpt-4o-mini-blank\endcsname{0.510}
\expandafter\def\csname regval:M-stress-e0-P1-gpt-4o-mini-blank.n\endcsname{51}
\expandafter\def\csname regval:M-stress-e0-P1-gpt-4o-mini-blank.d\endcsname{100}
\expandafter\def\csname regval:M-stress-e0-P1-gpt-4o-mini-id\endcsname{0.850}
\expandafter\def\csname regval:M-stress-e0-P1-gpt-4o-mini-id.n\endcsname{85}
\expandafter\def\csname regval:M-stress-e0-P1-gpt-4o-mini-id.d\endcsname{100}
\expandafter\def\csname regval:M-stress-e0-P1-gpt-4o-mini-od\endcsname{0.770}
\expandafter\def\csname regval:M-stress-e0-P1-gpt-4o-mini-od.n\endcsname{77}
\expandafter\def\csname regval:M-stress-e0-P1-gpt-4o-mini-od.d\endcsname{100}
\expandafter\def\csname regval:M-stress-e0-P1-gpt-4o-mini-acc\endcsname{0.490}
\expandafter\def\csname regval:M-stress-e0-P1-gpt-4o-mini-acc.n\endcsname{49}
\expandafter\def\csname regval:M-stress-e0-P1-gpt-4o-mini-acc.d\endcsname{100}
\expandafter\def\csname regval:M-stress-e0-P1-gpt-4o-mini-answered\endcsname{0.900}
\expandafter\def\csname regval:M-stress-e0-P1-gpt-4o-mini-answered.n\endcsname{90}
\expandafter\def\csname regval:M-stress-e0-P1-gpt-4o-mini-answered.d\endcsname{100}
\expandafter\def\csname regval:M-stress-e0-P2-gpt-5.5-iaimr\endcsname{0.340}
\expandafter\def\csname regval:M-stress-e0-P2-gpt-5.5-iaimr.n\endcsname{102}
\expandafter\def\csname regval:M-stress-e0-P2-gpt-5.5-iaimr.d\endcsname{300}
\expandafter\def\csname regval:M-stress-e0-P2-gpt-5.5-iaimr.ci\endcsname{[0.289, 0.395]}
\expandafter\def\csname regval:M-stress-e0-P2-gpt-5.5-blank\endcsname{0.530}
\expandafter\def\csname regval:M-stress-e0-P2-gpt-5.5-blank.n\endcsname{53}
\expandafter\def\csname regval:M-stress-e0-P2-gpt-5.5-blank.d\endcsname{100}
\expandafter\def\csname regval:M-stress-e0-P2-gpt-5.5-id\endcsname{0.340}
\expandafter\def\csname regval:M-stress-e0-P2-gpt-5.5-id.n\endcsname{34}
\expandafter\def\csname regval:M-stress-e0-P2-gpt-5.5-id.d\endcsname{100}
\expandafter\def\csname regval:M-stress-e0-P2-gpt-5.5-od\endcsname{0.150}
\expandafter\def\csname regval:M-stress-e0-P2-gpt-5.5-od.n\endcsname{15}
\expandafter\def\csname regval:M-stress-e0-P2-gpt-5.5-od.d\endcsname{100}
\expandafter\def\csname regval:M-stress-e0-P2-gpt-5.5-acc\endcsname{0.450}
\expandafter\def\csname regval:M-stress-e0-P2-gpt-5.5-acc.n\endcsname{45}
\expandafter\def\csname regval:M-stress-e0-P2-gpt-5.5-acc.d\endcsname{100}
\expandafter\def\csname regval:M-stress-e0-P2-gpt-5.5-answered\endcsname{0.980}
\expandafter\def\csname regval:M-stress-e0-P2-gpt-5.5-answered.n\endcsname{98}
\expandafter\def\csname regval:M-stress-e0-P2-gpt-5.5-answered.d\endcsname{100}
\expandafter\def\csname regval:M-stress-e0-P2-gpt-4o-mini-iaimr\endcsname{0.380}
\expandafter\def\csname regval:M-stress-e0-P2-gpt-4o-mini-iaimr.n\endcsname{114}
\expandafter\def\csname regval:M-stress-e0-P2-gpt-4o-mini-iaimr.d\endcsname{300}
\expandafter\def\csname regval:M-stress-e0-P2-gpt-4o-mini-iaimr.ci\endcsname{[0.327, 0.436]}
\expandafter\def\csname regval:M-stress-e0-P2-gpt-4o-mini-blank\endcsname{0.410}
\expandafter\def\csname regval:M-stress-e0-P2-gpt-4o-mini-blank.n\endcsname{41}
\expandafter\def\csname regval:M-stress-e0-P2-gpt-4o-mini-blank.d\endcsname{100}
\expandafter\def\csname regval:M-stress-e0-P2-gpt-4o-mini-id\endcsname{0.510}
\expandafter\def\csname regval:M-stress-e0-P2-gpt-4o-mini-id.n\endcsname{51}
\expandafter\def\csname regval:M-stress-e0-P2-gpt-4o-mini-id.d\endcsname{100}
\expandafter\def\csname regval:M-stress-e0-P2-gpt-4o-mini-od\endcsname{0.220}
\expandafter\def\csname regval:M-stress-e0-P2-gpt-4o-mini-od.n\endcsname{22}
\expandafter\def\csname regval:M-stress-e0-P2-gpt-4o-mini-od.d\endcsname{100}
\expandafter\def\csname regval:M-stress-e0-P2-gpt-4o-mini-acc\endcsname{0.390}
\expandafter\def\csname regval:M-stress-e0-P2-gpt-4o-mini-acc.n\endcsname{39}
\expandafter\def\csname regval:M-stress-e0-P2-gpt-4o-mini-acc.d\endcsname{100}
\expandafter\def\csname regval:M-stress-e0-P2-gpt-4o-mini-answered\endcsname{0.960}
\expandafter\def\csname regval:M-stress-e0-P2-gpt-4o-mini-answered.n\endcsname{96}
\expandafter\def\csname regval:M-stress-e0-P2-gpt-4o-mini-answered.d\endcsname{100}
\expandafter\def\csname regval:M-stress-e2-P0-gpt-5.5-iaimr\endcsname{0.773}
\expandafter\def\csname regval:M-stress-e2-P0-gpt-5.5-iaimr.n\endcsname{232}
\expandafter\def\csname regval:M-stress-e2-P0-gpt-5.5-iaimr.d\endcsname{300}
\expandafter\def\csname regval:M-stress-e2-P0-gpt-5.5-iaimr.ci\endcsname{[0.723, 0.817]}
\expandafter\def\csname regval:M-stress-e2-P0-gpt-5.5-blank\endcsname{0.640}
\expandafter\def\csname regval:M-stress-e2-P0-gpt-5.5-blank.n\endcsname{64}
\expandafter\def\csname regval:M-stress-e2-P0-gpt-5.5-blank.d\endcsname{100}
\expandafter\def\csname regval:M-stress-e2-P0-gpt-5.5-id\endcsname{0.830}
\expandafter\def\csname regval:M-stress-e2-P0-gpt-5.5-id.n\endcsname{83}
\expandafter\def\csname regval:M-stress-e2-P0-gpt-5.5-id.d\endcsname{100}
\expandafter\def\csname regval:M-stress-e2-P0-gpt-5.5-od\endcsname{0.850}
\expandafter\def\csname regval:M-stress-e2-P0-gpt-5.5-od.n\endcsname{85}
\expandafter\def\csname regval:M-stress-e2-P0-gpt-5.5-od.d\endcsname{100}
\expandafter\def\csname regval:M-stress-e2-P0-gpt-5.5-acc\endcsname{0.250}
\expandafter\def\csname regval:M-stress-e2-P0-gpt-5.5-acc.n\endcsname{25}
\expandafter\def\csname regval:M-stress-e2-P0-gpt-5.5-acc.d\endcsname{100}
\expandafter\def\csname regval:M-stress-e2-P0-gpt-5.5-answered\endcsname{0.790}
\expandafter\def\csname regval:M-stress-e2-P0-gpt-5.5-answered.n\endcsname{79}
\expandafter\def\csname regval:M-stress-e2-P0-gpt-5.5-answered.d\endcsname{100}
\expandafter\def\csname regval:M-stress-e2-P0-gpt-4o-mini-iaimr\endcsname{0.457}
\expandafter\def\csname regval:M-stress-e2-P0-gpt-4o-mini-iaimr.n\endcsname{137}
\expandafter\def\csname regval:M-stress-e2-P0-gpt-4o-mini-iaimr.d\endcsname{300}
\expandafter\def\csname regval:M-stress-e2-P0-gpt-4o-mini-iaimr.ci\endcsname{[0.401, 0.513]}
\expandafter\def\csname regval:M-stress-e2-P0-gpt-4o-mini-blank\endcsname{0.480}
\expandafter\def\csname regval:M-stress-e2-P0-gpt-4o-mini-blank.n\endcsname{48}
\expandafter\def\csname regval:M-stress-e2-P0-gpt-4o-mini-blank.d\endcsname{100}
\expandafter\def\csname regval:M-stress-e2-P0-gpt-4o-mini-id\endcsname{0.560}
\expandafter\def\csname regval:M-stress-e2-P0-gpt-4o-mini-id.n\endcsname{56}
\expandafter\def\csname regval:M-stress-e2-P0-gpt-4o-mini-id.d\endcsname{100}
\expandafter\def\csname regval:M-stress-e2-P0-gpt-4o-mini-od\endcsname{0.330}
\expandafter\def\csname regval:M-stress-e2-P0-gpt-4o-mini-od.n\endcsname{33}
\expandafter\def\csname regval:M-stress-e2-P0-gpt-4o-mini-od.d\endcsname{100}
\expandafter\def\csname regval:M-stress-e2-P0-gpt-4o-mini-acc\endcsname{0.210}
\expandafter\def\csname regval:M-stress-e2-P0-gpt-4o-mini-acc.n\endcsname{21}
\expandafter\def\csname regval:M-stress-e2-P0-gpt-4o-mini-acc.d\endcsname{100}
\expandafter\def\csname regval:M-stress-e2-P0-gpt-4o-mini-answered\endcsname{0.770}
\expandafter\def\csname regval:M-stress-e2-P0-gpt-4o-mini-answered.n\endcsname{77}
\expandafter\def\csname regval:M-stress-e2-P0-gpt-4o-mini-answered.d\endcsname{100}
\expandafter\def\csname regval:M-stress-e2-P1-gpt-5.5-iaimr\endcsname{0.990}
\expandafter\def\csname regval:M-stress-e2-P1-gpt-5.5-iaimr.n\endcsname{297}
\expandafter\def\csname regval:M-stress-e2-P1-gpt-5.5-iaimr.d\endcsname{300}
\expandafter\def\csname regval:M-stress-e2-P1-gpt-5.5-iaimr.ci\endcsname{[0.971, 0.997]}
\expandafter\def\csname regval:M-stress-e2-P1-gpt-5.5-blank\endcsname{1.000}
\expandafter\def\csname regval:M-stress-e2-P1-gpt-5.5-blank.n\endcsname{100}
\expandafter\def\csname regval:M-stress-e2-P1-gpt-5.5-blank.d\endcsname{100}
\expandafter\def\csname regval:M-stress-e2-P1-gpt-5.5-id\endcsname{0.990}
\expandafter\def\csname regval:M-stress-e2-P1-gpt-5.5-id.n\endcsname{99}
\expandafter\def\csname regval:M-stress-e2-P1-gpt-5.5-id.d\endcsname{100}
\expandafter\def\csname regval:M-stress-e2-P1-gpt-5.5-od\endcsname{0.980}
\expandafter\def\csname regval:M-stress-e2-P1-gpt-5.5-od.n\endcsname{98}
\expandafter\def\csname regval:M-stress-e2-P1-gpt-5.5-od.d\endcsname{100}
\expandafter\def\csname regval:M-stress-e2-P1-gpt-5.5-acc\endcsname{0.320}
\expandafter\def\csname regval:M-stress-e2-P1-gpt-5.5-acc.n\endcsname{32}
\expandafter\def\csname regval:M-stress-e2-P1-gpt-5.5-acc.d\endcsname{100}
\expandafter\def\csname regval:M-stress-e2-P1-gpt-5.5-answered\endcsname{0.930}
\expandafter\def\csname regval:M-stress-e2-P1-gpt-5.5-answered.n\endcsname{93}
\expandafter\def\csname regval:M-stress-e2-P1-gpt-5.5-answered.d\endcsname{100}
\expandafter\def\csname regval:M-stress-e2-P1-gpt-4o-mini-iaimr\endcsname{0.653}
\expandafter\def\csname regval:M-stress-e2-P1-gpt-4o-mini-iaimr.n\endcsname{196}
\expandafter\def\csname regval:M-stress-e2-P1-gpt-4o-mini-iaimr.d\endcsname{300}
\expandafter\def\csname regval:M-stress-e2-P1-gpt-4o-mini-iaimr.ci\endcsname{[0.598, 0.705]}
\expandafter\def\csname regval:M-stress-e2-P1-gpt-4o-mini-blank\endcsname{0.510}
\expandafter\def\csname regval:M-stress-e2-P1-gpt-4o-mini-blank.n\endcsname{51}
\expandafter\def\csname regval:M-stress-e2-P1-gpt-4o-mini-blank.d\endcsname{100}
\expandafter\def\csname regval:M-stress-e2-P1-gpt-4o-mini-id\endcsname{0.880}
\expandafter\def\csname regval:M-stress-e2-P1-gpt-4o-mini-id.n\endcsname{88}
\expandafter\def\csname regval:M-stress-e2-P1-gpt-4o-mini-id.d\endcsname{100}
\expandafter\def\csname regval:M-stress-e2-P1-gpt-4o-mini-od\endcsname{0.570}
\expandafter\def\csname regval:M-stress-e2-P1-gpt-4o-mini-od.n\endcsname{57}
\expandafter\def\csname regval:M-stress-e2-P1-gpt-4o-mini-od.d\endcsname{100}
\expandafter\def\csname regval:M-stress-e2-P1-gpt-4o-mini-acc\endcsname{0.300}
\expandafter\def\csname regval:M-stress-e2-P1-gpt-4o-mini-acc.n\endcsname{30}
\expandafter\def\csname regval:M-stress-e2-P1-gpt-4o-mini-acc.d\endcsname{100}
\expandafter\def\csname regval:M-stress-e2-P1-gpt-4o-mini-answered\endcsname{0.800}
\expandafter\def\csname regval:M-stress-e2-P1-gpt-4o-mini-answered.n\endcsname{80}
\expandafter\def\csname regval:M-stress-e2-P1-gpt-4o-mini-answered.d\endcsname{100}
\expandafter\def\csname regval:M-stress-e2-P2-gpt-5.5-iaimr\endcsname{0.683}
\expandafter\def\csname regval:M-stress-e2-P2-gpt-5.5-iaimr.n\endcsname{205}
\expandafter\def\csname regval:M-stress-e2-P2-gpt-5.5-iaimr.d\endcsname{300}
\expandafter\def\csname regval:M-stress-e2-P2-gpt-5.5-iaimr.ci\endcsname{[0.629, 0.733]}
\expandafter\def\csname regval:M-stress-e2-P2-gpt-5.5-blank\endcsname{0.440}
\expandafter\def\csname regval:M-stress-e2-P2-gpt-5.5-blank.n\endcsname{44}
\expandafter\def\csname regval:M-stress-e2-P2-gpt-5.5-blank.d\endcsname{100}
\expandafter\def\csname regval:M-stress-e2-P2-gpt-5.5-id\endcsname{0.800}
\expandafter\def\csname regval:M-stress-e2-P2-gpt-5.5-id.n\endcsname{80}
\expandafter\def\csname regval:M-stress-e2-P2-gpt-5.5-id.d\endcsname{100}
\expandafter\def\csname regval:M-stress-e2-P2-gpt-5.5-od\endcsname{0.810}
\expandafter\def\csname regval:M-stress-e2-P2-gpt-5.5-od.n\endcsname{81}
\expandafter\def\csname regval:M-stress-e2-P2-gpt-5.5-od.d\endcsname{100}
\expandafter\def\csname regval:M-stress-e2-P2-gpt-5.5-acc\endcsname{0.330}
\expandafter\def\csname regval:M-stress-e2-P2-gpt-5.5-acc.n\endcsname{33}
\expandafter\def\csname regval:M-stress-e2-P2-gpt-5.5-acc.d\endcsname{100}
\expandafter\def\csname regval:M-stress-e2-P2-gpt-5.5-answered\endcsname{0.800}
\expandafter\def\csname regval:M-stress-e2-P2-gpt-5.5-answered.n\endcsname{80}
\expandafter\def\csname regval:M-stress-e2-P2-gpt-5.5-answered.d\endcsname{100}
\expandafter\def\csname regval:M-stress-e2-P2-gpt-4o-mini-iaimr\endcsname{0.383}
\expandafter\def\csname regval:M-stress-e2-P2-gpt-4o-mini-iaimr.n\endcsname{115}
\expandafter\def\csname regval:M-stress-e2-P2-gpt-4o-mini-iaimr.d\endcsname{300}
\expandafter\def\csname regval:M-stress-e2-P2-gpt-4o-mini-iaimr.ci\endcsname{[0.330, 0.440]}
\expandafter\def\csname regval:M-stress-e2-P2-gpt-4o-mini-blank\endcsname{0.300}
\expandafter\def\csname regval:M-stress-e2-P2-gpt-4o-mini-blank.n\endcsname{30}
\expandafter\def\csname regval:M-stress-e2-P2-gpt-4o-mini-blank.d\endcsname{100}
\expandafter\def\csname regval:M-stress-e2-P2-gpt-4o-mini-id\endcsname{0.520}
\expandafter\def\csname regval:M-stress-e2-P2-gpt-4o-mini-id.n\endcsname{52}
\expandafter\def\csname regval:M-stress-e2-P2-gpt-4o-mini-id.d\endcsname{100}
\expandafter\def\csname regval:M-stress-e2-P2-gpt-4o-mini-od\endcsname{0.330}
\expandafter\def\csname regval:M-stress-e2-P2-gpt-4o-mini-od.n\endcsname{33}
\expandafter\def\csname regval:M-stress-e2-P2-gpt-4o-mini-od.d\endcsname{100}
\expandafter\def\csname regval:M-stress-e2-P2-gpt-4o-mini-acc\endcsname{0.280}
\expandafter\def\csname regval:M-stress-e2-P2-gpt-4o-mini-acc.n\endcsname{28}
\expandafter\def\csname regval:M-stress-e2-P2-gpt-4o-mini-acc.d\endcsname{100}
\expandafter\def\csname regval:M-stress-e2-P2-gpt-4o-mini-answered\endcsname{0.730}
\expandafter\def\csname regval:M-stress-e2-P2-gpt-4o-mini-answered.n\endcsname{73}
\expandafter\def\csname regval:M-stress-e2-P2-gpt-4o-mini-answered.d\endcsname{100}
\expandafter\def\csname regval:M-stress-e3-P0-gpt-5.5-iaimr\endcsname{0.093}
\expandafter\def\csname regval:M-stress-e3-P0-gpt-5.5-iaimr.n\endcsname{28}
\expandafter\def\csname regval:M-stress-e3-P0-gpt-5.5-iaimr.d\endcsname{300}
\expandafter\def\csname regval:M-stress-e3-P0-gpt-5.5-iaimr.ci\endcsname{[0.065, 0.132]}
\expandafter\def\csname regval:M-stress-e3-P0-gpt-5.5-blank\endcsname{0.000}
\expandafter\def\csname regval:M-stress-e3-P0-gpt-5.5-blank.n\endcsname{0}
\expandafter\def\csname regval:M-stress-e3-P0-gpt-5.5-blank.d\endcsname{100}
\expandafter\def\csname regval:M-stress-e3-P0-gpt-5.5-id\endcsname{0.240}
\expandafter\def\csname regval:M-stress-e3-P0-gpt-5.5-id.n\endcsname{24}
\expandafter\def\csname regval:M-stress-e3-P0-gpt-5.5-id.d\endcsname{100}
\expandafter\def\csname regval:M-stress-e3-P0-gpt-5.5-od\endcsname{0.040}
\expandafter\def\csname regval:M-stress-e3-P0-gpt-5.5-od.n\endcsname{4}
\expandafter\def\csname regval:M-stress-e3-P0-gpt-5.5-od.d\endcsname{100}
\expandafter\def\csname regval:M-stress-e3-P0-gpt-5.5-acc\endcsname{0.410}
\expandafter\def\csname regval:M-stress-e3-P0-gpt-5.5-acc.n\endcsname{41}
\expandafter\def\csname regval:M-stress-e3-P0-gpt-5.5-acc.d\endcsname{100}
\expandafter\def\csname regval:M-stress-e3-P0-gpt-5.5-answered\endcsname{0.820}
\expandafter\def\csname regval:M-stress-e3-P0-gpt-5.5-answered.n\endcsname{82}
\expandafter\def\csname regval:M-stress-e3-P0-gpt-5.5-answered.d\endcsname{100}
\expandafter\def\csname regval:M-stress-e3-P0-gpt-4o-mini-iaimr\endcsname{0.093}
\expandafter\def\csname regval:M-stress-e3-P0-gpt-4o-mini-iaimr.n\endcsname{28}
\expandafter\def\csname regval:M-stress-e3-P0-gpt-4o-mini-iaimr.d\endcsname{300}
\expandafter\def\csname regval:M-stress-e3-P0-gpt-4o-mini-iaimr.ci\endcsname{[0.065, 0.132]}
\expandafter\def\csname regval:M-stress-e3-P0-gpt-4o-mini-blank\endcsname{0.000}
\expandafter\def\csname regval:M-stress-e3-P0-gpt-4o-mini-blank.n\endcsname{0}
\expandafter\def\csname regval:M-stress-e3-P0-gpt-4o-mini-blank.d\endcsname{100}
\expandafter\def\csname regval:M-stress-e3-P0-gpt-4o-mini-id\endcsname{0.240}
\expandafter\def\csname regval:M-stress-e3-P0-gpt-4o-mini-id.n\endcsname{24}
\expandafter\def\csname regval:M-stress-e3-P0-gpt-4o-mini-id.d\endcsname{100}
\expandafter\def\csname regval:M-stress-e3-P0-gpt-4o-mini-od\endcsname{0.040}
\expandafter\def\csname regval:M-stress-e3-P0-gpt-4o-mini-od.n\endcsname{4}
\expandafter\def\csname regval:M-stress-e3-P0-gpt-4o-mini-od.d\endcsname{100}
\expandafter\def\csname regval:M-stress-e3-P0-gpt-4o-mini-acc\endcsname{0.400}
\expandafter\def\csname regval:M-stress-e3-P0-gpt-4o-mini-acc.n\endcsname{40}
\expandafter\def\csname regval:M-stress-e3-P0-gpt-4o-mini-acc.d\endcsname{100}
\expandafter\def\csname regval:M-stress-e3-P0-gpt-4o-mini-answered\endcsname{0.820}
\expandafter\def\csname regval:M-stress-e3-P0-gpt-4o-mini-answered.n\endcsname{82}
\expandafter\def\csname regval:M-stress-e3-P0-gpt-4o-mini-answered.d\endcsname{100}
\expandafter\def\csname regval:M-stress-e3-P1-gpt-5.5-iaimr\endcsname{0.127}
\expandafter\def\csname regval:M-stress-e3-P1-gpt-5.5-iaimr.n\endcsname{38}
\expandafter\def\csname regval:M-stress-e3-P1-gpt-5.5-iaimr.d\endcsname{300}
\expandafter\def\csname regval:M-stress-e3-P1-gpt-5.5-iaimr.ci\endcsname{[0.094, 0.169]}
\expandafter\def\csname regval:M-stress-e3-P1-gpt-5.5-blank\endcsname{0.100}
\expandafter\def\csname regval:M-stress-e3-P1-gpt-5.5-blank.n\endcsname{10}
\expandafter\def\csname regval:M-stress-e3-P1-gpt-5.5-blank.d\endcsname{100}
\expandafter\def\csname regval:M-stress-e3-P1-gpt-5.5-id\endcsname{0.240}
\expandafter\def\csname regval:M-stress-e3-P1-gpt-5.5-id.n\endcsname{24}
\expandafter\def\csname regval:M-stress-e3-P1-gpt-5.5-id.d\endcsname{100}
\expandafter\def\csname regval:M-stress-e3-P1-gpt-5.5-od\endcsname{0.040}
\expandafter\def\csname regval:M-stress-e3-P1-gpt-5.5-od.n\endcsname{4}
\expandafter\def\csname regval:M-stress-e3-P1-gpt-5.5-od.d\endcsname{100}
\expandafter\def\csname regval:M-stress-e3-P1-gpt-5.5-acc\endcsname{0.400}
\expandafter\def\csname regval:M-stress-e3-P1-gpt-5.5-acc.n\endcsname{40}
\expandafter\def\csname regval:M-stress-e3-P1-gpt-5.5-acc.d\endcsname{100}
\expandafter\def\csname regval:M-stress-e3-P1-gpt-5.5-answered\endcsname{0.770}
\expandafter\def\csname regval:M-stress-e3-P1-gpt-5.5-answered.n\endcsname{77}
\expandafter\def\csname regval:M-stress-e3-P1-gpt-5.5-answered.d\endcsname{100}
\expandafter\def\csname regval:M-stress-e3-P1-gpt-4o-mini-iaimr\endcsname{0.117}
\expandafter\def\csname regval:M-stress-e3-P1-gpt-4o-mini-iaimr.n\endcsname{35}
\expandafter\def\csname regval:M-stress-e3-P1-gpt-4o-mini-iaimr.d\endcsname{300}
\expandafter\def\csname regval:M-stress-e3-P1-gpt-4o-mini-iaimr.ci\endcsname{[0.085, 0.158]}
\expandafter\def\csname regval:M-stress-e3-P1-gpt-4o-mini-blank\endcsname{0.070}
\expandafter\def\csname regval:M-stress-e3-P1-gpt-4o-mini-blank.n\endcsname{7}
\expandafter\def\csname regval:M-stress-e3-P1-gpt-4o-mini-blank.d\endcsname{100}
\expandafter\def\csname regval:M-stress-e3-P1-gpt-4o-mini-id\endcsname{0.240}
\expandafter\def\csname regval:M-stress-e3-P1-gpt-4o-mini-id.n\endcsname{24}
\expandafter\def\csname regval:M-stress-e3-P1-gpt-4o-mini-id.d\endcsname{100}
\expandafter\def\csname regval:M-stress-e3-P1-gpt-4o-mini-od\endcsname{0.040}
\expandafter\def\csname regval:M-stress-e3-P1-gpt-4o-mini-od.n\endcsname{4}
\expandafter\def\csname regval:M-stress-e3-P1-gpt-4o-mini-od.d\endcsname{100}
\expandafter\def\csname regval:M-stress-e3-P1-gpt-4o-mini-acc\endcsname{0.400}
\expandafter\def\csname regval:M-stress-e3-P1-gpt-4o-mini-acc.n\endcsname{40}
\expandafter\def\csname regval:M-stress-e3-P1-gpt-4o-mini-acc.d\endcsname{100}
\expandafter\def\csname regval:M-stress-e3-P1-gpt-4o-mini-answered\endcsname{0.770}
\expandafter\def\csname regval:M-stress-e3-P1-gpt-4o-mini-answered.n\endcsname{77}
\expandafter\def\csname regval:M-stress-e3-P1-gpt-4o-mini-answered.d\endcsname{100}
\expandafter\def\csname regval:M-stress-e3-P2-gpt-5.5-iaimr\endcsname{0.087}
\expandafter\def\csname regval:M-stress-e3-P2-gpt-5.5-iaimr.n\endcsname{26}
\expandafter\def\csname regval:M-stress-e3-P2-gpt-5.5-iaimr.d\endcsname{300}
\expandafter\def\csname regval:M-stress-e3-P2-gpt-5.5-iaimr.ci\endcsname{[0.060, 0.124]}
\expandafter\def\csname regval:M-stress-e3-P2-gpt-5.5-blank\endcsname{0.000}
\expandafter\def\csname regval:M-stress-e3-P2-gpt-5.5-blank.n\endcsname{0}
\expandafter\def\csname regval:M-stress-e3-P2-gpt-5.5-blank.d\endcsname{100}
\expandafter\def\csname regval:M-stress-e3-P2-gpt-5.5-id\endcsname{0.230}
\expandafter\def\csname regval:M-stress-e3-P2-gpt-5.5-id.n\endcsname{23}
\expandafter\def\csname regval:M-stress-e3-P2-gpt-5.5-id.d\endcsname{100}
\expandafter\def\csname regval:M-stress-e3-P2-gpt-5.5-od\endcsname{0.030}
\expandafter\def\csname regval:M-stress-e3-P2-gpt-5.5-od.n\endcsname{3}
\expandafter\def\csname regval:M-stress-e3-P2-gpt-5.5-od.d\endcsname{100}
\expandafter\def\csname regval:M-stress-e3-P2-gpt-5.5-acc\endcsname{0.380}
\expandafter\def\csname regval:M-stress-e3-P2-gpt-5.5-acc.n\endcsname{38}
\expandafter\def\csname regval:M-stress-e3-P2-gpt-5.5-acc.d\endcsname{100}
\expandafter\def\csname regval:M-stress-e3-P2-gpt-5.5-answered\endcsname{0.800}
\expandafter\def\csname regval:M-stress-e3-P2-gpt-5.5-answered.n\endcsname{80}
\expandafter\def\csname regval:M-stress-e3-P2-gpt-5.5-answered.d\endcsname{100}
\expandafter\def\csname regval:M-stress-e3-P2-gpt-4o-mini-iaimr\endcsname{0.087}
\expandafter\def\csname regval:M-stress-e3-P2-gpt-4o-mini-iaimr.n\endcsname{26}
\expandafter\def\csname regval:M-stress-e3-P2-gpt-4o-mini-iaimr.d\endcsname{300}
\expandafter\def\csname regval:M-stress-e3-P2-gpt-4o-mini-iaimr.ci\endcsname{[0.060, 0.124]}
\expandafter\def\csname regval:M-stress-e3-P2-gpt-4o-mini-blank\endcsname{0.000}
\expandafter\def\csname regval:M-stress-e3-P2-gpt-4o-mini-blank.n\endcsname{0}
\expandafter\def\csname regval:M-stress-e3-P2-gpt-4o-mini-blank.d\endcsname{100}
\expandafter\def\csname regval:M-stress-e3-P2-gpt-4o-mini-id\endcsname{0.230}
\expandafter\def\csname regval:M-stress-e3-P2-gpt-4o-mini-id.n\endcsname{23}
\expandafter\def\csname regval:M-stress-e3-P2-gpt-4o-mini-id.d\endcsname{100}
\expandafter\def\csname regval:M-stress-e3-P2-gpt-4o-mini-od\endcsname{0.030}
\expandafter\def\csname regval:M-stress-e3-P2-gpt-4o-mini-od.n\endcsname{3}
\expandafter\def\csname regval:M-stress-e3-P2-gpt-4o-mini-od.d\endcsname{100}
\expandafter\def\csname regval:M-stress-e3-P2-gpt-4o-mini-acc\endcsname{0.370}
\expandafter\def\csname regval:M-stress-e3-P2-gpt-4o-mini-acc.n\endcsname{37}
\expandafter\def\csname regval:M-stress-e3-P2-gpt-4o-mini-acc.d\endcsname{100}
\expandafter\def\csname regval:M-stress-e3-P2-gpt-4o-mini-answered\endcsname{0.800}
\expandafter\def\csname regval:M-stress-e3-P2-gpt-4o-mini-answered.n\endcsname{80}
\expandafter\def\csname regval:M-stress-e3-P2-gpt-4o-mini-answered.d\endcsname{100}
\expandafter\def\csname regval:M-stress-e4-P0-gpt-5.5-iaimr\endcsname{0.513}
\expandafter\def\csname regval:M-stress-e4-P0-gpt-5.5-iaimr.n\endcsname{154}
\expandafter\def\csname regval:M-stress-e4-P0-gpt-5.5-iaimr.d\endcsname{300}
\expandafter\def\csname regval:M-stress-e4-P0-gpt-5.5-iaimr.ci\endcsname{[0.457, 0.569]}
\expandafter\def\csname regval:M-stress-e4-P0-gpt-5.5-blank\endcsname{0.010}
\expandafter\def\csname regval:M-stress-e4-P0-gpt-5.5-blank.n\endcsname{1}
\expandafter\def\csname regval:M-stress-e4-P0-gpt-5.5-blank.d\endcsname{100}
\expandafter\def\csname regval:M-stress-e4-P0-gpt-5.5-id\endcsname{0.780}
\expandafter\def\csname regval:M-stress-e4-P0-gpt-5.5-id.n\endcsname{78}
\expandafter\def\csname regval:M-stress-e4-P0-gpt-5.5-id.d\endcsname{100}
\expandafter\def\csname regval:M-stress-e4-P0-gpt-5.5-od\endcsname{0.750}
\expandafter\def\csname regval:M-stress-e4-P0-gpt-5.5-od.n\endcsname{75}
\expandafter\def\csname regval:M-stress-e4-P0-gpt-5.5-od.d\endcsname{100}
\expandafter\def\csname regval:M-stress-e4-P0-gpt-5.5-acc\endcsname{0.340}
\expandafter\def\csname regval:M-stress-e4-P0-gpt-5.5-acc.n\endcsname{34}
\expandafter\def\csname regval:M-stress-e4-P0-gpt-5.5-acc.d\endcsname{100}
\expandafter\def\csname regval:M-stress-e4-P0-gpt-5.5-answered\endcsname{0.850}
\expandafter\def\csname regval:M-stress-e4-P0-gpt-5.5-answered.n\endcsname{85}
\expandafter\def\csname regval:M-stress-e4-P0-gpt-5.5-answered.d\endcsname{100}
\expandafter\def\csname regval:M-stress-e4-P0-gpt-4o-mini-iaimr\endcsname{0.223}
\expandafter\def\csname regval:M-stress-e4-P0-gpt-4o-mini-iaimr.n\endcsname{67}
\expandafter\def\csname regval:M-stress-e4-P0-gpt-4o-mini-iaimr.d\endcsname{300}
\expandafter\def\csname regval:M-stress-e4-P0-gpt-4o-mini-iaimr.ci\endcsname{[0.180, 0.274]}
\expandafter\def\csname regval:M-stress-e4-P0-gpt-4o-mini-blank\endcsname{0.010}
\expandafter\def\csname regval:M-stress-e4-P0-gpt-4o-mini-blank.n\endcsname{1}
\expandafter\def\csname regval:M-stress-e4-P0-gpt-4o-mini-blank.d\endcsname{100}
\expandafter\def\csname regval:M-stress-e4-P0-gpt-4o-mini-id\endcsname{0.380}
\expandafter\def\csname regval:M-stress-e4-P0-gpt-4o-mini-id.n\endcsname{38}
\expandafter\def\csname regval:M-stress-e4-P0-gpt-4o-mini-id.d\endcsname{100}
\expandafter\def\csname regval:M-stress-e4-P0-gpt-4o-mini-od\endcsname{0.280}
\expandafter\def\csname regval:M-stress-e4-P0-gpt-4o-mini-od.n\endcsname{28}
\expandafter\def\csname regval:M-stress-e4-P0-gpt-4o-mini-od.d\endcsname{100}
\expandafter\def\csname regval:M-stress-e4-P0-gpt-4o-mini-acc\endcsname{0.310}
\expandafter\def\csname regval:M-stress-e4-P0-gpt-4o-mini-acc.n\endcsname{31}
\expandafter\def\csname regval:M-stress-e4-P0-gpt-4o-mini-acc.d\endcsname{100}
\expandafter\def\csname regval:M-stress-e4-P0-gpt-4o-mini-answered\endcsname{0.840}
\expandafter\def\csname regval:M-stress-e4-P0-gpt-4o-mini-answered.n\endcsname{84}
\expandafter\def\csname regval:M-stress-e4-P0-gpt-4o-mini-answered.d\endcsname{100}
\expandafter\def\csname regval:M-stress-e4-P1-gpt-5.5-iaimr\endcsname{0.530}
\expandafter\def\csname regval:M-stress-e4-P1-gpt-5.5-iaimr.n\endcsname{159}
\expandafter\def\csname regval:M-stress-e4-P1-gpt-5.5-iaimr.d\endcsname{300}
\expandafter\def\csname regval:M-stress-e4-P1-gpt-5.5-iaimr.ci\endcsname{[0.473, 0.586]}
\expandafter\def\csname regval:M-stress-e4-P1-gpt-5.5-blank\endcsname{0.030}
\expandafter\def\csname regval:M-stress-e4-P1-gpt-5.5-blank.n\endcsname{3}
\expandafter\def\csname regval:M-stress-e4-P1-gpt-5.5-blank.d\endcsname{100}
\expandafter\def\csname regval:M-stress-e4-P1-gpt-5.5-id\endcsname{0.710}
\expandafter\def\csname regval:M-stress-e4-P1-gpt-5.5-id.n\endcsname{71}
\expandafter\def\csname regval:M-stress-e4-P1-gpt-5.5-id.d\endcsname{100}
\expandafter\def\csname regval:M-stress-e4-P1-gpt-5.5-od\endcsname{0.850}
\expandafter\def\csname regval:M-stress-e4-P1-gpt-5.5-od.n\endcsname{85}
\expandafter\def\csname regval:M-stress-e4-P1-gpt-5.5-od.d\endcsname{100}
\expandafter\def\csname regval:M-stress-e4-P1-gpt-5.5-acc\endcsname{0.340}
\expandafter\def\csname regval:M-stress-e4-P1-gpt-5.5-acc.n\endcsname{34}
\expandafter\def\csname regval:M-stress-e4-P1-gpt-5.5-acc.d\endcsname{100}
\expandafter\def\csname regval:M-stress-e4-P1-gpt-5.5-answered\endcsname{0.840}
\expandafter\def\csname regval:M-stress-e4-P1-gpt-5.5-answered.n\endcsname{84}
\expandafter\def\csname regval:M-stress-e4-P1-gpt-5.5-answered.d\endcsname{100}
\expandafter\def\csname regval:M-stress-e4-P1-gpt-4o-mini-iaimr\endcsname{0.303}
\expandafter\def\csname regval:M-stress-e4-P1-gpt-4o-mini-iaimr.n\endcsname{91}
\expandafter\def\csname regval:M-stress-e4-P1-gpt-4o-mini-iaimr.d\endcsname{300}
\expandafter\def\csname regval:M-stress-e4-P1-gpt-4o-mini-iaimr.ci\endcsname{[0.254, 0.358]}
\expandafter\def\csname regval:M-stress-e4-P1-gpt-4o-mini-blank\endcsname{0.030}
\expandafter\def\csname regval:M-stress-e4-P1-gpt-4o-mini-blank.n\endcsname{3}
\expandafter\def\csname regval:M-stress-e4-P1-gpt-4o-mini-blank.d\endcsname{100}
\expandafter\def\csname regval:M-stress-e4-P1-gpt-4o-mini-id\endcsname{0.440}
\expandafter\def\csname regval:M-stress-e4-P1-gpt-4o-mini-id.n\endcsname{44}
\expandafter\def\csname regval:M-stress-e4-P1-gpt-4o-mini-id.d\endcsname{100}
\expandafter\def\csname regval:M-stress-e4-P1-gpt-4o-mini-od\endcsname{0.440}
\expandafter\def\csname regval:M-stress-e4-P1-gpt-4o-mini-od.n\endcsname{44}
\expandafter\def\csname regval:M-stress-e4-P1-gpt-4o-mini-od.d\endcsname{100}
\expandafter\def\csname regval:M-stress-e4-P1-gpt-4o-mini-acc\endcsname{0.310}
\expandafter\def\csname regval:M-stress-e4-P1-gpt-4o-mini-acc.n\endcsname{31}
\expandafter\def\csname regval:M-stress-e4-P1-gpt-4o-mini-acc.d\endcsname{100}
\expandafter\def\csname regval:M-stress-e4-P1-gpt-4o-mini-answered\endcsname{0.820}
\expandafter\def\csname regval:M-stress-e4-P1-gpt-4o-mini-answered.n\endcsname{82}
\expandafter\def\csname regval:M-stress-e4-P1-gpt-4o-mini-answered.d\endcsname{100}
\expandafter\def\csname regval:M-stress-e4-P2-gpt-5.5-iaimr\endcsname{0.477}
\expandafter\def\csname regval:M-stress-e4-P2-gpt-5.5-iaimr.n\endcsname{143}
\expandafter\def\csname regval:M-stress-e4-P2-gpt-5.5-iaimr.d\endcsname{300}
\expandafter\def\csname regval:M-stress-e4-P2-gpt-5.5-iaimr.ci\endcsname{[0.421, 0.533]}
\expandafter\def\csname regval:M-stress-e4-P2-gpt-5.5-blank\endcsname{0.010}
\expandafter\def\csname regval:M-stress-e4-P2-gpt-5.5-blank.n\endcsname{1}
\expandafter\def\csname regval:M-stress-e4-P2-gpt-5.5-blank.d\endcsname{100}
\expandafter\def\csname regval:M-stress-e4-P2-gpt-5.5-id\endcsname{0.690}
\expandafter\def\csname regval:M-stress-e4-P2-gpt-5.5-id.n\endcsname{69}
\expandafter\def\csname regval:M-stress-e4-P2-gpt-5.5-id.d\endcsname{100}
\expandafter\def\csname regval:M-stress-e4-P2-gpt-5.5-od\endcsname{0.730}
\expandafter\def\csname regval:M-stress-e4-P2-gpt-5.5-od.n\endcsname{73}
\expandafter\def\csname regval:M-stress-e4-P2-gpt-5.5-od.d\endcsname{100}
\expandafter\def\csname regval:M-stress-e4-P2-gpt-5.5-acc\endcsname{0.320}
\expandafter\def\csname regval:M-stress-e4-P2-gpt-5.5-acc.n\endcsname{32}
\expandafter\def\csname regval:M-stress-e4-P2-gpt-5.5-acc.d\endcsname{100}
\expandafter\def\csname regval:M-stress-e4-P2-gpt-5.5-answered\endcsname{0.850}
\expandafter\def\csname regval:M-stress-e4-P2-gpt-5.5-answered.n\endcsname{85}
\expandafter\def\csname regval:M-stress-e4-P2-gpt-5.5-answered.d\endcsname{100}
\expandafter\def\csname regval:M-stress-e4-P2-gpt-4o-mini-iaimr\endcsname{0.203}
\expandafter\def\csname regval:M-stress-e4-P2-gpt-4o-mini-iaimr.n\endcsname{61}
\expandafter\def\csname regval:M-stress-e4-P2-gpt-4o-mini-iaimr.d\endcsname{300}
\expandafter\def\csname regval:M-stress-e4-P2-gpt-4o-mini-iaimr.ci\endcsname{[0.162, 0.253]}
\expandafter\def\csname regval:M-stress-e4-P2-gpt-4o-mini-blank\endcsname{0.010}
\expandafter\def\csname regval:M-stress-e4-P2-gpt-4o-mini-blank.n\endcsname{1}
\expandafter\def\csname regval:M-stress-e4-P2-gpt-4o-mini-blank.d\endcsname{100}
\expandafter\def\csname regval:M-stress-e4-P2-gpt-4o-mini-id\endcsname{0.370}
\expandafter\def\csname regval:M-stress-e4-P2-gpt-4o-mini-id.n\endcsname{37}
\expandafter\def\csname regval:M-stress-e4-P2-gpt-4o-mini-id.d\endcsname{100}
\expandafter\def\csname regval:M-stress-e4-P2-gpt-4o-mini-od\endcsname{0.230}
\expandafter\def\csname regval:M-stress-e4-P2-gpt-4o-mini-od.n\endcsname{23}
\expandafter\def\csname regval:M-stress-e4-P2-gpt-4o-mini-od.d\endcsname{100}
\expandafter\def\csname regval:M-stress-e4-P2-gpt-4o-mini-acc\endcsname{0.300}
\expandafter\def\csname regval:M-stress-e4-P2-gpt-4o-mini-acc.n\endcsname{30}
\expandafter\def\csname regval:M-stress-e4-P2-gpt-4o-mini-acc.d\endcsname{100}
\expandafter\def\csname regval:M-stress-e4-P2-gpt-4o-mini-answered\endcsname{0.830}
\expandafter\def\csname regval:M-stress-e4-P2-gpt-4o-mini-answered.n\endcsname{83}
\expandafter\def\csname regval:M-stress-e4-P2-gpt-4o-mini-answered.d\endcsname{100}
\expandafter\def\csname regval:M-diag-e4-abstain-before\endcsname{0.021}
\expandafter\def\csname regval:M-diag-e4-abstain-before.d\endcsname{10}
\expandafter\def\csname regval:M-diag-e4-abstain-after\endcsname{1.675}
\expandafter\def\csname regval:M-diag-e4-abstain-after.d\endcsname{10}
\expandafter\def\csname regval:M-diag-e4-real-before\endcsname{0.332}
\expandafter\def\csname regval:M-diag-e4-real-before.d\endcsname{10}
\expandafter\def\csname regval:M-diag-e4-real-after\endcsname{2.611}
\expandafter\def\csname regval:M-diag-e4-real-after.d\endcsname{10}
\expandafter\def\csname regval:M-diag-e5-mean_A_positive-before\endcsname{4.963}
\expandafter\def\csname regval:M-diag-e5-mean_A_positive-before.d\endcsname{64}
\expandafter\def\csname regval:M-diag-e5-mean_A_positive-after\endcsname{12.1}
\expandafter\def\csname regval:M-diag-e5-mean_A_positive-after.d\endcsname{64}
\expandafter\def\csname regval:M-diag-e5-candidate_top1_accuracy-before\endcsname{0.750}
\expandafter\def\csname regval:M-diag-e5-candidate_top1_accuracy-before.d\endcsname{64}
\expandafter\def\csname regval:M-diag-e5-candidate_top1_accuracy-after\endcsname{0.938}
\expandafter\def\csname regval:M-diag-e5-candidate_top1_accuracy-after.d\endcsname{64}
\expandafter\def\csname regval:M-diag-e5-evidence_positive_rate-before\endcsname{0.625}
\expandafter\def\csname regval:M-diag-e5-evidence_positive_rate-before.d\endcsname{64}
\expandafter\def\csname regval:M-diag-e5-evidence_positive_rate-after\endcsname{0.984}
\expandafter\def\csname regval:M-diag-e5-evidence_positive_rate-after.d\endcsname{64}
\expandafter\def\csname regval:M-diag-e5-contrastive_loss-before\endcsname{0.760}
\expandafter\def\csname regval:M-diag-e5-contrastive_loss-before.d\endcsname{64}
\expandafter\def\csname regval:M-diag-e5-contrastive_loss-after\endcsname{0.145}
\expandafter\def\csname regval:M-diag-e5-contrastive_loss-after.d\endcsname{64}
\expandafter\def\csname regval:M-diag-e5-empty-grad\endcsname{0.000}
\expandafter\def\csname regval:M-diag-e5-empty-grad.d\endcsname{300}
\expandafter\def\csname regval:M-diag-e2-heldmatch\endcsname{0.924}
\expandafter\def\csname regval:M-diag-e2-heldmatch.n\endcsname{231}
\expandafter\def\csname regval:M-diag-e2-heldmatch.d\endcsname{250}
\expandafter\def\csname regval:M-poc-bclean-removed-total\endcsname{0.461}
\expandafter\def\csname regval:M-poc-bclean-removed-total.n\endcsname{2{,}190}
\expandafter\def\csname regval:M-poc-bclean-removed-total.d\endcsname{4{,}751}
\expandafter\def\csname regval:M-poc-bclean-lomo-min\endcsname{0.383}
\expandafter\def\csname regval:M-poc-bclean-lomo-min.d\endcsname{4{,}751}
\expandafter\def\csname regval:M-poc-bclean-lomo-max\endcsname{0.456}
\expandafter\def\csname regval:M-poc-bclean-lomo-max.d\endcsname{4{,}751}
\expandafter\def\csname regval:M-poc-union5-expected\endcsname{0.750}
\expandafter\def\csname regval:M-poc-union5-expected.d\endcsname{2{,}816}
\expandafter\def\csname regval:M-poc-bclean-removed-chancelevel\endcsname{0.658}
\expandafter\def\csname regval:M-poc-bclean-removed-chancelevel.d\endcsname{2{,}816}
\expandafter\def\csname regval:M-poc-bclean-removed-nochance\endcsname{0.175}
\expandafter\def\csname regval:M-poc-bclean-removed-nochance.d\endcsname{1{,}935}
\expandafter\def\csname regval:M-poc-chancelevel-n\endcsname{2{,}816}
\expandafter\def\csname regval:M-poc-chancelevel-n.n\endcsname{2{,}816}
\expandafter\def\csname regval:M-poc-chancelevel-n.d\endcsname{4{,}751}
\expandafter\def\csname regval:M-poc-bclean-removed-fmt-mcq\endcsname{0.610}
\expandafter\def\csname regval:M-poc-bclean-removed-fmt-mcq.d\endcsname{2{,}280}
\expandafter\def\csname regval:M-poc-fmt-n-mcq\endcsname{2{,}280}
\expandafter\def\csname regval:M-poc-fmt-n-mcq.n\endcsname{2{,}280}
\expandafter\def\csname regval:M-poc-fmt-n-mcq.d\endcsname{4{,}751}
\expandafter\def\csname regval:M-poc-bclean-removed-fmt-yesno\endcsname{0.862}
\expandafter\def\csname regval:M-poc-bclean-removed-fmt-yesno.d\endcsname{536}
\expandafter\def\csname regval:M-poc-fmt-n-yesno\endcsname{536}
\expandafter\def\csname regval:M-poc-fmt-n-yesno.n\endcsname{536}
\expandafter\def\csname regval:M-poc-fmt-n-yesno.d\endcsname{4{,}751}
\expandafter\def\csname regval:M-poc-bclean-removed-fmt-open\endcsname{0.175}
\expandafter\def\csname regval:M-poc-bclean-removed-fmt-open.d\endcsname{1{,}935}
\expandafter\def\csname regval:M-poc-fmt-n-open\endcsname{1{,}935}
\expandafter\def\csname regval:M-poc-fmt-n-open.n\endcsname{1{,}935}
\expandafter\def\csname regval:M-poc-fmt-n-open.d\endcsname{4{,}751}
\expandafter\def\csname regval:M-poc-cov-blind\endcsname{23{,}755}
\expandafter\def\csname regval:M-poc-cov-image\endcsname{23{,}723}
\expandafter\def\csname regval:M-poc-cov-image-missing\endcsname{32}
\expandafter\def\csname regval:M-poc-cov-paired\endcsname{19{,}607}
\expandafter\def\csname regval:M-poc-cov-paired.n\endcsname{19{,}607}
\expandafter\def\csname regval:M-poc-cov-paired.d\endcsname{23{,}755}
\expandafter\def\csname regval:M-poc-cov-paired-claude-sonnet-4.6\endcsname{4{,}702}
\expandafter\def\csname regval:M-poc-cov-paired-claude-sonnet-4.6.n\endcsname{4{,}702}
\expandafter\def\csname regval:M-poc-cov-paired-claude-sonnet-4.6.d\endcsname{4{,}751}
\expandafter\def\csname regval:M-poc-cov-paired-gemini-3.1-flash-lite\endcsname{4{,}751}
\expandafter\def\csname regval:M-poc-cov-paired-gemini-3.1-flash-lite.n\endcsname{4{,}751}
\expandafter\def\csname regval:M-poc-cov-paired-gemini-3.1-flash-lite.d\endcsname{4{,}751}
\expandafter\def\csname regval:M-poc-cov-paired-llama-4-scout\endcsname{4{,}601}
\expandafter\def\csname regval:M-poc-cov-paired-llama-4-scout.n\endcsname{4{,}601}
\expandafter\def\csname regval:M-poc-cov-paired-llama-4-scout.d\endcsname{4{,}751}
\expandafter\def\csname regval:M-poc-cov-paired-gpt-5.5\endcsname{805}
\expandafter\def\csname regval:M-poc-cov-paired-gpt-5.5.n\endcsname{805}
\expandafter\def\csname regval:M-poc-cov-paired-gpt-5.5.d\endcsname{4{,}751}
\expandafter\def\csname regval:M-poc-cov-paired-qwen2.5-vl-72b-instruct\endcsname{4{,}748}
\expandafter\def\csname regval:M-poc-cov-paired-qwen2.5-vl-72b-instruct.n\endcsname{4{,}748}
\expandafter\def\csname regval:M-poc-cov-paired-qwen2.5-vl-72b-instruct.d\endcsname{4{,}751}
\expandafter\def\csname regval:M-poc-cov-image-csv\endcsname{23{,}721}
\expandafter\def\csname regval:M-poc-cov-image-missing-claude-sonnet-4.6-infovqa\endcsname{28}
\expandafter\def\csname regval:M-poc-cov-image-missing-gpt-5.5-pathvqa\endcsname{4}
\expandafter\def\csname regval:M-poc-union5-fmt-mcq\endcsname{0.699}
\expandafter\def\csname regval:M-poc-union5-fmt-mcq.d\endcsname{2{,}280}
\expandafter\def\csname regval:M-poc-union5-fmt-yesno\endcsname{0.969}
\expandafter\def\csname regval:M-poc-union5-fmt-yesno.d\endcsname{536}
\expandafter\def\csname regval:M-poc-bclean-clean-n\endcsname{2{,}561}
\expandafter\def\csname regval:M-poc-bclean-clean-n.n\endcsname{2{,}561}
\expandafter\def\csname regval:M-poc-bclean-clean-n.d\endcsname{4{,}751}
\expandafter\def\csname regval:M-poc-cov-blind-unjudged\endcsname{51}
\expandafter\def\csname regval:M-poc-cov-image-unjudged\endcsname{123}
\expandafter\def\csname regval:M-poc-empty-blind-claude-sonnet-4.6\endcsname{0.004}
\expandafter\def\csname regval:M-poc-empty-blind-claude-sonnet-4.6.n\endcsname{21}
\expandafter\def\csname regval:M-poc-empty-blind-claude-sonnet-4.6.d\endcsname{4{,}751}
\expandafter\def\csname regval:M-poc-empty-blind-gemini-3.1-flash-lite\endcsname{0}
\expandafter\def\csname regval:M-poc-empty-blind-gemini-3.1-flash-lite.n\endcsname{0}
\expandafter\def\csname regval:M-poc-empty-blind-gemini-3.1-flash-lite.d\endcsname{4{,}751}
\expandafter\def\csname regval:M-poc-empty-blind-llama-4-scout\endcsname{0.008}
\expandafter\def\csname regval:M-poc-empty-blind-llama-4-scout.n\endcsname{39}
\expandafter\def\csname regval:M-poc-empty-blind-llama-4-scout.d\endcsname{4{,}751}
\expandafter\def\csname regval:M-poc-empty-blind-gpt-5.5\endcsname{0.789}
\expandafter\def\csname regval:M-poc-empty-blind-gpt-5.5.n\endcsname{3{,}751}
\expandafter\def\csname regval:M-poc-empty-blind-gpt-5.5.d\endcsname{4{,}751}
\expandafter\def\csname regval:M-poc-empty-blind-qwen2.5-vl-72b-instruct\endcsname{0.000}
\expandafter\def\csname regval:M-poc-empty-blind-qwen2.5-vl-72b-instruct.n\endcsname{2}
\expandafter\def\csname regval:M-poc-empty-blind-qwen2.5-vl-72b-instruct.d\endcsname{4{,}751}
\expandafter\def\csname regval:M-poc-empty-image-claude-sonnet-4.6\endcsname{0.004}
\expandafter\def\csname regval:M-poc-empty-image-claude-sonnet-4.6.n\endcsname{20}
\expandafter\def\csname regval:M-poc-empty-image-claude-sonnet-4.6.d\endcsname{4{,}751}
\expandafter\def\csname regval:M-poc-empty-image-gemini-3.1-flash-lite\endcsname{0}
\expandafter\def\csname regval:M-poc-empty-image-gemini-3.1-flash-lite.n\endcsname{0}
\expandafter\def\csname regval:M-poc-empty-image-gemini-3.1-flash-lite.d\endcsname{4{,}751}
\expandafter\def\csname regval:M-poc-empty-image-llama-4-scout\endcsname{0.024}
\expandafter\def\csname regval:M-poc-empty-image-llama-4-scout.n\endcsname{114}
\expandafter\def\csname regval:M-poc-empty-image-llama-4-scout.d\endcsname{4{,}751}
\expandafter\def\csname regval:M-poc-empty-image-gpt-5.5\endcsname{0.384}
\expandafter\def\csname regval:M-poc-empty-image-gpt-5.5.n\endcsname{1{,}826}
\expandafter\def\csname regval:M-poc-empty-image-gpt-5.5.d\endcsname{4{,}751}
\expandafter\def\csname regval:M-poc-empty-image-qwen2.5-vl-72b-instruct\endcsname{0.000}
\expandafter\def\csname regval:M-poc-empty-image-qwen2.5-vl-72b-instruct.n\endcsname{1}
\expandafter\def\csname regval:M-poc-empty-image-qwen2.5-vl-72b-instruct.d\endcsname{4{,}751}
\expandafter\def\csname regval:M-poc-empty-gpt55\endcsname{0.789}
\expandafter\def\csname regval:M-poc-empty-gpt55.n\endcsname{3{,}751}
\expandafter\def\csname regval:M-poc-empty-gpt55.d\endcsname{4{,}751}
\expandafter\def\csname regval:M-poc-bclean-n-chest_vqa\endcsname{451}
\expandafter\def\csname regval:M-poc-bclean-removed-chest_vqa\endcsname{0.612}
\expandafter\def\csname regval:M-poc-bclean-removed-chest_vqa.n\endcsname{276}
\expandafter\def\csname regval:M-poc-bclean-removed-chest_vqa.d\endcsname{451}
\expandafter\def\csname regval:M-poc-lomo-min-chest_vqa\endcsname{0.388}
\expandafter\def\csname regval:M-poc-lomo-min-chest_vqa.d\endcsname{451}
\expandafter\def\csname regval:M-poc-lomo-max-chest_vqa\endcsname{0.612}
\expandafter\def\csname regval:M-poc-lomo-max-chest_vqa.d\endcsname{451}
\expandafter\def\csname regval:M-poc-union5-chest_vqa\endcsname{0.539}
\expandafter\def\csname regval:M-poc-union5-chest_vqa.d\endcsname{451}
\expandafter\def\csname regval:M-poc-substantive-rate-chest_vqa\endcsname{0.482}
\expandafter\def\csname regval:M-poc-substantive-rate-chest_vqa.d\endcsname{451}
\expandafter\def\csname regval:M-poc-decline-rate-chest_vqa\endcsname{0.358}
\expandafter\def\csname regval:M-poc-decline-rate-chest_vqa.d\endcsname{451}
\expandafter\def\csname regval:M-poc-empty-rate-chest_vqa\endcsname{0.160}
\expandafter\def\csname regval:M-poc-empty-rate-chest_vqa.d\endcsname{451}
\expandafter\def\csname regval:M-poc-blind-acc-paired-chest_vqa\endcsname{0.308}
\expandafter\def\csname regval:M-poc-blind-acc-paired-chest_vqa.d\endcsname{451}
\expandafter\def\csname regval:M-poc-image-acc-paired-chest_vqa\endcsname{0.671}
\expandafter\def\csname regval:M-poc-image-acc-paired-chest_vqa.d\endcsname{451}
\expandafter\def\csname regval:M-poc-gain-paired-chest_vqa\endcsname{0.364}
\expandafter\def\csname regval:M-poc-gain-paired-chest_vqa.d\endcsname{451}
\expandafter\def\csname regval:M-poc-gain-paired-empty-incorrect-chest_vqa\endcsname{0.359}
\expandafter\def\csname regval:M-poc-gain-paired-empty-incorrect-chest_vqa.d\endcsname{451}
\expandafter\def\csname regval:M-poc-bclean-n-pathvqa\endcsname{500}
\expandafter\def\csname regval:M-poc-bclean-removed-pathvqa\endcsname{0.456}
\expandafter\def\csname regval:M-poc-bclean-removed-pathvqa.n\endcsname{228}
\expandafter\def\csname regval:M-poc-bclean-removed-pathvqa.d\endcsname{500}
\expandafter\def\csname regval:M-poc-lomo-min-pathvqa\endcsname{0.352}
\expandafter\def\csname regval:M-poc-lomo-min-pathvqa.d\endcsname{500}
\expandafter\def\csname regval:M-poc-lomo-max-pathvqa\endcsname{0.456}
\expandafter\def\csname regval:M-poc-lomo-max-pathvqa.d\endcsname{500}
\expandafter\def\csname regval:M-poc-union5-pathvqa\endcsname{0.488}
\expandafter\def\csname regval:M-poc-union5-pathvqa.d\endcsname{500}
\expandafter\def\csname regval:M-poc-substantive-rate-pathvqa\endcsname{0.490}
\expandafter\def\csname regval:M-poc-substantive-rate-pathvqa.d\endcsname{500}
\expandafter\def\csname regval:M-poc-decline-rate-pathvqa\endcsname{0.373}
\expandafter\def\csname regval:M-poc-decline-rate-pathvqa.d\endcsname{500}
\expandafter\def\csname regval:M-poc-empty-rate-pathvqa\endcsname{0.137}
\expandafter\def\csname regval:M-poc-empty-rate-pathvqa.d\endcsname{500}
\expandafter\def\csname regval:M-poc-blind-acc-paired-pathvqa\endcsname{0.276}
\expandafter\def\csname regval:M-poc-blind-acc-paired-pathvqa.d\endcsname{500}
\expandafter\def\csname regval:M-poc-image-acc-paired-pathvqa\endcsname{0.455}
\expandafter\def\csname regval:M-poc-image-acc-paired-pathvqa.d\endcsname{500}
\expandafter\def\csname regval:M-poc-gain-paired-pathvqa\endcsname{0.179}
\expandafter\def\csname regval:M-poc-gain-paired-pathvqa.d\endcsname{500}
\expandafter\def\csname regval:M-poc-gain-paired-empty-incorrect-pathvqa\endcsname{0.178}
\expandafter\def\csname regval:M-poc-gain-paired-empty-incorrect-pathvqa.d\endcsname{500}
\expandafter\def\csname regval:M-poc-bclean-n-rexvqa\endcsname{500}
\expandafter\def\csname regval:M-poc-bclean-removed-rexvqa\endcsname{0.834}
\expandafter\def\csname regval:M-poc-bclean-removed-rexvqa.n\endcsname{417}
\expandafter\def\csname regval:M-poc-bclean-removed-rexvqa.d\endcsname{500}
\expandafter\def\csname regval:M-poc-lomo-min-rexvqa\endcsname{0.766}
\expandafter\def\csname regval:M-poc-lomo-min-rexvqa.d\endcsname{500}
\expandafter\def\csname regval:M-poc-lomo-max-rexvqa\endcsname{0.824}
\expandafter\def\csname regval:M-poc-lomo-max-rexvqa.d\endcsname{500}
\expandafter\def\csname regval:M-poc-union5-rexvqa\endcsname{0.763}
\expandafter\def\csname regval:M-poc-union5-rexvqa.d\endcsname{500}
\expandafter\def\csname regval:M-poc-substantive-rate-rexvqa\endcsname{0.754}
\expandafter\def\csname regval:M-poc-substantive-rate-rexvqa.d\endcsname{500}
\expandafter\def\csname regval:M-poc-decline-rate-rexvqa\endcsname{0.126}
\expandafter\def\csname regval:M-poc-decline-rate-rexvqa.d\endcsname{500}
\expandafter\def\csname regval:M-poc-empty-rate-rexvqa\endcsname{0.121}
\expandafter\def\csname regval:M-poc-empty-rate-rexvqa.d\endcsname{500}
\expandafter\def\csname regval:M-poc-blind-acc-paired-rexvqa\endcsname{0.614}
\expandafter\def\csname regval:M-poc-blind-acc-paired-rexvqa.d\endcsname{500}
\expandafter\def\csname regval:M-poc-image-acc-paired-rexvqa\endcsname{0.744}
\expandafter\def\csname regval:M-poc-image-acc-paired-rexvqa.d\endcsname{500}
\expandafter\def\csname regval:M-poc-gain-paired-rexvqa\endcsname{0.130}
\expandafter\def\csname regval:M-poc-gain-paired-rexvqa.d\endcsname{500}
\expandafter\def\csname regval:M-poc-gain-paired-empty-incorrect-rexvqa\endcsname{0.138}
\expandafter\def\csname regval:M-poc-gain-paired-empty-incorrect-rexvqa.d\endcsname{500}
\expandafter\def\csname regval:M-poc-bclean-n-docvqa\endcsname{500}
\expandafter\def\csname regval:M-poc-bclean-removed-docvqa\endcsname{0.082}
\expandafter\def\csname regval:M-poc-bclean-removed-docvqa.n\endcsname{41}
\expandafter\def\csname regval:M-poc-bclean-removed-docvqa.d\endcsname{500}
\expandafter\def\csname regval:M-poc-lomo-min-docvqa\endcsname{0.070}
\expandafter\def\csname regval:M-poc-lomo-min-docvqa.d\endcsname{500}
\expandafter\def\csname regval:M-poc-lomo-max-docvqa\endcsname{0.080}
\expandafter\def\csname regval:M-poc-lomo-max-docvqa.d\endcsname{500}
\expandafter\def\csname regval:M-poc-union5-docvqa\endcsname{0.012}
\expandafter\def\csname regval:M-poc-union5-docvqa.d\endcsname{500}
\expandafter\def\csname regval:M-poc-substantive-rate-docvqa\endcsname{0.386}
\expandafter\def\csname regval:M-poc-substantive-rate-docvqa.d\endcsname{500}
\expandafter\def\csname regval:M-poc-decline-rate-docvqa\endcsname{0.426}
\expandafter\def\csname regval:M-poc-decline-rate-docvqa.d\endcsname{500}
\expandafter\def\csname regval:M-poc-empty-rate-docvqa\endcsname{0.188}
\expandafter\def\csname regval:M-poc-empty-rate-docvqa.d\endcsname{500}
\expandafter\def\csname regval:M-poc-blind-acc-paired-docvqa\endcsname{0.116}
\expandafter\def\csname regval:M-poc-blind-acc-paired-docvqa.d\endcsname{500}
\expandafter\def\csname regval:M-poc-image-acc-paired-docvqa\endcsname{0.888}
\expandafter\def\csname regval:M-poc-image-acc-paired-docvqa.d\endcsname{500}
\expandafter\def\csname regval:M-poc-gain-paired-docvqa\endcsname{0.773}
\expandafter\def\csname regval:M-poc-gain-paired-docvqa.d\endcsname{500}
\expandafter\def\csname regval:M-poc-gain-paired-empty-incorrect-docvqa\endcsname{0.818}
\expandafter\def\csname regval:M-poc-gain-paired-empty-incorrect-docvqa.d\endcsname{500}
\expandafter\def\csname regval:M-poc-bclean-n-textvqa\endcsname{500}
\expandafter\def\csname regval:M-poc-bclean-removed-textvqa\endcsname{0.166}
\expandafter\def\csname regval:M-poc-bclean-removed-textvqa.n\endcsname{83}
\expandafter\def\csname regval:M-poc-bclean-removed-textvqa.d\endcsname{500}
\expandafter\def\csname regval:M-poc-lomo-min-textvqa\endcsname{0.126}
\expandafter\def\csname regval:M-poc-lomo-min-textvqa.d\endcsname{500}
\expandafter\def\csname regval:M-poc-lomo-max-textvqa\endcsname{0.164}
\expandafter\def\csname regval:M-poc-lomo-max-textvqa.d\endcsname{500}
\expandafter\def\csname regval:M-poc-union5-textvqa\endcsname{0.052}
\expandafter\def\csname regval:M-poc-union5-textvqa.d\endcsname{500}
\expandafter\def\csname regval:M-poc-substantive-rate-textvqa\endcsname{0.331}
\expandafter\def\csname regval:M-poc-substantive-rate-textvqa.d\endcsname{500}
\expandafter\def\csname regval:M-poc-decline-rate-textvqa\endcsname{0.492}
\expandafter\def\csname regval:M-poc-decline-rate-textvqa.d\endcsname{500}
\expandafter\def\csname regval:M-poc-empty-rate-textvqa\endcsname{0.178}
\expandafter\def\csname regval:M-poc-empty-rate-textvqa.d\endcsname{500}
\expandafter\def\csname regval:M-poc-blind-acc-paired-textvqa\endcsname{0.092}
\expandafter\def\csname regval:M-poc-blind-acc-paired-textvqa.d\endcsname{500}
\expandafter\def\csname regval:M-poc-image-acc-paired-textvqa\endcsname{0.843}
\expandafter\def\csname regval:M-poc-image-acc-paired-textvqa.d\endcsname{500}
\expandafter\def\csname regval:M-poc-gain-paired-textvqa\endcsname{0.751}
\expandafter\def\csname regval:M-poc-gain-paired-textvqa.d\endcsname{500}
\expandafter\def\csname regval:M-poc-gain-paired-empty-incorrect-textvqa\endcsname{0.760}
\expandafter\def\csname regval:M-poc-gain-paired-empty-incorrect-textvqa.d\endcsname{500}
\expandafter\def\csname regval:M-poc-bclean-n-infovqa\endcsname{500}
\expandafter\def\csname regval:M-poc-bclean-removed-infovqa\endcsname{0.332}
\expandafter\def\csname regval:M-poc-bclean-removed-infovqa.n\endcsname{166}
\expandafter\def\csname regval:M-poc-bclean-removed-infovqa.d\endcsname{500}
\expandafter\def\csname regval:M-poc-lomo-min-infovqa\endcsname{0.296}
\expandafter\def\csname regval:M-poc-lomo-min-infovqa.d\endcsname{500}
\expandafter\def\csname regval:M-poc-lomo-max-infovqa\endcsname{0.332}
\expandafter\def\csname regval:M-poc-lomo-max-infovqa.d\endcsname{500}
\expandafter\def\csname regval:M-poc-union5-infovqa\endcsname{0.000}
\expandafter\def\csname regval:M-poc-union5-infovqa.d\endcsname{500}
\expandafter\def\csname regval:M-poc-substantive-rate-infovqa\endcsname{0.605}
\expandafter\def\csname regval:M-poc-substantive-rate-infovqa.d\endcsname{500}
\expandafter\def\csname regval:M-poc-decline-rate-infovqa\endcsname{0.235}
\expandafter\def\csname regval:M-poc-decline-rate-infovqa.d\endcsname{500}
\expandafter\def\csname regval:M-poc-empty-rate-infovqa\endcsname{0.160}
\expandafter\def\csname regval:M-poc-empty-rate-infovqa.d\endcsname{500}
\expandafter\def\csname regval:M-poc-blind-acc-paired-infovqa\endcsname{0.339}
\expandafter\def\csname regval:M-poc-blind-acc-paired-infovqa.d\endcsname{500}
\expandafter\def\csname regval:M-poc-image-acc-paired-infovqa\endcsname{0.877}
\expandafter\def\csname regval:M-poc-image-acc-paired-infovqa.d\endcsname{500}
\expandafter\def\csname regval:M-poc-gain-paired-infovqa\endcsname{0.537}
\expandafter\def\csname regval:M-poc-gain-paired-infovqa.d\endcsname{500}
\expandafter\def\csname regval:M-poc-gain-paired-empty-incorrect-infovqa\endcsname{0.627}
\expandafter\def\csname regval:M-poc-gain-paired-empty-incorrect-infovqa.d\endcsname{500}
\expandafter\def\csname regval:M-poc-bclean-n-microvqa\endcsname{500}
\expandafter\def\csname regval:M-poc-bclean-removed-microvqa\endcsname{0.738}
\expandafter\def\csname regval:M-poc-bclean-removed-microvqa.n\endcsname{369}
\expandafter\def\csname regval:M-poc-bclean-removed-microvqa.d\endcsname{500}
\expandafter\def\csname regval:M-poc-lomo-min-microvqa\endcsname{0.668}
\expandafter\def\csname regval:M-poc-lomo-min-microvqa.d\endcsname{500}
\expandafter\def\csname regval:M-poc-lomo-max-microvqa\endcsname{0.732}
\expandafter\def\csname regval:M-poc-lomo-max-microvqa.d\endcsname{500}
\expandafter\def\csname regval:M-poc-union5-microvqa\endcsname{0.714}
\expandafter\def\csname regval:M-poc-union5-microvqa.d\endcsname{500}
\expandafter\def\csname regval:M-poc-substantive-rate-microvqa\endcsname{0.863}
\expandafter\def\csname regval:M-poc-substantive-rate-microvqa.d\endcsname{500}
\expandafter\def\csname regval:M-poc-decline-rate-microvqa\endcsname{0.000}
\expandafter\def\csname regval:M-poc-decline-rate-microvqa.d\endcsname{500}
\expandafter\def\csname regval:M-poc-empty-rate-microvqa\endcsname{0.137}
\expandafter\def\csname regval:M-poc-empty-rate-microvqa.d\endcsname{500}
\expandafter\def\csname regval:M-poc-blind-acc-paired-microvqa\endcsname{0.572}
\expandafter\def\csname regval:M-poc-blind-acc-paired-microvqa.d\endcsname{500}
\expandafter\def\csname regval:M-poc-image-acc-paired-microvqa\endcsname{0.598}
\expandafter\def\csname regval:M-poc-image-acc-paired-microvqa.d\endcsname{500}
\expandafter\def\csname regval:M-poc-gain-paired-microvqa\endcsname{0.025}
\expandafter\def\csname regval:M-poc-gain-paired-microvqa.d\endcsname{500}
\expandafter\def\csname regval:M-poc-gain-paired-empty-incorrect-microvqa\endcsname{0.034}
\expandafter\def\csname regval:M-poc-gain-paired-empty-incorrect-microvqa.d\endcsname{500}
\expandafter\def\csname regval:M-poc-bclean-n-medxpertqa_mm\endcsname{500}
\expandafter\def\csname regval:M-poc-bclean-removed-medxpertqa_mm\endcsname{0.608}
\expandafter\def\csname regval:M-poc-bclean-removed-medxpertqa_mm.n\endcsname{304}
\expandafter\def\csname regval:M-poc-bclean-removed-medxpertqa_mm.d\endcsname{500}
\expandafter\def\csname regval:M-poc-lomo-min-medxpertqa_mm\endcsname{0.532}
\expandafter\def\csname regval:M-poc-lomo-min-medxpertqa_mm.d\endcsname{500}
\expandafter\def\csname regval:M-poc-lomo-max-medxpertqa_mm\endcsname{0.602}
\expandafter\def\csname regval:M-poc-lomo-max-medxpertqa_mm.d\endcsname{500}
\expandafter\def\csname regval:M-poc-union5-medxpertqa_mm\endcsname{0.672}
\expandafter\def\csname regval:M-poc-union5-medxpertqa_mm.d\endcsname{500}
\expandafter\def\csname regval:M-poc-substantive-rate-medxpertqa_mm\endcsname{0.802}
\expandafter\def\csname regval:M-poc-substantive-rate-medxpertqa_mm.d\endcsname{500}
\expandafter\def\csname regval:M-poc-decline-rate-medxpertqa_mm\endcsname{0.028}
\expandafter\def\csname regval:M-poc-decline-rate-medxpertqa_mm.d\endcsname{500}
\expandafter\def\csname regval:M-poc-empty-rate-medxpertqa_mm\endcsname{0.170}
\expandafter\def\csname regval:M-poc-empty-rate-medxpertqa_mm.d\endcsname{500}
\expandafter\def\csname regval:M-poc-blind-acc-paired-medxpertqa_mm\endcsname{0.413}
\expandafter\def\csname regval:M-poc-blind-acc-paired-medxpertqa_mm.d\endcsname{500}
\expandafter\def\csname regval:M-poc-image-acc-paired-medxpertqa_mm\endcsname{0.425}
\expandafter\def\csname regval:M-poc-image-acc-paired-medxpertqa_mm.d\endcsname{500}
\expandafter\def\csname regval:M-poc-gain-paired-medxpertqa_mm\endcsname{0.012}
\expandafter\def\csname regval:M-poc-gain-paired-medxpertqa_mm.d\endcsname{500}
\expandafter\def\csname regval:M-poc-gain-paired-empty-incorrect-medxpertqa_mm\endcsname{0.036}
\expandafter\def\csname regval:M-poc-gain-paired-empty-incorrect-medxpertqa_mm.d\endcsname{500}
\expandafter\def\csname regval:M-poc-bclean-n-mmmu_pro\endcsname{500}
\expandafter\def\csname regval:M-poc-bclean-removed-mmmu_pro\endcsname{0.298}
\expandafter\def\csname regval:M-poc-bclean-removed-mmmu_pro.n\endcsname{149}
\expandafter\def\csname regval:M-poc-bclean-removed-mmmu_pro.d\endcsname{500}
\expandafter\def\csname regval:M-poc-lomo-min-mmmu_pro\endcsname{0.136}
\expandafter\def\csname regval:M-poc-lomo-min-mmmu_pro.d\endcsname{500}
\expandafter\def\csname regval:M-poc-lomo-max-mmmu_pro\endcsname{0.298}
\expandafter\def\csname regval:M-poc-lomo-max-mmmu_pro.d\endcsname{500}
\expandafter\def\csname regval:M-poc-union5-mmmu_pro\endcsname{0.772}
\expandafter\def\csname regval:M-poc-union5-mmmu_pro.d\endcsname{500}
\expandafter\def\csname regval:M-poc-substantive-rate-mmmu_pro\endcsname{0.733}
\expandafter\def\csname regval:M-poc-substantive-rate-mmmu_pro.d\endcsname{500}
\expandafter\def\csname regval:M-poc-decline-rate-mmmu_pro\endcsname{0.084}
\expandafter\def\csname regval:M-poc-decline-rate-mmmu_pro.d\endcsname{500}
\expandafter\def\csname regval:M-poc-empty-rate-mmmu_pro\endcsname{0.183}
\expandafter\def\csname regval:M-poc-empty-rate-mmmu_pro.d\endcsname{500}
\expandafter\def\csname regval:M-poc-blind-acc-paired-mmmu_pro\endcsname{0.058}
\expandafter\def\csname regval:M-poc-blind-acc-paired-mmmu_pro.d\endcsname{500}
\expandafter\def\csname regval:M-poc-image-acc-paired-mmmu_pro\endcsname{0.103}
\expandafter\def\csname regval:M-poc-image-acc-paired-mmmu_pro.d\endcsname{500}
\expandafter\def\csname regval:M-poc-gain-paired-mmmu_pro\endcsname{0.045}
\expandafter\def\csname regval:M-poc-gain-paired-mmmu_pro.d\endcsname{500}
\expandafter\def\csname regval:M-poc-gain-paired-empty-incorrect-mmmu_pro\endcsname{0.050}
\expandafter\def\csname regval:M-poc-gain-paired-empty-incorrect-mmmu_pro.d\endcsname{500}
\expandafter\def\csname regval:M-poc-bclean-n-video_mmmu\endcsname{300}
\expandafter\def\csname regval:M-poc-bclean-removed-video_mmmu\endcsname{0.523}
\expandafter\def\csname regval:M-poc-bclean-removed-video_mmmu.n\endcsname{157}
\expandafter\def\csname regval:M-poc-bclean-removed-video_mmmu.d\endcsname{300}
\expandafter\def\csname regval:M-poc-lomo-min-video_mmmu\endcsname{0.410}
\expandafter\def\csname regval:M-poc-lomo-min-video_mmmu.d\endcsname{300}
\expandafter\def\csname regval:M-poc-lomo-max-video_mmmu\endcsname{0.520}
\expandafter\def\csname regval:M-poc-lomo-max-video_mmmu.d\endcsname{300}
\expandafter\def\csname regval:M-poc-union5-video_mmmu\endcsname{0.441}
\expandafter\def\csname regval:M-poc-union5-video_mmmu.d\endcsname{300}
\expandafter\def\csname regval:M-poc-substantive-rate-video_mmmu\endcsname{0.706}
\expandafter\def\csname regval:M-poc-substantive-rate-video_mmmu.d\endcsname{300}
\expandafter\def\csname regval:M-poc-decline-rate-video_mmmu\endcsname{0.114}
\expandafter\def\csname regval:M-poc-decline-rate-video_mmmu.d\endcsname{300}
\expandafter\def\csname regval:M-poc-empty-rate-video_mmmu\endcsname{0.180}
\expandafter\def\csname regval:M-poc-empty-rate-video_mmmu.d\endcsname{300}
\expandafter\def\csname regval:M-poc-blind-acc-paired-video_mmmu\endcsname{0.387}
\expandafter\def\csname regval:M-poc-blind-acc-paired-video_mmmu.d\endcsname{300}
\expandafter\def\csname regval:M-poc-image-acc-paired-video_mmmu\endcsname{0.493}
\expandafter\def\csname regval:M-poc-image-acc-paired-video_mmmu.d\endcsname{300}
\expandafter\def\csname regval:M-poc-gain-paired-video_mmmu\endcsname{0.106}
\expandafter\def\csname regval:M-poc-gain-paired-video_mmmu.d\endcsname{300}
\expandafter\def\csname regval:M-poc-gain-paired-empty-incorrect-video_mmmu\endcsname{0.170}
\expandafter\def\csname regval:M-poc-gain-paired-empty-incorrect-video_mmmu.d\endcsname{300}
\expandafter\def\csname regval:M-poc-full-claude-sonnet-4.6-chest_vqa\endcsname{0.616}
\expandafter\def\csname regval:M-poc-full-claude-sonnet-4.6-chest_vqa.d\endcsname{451}
\expandafter\def\csname regval:M-poc-full-claude-sonnet-4.6-docvqa\endcsname{0.136}
\expandafter\def\csname regval:M-poc-full-claude-sonnet-4.6-docvqa.d\endcsname{500}
\expandafter\def\csname regval:M-poc-full-claude-sonnet-4.6-infovqa\endcsname{0.420}
\expandafter\def\csname regval:M-poc-full-claude-sonnet-4.6-infovqa.d\endcsname{500}
\expandafter\def\csname regval:M-poc-full-claude-sonnet-4.6-medxpertqa_mm\endcsname{0.988}
\expandafter\def\csname regval:M-poc-full-claude-sonnet-4.6-medxpertqa_mm.d\endcsname{500}
\expandafter\def\csname regval:M-poc-full-claude-sonnet-4.6-microvqa\endcsname{0.958}
\expandafter\def\csname regval:M-poc-full-claude-sonnet-4.6-microvqa.d\endcsname{500}
\expandafter\def\csname regval:M-poc-full-claude-sonnet-4.6-mmmu_pro\endcsname{0.898}
\expandafter\def\csname regval:M-poc-full-claude-sonnet-4.6-mmmu_pro.d\endcsname{500}
\expandafter\def\csname regval:M-poc-full-claude-sonnet-4.6-pathvqa\endcsname{0.538}
\expandafter\def\csname regval:M-poc-full-claude-sonnet-4.6-pathvqa.d\endcsname{500}
\expandafter\def\csname regval:M-poc-full-claude-sonnet-4.6-rexvqa\endcsname{1.000}
\expandafter\def\csname regval:M-poc-full-claude-sonnet-4.6-rexvqa.d\endcsname{500}
\expandafter\def\csname regval:M-poc-full-claude-sonnet-4.6-textvqa\endcsname{0.110}
\expandafter\def\csname regval:M-poc-full-claude-sonnet-4.6-textvqa.d\endcsname{500}
\expandafter\def\csname regval:M-poc-full-claude-sonnet-4.6-video_mmmu\endcsname{0.867}
\expandafter\def\csname regval:M-poc-full-claude-sonnet-4.6-video_mmmu.d\endcsname{300}
\expandafter\def\csname regval:M-poc-full-gemini-3.1-flash-lite-chest_vqa\endcsname{0.191}
\expandafter\def\csname regval:M-poc-full-gemini-3.1-flash-lite-chest_vqa.d\endcsname{451}
\expandafter\def\csname regval:M-poc-full-gemini-3.1-flash-lite-docvqa\endcsname{0.182}
\expandafter\def\csname regval:M-poc-full-gemini-3.1-flash-lite-docvqa.d\endcsname{500}
\expandafter\def\csname regval:M-poc-full-gemini-3.1-flash-lite-infovqa\endcsname{0.640}
\expandafter\def\csname regval:M-poc-full-gemini-3.1-flash-lite-infovqa.d\endcsname{500}
\expandafter\def\csname regval:M-poc-full-gemini-3.1-flash-lite-medxpertqa_mm\endcsname{0.980}
\expandafter\def\csname regval:M-poc-full-gemini-3.1-flash-lite-medxpertqa_mm.d\endcsname{500}
\expandafter\def\csname regval:M-poc-full-gemini-3.1-flash-lite-microvqa\endcsname{1.000}
\expandafter\def\csname regval:M-poc-full-gemini-3.1-flash-lite-microvqa.d\endcsname{500}
\expandafter\def\csname regval:M-poc-full-gemini-3.1-flash-lite-mmmu_pro\endcsname{0.940}
\expandafter\def\csname regval:M-poc-full-gemini-3.1-flash-lite-mmmu_pro.d\endcsname{500}
\expandafter\def\csname regval:M-poc-full-gemini-3.1-flash-lite-pathvqa\endcsname{0.262}
\expandafter\def\csname regval:M-poc-full-gemini-3.1-flash-lite-pathvqa.d\endcsname{500}
\expandafter\def\csname regval:M-poc-full-gemini-3.1-flash-lite-rexvqa\endcsname{0.910}
\expandafter\def\csname regval:M-poc-full-gemini-3.1-flash-lite-rexvqa.d\endcsname{500}
\expandafter\def\csname regval:M-poc-full-gemini-3.1-flash-lite-textvqa\endcsname{0.068}
\expandafter\def\csname regval:M-poc-full-gemini-3.1-flash-lite-textvqa.d\endcsname{500}
\expandafter\def\csname regval:M-poc-full-gemini-3.1-flash-lite-video_mmmu\endcsname{0.930}
\expandafter\def\csname regval:M-poc-full-gemini-3.1-flash-lite-video_mmmu.d\endcsname{300}
\expandafter\def\csname regval:M-poc-full-llama-4-scout-chest_vqa\endcsname{0.488}
\expandafter\def\csname regval:M-poc-full-llama-4-scout-chest_vqa.d\endcsname{451}
\expandafter\def\csname regval:M-poc-full-llama-4-scout-docvqa\endcsname{0.654}
\expandafter\def\csname regval:M-poc-full-llama-4-scout-docvqa.d\endcsname{500}
\expandafter\def\csname regval:M-poc-full-llama-4-scout-infovqa\endcsname{0.798}
\expandafter\def\csname regval:M-poc-full-llama-4-scout-infovqa.d\endcsname{500}
\expandafter\def\csname regval:M-poc-full-llama-4-scout-medxpertqa_mm\endcsname{0.888}
\expandafter\def\csname regval:M-poc-full-llama-4-scout-medxpertqa_mm.d\endcsname{500}
\expandafter\def\csname regval:M-poc-full-llama-4-scout-microvqa\endcsname{0.994}
\expandafter\def\csname regval:M-poc-full-llama-4-scout-microvqa.d\endcsname{500}
\expandafter\def\csname regval:M-poc-full-llama-4-scout-mmmu_pro\endcsname{0.742}
\expandafter\def\csname regval:M-poc-full-llama-4-scout-mmmu_pro.d\endcsname{500}
\expandafter\def\csname regval:M-poc-full-llama-4-scout-pathvqa\endcsname{0.544}
\expandafter\def\csname regval:M-poc-full-llama-4-scout-pathvqa.d\endcsname{500}
\expandafter\def\csname regval:M-poc-full-llama-4-scout-rexvqa\endcsname{0.462}
\expandafter\def\csname regval:M-poc-full-llama-4-scout-rexvqa.d\endcsname{500}
\expandafter\def\csname regval:M-poc-full-llama-4-scout-textvqa\endcsname{0.594}
\expandafter\def\csname regval:M-poc-full-llama-4-scout-textvqa.d\endcsname{500}
\expandafter\def\csname regval:M-poc-full-llama-4-scout-video_mmmu\endcsname{0.647}
\expandafter\def\csname regval:M-poc-full-llama-4-scout-video_mmmu.d\endcsname{300}
\expandafter\def\csname regval:M-poc-full-gpt-5.5-chest_vqa\endcsname{0.162}
\expandafter\def\csname regval:M-poc-full-gpt-5.5-chest_vqa.d\endcsname{451}
\expandafter\def\csname regval:M-poc-full-gpt-5.5-docvqa\endcsname{0.068}
\expandafter\def\csname regval:M-poc-full-gpt-5.5-docvqa.d\endcsname{500}
\expandafter\def\csname regval:M-poc-full-gpt-5.5-infovqa\endcsname{0.204}
\expandafter\def\csname regval:M-poc-full-gpt-5.5-infovqa.d\endcsname{500}
\expandafter\def\csname regval:M-poc-full-gpt-5.5-medxpertqa_mm\endcsname{0.154}
\expandafter\def\csname regval:M-poc-full-gpt-5.5-medxpertqa_mm.d\endcsname{500}
\expandafter\def\csname regval:M-poc-full-gpt-5.5-microvqa\endcsname{0.362}
\expandafter\def\csname regval:M-poc-full-gpt-5.5-microvqa.d\endcsname{500}
\expandafter\def\csname regval:M-poc-full-gpt-5.5-mmmu_pro\endcsname{0.086}
\expandafter\def\csname regval:M-poc-full-gpt-5.5-mmmu_pro.d\endcsname{500}
\expandafter\def\csname regval:M-poc-full-gpt-5.5-pathvqa\endcsname{0.158}
\expandafter\def\csname regval:M-poc-full-gpt-5.5-pathvqa.d\endcsname{500}
\expandafter\def\csname regval:M-poc-full-gpt-5.5-rexvqa\endcsname{0.396}
\expandafter\def\csname regval:M-poc-full-gpt-5.5-rexvqa.d\endcsname{500}
\expandafter\def\csname regval:M-poc-full-gpt-5.5-textvqa\endcsname{0.094}
\expandafter\def\csname regval:M-poc-full-gpt-5.5-textvqa.d\endcsname{500}
\expandafter\def\csname regval:M-poc-full-gpt-5.5-video_mmmu\endcsname{0.103}
\expandafter\def\csname regval:M-poc-full-gpt-5.5-video_mmmu.d\endcsname{300}
\expandafter\def\csname regval:M-poc-full-qwen2.5-vl-72b-instruct-chest_vqa\endcsname{0.953}
\expandafter\def\csname regval:M-poc-full-qwen2.5-vl-72b-instruct-chest_vqa.d\endcsname{451}
\expandafter\def\csname regval:M-poc-full-qwen2.5-vl-72b-instruct-docvqa\endcsname{0.890}
\expandafter\def\csname regval:M-poc-full-qwen2.5-vl-72b-instruct-docvqa.d\endcsname{500}
\expandafter\def\csname regval:M-poc-full-qwen2.5-vl-72b-instruct-infovqa\endcsname{0.964}
\expandafter\def\csname regval:M-poc-full-qwen2.5-vl-72b-instruct-infovqa.d\endcsname{500}
\expandafter\def\csname regval:M-poc-full-qwen2.5-vl-72b-instruct-medxpertqa_mm\endcsname{1.000}
\expandafter\def\csname regval:M-poc-full-qwen2.5-vl-72b-instruct-medxpertqa_mm.d\endcsname{500}
\expandafter\def\csname regval:M-poc-full-qwen2.5-vl-72b-instruct-microvqa\endcsname{1.000}
\expandafter\def\csname regval:M-poc-full-qwen2.5-vl-72b-instruct-microvqa.d\endcsname{500}
\expandafter\def\csname regval:M-poc-full-qwen2.5-vl-72b-instruct-mmmu_pro\endcsname{0.998}
\expandafter\def\csname regval:M-poc-full-qwen2.5-vl-72b-instruct-mmmu_pro.d\endcsname{500}
\expandafter\def\csname regval:M-poc-full-qwen2.5-vl-72b-instruct-pathvqa\endcsname{0.950}
\expandafter\def\csname regval:M-poc-full-qwen2.5-vl-72b-instruct-pathvqa.d\endcsname{500}
\expandafter\def\csname regval:M-poc-full-qwen2.5-vl-72b-instruct-rexvqa\endcsname{1.000}
\expandafter\def\csname regval:M-poc-full-qwen2.5-vl-72b-instruct-rexvqa.d\endcsname{500}
\expandafter\def\csname regval:M-poc-full-qwen2.5-vl-72b-instruct-textvqa\endcsname{0.788}
\expandafter\def\csname regval:M-poc-full-qwen2.5-vl-72b-instruct-textvqa.d\endcsname{500}
\expandafter\def\csname regval:M-poc-full-qwen2.5-vl-72b-instruct-video_mmmu\endcsname{0.983}
\expandafter\def\csname regval:M-poc-full-qwen2.5-vl-72b-instruct-video_mmmu.d\endcsname{300}
\expandafter\def\csname regval:M-poc-sens-gpt-5.5-chest_vqa\endcsname{0.091}
\expandafter\def\csname regval:M-poc-sens-gpt-5.5-chest_vqa.d\endcsname{175}
\expandafter\def\csname regval:M-poc-sens-gpt-5.5-pathvqa\endcsname{0.125}
\expandafter\def\csname regval:M-poc-sens-gpt-5.5-pathvqa.d\endcsname{272}
\expandafter\def\csname regval:M-poc-sens-gpt-5.5-rexvqa\endcsname{0.084}
\expandafter\def\csname regval:M-poc-sens-gpt-5.5-rexvqa.d\endcsname{83}
\expandafter\def\csname regval:M-poc-sens-gpt-5.5-docvqa\endcsname{0.044}
\expandafter\def\csname regval:M-poc-sens-gpt-5.5-docvqa.d\endcsname{459}
\expandafter\def\csname regval:M-poc-sens-gpt-5.5-textvqa\endcsname{0.058}
\expandafter\def\csname regval:M-poc-sens-gpt-5.5-textvqa.d\endcsname{417}
\expandafter\def\csname regval:M-poc-sens-gpt-5.5-infovqa\endcsname{0.069}
\expandafter\def\csname regval:M-poc-sens-gpt-5.5-infovqa.d\endcsname{334}
\expandafter\def\csname regval:M-poc-sens-gpt-5.5-microvqa\endcsname{0.168}
\expandafter\def\csname regval:M-poc-sens-gpt-5.5-microvqa.d\endcsname{131}
\expandafter\def\csname regval:M-poc-sens-gpt-5.5-medxpertqa_mm\endcsname{0.056}
\expandafter\def\csname regval:M-poc-sens-gpt-5.5-medxpertqa_mm.d\endcsname{196}
\expandafter\def\csname regval:M-poc-sens-gpt-5.5-mmmu_pro\endcsname{0.046}
\expandafter\def\csname regval:M-poc-sens-gpt-5.5-mmmu_pro.d\endcsname{351}
\expandafter\def\csname regval:M-poc-sens-gpt-5.5-video_mmmu\endcsname{0.007}
\expandafter\def\csname regval:M-poc-sens-gpt-5.5-video_mmmu.d\endcsname{143}
\expandafter\def\csname regval:M-poc-sens-gemini-3.1-flash-lite-chest_vqa\endcsname{0.103}
\expandafter\def\csname regval:M-poc-sens-gemini-3.1-flash-lite-chest_vqa.d\endcsname{175}
\expandafter\def\csname regval:M-poc-sens-gemini-3.1-flash-lite-pathvqa\endcsname{0.129}
\expandafter\def\csname regval:M-poc-sens-gemini-3.1-flash-lite-pathvqa.d\endcsname{272}
\expandafter\def\csname regval:M-poc-sens-gemini-3.1-flash-lite-rexvqa\endcsname{0.771}
\expandafter\def\csname regval:M-poc-sens-gemini-3.1-flash-lite-rexvqa.d\endcsname{83}
\expandafter\def\csname regval:M-poc-sens-gemini-3.1-flash-lite-docvqa\endcsname{0.139}
\expandafter\def\csname regval:M-poc-sens-gemini-3.1-flash-lite-docvqa.d\endcsname{459}
\expandafter\def\csname regval:M-poc-sens-gemini-3.1-flash-lite-textvqa\endcsname{0.043}
\expandafter\def\csname regval:M-poc-sens-gemini-3.1-flash-lite-textvqa.d\endcsname{417}
\expandafter\def\csname regval:M-poc-sens-gemini-3.1-flash-lite-infovqa\endcsname{0.512}
\expandafter\def\csname regval:M-poc-sens-gemini-3.1-flash-lite-infovqa.d\endcsname{334}
\expandafter\def\csname regval:M-poc-sens-gemini-3.1-flash-lite-microvqa\endcsname{1.000}
\expandafter\def\csname regval:M-poc-sens-gemini-3.1-flash-lite-microvqa.d\endcsname{131}
\expandafter\def\csname regval:M-poc-sens-gemini-3.1-flash-lite-medxpertqa_mm\endcsname{0.985}
\expandafter\def\csname regval:M-poc-sens-gemini-3.1-flash-lite-medxpertqa_mm.d\endcsname{196}
\expandafter\def\csname regval:M-poc-sens-gemini-3.1-flash-lite-mmmu_pro\endcsname{0.943}
\expandafter\def\csname regval:M-poc-sens-gemini-3.1-flash-lite-mmmu_pro.d\endcsname{351}
\expandafter\def\csname regval:M-poc-sens-gemini-3.1-flash-lite-video_mmmu\endcsname{0.888}
\expandafter\def\csname regval:M-poc-sens-gemini-3.1-flash-lite-video_mmmu.d\endcsname{143}
\expandafter\def\csname regval:M-poc-sens-claude-sonnet-4.6-chest_vqa\endcsname{0.389}
\expandafter\def\csname regval:M-poc-sens-claude-sonnet-4.6-chest_vqa.d\endcsname{175}
\expandafter\def\csname regval:M-poc-sens-claude-sonnet-4.6-pathvqa\endcsname{0.346}
\expandafter\def\csname regval:M-poc-sens-claude-sonnet-4.6-pathvqa.d\endcsname{272}
\expandafter\def\csname regval:M-poc-sens-claude-sonnet-4.6-rexvqa\endcsname{1.000}
\expandafter\def\csname regval:M-poc-sens-claude-sonnet-4.6-rexvqa.d\endcsname{83}
\expandafter\def\csname regval:M-poc-sens-claude-sonnet-4.6-docvqa\endcsname{0.105}
\expandafter\def\csname regval:M-poc-sens-claude-sonnet-4.6-docvqa.d\endcsname{459}
\expandafter\def\csname regval:M-poc-sens-claude-sonnet-4.6-textvqa\endcsname{0.079}
\expandafter\def\csname regval:M-poc-sens-claude-sonnet-4.6-textvqa.d\endcsname{417}
\expandafter\def\csname regval:M-poc-sens-claude-sonnet-4.6-infovqa\endcsname{0.284}
\expandafter\def\csname regval:M-poc-sens-claude-sonnet-4.6-infovqa.d\endcsname{334}
\expandafter\def\csname regval:M-poc-sens-claude-sonnet-4.6-microvqa\endcsname{0.947}
\expandafter\def\csname regval:M-poc-sens-claude-sonnet-4.6-microvqa.d\endcsname{131}
\expandafter\def\csname regval:M-poc-sens-claude-sonnet-4.6-medxpertqa_mm\endcsname{0.980}
\expandafter\def\csname regval:M-poc-sens-claude-sonnet-4.6-medxpertqa_mm.d\endcsname{196}
\expandafter\def\csname regval:M-poc-sens-claude-sonnet-4.6-mmmu_pro\endcsname{0.892}
\expandafter\def\csname regval:M-poc-sens-claude-sonnet-4.6-mmmu_pro.d\endcsname{351}
\expandafter\def\csname regval:M-poc-sens-claude-sonnet-4.6-video_mmmu\endcsname{0.818}
\expandafter\def\csname regval:M-poc-sens-claude-sonnet-4.6-video_mmmu.d\endcsname{143}
\expandafter\def\csname regval:M-poc-sens-llama-4-scout-chest_vqa\endcsname{0.377}
\expandafter\def\csname regval:M-poc-sens-llama-4-scout-chest_vqa.d\endcsname{175}
\expandafter\def\csname regval:M-poc-sens-llama-4-scout-pathvqa\endcsname{0.301}
\expandafter\def\csname regval:M-poc-sens-llama-4-scout-pathvqa.d\endcsname{272}
\expandafter\def\csname regval:M-poc-sens-llama-4-scout-rexvqa\endcsname{0.361}
\expandafter\def\csname regval:M-poc-sens-llama-4-scout-rexvqa.d\endcsname{83}
\expandafter\def\csname regval:M-poc-sens-llama-4-scout-docvqa\endcsname{0.634}
\expandafter\def\csname regval:M-poc-sens-llama-4-scout-docvqa.d\endcsname{459}
\expandafter\def\csname regval:M-poc-sens-llama-4-scout-textvqa\endcsname{0.571}
\expandafter\def\csname regval:M-poc-sens-llama-4-scout-textvqa.d\endcsname{417}
\expandafter\def\csname regval:M-poc-sens-llama-4-scout-infovqa\endcsname{0.737}
\expandafter\def\csname regval:M-poc-sens-llama-4-scout-infovqa.d\endcsname{334}
\expandafter\def\csname regval:M-poc-sens-llama-4-scout-microvqa\endcsname{1.000}
\expandafter\def\csname regval:M-poc-sens-llama-4-scout-microvqa.d\endcsname{131}
\expandafter\def\csname regval:M-poc-sens-llama-4-scout-medxpertqa_mm\endcsname{0.832}
\expandafter\def\csname regval:M-poc-sens-llama-4-scout-medxpertqa_mm.d\endcsname{196}
\expandafter\def\csname regval:M-poc-sens-llama-4-scout-mmmu_pro\endcsname{0.744}
\expandafter\def\csname regval:M-poc-sens-llama-4-scout-mmmu_pro.d\endcsname{351}
\expandafter\def\csname regval:M-poc-sens-llama-4-scout-video_mmmu\endcsname{0.629}
\expandafter\def\csname regval:M-poc-sens-llama-4-scout-video_mmmu.d\endcsname{143}
\expandafter\def\csname regval:M-poc-sens-qwen2.5-vl-72b-instruct-chest_vqa\endcsname{0.891}
\expandafter\def\csname regval:M-poc-sens-qwen2.5-vl-72b-instruct-chest_vqa.d\endcsname{175}
\expandafter\def\csname regval:M-poc-sens-qwen2.5-vl-72b-instruct-pathvqa\endcsname{0.908}
\expandafter\def\csname regval:M-poc-sens-qwen2.5-vl-72b-instruct-pathvqa.d\endcsname{272}
\expandafter\def\csname regval:M-poc-sens-qwen2.5-vl-72b-instruct-rexvqa\endcsname{1.000}
\expandafter\def\csname regval:M-poc-sens-qwen2.5-vl-72b-instruct-rexvqa.d\endcsname{83}
\expandafter\def\csname regval:M-poc-sens-qwen2.5-vl-72b-instruct-docvqa\endcsname{0.882}
\expandafter\def\csname regval:M-poc-sens-qwen2.5-vl-72b-instruct-docvqa.d\endcsname{459}
\expandafter\def\csname regval:M-poc-sens-qwen2.5-vl-72b-instruct-textvqa\endcsname{0.770}
\expandafter\def\csname regval:M-poc-sens-qwen2.5-vl-72b-instruct-textvqa.d\endcsname{417}
\expandafter\def\csname regval:M-poc-sens-qwen2.5-vl-72b-instruct-infovqa\endcsname{0.955}
\expandafter\def\csname regval:M-poc-sens-qwen2.5-vl-72b-instruct-infovqa.d\endcsname{334}
\expandafter\def\csname regval:M-poc-sens-qwen2.5-vl-72b-instruct-microvqa\endcsname{1.000}
\expandafter\def\csname regval:M-poc-sens-qwen2.5-vl-72b-instruct-microvqa.d\endcsname{131}
\expandafter\def\csname regval:M-poc-sens-qwen2.5-vl-72b-instruct-medxpertqa_mm\endcsname{1.000}
\expandafter\def\csname regval:M-poc-sens-qwen2.5-vl-72b-instruct-medxpertqa_mm.d\endcsname{196}
\expandafter\def\csname regval:M-poc-sens-qwen2.5-vl-72b-instruct-mmmu_pro\endcsname{0.997}
\expandafter\def\csname regval:M-poc-sens-qwen2.5-vl-72b-instruct-mmmu_pro.d\endcsname{351}
\expandafter\def\csname regval:M-poc-sens-qwen2.5-vl-72b-instruct-video_mmmu\endcsname{0.965}
\expandafter\def\csname regval:M-poc-sens-qwen2.5-vl-72b-instruct-video_mmmu.d\endcsname{143}
\expandafter\def\csname regval:M-poc-corr-full50\endcsname{-0.411}
\expandafter\def\csname regval:M-poc-corr-full50.d\endcsname{50}
\expandafter\def\csname regval:M-poc-corr-full40\endcsname{-0.691}
\expandafter\def\csname regval:M-poc-corr-full40.d\endcsname{40}
\expandafter\def\csname regval:M-poc-corr-full50-decline\endcsname{0.655}
\expandafter\def\csname regval:M-poc-corr-full50-decline.d\endcsname{50}
\expandafter\def\csname regval:M-poc-corr-full50-nonempty\endcsname{-0.667}
\expandafter\def\csname regval:M-poc-corr-full50-nonempty.d\endcsname{50}
\expandafter\def\csname regval:M-poc-corr-full50-emptyincorrect\endcsname{-0.523}
\expandafter\def\csname regval:M-poc-corr-full50-emptyincorrect.d\endcsname{50}
\expandafter\def\csname regval:M-poc-corr-blindacc50\endcsname{-0.023}
\expandafter\def\csname regval:M-poc-corr-blindacc50.d\endcsname{50}
\expandafter\def\csname regval:M-poc-boot-domains-lo\endcsname{-0.564}
\expandafter\def\csname regval:M-poc-boot-domains-lo.d\endcsname{2{,}000}
\expandafter\def\csname regval:M-poc-boot-domains-hi\endcsname{-0.121}
\expandafter\def\csname regval:M-poc-boot-domains-hi.d\endcsname{2{,}000}
\expandafter\def\csname regval:M-poc-boot-models-lo\endcsname{-0.867}
\expandafter\def\csname regval:M-poc-boot-models-lo.d\endcsname{2{,}000}
\expandafter\def\csname regval:M-poc-boot-models-hi\endcsname{-0.015}
\expandafter\def\csname regval:M-poc-boot-models-hi.d\endcsname{2{,}000}
\expandafter\def\csname regval:M-poc-corr-domainmeans10\endcsname{-0.867}
\expandafter\def\csname regval:M-poc-corr-domainmeans10.d\endcsname{10}
\expandafter\def\csname regval:M-poc-corr-modelmeans5\endcsname{0.571}
\expandafter\def\csname regval:M-poc-corr-modelmeans5.d\endcsname{5}
\expandafter\def\csname regval:M-poc-spearman-full-gpt-5.5\endcsname{-0.190}
\expandafter\def\csname regval:M-poc-spearman-full-gpt-5.5.d\endcsname{10}
\expandafter\def\csname regval:M-poc-spearman-full-gemini-3.1-flash-lite\endcsname{-0.903}
\expandafter\def\csname regval:M-poc-spearman-full-gemini-3.1-flash-lite.d\endcsname{10}
\expandafter\def\csname regval:M-poc-spearman-full-claude-sonnet-4.6\endcsname{-0.952}
\expandafter\def\csname regval:M-poc-spearman-full-claude-sonnet-4.6.d\endcsname{10}
\expandafter\def\csname regval:M-poc-spearman-full-llama-4-scout\endcsname{-0.455}
\expandafter\def\csname regval:M-poc-spearman-full-llama-4-scout.d\endcsname{10}
\expandafter\def\csname regval:M-poc-spearman-full-qwen2.5-vl-72b-instruct\endcsname{-0.841}
\expandafter\def\csname regval:M-poc-spearman-full-qwen2.5-vl-72b-instruct.d\endcsname{10}
\expandafter\def\csname regval:M-poc-corr-fmt-mcq\endcsname{-0.069}
\expandafter\def\csname regval:M-poc-corr-fmt-mcq.d\endcsname{25}
\expandafter\def\csname regval:M-poc-corr-fmt-yesno\endcsname{-0.478}
\expandafter\def\csname regval:M-poc-corr-fmt-yesno.d\endcsname{10}
\expandafter\def\csname regval:M-poc-corr-fmt-open\endcsname{-0.054}
\expandafter\def\csname regval:M-poc-corr-fmt-open.d\endcsname{25}
\expandafter\def\csname regval:M-poc-corr-sens50\endcsname{-0.340}
\expandafter\def\csname regval:M-poc-corr-sens50.d\endcsname{50}
\expandafter\def\csname regval:M-poc-corr-sens40\endcsname{-0.494}
\expandafter\def\csname regval:M-poc-corr-sens40.d\endcsname{40}
\expandafter\def\csname regval:M-poc-corr-sens50-nonempty\endcsname{-0.406}
\expandafter\def\csname regval:M-poc-corr-sens50-nonempty.d\endcsname{50}
\expandafter\def\csname regval:M-poc-corr-subset20\endcsname{-0.767}
\expandafter\def\csname regval:M-poc-corr-subset20.d\endcsname{20}
\expandafter\def\csname regval:v5.qwen.E0.bench.imr\endcsname{\unmeasured{v5.qwen.E0.bench.imr}}
\expandafter\def\csname regval:v5.qwen.E0.fresh.imr\endcsname{\unmeasured{v5.qwen.E0.fresh.imr}}
\expandafter\def\csname regval:v5.qwen.E0.cmp.imr\endcsname{\unmeasured{v5.qwen.E0.cmp.imr}}
\expandafter\def\csname regval:v5.qwen.E0.bench.blank\endcsname{\unmeasured{v5.qwen.E0.bench.blank}}
\expandafter\def\csname regval:v5.qwen.E0.fresh.blank\endcsname{\unmeasured{v5.qwen.E0.fresh.blank}}
\expandafter\def\csname regval:v5.qwen.E0.cmp.blank\endcsname{\unmeasured{v5.qwen.E0.cmp.blank}}
\expandafter\def\csname regval:v5.qwen.E0.bench.id\endcsname{\unmeasured{v5.qwen.E0.bench.id}}
\expandafter\def\csname regval:v5.qwen.E0.fresh.id\endcsname{\unmeasured{v5.qwen.E0.fresh.id}}
\expandafter\def\csname regval:v5.qwen.E0.cmp.id\endcsname{\unmeasured{v5.qwen.E0.cmp.id}}
\expandafter\def\csname regval:v5.qwen.E0.bench.od\endcsname{\unmeasured{v5.qwen.E0.bench.od}}
\expandafter\def\csname regval:v5.qwen.E0.fresh.od\endcsname{\unmeasured{v5.qwen.E0.fresh.od}}
\expandafter\def\csname regval:v5.qwen.E0.cmp.od\endcsname{\unmeasured{v5.qwen.E0.cmp.od}}
\expandafter\def\csname regval:v5.qwen.E0.bench.iaimr\endcsname{\unmeasured{v5.qwen.E0.bench.iaimr}}
\expandafter\def\csname regval:v5.qwen.E0.fresh.iaimr\endcsname{\unmeasured{v5.qwen.E0.fresh.iaimr}}
\expandafter\def\csname regval:v5.qwen.E0.cmp.iaimr\endcsname{\unmeasured{v5.qwen.E0.cmp.iaimr}}
\expandafter\def\csname regval:v5.qwen.E0.bench.acc\endcsname{\unmeasured{v5.qwen.E0.bench.acc}}
\expandafter\def\csname regval:v5.qwen.E0.fresh.acc\endcsname{\unmeasured{v5.qwen.E0.fresh.acc}}
\expandafter\def\csname regval:v5.qwen.E0.cmp.acc\endcsname{\unmeasured{v5.qwen.E0.cmp.acc}}
\expandafter\def\csname regval:v5.qwen.E0.bench.fa\endcsname{\unmeasured{v5.qwen.E0.bench.fa}}
\expandafter\def\csname regval:v5.qwen.E0.fresh.fa\endcsname{\unmeasured{v5.qwen.E0.fresh.fa}}
\expandafter\def\csname regval:v5.qwen.E0.cmp.fa\endcsname{\unmeasured{v5.qwen.E0.cmp.fa}}
\expandafter\def\csname regval:v5.llava.E0.imr\endcsname{\unmeasured{v5.llava.E0.imr}}
\expandafter\def\csname regval:v5.llava.E0.acc\endcsname{\unmeasured{v5.llava.E0.acc}}
\expandafter\def\csname regval:v5.internvl.E0.imr\endcsname{\unmeasured{v5.internvl.E0.imr}}
\expandafter\def\csname regval:v5.internvl.E0.acc\endcsname{\unmeasured{v5.internvl.E0.acc}}
\expandafter\def\csname regval:v5.qwen3.E0.imr\endcsname{\unmeasured{v5.qwen3.E0.imr}}
\expandafter\def\csname regval:v5.qwen3.E0.acc\endcsname{\unmeasured{v5.qwen3.E0.acc}}
\expandafter\def\csname regval:v5.gemma.E0.imr\endcsname{\unmeasured{v5.gemma.E0.imr}}
\expandafter\def\csname regval:v5.gemma.E0.acc\endcsname{\unmeasured{v5.gemma.E0.acc}}
\expandafter\def\csname regval:stress5.E0.P0\endcsname{\unmeasured{stress5.E0.P0}}
\expandafter\def\csname regval:stress5.E0.P1\endcsname{\unmeasured{stress5.E0.P1}}
\expandafter\def\csname regval:stress5.E0.P2\endcsname{\unmeasured{stress5.E0.P2}}
\expandafter\def\csname regval:v5.qwen.TIDE.bench.imr\endcsname{\unmeasured{v5.qwen.TIDE.bench.imr}}
\expandafter\def\csname regval:v5.qwen.TIDE.fresh.imr\endcsname{\unmeasured{v5.qwen.TIDE.fresh.imr}}
\expandafter\def\csname regval:v5.qwen.TIDE.cmp.imr\endcsname{\unmeasured{v5.qwen.TIDE.cmp.imr}}
\expandafter\def\csname regval:v5.qwen.TIDE.bench.blank\endcsname{\unmeasured{v5.qwen.TIDE.bench.blank}}
\expandafter\def\csname regval:v5.qwen.TIDE.fresh.blank\endcsname{\unmeasured{v5.qwen.TIDE.fresh.blank}}
\expandafter\def\csname regval:v5.qwen.TIDE.cmp.blank\endcsname{\unmeasured{v5.qwen.TIDE.cmp.blank}}
\expandafter\def\csname regval:v5.qwen.TIDE.bench.id\endcsname{\unmeasured{v5.qwen.TIDE.bench.id}}
\expandafter\def\csname regval:v5.qwen.TIDE.fresh.id\endcsname{\unmeasured{v5.qwen.TIDE.fresh.id}}
\expandafter\def\csname regval:v5.qwen.TIDE.cmp.id\endcsname{\unmeasured{v5.qwen.TIDE.cmp.id}}
\expandafter\def\csname regval:v5.qwen.TIDE.bench.od\endcsname{\unmeasured{v5.qwen.TIDE.bench.od}}
\expandafter\def\csname regval:v5.qwen.TIDE.fresh.od\endcsname{\unmeasured{v5.qwen.TIDE.fresh.od}}
\expandafter\def\csname regval:v5.qwen.TIDE.cmp.od\endcsname{\unmeasured{v5.qwen.TIDE.cmp.od}}
\expandafter\def\csname regval:v5.qwen.TIDE.bench.iaimr\endcsname{\unmeasured{v5.qwen.TIDE.bench.iaimr}}
\expandafter\def\csname regval:v5.qwen.TIDE.fresh.iaimr\endcsname{\unmeasured{v5.qwen.TIDE.fresh.iaimr}}
\expandafter\def\csname regval:v5.qwen.TIDE.cmp.iaimr\endcsname{\unmeasured{v5.qwen.TIDE.cmp.iaimr}}
\expandafter\def\csname regval:v5.qwen.TIDE.bench.acc\endcsname{\unmeasured{v5.qwen.TIDE.bench.acc}}
\expandafter\def\csname regval:v5.qwen.TIDE.fresh.acc\endcsname{\unmeasured{v5.qwen.TIDE.fresh.acc}}
\expandafter\def\csname regval:v5.qwen.TIDE.cmp.acc\endcsname{\unmeasured{v5.qwen.TIDE.cmp.acc}}
\expandafter\def\csname regval:v5.qwen.TIDE.bench.fa\endcsname{\unmeasured{v5.qwen.TIDE.bench.fa}}
\expandafter\def\csname regval:v5.qwen.TIDE.fresh.fa\endcsname{\unmeasured{v5.qwen.TIDE.fresh.fa}}
\expandafter\def\csname regval:v5.qwen.TIDE.cmp.fa\endcsname{\unmeasured{v5.qwen.TIDE.cmp.fa}}
\expandafter\def\csname regval:v5.llava.TIDE.imr\endcsname{\unmeasured{v5.llava.TIDE.imr}}
\expandafter\def\csname regval:v5.llava.TIDE.acc\endcsname{\unmeasured{v5.llava.TIDE.acc}}
\expandafter\def\csname regval:v5.internvl.TIDE.imr\endcsname{\unmeasured{v5.internvl.TIDE.imr}}
\expandafter\def\csname regval:v5.internvl.TIDE.acc\endcsname{\unmeasured{v5.internvl.TIDE.acc}}
\expandafter\def\csname regval:v5.qwen3.TIDE.imr\endcsname{\unmeasured{v5.qwen3.TIDE.imr}}
\expandafter\def\csname regval:v5.qwen3.TIDE.acc\endcsname{\unmeasured{v5.qwen3.TIDE.acc}}
\expandafter\def\csname regval:v5.gemma.TIDE.imr\endcsname{\unmeasured{v5.gemma.TIDE.imr}}
\expandafter\def\csname regval:v5.gemma.TIDE.acc\endcsname{\unmeasured{v5.gemma.TIDE.acc}}
\expandafter\def\csname regval:stress5.TIDE.P0\endcsname{\unmeasured{stress5.TIDE.P0}}
\expandafter\def\csname regval:stress5.TIDE.P1\endcsname{\unmeasured{stress5.TIDE.P1}}
\expandafter\def\csname regval:stress5.TIDE.P2\endcsname{\unmeasured{stress5.TIDE.P2}}
\expandafter\def\csname regval:v5.qwen.E1.bench.imr\endcsname{\unmeasured{v5.qwen.E1.bench.imr}}
\expandafter\def\csname regval:v5.qwen.E1.fresh.imr\endcsname{\unmeasured{v5.qwen.E1.fresh.imr}}
\expandafter\def\csname regval:v5.qwen.E1.cmp.imr\endcsname{\unmeasured{v5.qwen.E1.cmp.imr}}
\expandafter\def\csname regval:v5.qwen.E1.bench.blank\endcsname{\unmeasured{v5.qwen.E1.bench.blank}}
\expandafter\def\csname regval:v5.qwen.E1.fresh.blank\endcsname{\unmeasured{v5.qwen.E1.fresh.blank}}
\expandafter\def\csname regval:v5.qwen.E1.cmp.blank\endcsname{\unmeasured{v5.qwen.E1.cmp.blank}}
\expandafter\def\csname regval:v5.qwen.E1.bench.id\endcsname{\unmeasured{v5.qwen.E1.bench.id}}
\expandafter\def\csname regval:v5.qwen.E1.fresh.id\endcsname{\unmeasured{v5.qwen.E1.fresh.id}}
\expandafter\def\csname regval:v5.qwen.E1.cmp.id\endcsname{\unmeasured{v5.qwen.E1.cmp.id}}
\expandafter\def\csname regval:v5.qwen.E1.bench.od\endcsname{\unmeasured{v5.qwen.E1.bench.od}}
\expandafter\def\csname regval:v5.qwen.E1.fresh.od\endcsname{\unmeasured{v5.qwen.E1.fresh.od}}
\expandafter\def\csname regval:v5.qwen.E1.cmp.od\endcsname{\unmeasured{v5.qwen.E1.cmp.od}}
\expandafter\def\csname regval:v5.qwen.E1.bench.iaimr\endcsname{\unmeasured{v5.qwen.E1.bench.iaimr}}
\expandafter\def\csname regval:v5.qwen.E1.fresh.iaimr\endcsname{\unmeasured{v5.qwen.E1.fresh.iaimr}}
\expandafter\def\csname regval:v5.qwen.E1.cmp.iaimr\endcsname{\unmeasured{v5.qwen.E1.cmp.iaimr}}
\expandafter\def\csname regval:v5.qwen.E1.bench.acc\endcsname{\unmeasured{v5.qwen.E1.bench.acc}}
\expandafter\def\csname regval:v5.qwen.E1.fresh.acc\endcsname{\unmeasured{v5.qwen.E1.fresh.acc}}
\expandafter\def\csname regval:v5.qwen.E1.cmp.acc\endcsname{\unmeasured{v5.qwen.E1.cmp.acc}}
\expandafter\def\csname regval:v5.qwen.E1.bench.fa\endcsname{\unmeasured{v5.qwen.E1.bench.fa}}
\expandafter\def\csname regval:v5.qwen.E1.fresh.fa\endcsname{\unmeasured{v5.qwen.E1.fresh.fa}}
\expandafter\def\csname regval:v5.qwen.E1.cmp.fa\endcsname{\unmeasured{v5.qwen.E1.cmp.fa}}
\expandafter\def\csname regval:v5.llava.E1.imr\endcsname{\unmeasured{v5.llava.E1.imr}}
\expandafter\def\csname regval:v5.llava.E1.acc\endcsname{\unmeasured{v5.llava.E1.acc}}
\expandafter\def\csname regval:v5.internvl.E1.imr\endcsname{\unmeasured{v5.internvl.E1.imr}}
\expandafter\def\csname regval:v5.internvl.E1.acc\endcsname{\unmeasured{v5.internvl.E1.acc}}
\expandafter\def\csname regval:v5.qwen3.E1.imr\endcsname{\unmeasured{v5.qwen3.E1.imr}}
\expandafter\def\csname regval:v5.qwen3.E1.acc\endcsname{\unmeasured{v5.qwen3.E1.acc}}
\expandafter\def\csname regval:v5.gemma.E1.imr\endcsname{\unmeasured{v5.gemma.E1.imr}}
\expandafter\def\csname regval:v5.gemma.E1.acc\endcsname{\unmeasured{v5.gemma.E1.acc}}
\expandafter\def\csname regval:stress5.E1.P0\endcsname{\unmeasured{stress5.E1.P0}}
\expandafter\def\csname regval:stress5.E1.P1\endcsname{\unmeasured{stress5.E1.P1}}
\expandafter\def\csname regval:stress5.E1.P2\endcsname{\unmeasured{stress5.E1.P2}}
\expandafter\def\csname regval:v5.qwen.E2.bench.imr\endcsname{\unmeasured{v5.qwen.E2.bench.imr}}
\expandafter\def\csname regval:v5.qwen.E2.fresh.imr\endcsname{\unmeasured{v5.qwen.E2.fresh.imr}}
\expandafter\def\csname regval:v5.qwen.E2.cmp.imr\endcsname{\unmeasured{v5.qwen.E2.cmp.imr}}
\expandafter\def\csname regval:v5.qwen.E2.bench.blank\endcsname{\unmeasured{v5.qwen.E2.bench.blank}}
\expandafter\def\csname regval:v5.qwen.E2.fresh.blank\endcsname{\unmeasured{v5.qwen.E2.fresh.blank}}
\expandafter\def\csname regval:v5.qwen.E2.cmp.blank\endcsname{\unmeasured{v5.qwen.E2.cmp.blank}}
\expandafter\def\csname regval:v5.qwen.E2.bench.id\endcsname{\unmeasured{v5.qwen.E2.bench.id}}
\expandafter\def\csname regval:v5.qwen.E2.fresh.id\endcsname{\unmeasured{v5.qwen.E2.fresh.id}}
\expandafter\def\csname regval:v5.qwen.E2.cmp.id\endcsname{\unmeasured{v5.qwen.E2.cmp.id}}
\expandafter\def\csname regval:v5.qwen.E2.bench.od\endcsname{\unmeasured{v5.qwen.E2.bench.od}}
\expandafter\def\csname regval:v5.qwen.E2.fresh.od\endcsname{\unmeasured{v5.qwen.E2.fresh.od}}
\expandafter\def\csname regval:v5.qwen.E2.cmp.od\endcsname{\unmeasured{v5.qwen.E2.cmp.od}}
\expandafter\def\csname regval:v5.qwen.E2.bench.iaimr\endcsname{\unmeasured{v5.qwen.E2.bench.iaimr}}
\expandafter\def\csname regval:v5.qwen.E2.fresh.iaimr\endcsname{\unmeasured{v5.qwen.E2.fresh.iaimr}}
\expandafter\def\csname regval:v5.qwen.E2.cmp.iaimr\endcsname{\unmeasured{v5.qwen.E2.cmp.iaimr}}
\expandafter\def\csname regval:v5.qwen.E2.bench.acc\endcsname{\unmeasured{v5.qwen.E2.bench.acc}}
\expandafter\def\csname regval:v5.qwen.E2.fresh.acc\endcsname{\unmeasured{v5.qwen.E2.fresh.acc}}
\expandafter\def\csname regval:v5.qwen.E2.cmp.acc\endcsname{\unmeasured{v5.qwen.E2.cmp.acc}}
\expandafter\def\csname regval:v5.qwen.E2.bench.fa\endcsname{\unmeasured{v5.qwen.E2.bench.fa}}
\expandafter\def\csname regval:v5.qwen.E2.fresh.fa\endcsname{\unmeasured{v5.qwen.E2.fresh.fa}}
\expandafter\def\csname regval:v5.qwen.E2.cmp.fa\endcsname{\unmeasured{v5.qwen.E2.cmp.fa}}
\expandafter\def\csname regval:v5.llava.E2.imr\endcsname{\unmeasured{v5.llava.E2.imr}}
\expandafter\def\csname regval:v5.llava.E2.acc\endcsname{\unmeasured{v5.llava.E2.acc}}
\expandafter\def\csname regval:v5.internvl.E2.imr\endcsname{\unmeasured{v5.internvl.E2.imr}}
\expandafter\def\csname regval:v5.internvl.E2.acc\endcsname{\unmeasured{v5.internvl.E2.acc}}
\expandafter\def\csname regval:v5.qwen3.E2.imr\endcsname{\unmeasured{v5.qwen3.E2.imr}}
\expandafter\def\csname regval:v5.qwen3.E2.acc\endcsname{\unmeasured{v5.qwen3.E2.acc}}
\expandafter\def\csname regval:v5.gemma.E2.imr\endcsname{\unmeasured{v5.gemma.E2.imr}}
\expandafter\def\csname regval:v5.gemma.E2.acc\endcsname{\unmeasured{v5.gemma.E2.acc}}
\expandafter\def\csname regval:stress5.E2.P0\endcsname{\unmeasured{stress5.E2.P0}}
\expandafter\def\csname regval:stress5.E2.P1\endcsname{\unmeasured{stress5.E2.P1}}
\expandafter\def\csname regval:stress5.E2.P2\endcsname{\unmeasured{stress5.E2.P2}}
\expandafter\def\csname regval:v5.qwen.E3.bench.imr\endcsname{\unmeasured{v5.qwen.E3.bench.imr}}
\expandafter\def\csname regval:v5.qwen.E3.fresh.imr\endcsname{\unmeasured{v5.qwen.E3.fresh.imr}}
\expandafter\def\csname regval:v5.qwen.E3.cmp.imr\endcsname{\unmeasured{v5.qwen.E3.cmp.imr}}
\expandafter\def\csname regval:v5.qwen.E3.bench.blank\endcsname{\unmeasured{v5.qwen.E3.bench.blank}}
\expandafter\def\csname regval:v5.qwen.E3.fresh.blank\endcsname{\unmeasured{v5.qwen.E3.fresh.blank}}
\expandafter\def\csname regval:v5.qwen.E3.cmp.blank\endcsname{\unmeasured{v5.qwen.E3.cmp.blank}}
\expandafter\def\csname regval:v5.qwen.E3.bench.id\endcsname{\unmeasured{v5.qwen.E3.bench.id}}
\expandafter\def\csname regval:v5.qwen.E3.fresh.id\endcsname{\unmeasured{v5.qwen.E3.fresh.id}}
\expandafter\def\csname regval:v5.qwen.E3.cmp.id\endcsname{\unmeasured{v5.qwen.E3.cmp.id}}
\expandafter\def\csname regval:v5.qwen.E3.bench.od\endcsname{\unmeasured{v5.qwen.E3.bench.od}}
\expandafter\def\csname regval:v5.qwen.E3.fresh.od\endcsname{\unmeasured{v5.qwen.E3.fresh.od}}
\expandafter\def\csname regval:v5.qwen.E3.cmp.od\endcsname{\unmeasured{v5.qwen.E3.cmp.od}}
\expandafter\def\csname regval:v5.qwen.E3.bench.iaimr\endcsname{\unmeasured{v5.qwen.E3.bench.iaimr}}
\expandafter\def\csname regval:v5.qwen.E3.fresh.iaimr\endcsname{\unmeasured{v5.qwen.E3.fresh.iaimr}}
\expandafter\def\csname regval:v5.qwen.E3.cmp.iaimr\endcsname{\unmeasured{v5.qwen.E3.cmp.iaimr}}
\expandafter\def\csname regval:v5.qwen.E3.bench.acc\endcsname{\unmeasured{v5.qwen.E3.bench.acc}}
\expandafter\def\csname regval:v5.qwen.E3.fresh.acc\endcsname{\unmeasured{v5.qwen.E3.fresh.acc}}
\expandafter\def\csname regval:v5.qwen.E3.cmp.acc\endcsname{\unmeasured{v5.qwen.E3.cmp.acc}}
\expandafter\def\csname regval:v5.qwen.E3.bench.fa\endcsname{\unmeasured{v5.qwen.E3.bench.fa}}
\expandafter\def\csname regval:v5.qwen.E3.fresh.fa\endcsname{\unmeasured{v5.qwen.E3.fresh.fa}}
\expandafter\def\csname regval:v5.qwen.E3.cmp.fa\endcsname{\unmeasured{v5.qwen.E3.cmp.fa}}
\expandafter\def\csname regval:v5.llava.E3.imr\endcsname{\unmeasured{v5.llava.E3.imr}}
\expandafter\def\csname regval:v5.llava.E3.acc\endcsname{\unmeasured{v5.llava.E3.acc}}
\expandafter\def\csname regval:v5.internvl.E3.imr\endcsname{\unmeasured{v5.internvl.E3.imr}}
\expandafter\def\csname regval:v5.internvl.E3.acc\endcsname{\unmeasured{v5.internvl.E3.acc}}
\expandafter\def\csname regval:v5.qwen3.E3.imr\endcsname{\unmeasured{v5.qwen3.E3.imr}}
\expandafter\def\csname regval:v5.qwen3.E3.acc\endcsname{\unmeasured{v5.qwen3.E3.acc}}
\expandafter\def\csname regval:v5.gemma.E3.imr\endcsname{\unmeasured{v5.gemma.E3.imr}}
\expandafter\def\csname regval:v5.gemma.E3.acc\endcsname{\unmeasured{v5.gemma.E3.acc}}
\expandafter\def\csname regval:stress5.E3.P0\endcsname{\unmeasured{stress5.E3.P0}}
\expandafter\def\csname regval:stress5.E3.P1\endcsname{\unmeasured{stress5.E3.P1}}
\expandafter\def\csname regval:stress5.E3.P2\endcsname{\unmeasured{stress5.E3.P2}}
\expandafter\def\csname regval:abl.A1.imr\endcsname{\unmeasured{abl.A1.imr}}
\expandafter\def\csname regval:abl.A1.blank\endcsname{\unmeasured{abl.A1.blank}}
\expandafter\def\csname regval:abl.A1.id\endcsname{\unmeasured{abl.A1.id}}
\expandafter\def\csname regval:abl.A1.od\endcsname{\unmeasured{abl.A1.od}}
\expandafter\def\csname regval:abl.A1.acc\endcsname{\unmeasured{abl.A1.acc}}
\expandafter\def\csname regval:abl.A1.fa\endcsname{\unmeasured{abl.A1.fa}}
\expandafter\def\csname regval:abl.A2.imr\endcsname{\unmeasured{abl.A2.imr}}
\expandafter\def\csname regval:abl.A2.blank\endcsname{\unmeasured{abl.A2.blank}}
\expandafter\def\csname regval:abl.A2.id\endcsname{\unmeasured{abl.A2.id}}
\expandafter\def\csname regval:abl.A2.od\endcsname{\unmeasured{abl.A2.od}}
\expandafter\def\csname regval:abl.A2.acc\endcsname{\unmeasured{abl.A2.acc}}
\expandafter\def\csname regval:abl.A2.fa\endcsname{\unmeasured{abl.A2.fa}}
\expandafter\def\csname regval:abl.A3.imr\endcsname{\unmeasured{abl.A3.imr}}
\expandafter\def\csname regval:abl.A3.blank\endcsname{\unmeasured{abl.A3.blank}}
\expandafter\def\csname regval:abl.A3.id\endcsname{\unmeasured{abl.A3.id}}
\expandafter\def\csname regval:abl.A3.od\endcsname{\unmeasured{abl.A3.od}}
\expandafter\def\csname regval:abl.A3.acc\endcsname{\unmeasured{abl.A3.acc}}
\expandafter\def\csname regval:abl.A3.fa\endcsname{\unmeasured{abl.A3.fa}}
\expandafter\def\csname regval:abl.A4.imr\endcsname{\unmeasured{abl.A4.imr}}
\expandafter\def\csname regval:abl.A4.blank\endcsname{\unmeasured{abl.A4.blank}}
\expandafter\def\csname regval:abl.A4.id\endcsname{\unmeasured{abl.A4.id}}
\expandafter\def\csname regval:abl.A4.od\endcsname{\unmeasured{abl.A4.od}}
\expandafter\def\csname regval:abl.A4.acc\endcsname{\unmeasured{abl.A4.acc}}
\expandafter\def\csname regval:abl.A4.fa\endcsname{\unmeasured{abl.A4.fa}}
\expandafter\def\csname regval:abl.A5.imr\endcsname{\unmeasured{abl.A5.imr}}
\expandafter\def\csname regval:abl.A5.blank\endcsname{\unmeasured{abl.A5.blank}}
\expandafter\def\csname regval:abl.A5.id\endcsname{\unmeasured{abl.A5.id}}
\expandafter\def\csname regval:abl.A5.od\endcsname{\unmeasured{abl.A5.od}}
\expandafter\def\csname regval:abl.A5.acc\endcsname{\unmeasured{abl.A5.acc}}
\expandafter\def\csname regval:abl.A5.fa\endcsname{\unmeasured{abl.A5.fa}}
\expandafter\def\csname regval:abl.A6.imr\endcsname{\unmeasured{abl.A6.imr}}
\expandafter\def\csname regval:abl.A6.blank\endcsname{\unmeasured{abl.A6.blank}}
\expandafter\def\csname regval:abl.A6.id\endcsname{\unmeasured{abl.A6.id}}
\expandafter\def\csname regval:abl.A6.od\endcsname{\unmeasured{abl.A6.od}}
\expandafter\def\csname regval:abl.A6.acc\endcsname{\unmeasured{abl.A6.acc}}
\expandafter\def\csname regval:abl.A6.fa\endcsname{\unmeasured{abl.A6.fa}}
\expandafter\def\csname regval:abl.A7.imr\endcsname{\unmeasured{abl.A7.imr}}
\expandafter\def\csname regval:abl.A7.blank\endcsname{\unmeasured{abl.A7.blank}}
\expandafter\def\csname regval:abl.A7.id\endcsname{\unmeasured{abl.A7.id}}
\expandafter\def\csname regval:abl.A7.od\endcsname{\unmeasured{abl.A7.od}}
\expandafter\def\csname regval:abl.A7.acc\endcsname{\unmeasured{abl.A7.acc}}
\expandafter\def\csname regval:abl.A7.fa\endcsname{\unmeasured{abl.A7.fa}}
\expandafter\def\csname regval:hb.E0.aAcc\endcsname{\unmeasured{hb.E0.aAcc}}
\expandafter\def\csname regval:hb.E0.qAcc\endcsname{\unmeasured{hb.E0.qAcc}}
\expandafter\def\csname regval:hb.E0.fAcc\endcsname{\unmeasured{hb.E0.fAcc}}
\expandafter\def\csname regval:hb.TIDE.aAcc\endcsname{\unmeasured{hb.TIDE.aAcc}}
\expandafter\def\csname regval:hb.TIDE.qAcc\endcsname{\unmeasured{hb.TIDE.qAcc}}
\expandafter\def\csname regval:hb.TIDE.fAcc\endcsname{\unmeasured{hb.TIDE.fAcc}}
\expandafter\def\csname regval:hb.E2.aAcc\endcsname{\unmeasured{hb.E2.aAcc}}
\expandafter\def\csname regval:hb.E2.qAcc\endcsname{\unmeasured{hb.E2.qAcc}}
\expandafter\def\csname regval:hb.E2.fAcc\endcsname{\unmeasured{hb.E2.fAcc}}
\expandafter\def\csname regval:hb.E3.aAcc\endcsname{\unmeasured{hb.E3.aAcc}}
\expandafter\def\csname regval:hb.E3.qAcc\endcsname{\unmeasured{hb.E3.qAcc}}
\expandafter\def\csname regval:hb.E3.fAcc\endcsname{\unmeasured{hb.E3.fAcc}}
\expandafter\def\csname regval:hj.iaa\endcsname{\unmeasured{hj.iaa}}
\expandafter\def\csname regval:hj.ruleb\endcsname{\unmeasured{hj.ruleb}}
\expandafter\def\csname regval:hj.ruleb2\endcsname{\unmeasured{hj.ruleb2}}
\expandafter\def\csname regval:hj.judgeu\endcsname{\unmeasured{hj.judgeu}}
\expandafter\def\csname regval:hard.E0.low\endcsname{\unmeasured{hard.E0.low}}
\expandafter\def\csname regval:hard.TIDE.low\endcsname{\unmeasured{hard.TIDE.low}}
\expandafter\def\csname regval:hard.E0.mid\endcsname{\unmeasured{hard.E0.mid}}
\expandafter\def\csname regval:hard.TIDE.mid\endcsname{\unmeasured{hard.TIDE.mid}}
\expandafter\def\csname regval:hard.E0.high\endcsname{\unmeasured{hard.E0.high}}
\expandafter\def\csname regval:hard.TIDE.high\endcsname{\unmeasured{hard.TIDE.high}}

%% file: sec/0_abstract.tex
\begin{abstract}
Vision-language models answer questions about images they were never shown, a failure called mirage.
Such an answer can be correct without being justified by the evidence, and the difference matters
wherever a model's output is checked by someone who did not see its input. We locate the failure in
the training objective: cross-entropy separates a model that reads the image from one that ignores it
by only the conditional mutual information between answer and image, and no training example rewards
declining. We formalise support completion, TIDE, which also shows each training question with a blank
image, a same-domain wrong image and an out-of-domain image, with targets that decline, and prove that
under TIDE every image-blind predictor is suboptimal by at least $1.04$ nats per example, that the
optimum declines exactly on the counterfactual kinds while keeping its behaviour on valid images, and
that for any finite model the error of its answer-or-decline decision, declining on valid input plus
answering on invalid input, is bounded below by how well its representation separates valid from
invalid evidence. On Qwen2.5-VL-7B under a frozen held-out
protocol with two judges, the three-condition predecessor of TIDE reduced the invalid-evidence mirage
rate from $0.759$ to $0.097$ and the blank-image rate from $0.700$ to $0.003$ at a false-abstention
cost of $0.126$; label-free grounding raised the rate to $0.859$, a likelihood-advantage contrast
moved it only to $0.718$ because its empty-image term carries no gradient, and grounding applied after
support completion re-opened the mirage. The four-condition protocol on three base models, with seeds
and component ablations, is in progress and not reported here.
\end{abstract}

%% file: sec/1_intro.tex
\begin{figure*}[t]
\centering
\includegraphics[width=\textwidth]{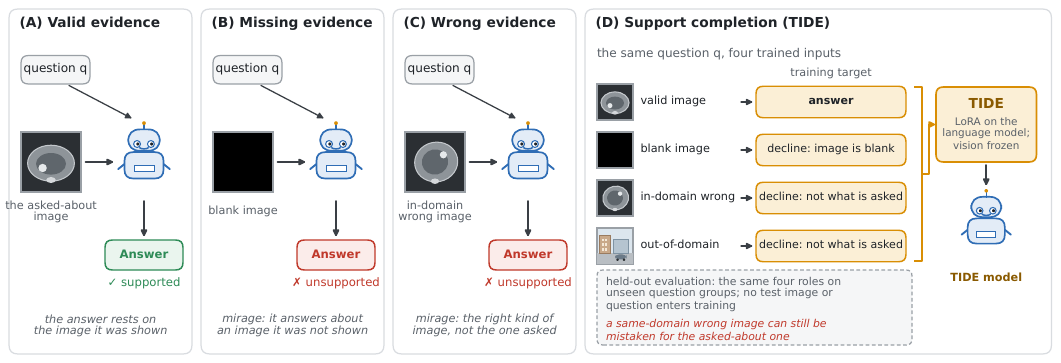}
\caption{\textbf{Mirage, and support completion.} A question with (A) the image it is about, (B) a
blank image and (C) a real image of the right kind that is not the one asked about: a stock model
answers in all three cases, and only the first answer rests on evidence. (D) Support completion (TIDE)
trains the same question with its valid image and three counterfactual inputs whose targets decline;
only the language model's adapters are trained, and evaluation uses the same four roles on groups
that never enter training. No measured quantity is shown.}
\label{fig:concept}
\end{figure*}

\section{Introduction}
\label{sec:intro}

Give a vision-language model a question about a scan and a blank image, and it will often describe a
finding anyway. Give it a genuine scan that is not the one the question refers to, and it will describe a
finding in that one. \citet{asadi2026mirage} named this behaviour \emph{mirage} and
separated it from hallucination: a hallucinating model invents detail inside a valid frame, whereas a
model in mirage mode treats absent or unrelated visual input as the evidence it was asked about. The
wrong-image case is the dangerous one. A mis-routed study or a stale screenshot produces a real image
of the expected kind, nothing about it looks anomalous, and a confident description of it has the
same form as a correct one. Inference-time detectors are weakest there, on the in-domain hard
negative~\cite{chowdhury2026tclia}.

\noindent\textbf{Why mirage is a safety problem.} An answer produced without the evidence can still
be right, because many questions are predictable from their text, so benchmark accuracy does not
reveal whether the model read the image; \citet{asadi2026mirage} found that on three multimodal
benchmarks three quarters of the questions were answered correctly without the image by at least one
frontier model. In deployment the difference between a correct answer and a justified one decides
whether an error can be caught: a radiology report, an invoice value or an agent's action on a
screenshot is checked, if at all, by a person or a program that did not see the model's input. For
that overseer four things matter: \emph{evidence provenance}, whether the output traces to
the input it names; \emph{unsupported confidence}, whether the output signals when it does not;
\emph{error detectability}, whether a wrong output differs in form from a right one; and the
\emph{value of abstention}, whether the system can hand the case back. A model that reports evidence
it never received defeats all four, because its failure is silent; oversight is cheap only when
outputs can be checked without redoing the task~\cite{amodei2016concrete,bowman2022scalable}, and
reporting what one observed is honesty in its narrowest sense~\cite{evans2021truthful,askell2021general}.

\noindent\textbf{The argument in seven steps.}
\begin{enumerate}
\itemsep-0.1em
\item Mirage is a relation between an output and the evidence shown with it (\cref{sec:problem}).
\item Standard training data contain no input on which declining is the right response.
\item The objective separates a model that reads the image from one that ignores it by only
$I(A;V\mid Q)$ (\cref{prop:starvation}), small when answers are predictable from the question.
\item Hence an intervention that leaves the training targets unchanged cannot make declining optimal
(\cref{sec:consequence}).
\item Support completion changes the distribution: under TIDE the optimum declines on every
counterfactual kind, keeps its behaviour on valid images, and every image-blind predictor loses at
least $H(B)=\tfrac32\ln 2$ nats per example (\cref{prop:tide}).
\item For any finite model, the error of its answer-or-decline decision (declining on valid inputs,
answering on counterfactual ones) is bounded below by the Bayes error of deciding validity from its
representation (\cref{prop:rep}).
\item These statements become hypotheses about where a trained model should and should not reach its
optimum, tested under a frozen protocol with two judges (\cref{sec:results}).
\end{enumerate}

\noindent\textbf{What is measured.} The frozen record of this paper is a single-model study of
Qwen2.5-VL-7B-Instruct~\cite{bai2025qwen25vl} under protocol v4.1, in which the TIDE adapter was
trained on three of the four conditions (out-of-domain held out of training): one fine-tune took the
invalid-evidence mirage rate from $\rv{M-e0-U-imr}$ to $\rv{M-e1-U-imr}$, leaving a residual on
in-domain wrong images and a false-abstention cost, and three target-preserving interventions behaved
as step 4 predicts. The four-condition protocol v5, further base models, seeds, component ablations
and a human validation of the judges are in progress and appear here only as designs.

\noindent\textbf{Contributions.} (1) A target-level account of mirage (\cref{prop:starvation}) and the
objective-behaviour decoupling it predicts. (2) TIDE, with a guarantee on its optimum
(\cref{prop:tide}), a representation-limited bound for finite models (\cref{prop:rep}) and the
token-mean form that is trained (\cref{cor:token}). (3) A frozen, leakage-safe protocol with two
judges, and a measured single-model test of the hypotheses (the component-and-order grid, the
evidence-advantage contrast and its gradient cancellation, a pre-registered prompt-pressure pilot). (4) A registry mapping every number to a cell, checkpoint, judge and
artifact, and a build that excludes anything unmeasured.

%% file: sec/2_related.tex
\section{Related Work}
\label{sec:related}

\noindent\textbf{The mirage failure mode.} \citet{asadi2026mirage} define mirage reasoning, separate
it from object hallucination, show that frontier systems describe images they were never given, and
report that common evaluation instructions raise the rate (their Figure S.1 instruction is our stress
condition P2). We adopt their framing and their response-only judge.

\noindent\textbf{Detecting rather than curing.} \citet{chowdhury2026tclia} gate a model at inference
with a probe of the layerwise image-question alignment in a frozen contrastive encoder, fused with a
blank-input gate and the model's self-assessment; the signal degrades towards chance on in-domain
hard negatives, where our residual also sits. Their mirage rate is a detector metric over all
evaluation inputs; our TC-LIA MR borrows only its denominator. \citet{benlevi2026mirageprobes}
train linear and MLP probes on hidden activations to predict whether a response is a mirage, defined
there by consistency between the with-image and without-image answers, on two models and three
benchmarks, with difference-of-activation probes the most accurate. They propose two mechanisms, spurious internal images and
textual biases, evaluate no intervention, and call their labels heuristic and their findings
correlational. Our account is complementary: both mechanisms answer from an input that does not
support the question and \cref{prop:starvation} says why neither is penalised; their probes target a
response-level mirage label, whereas \cref{prop:rep} concerns the detectability of input validity, a
different target (row-by-row comparison in the supplementary material).

\noindent\textbf{Hallucination benchmarks and decoding.} POPE~\cite{li2023pope} established polling
for object hallucination, whose probe structure our image-aware judge adopts;
HallusionBench~\cite{guan2024hallusionbench} tests visual dependence with paired questions and is the
external benchmark of our plan; visual contrastive decoding~\cite{leng2024vcd} contrasts clean and
distorted images at inference. Hallucination metrics measure invented content inside a valid frame, so we report them separately.

\noindent\textbf{Abstention and selective prediction.} Selective prediction treats declining as a
decision to be learned or thresholded, trading coverage for
risk~\cite{elyaniv2010selective,geifman2017selective,hendrycks2017baseline}, under domain
shift~\cite{kamath2020selective} and in visual question answering with a learned
selector~\cite{whitehead2022reliable}; refusal tuning teaches language models to say they do not
know~\cite{zhang2024rtuning}, and VizWiz~\cite{gurari2018vizwiz} contains unanswerable visual
questions. TIDE differs in what triggers declining: whether the input supports the question, rather than the
expected correctness of a guess, with counterfactual examples that differ from answerable ones only in
the image.

\noindent\textbf{Contrastive grounding and the prior study.} Contrastive image-text
objectives~\cite{radford2021clip,oord2018cpc} are the standard alignment tool and our grounding cell is
one, applied to the projector of an in-sequence model~\cite{liu2023llava,bai2025qwen25vl,chen2024internvl}.
Gradient starvation~\cite{pezeshki2021gradient} describes how an early-explaining feature suppresses
others; a prior single-model study measured a modality-level instance and argued that
representation-level interventions do not move the answer-versus-decline
decision~\cite{anon2026dissociation}; none of its numbers enters this paper, and \cref{sec:e5} re-tests
its likelihood-advantage contrast.

%% file: sec/3_problem.tex
\section{Problem Formulation}
\label{sec:problem}

We write $Q$ for the question, $V$ for the image or image set shown with it, and $A$ for the response.
For each question $q$ let $S(q)$ be its \emph{support}, the set of images the question is about; the
training distribution $D$ of visual instruction tuning places all its mass on $V\in S(Q)$. An
\emph{invalid-evidence input} is a pair $(q,v)$ with $v\notin S(q)$, of three kinds here: a blank
image (all black, supplied in every image slot), an \emph{in-domain} wrong image (a genuine image of
the same domain that is not the one asked about) and an \emph{out-of-domain} image. A response to
such an input is a \emph{mirage} when it answers as if the input supported the question, and not one
when it declines, states the problem, or is about what the shown image contains. The same string can
be a mirage under a blank image and justified under the right one, and a mirage can be factually
correct. A judge decides the event on generated text (\cref{sec:protocol}); the theory concerns the
targets that training rewards.

A predictor is a conditional distribution $p(a\mid q,v)$; its risk under $D$ is the expected log-loss
$R_D(p)=\mathbb{E}_{(Q,V,A)\sim D}[-\log p(A\mid Q,V)]$. A predictor is \emph{image-blind} if
$p(\cdot\mid q,v)=p(\cdot\mid q,v')$ for all $q,v,v'$. Minima are over all conditional distributions
or over all image-blind ones, with no restriction to a model class; \cref{prop:rep} returns to that
restriction. On valid rows we also report accuracy, the answered rate and the \emph{false abstention} rate, the
fraction of valid inputs on which the model declines.

%% file: sec/4_tide.tex
\section{Support Completion: Objective and Theory}
\label{sec:tide}

\subsection{Why the standard objective admits mirage}
\label{sec:theory}

\begin{proposition}[Evidence starvation of the objective]
\label{prop:starvation}
For any distribution $D$ with finite conditional entropies,
\begin{equation}
\Delta(D) := \min_{p\ \text{image-blind}} R_D(p) - \min_{p} R_D(p) = I_D(A;V\mid Q).
\label{eq:gap}
\end{equation}
\end{proposition}
\begin{proof}
For any $p$, $R_D(p)=H_D(A\mid Q,V)+\mathbb{E}_{Q,V}\,\mathrm{KL}\big(D(\cdot\mid Q,V)\,\|\,p(\cdot\mid
Q,V)\big)$, so the unrestricted minimum is $H_D(A\mid Q,V)$. An image-blind $p$ is a function
$r(\cdot\mid q)$ of the question, and the same decomposition with $V$ marginalised gives the minimum
$H_D(A\mid Q)$ at $r=D(A\mid Q)$. The difference is $I_D(A;V\mid Q)$.
\end{proof}

When the answer is largely predictable from the question, the image-blind predictor sits within a
small $I_D(A;V\mid Q)$ of the optimum, and multimodal benchmarks are in this
regime~\cite{asadi2026mirage}; on the ten domains of this paper, at least one of five API models produced a
correct answer without receiving the image on a fraction $\rv{M-poc-bclean-removed-total}$ of the questions
(supplementary material), a sign of question-only correctness and of reliance on textual priors, not a
measurement of $I_D(A;V\mid Q)$. A second fact explains why the failure is a confident answer rather
than a refusal. Every image in $D$ is one its question is about, so on any input outside that support
$R_D$ places no constraint on $p$: answering there is never penalised and declining is never rewarded. A
model near the image-blind solution meets a blank or wrong image by doing what it does everywhere,
answering from the question.

\noindent\textbf{Consequence for representation-level fixes.}\label{sec:consequence} Both facts are
properties of $D$ alone. An intervention that changes the model but not the targets leaves the gap in
\cref{eq:gap} unchanged and adds no reward for declining. It can change what the model does on an
invalid input, in either direction, because behaviour there is extrapolation; it cannot make
declining optimal. We therefore do not claim that representation changes cannot reduce
mirage; we predict that they do not bring blank-image mirage near zero and do not produce the
abstention a target-level objective produces (\cref{sec:results}).

\subsection{TIDE: completing the support}
\label{sec:objective}

For each kind $k\in\{\mathrm{blank},\mathrm{id},\mathrm{od}\}$ let a \emph{counterfactual kernel}
$C_k(\cdot\mid q)$ be a distribution over images whose support is disjoint from $S(q)$ and from the
other kinds' supports. Let $\dec{blank}$ be a response stating that the image is blank and
$\dec{unrel}$ a response stating that the image does not show what the question asks about; both are
distinct and never valid answers, and the in-domain and out-of-domain kinds share the second (an
unrelated image is unrelated whichever domain it comes from). The completed distribution $\Dc$ draws a
kind $K$ uniformly from $\{\mathrm{real},\mathrm{blank},\mathrm{id},\mathrm{od}\}$, independently of
everything else; if $K=\mathrm{real}$ it draws $(Q,V,A)\sim D$, otherwise $Q\sim D$,
$V\sim C_K(\cdot\mid Q)$ and $A=\dec{K}$ with $\dec{id}=\dec{od}=\dec{unrel}$. TIDE is cross-entropy on
the completed distribution,
\begin{equation}
R_{\Dc}(p)=\tfrac14 R_D(p)+\tfrac14\sum_{k}\mathbb{E}_{k}\big[-\log p(\dec{k}\mid Q,V)\big],
\label{eq:tide}
\end{equation}
where $k$ runs over the three counterfactual kinds and $\mathbb{E}_k$ is over $Q\sim D$, $V\sim
C_k(\cdot\mid Q)$. TIDE adds no auxiliary head, contrastive term, detector or architectural change. The four examples built from one question differ
only in the image and the target. Let $B=g(A)\in\{\ans,\dec{blank},\dec{unrel}\}$ be the
\emph{behaviour class} of a response, with masses $(\tfrac14,\tfrac14,\tfrac12)$ under $\Dc$, so
$H(B)=\tfrac32\ln 2\approx 1.040$ nats.

\begin{proposition}[TIDE removes the image-blind optimum and preserves valid behaviour]
\label{prop:tide}
Under the conditions above: (i) every image-blind predictor is at least $H(B)$ nats per example worse
than the optimum,
\begin{equation}
\Delta(\Dc)=H(B)+\tfrac14 I_D(A;V\mid Q)\;\ge\;\tfrac32\ln 2;
\label{eq:tidegap}
\end{equation}
(ii) the minimiser $p^\star$ satisfies $p^\star(\dec{k}\mid q,v)=1$ for every $v$ in the support of
$C_k(\cdot\mid q)$, and $p^\star(\cdot\mid q,v)=D(\cdot\mid q,v)$ for every $v\in S(q)$.
\end{proposition}
\begin{proof}[Proof sketch]
$\Delta(\Dc)=I_{\Dc}(A;V\mid Q)$ by \cref{prop:starvation}. Since $B=g(A)$, the chain rule gives
$I(A;V\mid Q)=I(B;V\mid Q)+I(A;V\mid Q,B)$; $B$ is a function of $(Q,V)$ by disjointness and
independent of $Q$, so the first term is $H(B)$, and the second is $\tfrac14 I_D(A;V\mid Q)$ because a
decline class has a constant response. (ii) The minimiser is $\Dc(A\mid Q,V)$, which has one branch on
each input. Full proof in the supplementary material.
\end{proof}

$H(B)$ depends on the kind weights and on the number of distinct decline strings (three would give
$\ln 4$, one shared string $h(\tfrac14)\approx 0.56$ nats); the decisions in (ii) are the same in
every variant.

\begin{corollary}[The token-mean loss that is trained]
\label{cor:token}
The implementation averages the log-loss over the $T(a)$ labelled tokens of each target. That
objective equals $\bar w\,R_{\tilde D}(p)$ for the length-tilted distribution
$\tilde D(q,v,a)\propto \Dc(q,v,a)/T(a)$, $\bar w=\mathbb{E}_{\Dc}[1/T(A)]$, so its image-blind gap is
$\bar w\,I_{\tilde D}(A;V\mid Q)\ge H(B)/T_{\max}$ nats per labelled token when every target has at
most $T_{\max}$ labelled tokens. Its optimum declines with probability one on every counterfactual
input and on valid inputs equals $D(\cdot\mid q,v)$ tilted towards shorter answers (unchanged when
competing answers have equal length).
\end{corollary}

\subsection{What a finite model can achieve}
\label{sec:realise}

\Cref{prop:tide} is a statement about the unrestricted optimum; a trained model is a function of its
representation. The next result says what that costs, in the two quantities we measure.

\begin{proposition}[Representation-limited optimum]
\label{prop:rep}
Let $\phi$ map inputs $(q,v)$ to a representation $Z=\phi(Q,V)$ and let $\mathcal P_\phi$ be the
predictors that factor through $\phi$. Call a predictor's \emph{decision} on an input binary: decline
(a decline string) or answer (any other response). Under $\Dc$, write
$\pi(z)=\Pr(K=\mathrm{real}\mid Z=z)$, $M$ for the probability of the answer decision on a
counterfactual input and $F$ for the probability of the decline decision on a valid one. Then
(i) the best risk in $\mathcal P_\phi$ exceeds the unrestricted optimum by $I(A;Q,V\mid Z)\ge H(B\mid Z)$;
(ii) for every predictor in $\mathcal P_\phi$ and every decoding rule,
\begin{equation}
\tfrac14 F+\tfrac34 M\;\ge\;e(\phi):=\mathbb{E}\big[\min(\pi(Z),1-\pi(Z))\big],
\label{eq:bayes}
\end{equation}
the Bayes error of deciding validity from $Z$, with equality for the rule that answers iff
$\pi(Z)>\tfrac12$;
(iii) if responses are sampled from the class optimum, $\tfrac14F=\tfrac34M=\mathbb{E}[\pi(Z)(1-\pi(Z))]$.
Both sides vanish iff validity is determined by $Z$ almost surely.
\end{proposition}
\begin{proof}[Proof sketch]
(i) is the class-restricted form of \cref{prop:starvation} with $H(B\mid Q,V)=0$. (ii) The left-hand
side is the probability that the answer-or-decline decision is wrong, and any decision based on $Z$
errs with probability at least $e(\phi)$. (iii) The class optimum puts mass $1-\pi(z)$ on declining.
\end{proof}

The proposition turns the residual into a property of the representation: if $\phi$ does not separate
a valid image from its in-domain counterfactual, the model must answer on the wrong one or decline on
the right one, and no further
training of the same representation can push the weighted sum in \cref{eq:bayes} below $e(\phi)$.
The bound does not fix the ratio of the two residuals; (iii) gives it for one decoding rule, and
greedy decoding is neither rule exactly. $F$ and $M$ are rates of the string-level decision (is the
response a trained decline string), not judge verdicts: I-MR and IA-IMR count judged mirages, and a
refusal in other words is an answer under the rule; \cref{sec:mechanism} computes the rates from the
stored responses.

\noindent\textbf{Scope.} \Cref{prop:tide} needs disjoint supports: if a question is ambiguous enough
that a counterfactual image could be valid for it, the gap shrinks by $H(B\mid Q,V)$ and the optimum
on the overlap is a mixture of answering and declining. The realised training set is disjoint at the
level of image content and question strings (\cref{sec:protocol}), not of meaning. The guarantee
covers the trained kinds only, and says nothing about finite samples, optimisation or seeds.

\noindent\textbf{Hypotheses.} The results become predictions about where a trained model should and
should not match its optimum. \textbf{H1}: blank inputs are separable without reading the question,
so blank-image mirage should reach zero. \textbf{H2}: declining should not spread to valid inputs.
\textbf{H3}: an out-of-domain image needs only its domain recognised, so its rate should sit near the
blank level once the kind is trained; under the three-condition predecessor any reduction there is
generalisation. \textbf{H4}: the residual should concentrate on in-domain wrong images.
\textbf{H5}: target-preserving interventions should not bring blank-image mirage near zero.
\textbf{H6} (empirical, not implied by \cref{prop:rep}, which constrains only a weighted sum): if
both the in-domain residual and false abstention track how hard a domain's valid and wrong images are
to tell apart, they should rise and fall together across domains.

%% file: sec/5_protocol.tex
\section{Protocol}
\label{sec:protocol}

\noindent\textbf{Benchmark and roles.} Each item is a \emph{group}: one question shown in four roles,
\textbf{real} (the correct image or images), \textbf{blank} (an all-black image in every image slot;
supplied, not removed), \textbf{in-domain} (a different genuine image set from the same domain) and
\textbf{out-of-domain} (an image set from another domain). Ten domains contribute (chest radiographs,
pathology, microscopy, medical and general multiple-choice sets, documents, infographics, scene text
and video frames); options are rendered once and every image is placed at its own marker.

\noindent\textbf{Leakage-safe tiers.} The 4{,}746 eligible groups are joined into 3{,}141
\emph{content components} by shared real-image content or identical question string; whole
components are assigned to tiers before any counterfactual image is chosen, and each counterfactual is
drawn from the same tier and a different component, so no held-out image or question string reaches
training in any role. Protocol v5 (frozen
2026-10-09) has a train tier of 3{,}000 groups (250 double as the grounding probe) yielding 12{,}000
TIDE examples; an eval tier of 500 groups; a \emph{full held-out benchmark} of 1{,}000 groups
(4{,}000 rows), held out from training, of which 712 groups (585 by key), including the 388 v4.1
test anchors it embeds, appear by key, image content or component in some earlier evaluation set,
result file or document (a frozen pre-v5 list of 1{,}170 keys) and are classed as previously
evaluated; a \emph{fresh sensitivity holdout} of 288 groups over nine domains whose anchor and
counterfactual components appear in no such artifact; and the 390-group \emph{v4.1 comparison
subset}. The predecessor protocol v4.1, whose results are the measured record of
this paper, used a 390-group final test (39 per domain, 1{,}560 rows) packed from components outside a 736-key
inspection list and never generated, read or judged before the frozen runs; its rows have never been
read by content, and only their judged aggregates and verdict labels enter this paper. Its TIDE adapter
was trained on three conditions, the out-of-domain image held out of training (9{,}000 examples). An earlier, defective pipeline is documented in the supplementary
material; no number from it is reused.

\noindent\textbf{Cells.} The frozen vocabulary names four principal cells and one appendix cell.
\textbf{E0 Base} is the released checkpoint. \textbf{E1 Grounding} is the target-preserving
intervention of \cref{sec:consequence}: it trains only the vision projector (language model and
encoder frozen, no answer read) by a joint softmax over the cosine scores of one anchor against five
candidate image sets (the real image, its two counterfactuals and two further negatives; under v4.1
three in-batch real images stood in for the explicit negatives), temperature 0.2, batch 16, 300 fixed
steps, fp32 master weights.
\textbf{E2 Grounding then TIDE} freezes the E1 projector and trains a fresh TIDE adapter;
\textbf{E3 TIDE then Grounding} trains the TIDE adapter first and then the projector on exactly the E1
pool, candidate policy, schedule and budget. \textbf{TIDE only}, the appendix cell, is the adapter
alone (E3's first stage): LoRA~\cite{hu2022lora} of rank 8 on the language model's q/k/v/o
projections, learning rate $10^{-4}$, 800 steps with accumulation 4, per-example token-mean loss, on
the support-completion examples, vision pathway untouched. \textbf{E5}, a v4.1 diagnostic, is the
evidence-advantage contrast of \cref{sec:e5}. Every stage uses its fixed final step and nothing is
selected on a validation metric.

\noindent\textbf{Metrics.} Let $M_c$ and $N_c$ be the mirage count and row count of an invalid role
$c\in\{\mathrm{blank},\mathrm{id},\mathrm{od}\}$ and $N_\mathrm{real}$ the real rows. We report
\begin{equation}
\begin{aligned}
\text{I-MR} &= \frac{M_\mathrm{blank}+M_\mathrm{id}+M_\mathrm{od}}{N_\mathrm{blank}+N_\mathrm{id}+N_\mathrm{od}},\\
\text{TC-LIA MR} &= \frac{M_\mathrm{blank}+M_\mathrm{id}+M_\mathrm{od}}{N_\mathrm{real}+N_\mathrm{blank}+N_\mathrm{id}+N_\mathrm{od}},
\end{aligned}
\label{eq:metrics}
\end{equation}
the \emph{invalid-evidence mirage rate}, our primary metric, and the full-set rate that shares the
denominator of \citet{chowdhury2026tclia} and is the only number of ours that may be placed next to
theirs. Two judges decide which rows are mirages. \textbf{Judge U} is
the response-only uncertainty judge of \citet{asadi2026mirage}: their template, gpt-4o-mini at
temperature 0, the response text alone; it defines I-MR and TC-LIA MR. \textbf{Rule B} is an
image-aware judge: it receives the rendered question, a caption of every image the model actually
saw, and the role, and answers four yes-or-no probes in the polling style of POPE~\cite{li2023pope};
an invalid-evidence response is a mirage unless it acknowledges the image problem or, on an unrelated
image, is about what that image shows. Its rates on the same denominators are IA-IMR and IA-FMR, and
on real rows it scores the answered rate, false abstention, accuracy against every gold string, and
accuracy conditional on answering. The primary rule-B model is gpt-5.5, the secondary gpt-4o-mini;
captions come from Qwen3-VL-32B-Instruct~\cite{qwen3vl2025}; the judges' known disagreement is
analysed in \cref{sec:mechanism}. An unparseable judge reply is counted, never dropped and never a
mirage. Rule B with gpt-5.5 agreed with human labels on 38 earlier-era rows, in sample; a blinded
human validation of both judges on 300 to 500 rows is designed (supplementary material) and not run.

\noindent\textbf{Statistics and status.} Per-rate intervals are 95\% Wilson intervals from the stored
counts; contrasts between cells carry paired bootstrap intervals over question groups (2{,}000
resamples; the four rows of a group move together). No hypothesis test was pre-registered and none is
reported. Every number in this paper is for Qwen2.5-VL-7B-Instruct at seed 0 under protocol v4.1.
Protocol v5 runs (the four-condition TIDE on Qwen2.5-VL-7B, LLaVA-OneVision-7B
\cite{li2024llavaonevision} and InternVL3.5-8B~\cite{wang2025internvl35}, three seeds, the component
ablations of \cref{sec:ablation}) are in progress; the registry marks them IN\_PROGRESS or PLANNED and
this build prints none of their values.

%% file: sec/6_results.tex
\section{Main Results}
\label{sec:results}

\begin{table*}[t]
\caption{\textbf{The measured record: one model, the component-and-order grid plus E5, one frozen
held-out test.} Qwen2.5-VL-7B-Instruct, seed 0, protocol v4.1 final test (390 groups, 1{,}560 rows);
the TIDE stage is the three-condition predecessor. Left: judge U, I-MR over the 1{,}170
invalid-evidence rows, the rate per role (390 rows each) and TC-LIA MR (denominator 1{,}560). Right:
rule B (gpt-5.5), IA-IMR on the same rows, and accuracy, answered rate and false abstention on the
390 real rows. No judge reply was unparseable; E5 was judged with judge U only. Bold: the
support-completion cells.}
\label{tab:main}
\centering\scriptsize\setlength{\tabcolsep}{5pt}
\input{tables/arxiv/tab_main_v41}
\end{table*}

\begin{figure*}[t]
\centering
\includegraphics[width=\textwidth]{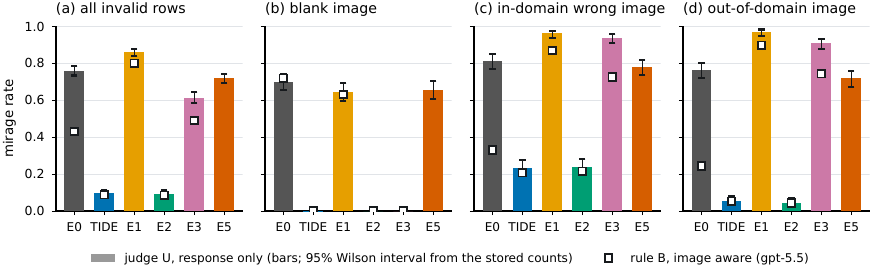}
\caption{\textbf{Mirage rate by cell and role (measured, v4.1 final test).} Bars: judge U with 95\%
Wilson intervals; hollow squares: rule B (gpt-5.5). Support completion (the TIDE-only cell and E2, Grounding then TIDE) empties the
blank role, nearly empties the untrained out-of-domain role and leaves the in-domain role as the
residual; grounding after TIDE (E3) keeps blank at zero and re-opens both unrelated-image roles.}
\label{fig:roles}
\end{figure*}

\Cref{tab:main} and \cref{fig:roles} hold every measured number; we read them against the hypotheses
of \cref{sec:realise}.

\noindent\textbf{Support completion removes most of the mirage (H1, H3, H4).} The stock model answers
non-uncertainly on $\rv{M-e0-U-imr}$ of the invalid-evidence rows (blank $\rv{M-e0-U-blank}$,
in-domain $\rv{M-e0-U-id_unrelated}$, out-of-domain $\rv{M-e0-U-od_unrelated}$) and, under the
image-aware judge, commits as if the shown image supported the question on $\rv{M-e0-B-iaimr}$ of
them. One support-completion fine-tune takes I-MR to $\rv{M-e1-U-imr}$ and IA-IMR to
$\rv{M-e1-B-iaimr}$ (paired difference $\rv{M-contrast-U-e0-e1}$ $\rv{M-contrast-U-e0-e1.ci}$). On blank images the mirage is gone, $\rv{M-e1-U-blank.n}$ of 390 rows under either judge,
as H1 predicts. On out-of-domain images,
which this predecessor never trained on, the rate falls from $\rv{M-e0-U-od_unrelated}$ to
$\rv{M-e1-U-od_unrelated}$, the generalisation H3 allows but does not guarantee; whether the trained
kind reaches the blank level is a v5 question. The residual sits where H4 places it: in-domain wrong
images at $\rv{M-e1-U-id_unrelated}$ under judge U and $\rv{M-e1-B-id_unrelated}$ under rule B.

\noindent\textbf{The cost on valid images (H2 not supported).} Declining does spread to valid
inputs, modestly: the answered rate falls from $\rv{M-e0-B-answered}$ to $\rv{M-e1-B-answered}$,
false abstention rises from $\rv{M-e0-B-fa}$ to $\rv{M-e1-B-fa}$, and accuracy over all 390 real rows
falls from $\rv{M-e0-B-acc}$ to $\rv{M-e1-B-acc}$ while accuracy conditional on answering rises from
$\rv{M-e0-B-acc-cond}$ to $\rv{M-e1-B-acc-cond}$. Whether the residual and the false abstention co-vary across domains is H6
(\cref{sec:mechanism}).

\noindent\textbf{Grounding alone (H5).} Grounding (E1) raises I-MR to $\rv{M-e2-U-imr}$ (in-domain
$\rv{M-e2-U-id_unrelated}$, out-of-domain $\rv{M-e2-U-od_unrelated}$) and IA-IMR to
$\rv{M-e2-B-iaimr}$, and lowers accuracy to $\rv{M-e2-B-acc}$, although its projector converged
(held-match $\rv{M-diag-e2-heldmatch}$ on the train-owned probe). Blank-image mirage falls slightly,
from $\rv{M-e0-U-blank}$ to $\rv{M-e2-U-blank}$, so a stricter earlier wording of H5 (no reduction at
all) fails while the stated prediction (not near zero) holds ($\rv{M-e2-U-blank}$ against
$\rv{M-e1-U-blank}$).

\noindent\textbf{Order.} Grounding then TIDE (E2) is indistinguishable from TIDE only (I-MR
$\rv{M-e3-U-imr}$ against $\rv{M-e1-U-imr}$, paired difference $\rv{M-contrast-U-e1-e3}$
$\rv{M-contrast-U-e1-e3.ci}$); we do not claim that grounding first helps. TIDE then Grounding (E3)
keeps the blank role at $\rv{M-e4-U-blank}$ and re-opens both unrelated-image roles (in-domain
$\rv{M-e4-U-id_unrelated}$, out-of-domain $\rv{M-e4-U-od_unrelated}$). The two orders differ in order alone (\cref{sec:order}).

%% file: tables/arxiv/tab_main_v41.tex
\begin{tabular}{@{}lccccc|cccc@{}}
\toprule
& \multicolumn{5}{c|}{judge U (response only)} & \multicolumn{4}{c}{rule B (gpt-5.5, image aware)} \\
cell & I-MR & blank & in-dom. & out-dom. & TC-LIA MR & IA-IMR & acc. & answered & false abst. \\
\midrule
E0 Base & 0.759 & 0.700 & 0.813 & 0.764 & 0.569 & 0.431 & 0.577 & 0.982 & 0.018 \\
TIDE only (three-condition) & \textbf{0.097} & \textbf{0.003} & \textbf{0.233} & \textbf{0.054} & \textbf{0.072} & \textbf{0.088} & \textbf{0.559} & \textbf{0.874} & \textbf{0.126} \\
E1 Grounding & 0.859 & 0.646 & 0.962 & 0.969 & 0.644 & 0.799 & 0.377 & 0.774 & 0.226 \\
E2 Grounding then TIDE (three-cond.) & \textbf{0.094} & \textbf{0.000} & \textbf{0.236} & \textbf{0.046} & \textbf{0.070} & \textbf{0.086} & \textbf{0.526} & \textbf{0.849} & \textbf{0.151} \\
E3 TIDE then Grounding (three-cond.) & 0.615 & 0.000 & 0.936 & 0.908 & 0.461 & 0.490 & 0.423 & 0.821 & 0.179 \\
E5 evidence-advantage contrast & 0.718 & 0.656 & 0.779 & 0.718 & 0.538 & -- & -- & -- & -- \\
\bottomrule
\end{tabular}

%% file: sec/7_ablation.tex
\section{Ablation and Generalisation}
\label{sec:ablation}

\begin{figure}[t]
\centering
\includegraphics[width=\linewidth]{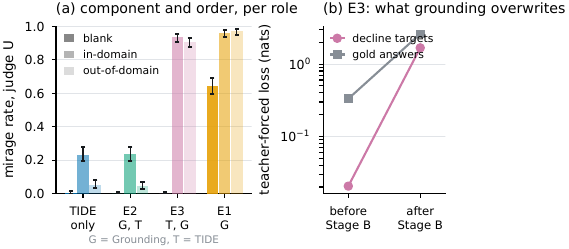}
\caption{\textbf{Training order (measured, v4.1).} (a) Per-role mirage rate for the four trained
cells. (b) Teacher-forced loss of the E3 model on ten train-owned probe keys before and after its
grounding stage: the decline targets are forgotten (a diagnostic, not a held-out result).}
\label{fig:order}
\end{figure}

\subsection{Training order}
\label{sec:order}
In the component-and-order grid (\cref{fig:order}a), grounding before TIDE leaves the result unchanged; grounding after TIDE undoes it on exactly the roles that need
image content and not on the blank role, which needs none. A train-owned retention diagnostic agrees
(\cref{fig:order}b): the fine-tuned model's loss on its own decline targets rises from
$\rv{M-diag-e4-abstain-before}$ to $\rv{M-diag-e4-abstain-after}$ nats during the grounding stage (and
on the gold answers from $\rv{M-diag-e4-real-before}$ to $\rv{M-diag-e4-real-after}$). The projector
update moves the inputs the adapter conditions on, and the adapter, frozen in Stage B, cannot follow.

\subsection{The evidence-advantage contrast}
\label{sec:e5}
The proposal from which TIDE grew also had a contrastive component: the \emph{evidence advantage}
$\mathcal{A}(v,q,a)=\log p(a\mid q,v)-\log p(a\mid q,\emptyset)$, contrasted across candidate images so
that the gold answer's advantage is largest under the correct image. E5 implements it literally: for
an anchor $(q_i,a_i)$ with the same five candidates $v_k$ as Grounding and its blank role
$\emptyset_i$, $S_{ik}=\log p(a_i\mid q_i,v_k)-\log p(a_i\mid q_i,\emptyset_i)$ and
$\mathcal{L}_i=-\log(e^{S_{i0}}/\sum_k e^{S_{ik}})$, projector only, 300 steps from stock. The
empty-image term is constant along a row and softmax is shift invariant, so neither the loss nor its
gradient depends on it: the autograd gradient through $\log p(a_i\mid q_i,\emptyset_i)$ never exceeds
$\rv{M-diag-e5-empty-grad}$. The objective trains only which candidate image best supports the gold
answer, and it did (mean advantage on 64 train-owned anchors from
$\rv{M-diag-e5-mean_A_positive-before}$ to $\rv{M-diag-e5-mean_A_positive-after}$ nats; top-1 from
$\rv{M-diag-e5-candidate_top1_accuracy-before}$ to $\rv{M-diag-e5-candidate_top1_accuracy-after}$).
Behaviour moved little: I-MR $\rv{M-e5-U-imr}$ against $\rv{M-e0-U-imr}$ for stock, far from the
support-completion cells, which is the decoupling \cref{sec:consequence} predicts.

\subsection{Prompt-induced stress: a dev pilot}
\label{sec:stress}

\begin{figure}[t]
\centering
\includegraphics[width=\linewidth]{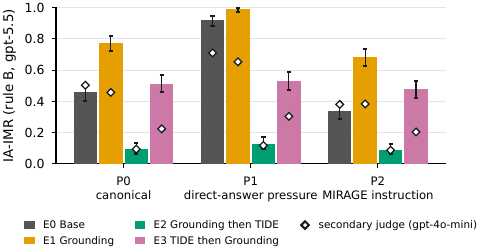}
\caption{\textbf{Stress pilot (measured, DEV, 100 groups): IA-IMR by prompt and cell.} Bars: primary
judge (gpt-5.5) with 95\% Wilson intervals; diamonds: secondary judge. A dev pilot, not a final-test
result. }
\label{fig:stress}
\end{figure}

\citet{asadi2026mirage} report that common evaluation instructions raise mirage rates. We ask
whether a one-line instruction changes what the trained cells do, with everything else fixed, under a
protocol frozen and hashed before any stress generation: \textbf{P0}, the canonical rendering;
\textbf{P1}, direct-answer pressure in our wording (``Directly answer the question based on the
image \ldots''); \textbf{P2}, the MIRAGE Figure S.1 instruction verbatim. The pilot covers the first
ten dev groups of each domain (100 groups, 400 rows per cell and prompt) for E0, E1, E2 and E3
(\cref{fig:stress}). Direct-answer pressure raises IA-IMR in every cell: stock from
$\rv{M-stress-e0-P0-gpt-5.5-iaimr}$ to $\rv{M-stress-e0-P1-gpt-5.5-iaimr}$, grounding then TIDE from
$\rv{M-stress-e3-P0-gpt-5.5-iaimr}$ to $\rv{M-stress-e3-P1-gpt-5.5-iaimr}$. The MIRAGE instruction
lowers it in every cell, the opposite direction to what \citet{asadi2026mirage} report for frontier
API models with the image removed; the settings differ enough (black image, user-turn prefix, 7B
open model) that this pilot neither reproduces nor refutes their finding. A stock model whose rate
moves from $\rv{M-stress-e0-P2-gpt-5.5-iaimr}$ to $\rv{M-stress-e0-P1-gpt-5.5-iaimr}$ with one
instruction has no stable disposition toward evidence; support completion has one under every
instruction tried. The TIDE-only cell was outside the pilot budget and comes first in the v5 stage.

\subsection{Planned studies}
\label{sec:plan}
The component ablations A0 to A7, leave-one-domain-out training, across-model evaluation, seeds,
thresholded selective-prediction baselines, HallusionBench~\cite{guan2024hallusionbench} under its
official grader, the representation probes and the human audits of judges and counterfactuals are
specified with budgets in the supplementary material; none has a result here.

%% file: sec/8_mechanism.tex
\section{Mechanistic Analysis}
\label{sec:mechanism}

\noindent\textbf{Where the decision lives.} Three completed observations are consistent with the
answer-or-decline decision living in the language model's mapping from representation to response
rather than in the visual representation; none of them localises it causally: the evidence-advantage contrast's empty-image term carries no gradient, so that
objective can only re-rank candidate images (\cref{sec:e5}); grounding after TIDE erases the decline
targets of unrelated images while leaving the blank decline, which needs no image content, intact
(\cref{sec:order}); and grounding alone, with a converged projector, makes the model commit on more
wrong images, not fewer. Representation changes still matter: they bound the residual (\cref{prop:rep}) and can destroy a
learned decision.

\noindent\textbf{The residual and false abstention across domains (H6).} \Cref{prop:rep} bounds a
weighted sum of the two decision errors from below by the representation's validity error; it does not
say whether, across domains, the two rise together or trade off. H6 is the empirical hypothesis that
they rise together because both track how separable a domain's valid and wrong images are. On the ten
domains of the v4.1 test the TIDE-only cell's in-domain mirage rate and its false-abstention rate have
a Spearman rank correlation of $\rv{M-e1-H6-spearman}$ (E2: $\rv{M-e3-H6-spearman}$), with pathology
and radiology highest on both and documents and scene text lowest (supplementary material); a post-hoc
analysis of ten points from 39 rows each, not a pre-registered test. The proposition's quantity is computed from the
stored responses with its own rule: decline iff the response is a trained decline string (exact, or
its opening clause; one row), answer otherwise; the test has 390 rows per role, so its prior is that of
$\Dc$. For TIDE only, $F=\rv{M-e1-dec-F}$ ($\rv{M-e1-dec-F.n}$ of 390 real rows),
$M=\rv{M-e1-dec-M}$ ($\rv{M-e1-dec-M.n}$ of 1{,}170 counterfactual rows) and
$\tfrac14F+\tfrac34M=\rv{M-e1-dec-sum}$, paired question-group bootstrap 95\% interval
$\rv{M-e1-dec-sum.ci}$ (exact matches only: $\rv{M-e1-dec-exact-sum}$; E2: $\rv{M-e3-dec-sum}$). In these two cells every response is a trained decline string or an answer,
with no hedge, other refusal or empty output, and the rule agrees row for row with the image-aware
judge's answered-or-abstained category (supplementary material). If
the response is a function of the question and the visual tokens, the $Z$ of \cref{prop:rep}, and the
test is a sample from the completed distribution, then by \cref{eq:bayes} $\rv{M-e1-dec-sum}$
estimates an upper bound on the Bayes error of deciding validity from the TIDE-only representation;
equality would need the class optimum, which we cannot verify. Cells that never or rarely emit a trained
decline string (stock, grounding, E3, E5) get a value above the $0.25$ of always declining, which
bounds nothing; the judge's substantive-answer probe, which sees the question, captions and role, is a
descriptive proxy only.

\noindent\textbf{Two judges, two events.} Judge U asks whether a response shows uncertainty about a
missing image; rule B asks whether it treats the shown input as the evidence. On stock-model rows that
acknowledge an unrelated image, judge U still calls 134 of 190 in-domain and 171 of 257 out-of-domain
rows mirages, so I-MR exceeds IA-IMR by about $0.33$ for the stock model and by at most $0.01$ for
the support-completion cells; both judges agree on every ranking.

\noindent\textbf{Relation to probing accounts.} The two mechanisms of \citet{benlevi2026mirageprobes},
internally spurious images and textual priors, both produce $B=\ans$ on an input outside $S(q)$, and
\cref{prop:starvation} says why the objective penalises neither. Their probes decode a response-level
mirage label from activations; \cref{prop:rep} concerns the detectability of input validity, a
different target, and neither implies the other without the matched probe experiment of the plan
(supplementary material); a decodable representation does not decline by itself, as the grounding
cell shows.

%% file: sec/9_limitations.tex
\section{Limitations and Conclusion}
\label{sec:limitations}

Every measured result is for one base model, Qwen2.5-VL-7B-Instruct, at one seed, under the
three-condition predecessor of the objective we define. A per-role rate on 390 rows carries a binomial standard error near 0.03, and no statistical test
was pre-registered. Rule B reads automatic captions and was validated in sample on 38 earlier-era rows,
judge U was not validated against humans, and the v4.1 test rows, never read by content, have been
judged and their aggregates used while writing this paper, so v5 classes them as previously evaluated
and adds a never-inspected fresh tier.
\Cref{prop:tide} concerns the population optimum and needs supports that are disjoint in meaning, not
only in content; \cref{prop:rep} bounds a weighted sum of string-level decision errors, not each residual and not a
judge-defined mirage rate; the token-mean corollary
is a bound per labelled token, not a rate.
TIDE covers the kinds of invalid input it trains on, the residual it leaves falls on the case the
introduction calls most dangerous, and a support-completed model is not thereby safe for any
deployment named there.

Models see mirages because the standard objective barely separates reading the image from ignoring it
and never rewards declining. Completing each question's support changes the optimum, provably, and
a finite model's answer-or-decline error is bounded by how well its representation separates valid
from invalid evidence. On one 7B model under a frozen protocol the predecessor of this objective took
blank-image mirage from $\rv{M-e0-U-blank}$ to $\rv{M-e1-U-blank}$ and the invalid-evidence rate from
$\rv{M-e0-U-imr}$ to $\rv{M-e1-U-imr}$, with the residual and its cost where the theory places them.
The four-condition objective, other base models and further seeds are the next measurements.

%% file: sec/X_suppl.tex
\clearpage
\setcounter{page}{1}
\maketitlesupplementary
\appendix

This supplementary material is self-contained. It records the protocol repair, the leakage-safe
tiers, every recipe, the two judges and their validation, the full tables behind every number in the
main paper, the stress pilot with both judge models, the training-time diagnostics, a qualitative
example from dev data, the canonical cell registry, the experiment registry with the planned
ablations, a methodological comparison with Mirage Probes, the full proofs and research notes behind
the theory, and the compute used. Every number comes from a protocol-hashed artifact of the v4.1 run
through the registry; no benchmark row is shown or described by content anywhere in this paper.

\section{Protocol repair}
\label{sec:supp-protocol}
An earlier pipeline, used in a prior draft of this work, carried six defects that a forensic audit
established and that invalidate every number produced with it (\cref{tab:supp-protocol}): options of
multiple-choice questions never reached the model, every image after the first was dropped, all but
the first gold string was discarded, counterfactual images were assigned before any split existed
(803 of 1{,}000 test groups had their real image inside training as a decline target), decoding was
capped at 64 tokens, and checkpoints were identified by file name. Protocol v4 replaced all of it with
one canonical record per example, one renderer, a documented target policy per domain, a
content-component split computed before counterfactual assignment, a recorded decoding cap and a
protocol record on every artifact; v4.1 added the 390-group final test packed from whole uninspected
components, fp32 master weights for projector training, fixed final steps and the cap of 3{,}072 after
a dev-only audit; v5 adds the fourth trained condition, the five-candidate grounding row, and the
three evaluation tiers of \cref{sec:supp-split}.

\begin{table*}[h]
\caption{The earlier pipeline against the frozen protocols. No number from the left column is reused.}
\label{tab:supp-protocol}
\centering\footnotesize
\input{tables/arxiv/supp_protocol}

\end{table*}

\section{Tiers, split and training set}
\label{sec:supp-split}
\Cref{fig:tiers} shows the construction and \cref{tab:supp-split} the counts and fingerprints. The
v5 training set was rebuilt from the frozen protocol by a read-only check that verifies, on the
realised 12{,}000 examples, that each group contributes exactly one example of each kind, that the two
decline strings are distinct and equal no gold string, that no counterfactual shares an image, by path
or by content, with a valid image of the same question string or with another kind, and that
$H(B)=\tfrac32\ln 2$ and $H(B\mid Q,V)=0$ exactly on the realised set. The check also records the
tokenizer facts behind \cref{cor:token}: the blank decline has 16 labelled tokens, the unrelated
decline 23, gold answers 3 to 23 (median 3), so $T_{\max}=23$ and $H(B)/T_{\max}=0.0452$ nats per
labelled token; the two decline strings share a nine-token prefix and diverge at the tenth token.

\begin{figure*}[h]
\centering
\includegraphics[width=\textwidth]{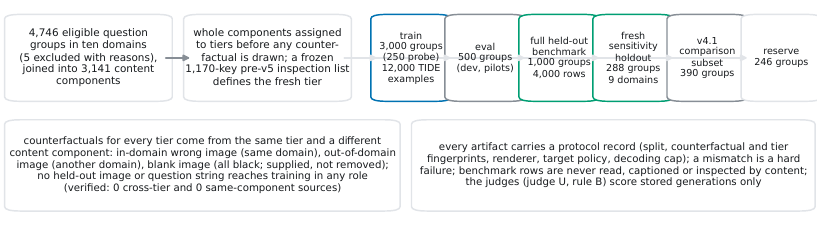}
\caption{\textbf{Leakage-safe construction of the protocol v5 tiers.} Whole content components are
assigned to tiers before any counterfactual is drawn; counterfactuals come from the same tier and a
different component.}
\label{fig:tiers}
\end{figure*}

\begin{table}[h]
\caption{Split, holdout and training-set facts for protocols v4.1 and v5.}
\label{tab:supp-split}
\centering\scriptsize
\input{tables/arxiv/supp_split}
\end{table}

\section{Text-sufficiency of the ten domains (five API models)}
\label{sec:supp-textsuff}
Before the protocol was built, five API models (GPT-5.5, Gemini-3.1-Flash-Lite, Claude-Sonnet-4.6, Llama-4-Scout
and Qwen2.5-VL-72B) answered the \rv{M-poc-bclean-removed-total.d} real-image questions of the benchmark twice through one neutral prompt
(options rendered for multiple choice; 128 output tokens, 512 for Gemini), once with every image attached and once
with the image removed. These are descriptive, exploratory observations from a pre-protocol pipeline: the image
was removed rather than replaced by a black image, the output cap truncated long answers, the judges are
response-only, and the models are instruction-tuned API systems whose refusal behaviour we do not control. They
say nothing about any trained cell, and the accuracy gain below is a model-specific 0-1 quantity with a chance
floor, not an estimator of the conditional mutual information of \cref{prop:starvation}. Question formats are taken
row by row: \rv{M-poc-fmt-n-mcq} multiple-choice questions (two to ten options), \rv{M-poc-fmt-n-yesno} yes/no
questions without options (most of them in the chest and pathology domains) and \rv{M-poc-fmt-n-open} other
open-ended questions.

\noindent\textbf{Records, categories and pairing.} There are \rv{M-poc-cov-blind} blind response records and
\rv{M-poc-cov-image} with-image records; \rv{M-poc-cov-image-missing} questions have no with-image response, every
one with logged request failures (\rv{M-poc-cov-image-missing-claude-sonnet-4.6-infovqa} Claude-Sonnet-4.6 InfoVQA
requests rejected by the provider, \rv{M-poc-cov-image-missing-gpt-5.5-pathvqa} GPT-5.5 PathVQA requests without a
completion), and the scored file holds \rv{M-poc-cov-image-csv} rows because it predates a resumed run. Each
blind response is one of three categories with every evaluated question as the denominator: a \emph{substantive
answer} (non-empty, and a gpt-4o-mini uncertainty judge found no acknowledgement of the missing image), an
\emph{explicit decline} (non-empty and acknowledged), or \emph{empty}. Empty output is never counted as abstention.
GPT-5.5 returned an empty blind response in a fraction $\rv{M-poc-empty-gpt55}$ of cases and an empty with-image
response in $\rv{M-poc-empty-image-gpt-5.5}$; Llama-4-Scout returned $\rv{M-poc-empty-image-llama-4-scout}$ empty
with-image responses; the other models fewer than 1\%. Correctness is exact match for multiple choice and a
gpt-4o-mini judge against every gold string for open-ended questions; the judge was not run on empty responses
(\rv{M-poc-cov-blind-unjudged} blind and \rv{M-poc-cov-image-unjudged} with-image rows), and no missing judgment
is counted as incorrect. The accuracy gain is computed on \emph{paired complete cases}, questions whose blind and
with-image responses are both present, non-empty and judged (\rv{M-poc-cov-paired} of \rv{M-poc-cov-blind}; only
\rv{M-poc-cov-paired-gpt-5.5} for GPT-5.5, so its gains rest on a small and non-random subset); the inherited rule that counts empty
responses as incorrect is kept as a sensitivity.

\noindent\textbf{Question-only correctness (\cref{tab:textsuff}, \cref{fig:textsuff}a).} The B-Clean rule of
\citet{asadi2026mirage} flags a question when at least one evaluated model answers it correctly without the image.
At least one of the five models produced a correct answer without receiving the image on
\rv{M-poc-bclean-removed-total.n} of \rv{M-poc-bclean-removed-total.d} questions (a fraction
$\rv{M-poc-bclean-removed-total}$; leaving out any one model, $\rv{M-poc-bclean-lomo-min}$ to
$\rv{M-poc-bclean-lomo-max}$). This is an operational count, not a measure of intrinsic text-sufficiency, because a
guess can satisfy it and the union over five models compounds the chance. By format, the rule flags
$\rv{M-poc-bclean-removed-fmt-mcq}$ of the multiple-choice questions, where independent uniform guessing by five
models would flag $\rv{M-poc-union5-fmt-mcq}$; $\rv{M-poc-bclean-removed-fmt-yesno}$ of the yes/no questions, against
$\rv{M-poc-union5-fmt-yesno}$ under guessing; and $\rv{M-poc-bclean-removed-fmt-open}$ of the other open-ended questions, which have no chance level.
The yes/no and multiple-choice counts therefore say little; the open-ended count and the model-mean blind accuracy
on paired cases, listed per domain, are the better-defined quantities. The accuracy gained from the image on paired
cases ranges from $\rv{M-poc-gain-paired-medxpertqa_mm}$ (MedXpertQA-MM) to $\rv{M-poc-gain-paired-docvqa}$
(DocVQA).

\noindent\textbf{Blind answering against accuracy gain (\cref{tab:textsuffassoc}, \cref{fig:textsuff}b).} The
primary analysis pairs the substantive blind-answer rate (all questions) with the paired complete-case gain for
each of the \rv{M-poc-corr-full50.d} (model, domain) pairs: Pearson $\rv{M-poc-corr-full50}$ ($\rv{M-poc-corr-full40}$ without GPT-5.5,
whose substantive rate is lowered by its empty responses). The pairs are five models crossed with ten domains, not
independent samples, so no $p$-value is reported; a cluster bootstrap that resamples domains gives a 95\% interval of
$[\rv{M-poc-boot-domains-lo}, \rv{M-poc-boot-domains-hi}]$ for the pooled $r$, and one that resamples models
$[\rv{M-poc-boot-models-lo}, \rv{M-poc-boot-models-hi}]$. The association is a contrast between question formats:
within cells of one format with at least 30 questions it is $\rv{M-poc-corr-fmt-mcq}$ for multiple choice
(\rv{M-poc-corr-fmt-mcq.d} cells), $\rv{M-poc-corr-fmt-yesno}$ for yes/no (\rv{M-poc-corr-fmt-yesno.d} cells, a small
stratum) and $\rv{M-poc-corr-fmt-open}$ for the other open-ended questions (\rv{M-poc-corr-fmt-open.d} cells), while over the ten domain means it is $\rv{M-poc-corr-domainmeans10}$. The
reading is that the document and scene-text domains, where the image carries most of the answer, are also the
domains where these models most often say they need the image, and the knowledge-heavy multiple-choice domains are
answered from the question; within a format the two quantities do not track each other. Within models the Spearman
correlation across domains ranges from $\rv{M-poc-spearman-full-claude-sonnet-4.6}$ (Claude) to
$\rv{M-poc-spearman-full-gpt-5.5}$ (GPT-5.5). The rate conditional on non-empty responses, kept as a supplementary
diagnostic, gives $\rv{M-poc-corr-full50-nonempty}$, and counting empty responses as incorrect in the gain gives
$\rv{M-poc-corr-full50-emptyincorrect}$.

\noindent\textbf{Sensitivity: the B-Clean subset.} An earlier draft computed the association on the
\rv{M-poc-bclean-clean-n} questions that no model answered correctly blind. That subset is selection-conditioned: blind accuracy is zero on
it by construction, so its accuracy gain equals its with-image accuracy, and the blind-answer rate is measured on
questions chosen by the models' own blind behaviour. On that subset the pooled correlation is
$\rv{M-poc-corr-sens50}$ ($\rv{M-poc-corr-sens40}$ without GPT-5.5). The earlier draft's figure further excluded
five domains and GPT-5.5 after the fact, giving $\rv{M-poc-corr-subset20}$ on \rv{M-poc-corr-subset20.d} pairs; that figure is not used.

\begin{table*}[h]
\caption{Text-sufficiency of the ten domains, five API models, image removed (2026-07 pipeline, all questions).
Row-level formats (MC multiple choice, yes/no without options, OE other open-ended); model-mean composition of the
blind responses with every evaluated question as the denominator; model-mean blind and with-image accuracy and
their difference on paired complete cases; the number and fraction of questions on which at least one model was
correct without the image, with the fraction the same rule would flag if the questions with a chance level were
guessed uniformly (guessing ref.) and, for the total, the leave-one-model-out range. Descriptive; not a protocol
result.}
\label{tab:textsuff}
\centering\scriptsize\setlength{\tabcolsep}{4pt}
\input{tables/arxiv/tab_textsuff}
\end{table*}

\begin{table}[h]
\caption{Exploratory associations between the substantive blind-answer rate and the paired accuracy gain across
(model, domain) pairs or (model, domain, format) cells. Observations are clustered by model and domain; no
hypothesis test is claimed.}
\label{tab:textsuffassoc}
\centering\scriptsize\setlength{\tabcolsep}{3pt}
\input{tables/arxiv/tab_textsuff_assoc}
\end{table}

\begin{figure*}[h]
\centering
\includegraphics[width=\textwidth]{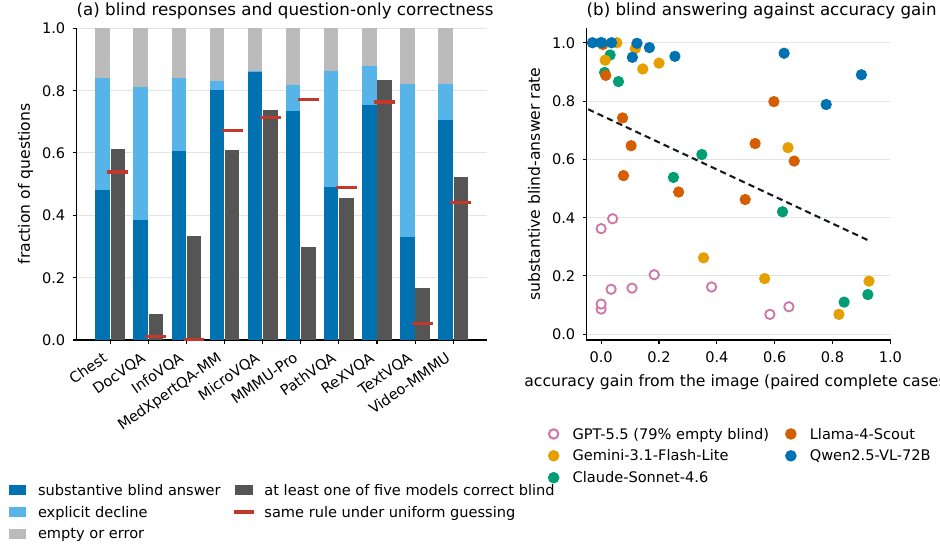}
\caption{\textbf{Text-sufficiency of the ten domains (five API models, image removed; descriptive).} (a) Per
domain, model mean: the blind responses split into substantive answers, explicit declines and empty output, next
to the fraction of questions on which at least one model was correct without the image and the fraction the same
rule would flag if the questions with a chance level were guessed uniformly. (b) Substantive blind-answer rate
against the accuracy gain from the image on paired complete cases, one point per (model, domain) pair; the dashed
line is the pooled fit over all \rv{M-poc-corr-full50.d} pairs (Pearson $\rv{M-poc-corr-full50}$); GPT-5.5 is hollow
because a fraction $\rv{M-poc-empty-gpt55}$ of its blind responses are empty.}
\label{fig:textsuff}
\end{figure*}

\section{Training and generation settings}
\label{sec:supp-hparams}
\Cref{fig:method} shows where each stage acts and how the cells are composed; \cref{tab:supp-hparams}
lists every recipe. All projector training uses fp32 master weights with a bf16 forward; a
CPU-versus-GPU comparison of the master updates found only bf16 kernel nondeterminism.

\begin{figure}[h]
\centering
\includegraphics[width=\linewidth]{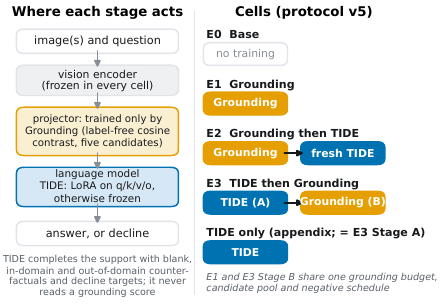}
\caption{\textbf{Where each stage acts, and how the cells are composed.} TIDE trains LoRA adapters on
the language model's attention projections on the completed support and never reads a grounding
score; Grounding trains only the vision projector with a label-free cosine contrast and never reads
an answer. E1 and E3 Stage B share one grounding budget, candidate pool and negative schedule, so E2
and E3 differ in order only.}
\label{fig:method}
\end{figure}

\begin{table*}[h]
\caption{Recipes of every cell, read from the stored protocol records.}
\label{tab:supp-hparams}
\centering\footnotesize
\input{tables/arxiv/supp_hparams}
\end{table*}

\section{Metric definitions}
\label{sec:supp-metrics}
Let $M_c$ be the number of rows of invalid role $c$ that a judge calls a mirage and $N_c$ the number of
rows of that role; $N_\mathrm{real}$ is the number of real rows. With judge U in the numerator,
I-MR and TC-LIA MR are as in \cref{eq:metrics}; with rule B in the numerator the same two quotients
are IA-IMR and IA-FMR. Real rows contribute no mirage and enter only the full-set denominators. On a
balanced test the full-set rates are $0.75$ times the invalid-row rates by composition, and the
count-based I-MR equals the mean of the three per-role rates; the count-based form is the definition.
On real rows, the answered rate counts answered and hedged verdicts, false abstention is its
complement, accuracy is the fraction of all real rows whose answer matches any gold string, and
accuracy conditional on answering divides by the answered rows. TC-LIA MR shares the denominator of
the detector mirage rate of \citet{chowdhury2026tclia}, the fraction of all evaluation inputs on which a
non-related input is wrongly allowed, but its numerator is a judgement of a generated response rather
than a detector decision; it is the metric to place beside theirs, and it does not reproduce their
experiment.
Per-rate uncertainty is a 95\% Wilson interval from the stored counts; contrasts carry a paired
bootstrap over question groups (2{,}000 resamples, groups drawn with replacement, all four rows of a
group moving together), which respects the dependence between the roles of one question.

\section{The two judges}
\label{sec:supp-judge}
\noindent\textbf{Judge U.} The uncertainty-judge template of \citet{asadi2026mirage} (their Phantom-0
judge), applied with gpt-4o-mini at temperature 0 to the response text alone, with no question,
caption or image; a row is a mirage when the judge finds no uncertainty about an absent or invalid
image. The template was written for a removed-image setting and asks whether the response says it
cannot see an image, that no image was provided, or that it needs an image; on unrelated-image rows it
often does not credit an explicit statement that the shown image is unrelated. On stock-model rows
that rule B's acknowledgement probe marks as such, judge U still calls 134 of 190 in-domain rows
(70.5\%), 171 of 257 out-of-domain rows (66.5\%) and 33 of 136 blank rows (24.3\%) mirages; rows that
merely describe the shown image without acknowledgement are mirages under judge U as intended (98 of
110 in-domain, 39 of 44 out-of-domain). Judge U was not validated against human labels.

\noindent\textbf{Rule B.} For each response the image-aware judge receives the rendered question with
visible options, a caption of every image the model actually saw, in order, the role's case text (a
blank row states that the image was blank), and the response, and answers four probes in one call:
P1, does the response give a substantive answer; P2, does it state that the image is blank, missing,
unrelated, or that it cannot answer from this image; P3, is it about what the image actually shows
(unrelated-image rows only); P4, does it describe visual content of the question's subject as if it
were present. Rule B: an invalid-evidence row is a mirage unless P2 holds or, on an unrelated image,
P3 holds; P4 is always a mirage. A reply that does not parse is recorded as unparseable, enters the
denominator and is never a mirage; the final test has none after four re-issued calls, and the stress
pilot has 3 of 4{,}800 primary verdicts. Real rows receive answered, hedged or abstained, never mirage;
accuracy comes from a correctness probe of the same run against every gold string. Captions come
from Qwen3-VL-32B-Instruct: 1{,}311 automatic and 69 verified captions cover the 1{,}380 image slots of
the final test, none missing. \Cref{tab:supp-judge} records the validation on 38 human-audited
earlier-era rows (21 in-domain, 9 out-of-domain, 8 blank; 5 mirages); the gpt-5.5 figures are in
sample because the rule was adjusted on these rows. This record is insufficient as evidence of
validity, which is why the study of \cref{sec:supp-humanjudge} is designed.

\begin{table}[h]
\caption{Rule-B judge models on the 38 audited rows (agreement and Cohen's $\kappa$ on non-real rows,
precision and recall of the mirage label). In sample for gpt-5.5.}
\label{tab:supp-judge}
\centering\footnotesize
\input{tables/arxiv/supp_judge_record}
\end{table}

\subsection{Human validation of the judges (design)}
\label{sec:supp-humanjudge}
\begin{figure}[h]
\centering
\includegraphics[width=\linewidth]{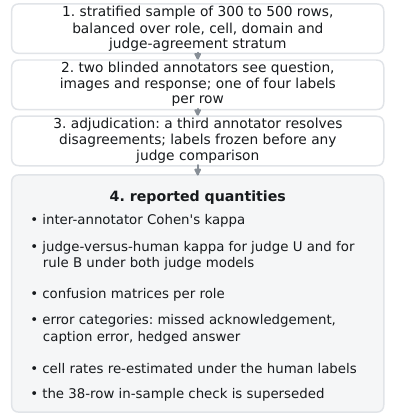}
\caption{\textbf{Design of the blinded human validation of both judges.} Not run; no number exists.}
\label{fig:judgedesign}
\end{figure}
\Cref{fig:judgedesign} shows the design. A stratified sample of 300 to 500 rows is drawn from the v5
eval tier, balanced over the three invalid roles, the five cells, the ten domains and the four
judge-agreement strata (U and B agree on mirage, agree on not mirage, U-only mirage, B-only mirage),
so that disagreements are over-represented. Two annotators, blind to cell and judge verdicts, see the
rendered question, every image the model saw and the response, and assign one of four labels (mirage;
acknowledges the image problem; about the shown image; unparseable or empty); a third adjudicates
disagreements and the labels are frozen before any judge output is consulted. Reported quantities:
inter-annotator Cohen's $\kappa$; judge-versus-human $\kappa$ for judge U and for rule B with both
judge models; confusion matrices per role; error categories (missed acknowledgement, caption error,
hedged answer, option-letter-only responses); and a re-estimate of each cell's rate under the human
labels for the audited strata. The sample size gives a standard error of about 0.04 on a $\kappa$
near 0.7.

\section{Full results}
\label{sec:supp-full}
\Cref{tab:supp-u} gives the judge-U counts and both rates for every cell; \cref{tab:supp-ruleb} every
rule-B quantity under both judge models; \cref{tab:supp-probes} the probe yes-rates per role;
\cref{tab:supp-decision} the string-level decision statistic of \cref{prop:rep} beside the judge
proxy; \cref{tab:contrasts} the paired group-level contrasts; \cref{tab:domains} and
\cref{fig:domains} the domain breakdown behind H6.

\begin{table*}[h]
\caption{The decision statistic of \cref{prop:rep}, v4.1 final test. \emph{decline}: the response equals
one of the two trained v4.1 decline strings or begins with their shared clause ``I can't answer this
question'' (\emph{near}); \emph{answer}: anything else, including refusals in other words (\emph{refusal
wording}: a fixed list of refusal phrases, counted for description only and entering no rate), hedges and
empty outputs. Hedged replies (rule B's P1 and P2 both yes) are answers under the rule; none occurs in the
support-completion cells. $F$ is the decline count on the 390 real rows over 390, $M$ the answer count
on the 1{,}170 counterfactual rows over 1{,}170; the interval is a paired question-group bootstrap
(2{,}000 resamples of the 390 groups, the four roles of a group kept together, percentile); \emph{exact
only} counts the near matches as answers. Counts come from exact string operations on the stored
generation files; no row was read by a person. \emph{Judge proxy}: $\tfrac14(1-\text{P1}_\text{real})
+\tfrac34\,\overline{\text{P1}}_\text{cf.}$ from rule B (gpt-5.5), whose judge sees the question, the
captions and the role; descriptive only. \emph{Agree}: real rows on which the string decline and the
judge's ``abstained'' coincide, and counterfactual rows declined under the string rule that rule B also
calls not a mirage.}
\label{tab:supp-decision}
\centering\scriptsize\setlength{\tabcolsep}{1.9pt}
\input{tables/arxiv/supp_decision}
\end{table*}

\begin{table}[h]
\caption{Judge U, v4.1 final test: mirage counts per invalid role (390 rows each), their sum over the
1{,}170 invalid-evidence rows, I-MR and TC-LIA MR (denominator 1{,}560).}
\label{tab:supp-u}
\centering\scriptsize\setlength{\tabcolsep}{2.5pt}
\input{tables/arxiv/supp_u}
\end{table}

\begin{table*}[h]
\caption{Rule B, v4.1 final test, both judge models: mirage rate per invalid role, IA-IMR (1{,}170
rows), IA-FMR (1{,}560 rows), real-image accuracy, answered rate and false abstention (390 rows), and
unparseable replies.}
\label{tab:supp-ruleb}
\centering\scriptsize\setlength{\tabcolsep}{4pt}
\input{tables/arxiv/supp_ruleb}
\end{table*}

\begin{table*}[h]
\caption{Probe yes-rates per role, rule B, gpt-5.5. P3 is not asked on blank rows.}
\label{tab:supp-probes}
\centering\scriptsize\setlength{\tabcolsep}{3pt}
\input{tables/arxiv/supp_probes}
\end{table*}

\begin{table*}[h]
\caption{Paired group-level bootstrap contrasts (2{,}000 resamples over the 390 question groups) for
the differences the main text reports.}
\label{tab:contrasts}
\centering\scriptsize\setlength{\tabcolsep}{3pt}
\input{tables/arxiv/tab_contrasts}
\end{table*}

\begin{table*}[h]
\caption{Domain breakdown, v4.1 final test (117 invalid rows and 39 real rows per domain): pooled
mirage rates under both judges, the in-domain rate and the false-abstention rate of the two
support-completion cells. Spearman rank correlation across domains between in-domain mirage and false
abstention: TIDE only $\rv{M-e1-H6-spearman}$, E2 $\rv{M-e3-H6-spearman}$.}
\label{tab:domains}
\centering\scriptsize\setlength{\tabcolsep}{3pt}
\input{tables/arxiv/tab_domains}
\end{table*}

\begin{figure}[h]
\centering
\includegraphics[width=\linewidth]{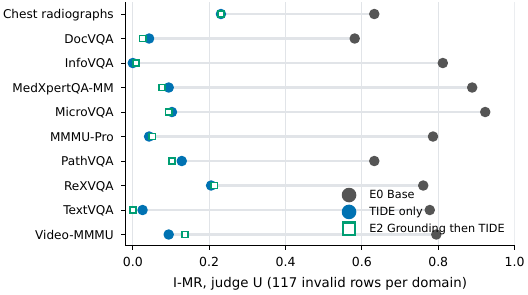}
\caption{\textbf{I-MR per domain (measured, v4.1).} Stock model, TIDE only and Grounding then TIDE.}
\label{fig:domains}
\end{figure}

\section{Stress pilot, both judge models}
\label{sec:supp-stress}
The three conditions, verbatim and complete: P0 is the canonical rendering with no added text; P1
prepends \emph{``Directly answer the question based on the image. If options are provided, respond
with the option letter only. Do not provide an explanation.''} to the user turn followed by one blank
line; P2 prepends \emph{``You are a helpful vision-language assistant. Base your statements on the
visual evidence in the data.''} in the same way (MIRAGE supplementary material, Figure S.1). No other
wording and no other placement is part of the protocol; the protocol file's frozen block is hashed
(sha256 \texttt{955d501c1690d1d3\allowbreak cd4c9b15a54b8a4d\allowbreak b97ed6aeaf8ce2ac\allowbreak
8323973c8865459f}). The pilot subset is the first ten dev groups of each domain in the frozen dev
order (100 groups, fingerprint \texttt{5f928385bde4}), disjoint from every test tier.
\Cref{tab:supp-stress} gives every pre-registered quantity under both judge models. A stock rate above
0.70 under some prompt was recorded in the protocol as a non-criterion: reported, not selected for.
The v5 stress stage must include the TIDE-only cell and the eval tier before any benchmark stage.

\begin{table*}[h]
\caption{Mirage Stress Test V1, DEV pilot, 100 groups: every pre-registered quantity, both judge
models. Preliminary; not a benchmark result.}
\label{tab:supp-stress}
\centering\scriptsize\setlength{\tabcolsep}{3pt}
\input{tables/arxiv/supp_stress_full}
\end{table*}

\section{Training-time diagnostics}
\label{sec:supp-diag}
\Cref{tab:supp-diag} lists the train-owned diagnostics; they measure retention or convergence on
training-owned keys and are not held-out results. \Cref{fig:e5} shows the E5 training loss and the
before-and-after probe rates.

\begin{table}[h]
\caption{Training-time diagnostics (train-owned keys; not held-out results).}
\label{tab:supp-diag}
\centering\footnotesize
\input{tables/arxiv/supp_diag}
\end{table}

\begin{figure}[h]
\centering
\includegraphics[width=\linewidth]{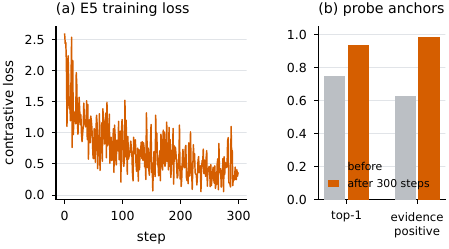}
\caption{E5 on train-owned probe anchors: per-step contrastive loss (left) and the candidate top-1
and evidence-positive rates before and after 300 steps (right).}
\label{fig:e5}
\end{figure}

\section{A qualitative example from dev data}
\label{sec:supp-qual}
\Cref{tab:supp-qual} shows one dev group from the stress pilot's canonical-prompt cells, selected by a
fixed rule: the first group, in key order, on which the stock model is a rule-B mirage on both the
blank and the in-domain image while the grounding-then-TIDE model is not on either and answers the
real image. The question asks about hyperintensities surrounding the aorta in a CT image. On the blank
image the stock model describes the image it was not given; on the in-domain wrong image it describes
that image as if it were the one asked about; on the out-of-domain image (a whiteboard) it
acknowledges the mismatch. The trained model answers the real image and emits its two decline strings
on the counterfactuals. Outputs are verbatim apart from truncation, dashes normalised to a hyphen and
the removal of all other non-ASCII characters. No benchmark row is shown.

\begin{table*}[h]
\caption{Outputs on one dev group (stress pilot, P0), with the rule-B verdict (gpt-5.5).}
\label{tab:supp-qual}
\centering\footnotesize
\input{tables/arxiv/supp_qual}

\end{table*}

\section{Canonical cell registry}
\label{sec:supp-cells}
\Cref{tab:cells} is the registry of every cell, with the internal tag that names its artifacts and
the historical aliases that earlier drafts used and this paper never does. The historical ``E1'' of
the v4.1 draft was the three-condition TIDE-only cell; the current E1 is Grounding.

\begin{table*}[h]
\caption{Canonical cell registry (from \texttt{registry/cells.json}).}
\label{tab:cells}
\centering\scriptsize\setlength{\tabcolsep}{3pt}
\input{tables/arxiv/tab_cells}
\end{table*}

\section{Experiment registry and ablation plan}
\label{sec:supp-plan}
\Cref{tab:experiments} lists every experiment with its priority and status on 2026-10-10;
\cref{tab:abldesign} the component ablations. Priorities: P0, essential to the paper's validity
(core ablations, competitive abstention baselines, across-model evaluation, three seeds, judge
validation, valid-image utility, counterfactual validity audit); P1, mechanism and generalisation
(counterfactual hardness, negative composition, decline wording, leave-one-domain-out, visual-swap
sensitivity, evidence margin, representation separability, training-order forgetting, prompt stress
including the TIDE-only cell, actual image removal versus the black image); P2, sensitivity (LoRA
rank, learning rate, steps, training-set size, modules, token-loss weighting, image corruption,
paraphrases, counterfactual ratios, joint variants). The staged plan runs P0 on Qwen2.5-VL-7B first
(eight ablation cells at 800 steps each, three seeds of the four trainable principal cells, the
two baselines and the human audit), then the two further base models, then P1 on the single model, and
P2 only on the TIDE-only cell. Every planned cell is evaluated on the full held-out benchmark and the
fresh holdout with the same two judges and reported with the metrics of \cref{sec:supp-metrics};
comparisons between cells use the paired group bootstrap; selective-prediction baselines are reported
as risk-coverage curves over the stock model's confidence.

\begin{table*}[h]
\caption{Experiment registry (from \texttt{registry/experiments.json}).}
\label{tab:experiments}
\centering\scriptsize\setlength{\tabcolsep}{3pt}
\input{tables/arxiv/tab_experiments}
\end{table*}

\begin{table*}[h]
\caption{Component ablations A0 to A7: composition and hypothesis. A4 is the TIDE-only cell; A0 is E0.
Matched optimizer steps, LoRA budget and decline strings; separate from the training-order grid.}
\label{tab:abldesign}
\centering\scriptsize\setlength{\tabcolsep}{3pt}
\input{tables/arxiv/tab_abl_design}
\end{table*}

\section{Comparison with Mirage Probes}
\label{sec:supp-probes}
\Cref{tab:probescmp} compares this paper with \citet{benlevi2026mirageprobes} row by row; every entry
in the right column was checked against their arXiv v1 text.

\begin{table*}[h]
\caption{Methodological comparison with Mirage Probes (arXiv:2606.13870v1).}
\label{tab:probescmp}
\centering\scriptsize
\begin{tabular}{@{}p{0.16\textwidth}p{0.4\textwidth}p{0.4\textwidth}@{}}
\toprule
aspect & this paper & Mirage Probes \\
\midrule
goal & remove mirage by changing the training objective; prove what the optimum does & detect mirage responses from internal activations; characterise mechanisms \\
mirage label & a judge's verdict on the generated response given the role: response-only (judge U) and image-aware (rule B); invalid inputs are constructed (blank, in-domain wrong, out-of-domain) & consistency between the with-image and the without-image generation (bag-of-words cosine above 0.7), uncertainty detected by regular expressions; no wrong-image condition \\
models & Qwen2.5-VL-7B measured; LLaVA-OneVision-7B and InternVL3.5-8B in progress & Ovis2.5-2B and Qwen3-VL-32B; GLM-4.6-flash dropped for too few non-mirage responses \\
data & ten domains, frozen leakage-safe tiers, 390 groups measured (1{,}560 rows) & VQA-RAD, MMMU-Pro, MedXpertQA; MicroVQA dropped for too few non-mirage responses \\
intervention & TIDE (support completion) and the grounding and contrast baselines; held-out behavioural evaluation & none evaluated; discussion of text-distribution cleaning and representational interventions \\
mechanistic claim & the decision lives in the mapping from representation to response (order grid, gradient cancellation); the residual is bounded by validity separability (Prop.~3) & two mechanisms, spurious internal images and textual biases, stated as a hypothesis \\
reported accuracies & mirage and utility rates with exact denominators and group-level intervals & probe accuracy: linear 57 to 75\%, MLP 60 to 76\%, difference-of-activation 71 to 98\%, naive Bayes 45 to 82\% \\
stated limitations & one model, one seed, three-condition predecessor, judges not human-validated & heuristic labels, correlational findings, with-image mirages not causally validated \\
relation & complementary; different probe targets (their response-level mirage label against the input-validity detectability of Prop.~3), related only by an experiment not yet run & \\
\bottomrule
\end{tabular}
\end{table*}

\section{Theory: full proofs and research notes}
\label{sec:supp-theory}

\noindent\textbf{Proof of \cref{prop:tide}.} (i) By \cref{prop:starvation} applied to $\Dc$,
$\Delta(\Dc)=I_{\Dc}(A;V\mid Q)$. The class $B=g(A)$ is a function of the response because the two
decline strings are distinct from each other and from every valid answer. By the chain rule,
$I(A;V\mid Q)=I(A,B;V\mid Q)=I(B;V\mid Q)+I(A;V\mid Q,B)$. The kind $K$ is a function of $(Q,V)$
because the four supports are disjoint, and $B$ is a function of $K$, so $H(B\mid Q,V)=0$ and
$I(B;V\mid Q)=H(B\mid Q)-H(B\mid Q,V)=H(B\mid Q)=H(B)$, the last step because $K$, hence $B$, is
independent of $Q$. For the second term, $I(A;V\mid Q,B)=\sum_b\Pr(B=b)\,I(A;V\mid Q,B=b)$; given a
decline class the response is a constant, so those terms vanish, and given $B=\ans$ the triple
$(Q,V,A)$ has law $D$, contributing $\tfrac14 I_D(A;V\mid Q)$. With masses $(\tfrac14,\tfrac14,\tfrac12)$,
$H(B)=\tfrac12\ln 4+\tfrac12\ln 2=\tfrac32\ln 2$. (ii) The unrestricted minimiser of a log-loss is the
conditional $\Dc(A\mid Q,V)$, unique on the support of $\Dc$. On an input $(q,v)$ with $v$ in the
support of $C_k(\cdot\mid q)$ only the branch $K=k$ has mass, and its response is $\dec{k}$; on an
input with $v\in S(q)$ only the real branch has mass, so the conditional is $D(\cdot\mid q,v)$.
\qed

\noindent\textbf{General weights and targets.} If the kinds are drawn with weights $w$ instead of
uniformly, and the decline strings induce classes with masses $\pi$, the same proof gives
$\Delta=H(\pi)+w_{\mathrm{real}}I_D(A;V\mid Q)$ and the same optimum decisions. Three distinct decline
strings give $H(\pi)=\ln 4$; one shared string gives $h(\tfrac14)=0.562$ nats; the realised choice
(shared string for the two unrelated kinds) gives $\tfrac32\ln 2$. None of these constants is a
per-token quantity, and none should be compared with $\ln 3$, which was the constant of the
three-condition predecessor with two distinct strings.

\noindent\textbf{Proof of \cref{cor:token}.} The per-example token-mean loss equals the sequence
log-loss divided by the number $T(a)$ of labelled tokens, so the objective is
$R_w(p)=\mathbb{E}_{\Dc}[w(A)(-\log p(A\mid Q,V))]$ with $w=1/T$. Let $W(q,v)=\mathbb{E}[w(A)\mid q,v]$,
$\bar w=\mathbb{E}[w(A)]$ and $\tilde D(q,v,a)=\Dc(q,v,a)w(a)/\bar w$. Then
$R_w(p)=\bar w\,\mathbb{E}_{\tilde D}[-\log p(A\mid Q,V)]=\bar w R_{\tilde D}(p)$, so by
\cref{prop:starvation} the image-blind gap is $\bar w I_{\tilde D}(A;V\mid Q)$, and the unrestricted
minimiser is $\tilde D(a\mid q,v)\propto D_4(a\mid q,v)/T(a)$: a point mass on every counterfactual
input and the length-tilted conditional on valid inputs. For the lower bound, write
$R_w(p)=R_w(p_w)+\sum_{q,v}\Dc(q,v)W(q,v)\,\mathrm{KL}(\tilde D(\cdot\mid q,v)\|p(\cdot\mid q,v))$ for the
tilted optimum $p_w$. For an image-blind $r$, $W(q,v)\ge 1/T_{\max}$ gives
$R_w(r)-R_w(p_w)\ge T_{\max}^{-1}\mathbb{E}_q\sum_v\Dc(v\mid q)\mathrm{KL}(\tilde D_{q,v}\|r_q)$, and by the
compensation identity $\sum_v\lambda_v\mathrm{KL}(P_v\|r)=\sum_v\lambda_v\mathrm{KL}(P_v\|\bar P)+\mathrm{KL}(\bar P\|r)$
the inner sum is at least the mutual information between $V$ and $A$ under the joint
$\Dc(v\mid q)\tilde D(a\mid q,v)$. Under that joint the kind is still a function of $v$ and
$B$ a function of $a$, so by data processing the mutual information is at least $H(B\mid q)=H(B)$
(the tilt acts on $A$ given $(q,v)$ and leaves the kind weights uniform). Hence the gap is at least
$H(B)/T_{\max}$. \qed

\noindent\textbf{Proof of \cref{prop:rep}.} (i) For $p=r\circ\phi$,
$R(p)=H(A\mid Z)+\mathbb{E}_Z\mathrm{KL}(\Pr(A\mid Z)\|r(\cdot\mid Z))$, so the class minimum is
$H(A\mid Z)$ and the excess over $H(A\mid Q,V)$ is $I(A;Q,V\mid Z)$; since $B=g(A)$,
$I(A;Q,V\mid Z)\ge I(B;Q,V\mid Z)=H(B\mid Z)-H(B\mid Q,V,Z)=H(B\mid Z)$. (ii) Any predictor in
$\mathcal P_\phi$ with any decoding rule induces a (possibly randomised) decision $d(Z)\in\{\ans,
\mathsf{decline}\}$; $\tfrac14F+\tfrac34M=\Pr(d=\mathsf{decline},K=\mathrm{real})+\Pr(d=\ans,K\ne\mathrm{real})$
is the probability of a wrong decision, which for every rule based on $Z$ is at least
$\mathbb{E}[\min(\pi(Z),1-\pi(Z))]$, with equality for the rule that answers iff $\pi(Z)>\tfrac12$.
(iii) The class optimum $r^\star(a\mid z)=\Pr(A=a\mid Z=z)$ puts mass $1-\pi(z)$ on the decline
strings, so sampling from it gives $\Pr(\mathsf{decline},K=\mathrm{real})=\mathbb{E}[\pi(Z)(1-\pi(Z))]$
and $\Pr(\ans,K\ne\mathrm{real})=\mathbb{E}[(1-\pi(Z))\pi(Z)]$. Both sides vanish iff
$\pi(Z)\in\{0,1\}$ almost surely. \qed

\noindent\textbf{Research notes: derivations that were tried and not used.} (a) A Fano-type bound,
$H(B\mid Z)\le h(e)+e\ln 2$, relates (i) and (ii) but is loose at the error levels observed and is not
stated. (b) Under greedy decoding over tokens the decision is neither the class-MAP rule nor
posterior sampling: the two decline strings share a nine-token prefix, so their mass is pooled for
the first nine tokens while the mass of competing gold answers is split across strings; a model can
therefore decline under greedy decoding even when $\pi(z)>\tfrac12$. No clean identity for the ratio
$F/M$ under greedy decoding was found, and none is claimed; (iii) is stated for posterior sampling
only. (c) The identity in (iii) holds for posterior sampling; the cells are decoded greedily, so it is not
expected to hold for them and is not tested. The decision rates of \cref{sec:mechanism} use the
proposition's own rule on the stored responses (decline iff the response is a trained decline string,
exact or by its opening clause); a refusal in other words, a hedge and an empty output are answers under
that rule by definition, not by judgement. Applying (ii) to them needs three things: that the response
is a function of $Z$ (true for $Z$ = the question with the visual tokens, from which the language
model decodes), that the test is a sample from $\Dc$ with the test's counterfactual kernels (its roles
are balanced, so the prior matches), and that sampling error is carried (the paired question-group bootstrap interval of
\cref{tab:supp-decision}, which keeps the four roles of a group together; the binomial standard error
of about $0.008$ that treats rows as independent is not used). The result is an estimate of an upper bound on
$e(\phi)$ and nothing further. The image-aware judge's substantive-answer probe is a different quantity:
the judge sees the question, the captions and the role, so its decision does not factor through $Z$ and
inherits no bound; it coincides with the string rule row for row on the two support-completion cells
(\cref{tab:supp-decision}) and not on the others, and it is reported as a descriptive proxy only. (d) Replacing disjoint supports by overlapping ones (an
ambiguous question answerable from many images) turns $H(B\mid Q,V)$ positive and the bound in (i)
into $H(B)-H(B\mid Q,V)$; we did not attempt to estimate semantic overlap and treat it as a
limitation. (e) A per-token constant for the token-mean objective (``1.04 nats per token'') is
false; only the bound per labelled token of \cref{cor:token} holds. (f) A generalisation statement
for untrained kinds of invalid input was attempted through a Lipschitz assumption on $\phi$ and
abandoned: it requires assumptions about the representation that nothing in the protocol verifies.

\section{Compute}
\label{sec:supp-compute}
Training and generation ran on one H100 per job, judging on a local machine. Recorded wall-clock
times for the v4.1 study: the grounding projector 2 h 01 min; the stock cell's 1{,}560 generations
4{,}098 s; the grounding and grounding-then-TIDE cells together 4{,}471 s; the TIDE-only and
TIDE-then-grounding cells together, including both training stages, 11{,}799 s; E5 training 2 h 39 min
and its generation 76 min; the stress pilot 34 min on three GPUs and its judging 37 min at
concurrency 8.

%% file: tables/arxiv/supp_protocol.tex
\begin{tabular}{@{}p{0.15\textwidth}p{0.4\textwidth}p{0.4\textwidth}@{}}
\toprule
aspect & earlier pipeline (superseded; no number reused) & protocols v4.1 and v5 (this paper) \\
\midrule
MCQ options & never reached the model (1{,}131 train / 399 test groups) & rendered once, visibly; the MCQ target is the option letter \\
multi-image records & cut to the first image (477 train / 121 test groups) & every image at its own marker; all roles count-matched \\
gold answers & collapsed to the first string (9 of 10 TextVQA annotators discarded) & per-domain target policy; scoring keeps every gold string \\
counterfactual images & assigned before the split: 803 of 1{,}000 test groups had their real image in training & content-component split first, counterfactuals after, from a different component \\
decoding & hard cap of 64 tokens (stock and grounding cells truncated on 71 to 95\% of rows) & greedy, cap 3{,}072, frozen after a dev-only audit \\
artifact identity & a file name & a protocol record; a mismatch is a hard failure \\
held-out test & 1{,}000 groups overlapping inspected material & 390 groups packed from whole uninspected components; never inspected by content \\
\bottomrule
\end{tabular}

%% file: tables/arxiv/supp_split.tex
\begin{tabular}{@{}p{0.5\linewidth}p{0.47\linewidth}@{}}
\toprule
fact & value \\
\midrule
protocol versions & v4.1 (measured study); v5.0 (frozen 2026-10-09; runs in progress) \\
eligible groups (five recorded exclusions) & 4{,}746 \\
content components (shared real-image content or identical question string) & 3{,}141 \\
v4.1: train / dev / final test groups & 3{,}000 / 484 / 390 (39 per domain; 1{,}560 rows) \\
v4.1: TIDE training examples (three-condition) & 9{,}000 (fingerprint \texttt{a2a16a8fe2eb}) \\
v4.1: split / counterfactual / final-test fingerprints & \texttt{be1e7673f79a} / \texttt{v4.0:1f894eacffbd} / \texttt{a84e616b5c82} \\
v5: train groups (grounding probe inside) & 3{,}000 (250); fingerprint \texttt{9fc0b02aad22} \\
v5: TIDE training examples (four conditions) & 12{,}000; $H(B)=1.0397$ nats, $H(B\mid Q,V)=0$ (verified on the realised set) \\
v5: eval tier & 500 groups; fingerprint \texttt{0eb43b470f71} \\
v5: full held-out benchmark & 1{,}000 groups, 4{,}000 rows; fingerprint \texttt{94f54eb34ca5}; held out from training; 712 groups (585 by key) appear in an earlier evaluation set, result file or document (pre-v5 inspection rule, manifest \texttt{n\_test\_inspected\_pre\_v5}), of which 388 are v4.1 test anchors and 324 were evaluated under the superseded v3 pipeline \\
v5: fresh sensitivity holdout & 288 groups, 1{,}152 rows, nine domains (MicroVQA absent); fingerprint \texttt{f312268efaf0}; anchor and counterfactual components in no pre-v5 artifact \\
v5: v4.1 comparison subset & 390 groups (the v4.1 final test), 1{,}560 rows; descriptive only \\
v5: reserve / excluded & 246 / 5 \\
v5: split / counterfactual fingerprints & \texttt{b6cfc0d501d9} / \texttt{v4.0:46db3c6d2040} \\
v5: frozen pre-v5 inspection list (defines freshness; not an exclusion from the benchmark) & 1{,}170 keys in 979 components (\texttt{9ec78ec660c6}); rule: a group is listed if its key, real-image content or component appears in any evaluation set, result file, judge output, debug trace, audit, document or paper source before 2026-10-09 \\
v4.1: inspection list used to pack the final test & 736 keys in 639 components (\texttt{992e3a75f73c}); overlap with the final test 0 \\
\bottomrule
\end{tabular}

%% file: tables/arxiv/tab_textsuff.tex
\begin{tabular}{@{}lrl|ccc|ccc|rr@{}}
\toprule
& & & \multicolumn{3}{c|}{blind responses (model mean)} & \multicolumn{3}{c|}{paired complete cases (model mean)} & \multicolumn{2}{c}{$\ge 1$ model correct blind} \\
domain & $n$ & formats & answer & decline & empty & blind acc. & image acc. & gain & observed & guessing ref. \\
\midrule
Chest radiographs & 451 & 251 yes/no, 200 OE & 0.48 & 0.36 & 0.16 & 0.31 & 0.67 & 0.36 & 276 (0.61) & 0.54 \\
DocVQA & 500 & 494 OE, 6 yes/no & 0.39 & 0.43 & 0.19 & 0.12 & 0.89 & 0.77 & 41 (0.08) & 0.01 \\
InfoVQA & 500 & 500 OE & 0.61 & 0.23 & 0.16 & 0.34 & 0.88 & 0.54 & 166 (0.33) & 0.00 \\
MedXpertQA-MM & 500 & 500 MC & 0.80 & 0.03 & 0.17 & 0.41 & 0.43 & 0.01 & 304 (0.61) & 0.67 \\
MicroVQA & 500 & 500 MC & 0.86 & 0.00 & 0.14 & 0.57 & 0.60 & 0.03 & 369 (0.74) & 0.71 \\
MMMU-Pro & 500 & 500 MC & 0.73 & 0.08 & 0.18 & 0.06 & 0.10 & 0.04 & 149 (0.30) & 0.77 \\
PathVQA & 500 & 252 yes/no, 248 OE & 0.49 & 0.37 & 0.14 & 0.28 & 0.46 & 0.18 & 228 (0.46) & 0.49 \\
ReXVQA & 500 & 500 MC & 0.75 & 0.13 & 0.12 & 0.61 & 0.74 & 0.13 & 417 (0.83) & 0.76 \\
TextVQA & 500 & 473 OE, 27 yes/no & 0.33 & 0.49 & 0.18 & 0.09 & 0.84 & 0.75 & 83 (0.17) & 0.05 \\
Video-MMMU & 300 & 280 MC, 20 OE & 0.71 & 0.11 & 0.18 & 0.39 & 0.49 & 0.11 & 157 (0.52) & 0.44 \\
\midrule
all ten & 4{,}751 & & & & & & & & 2{,}190 (0.46) & LOMO 0.38 to 0.46 \\
\bottomrule
\end{tabular}

%% file: tables/arxiv/tab_textsuff_assoc.tex
\begin{tabular}{@{}p{0.63\linewidth}rcc@{}}
\toprule
pairs or cells & $n$ & Pearson $r$ & Spearman $\rho$ \\
\midrule
full benchmark, all 50 pairs (primary) & 50 & -0.41 & -0.43 \\
same, 40 pairs without GPT-5.5 & 40 & -0.69 & -0.72 \\
same, rate conditional on non-empty responses & 50 & -0.67 & -0.72 \\
same, empty responses counted incorrect in the gain & 50 & -0.52 & -0.59 \\
paired gain vs explicit-decline rate, 50 pairs & 50 & +0.66 & +0.70 \\
within multiple-choice cells (25 cells, 5 domains) & 25 & -0.07 & -0.03 \\
within yes/no cells (10 cells, 2 domains) & 10 & -0.48 & -0.43 \\
within open-ended cells (25 cells, 5 domains) & 25 & -0.05 & -0.16 \\
B-Clean subset, 50 pairs (selection-conditioned) & 50 & -0.34 & -0.34 \\
B-Clean subset, 40 pairs without GPT-5.5 & 40 & -0.49 & -0.52 \\
earlier draft, 20 pairs after post hoc exclusions (not used) & 20 & -0.77 & -0.47 \\
\midrule
within GPT-5.5 (79\% empty blind responses), ten domains & 10 & -0.45 & -0.19 \\
within Gemini-3.1-Flash-Lite, ten domains & 10 & -0.86 & -0.90 \\
within Claude-Sonnet-4.6, ten domains & 10 & -0.98 & -0.95 \\
within Llama-4-Scout, ten domains & 10 & -0.41 & -0.46 \\
within Qwen2.5-VL-72B, ten domains & 10 & -0.80 & -0.84 \\
domain means & 10 & -0.87 & -0.90 \\
model means & 5 & +0.57 & +0.10 \\
\bottomrule
\end{tabular}

%% file: tables/arxiv/supp_hparams.tex
\begin{tabular}{@{}p{0.2\textwidth}p{0.77\textwidth}@{}}
\toprule
TIDE (the adapter of TIDE only, E2 stage 2 and E3 Stage A) & LoRA rank 8, $\alpha$ 16, dropout 0, on the language model's q/k/v/o projections; lr $10^{-4}$, 800 steps, accumulation 4; per-example token-mean cross-entropy; v5: 12{,}000 examples from 3{,}000 train groups (real image to the gold answer; blank image to the blank decline; in-domain and out-of-domain wrong images to the shared unrelated decline); v4.1: 9{,}000 examples, out-of-domain never trained \\
Grounding (E1; E3 Stage B) & vision projector only, language model and vision encoder frozen, no answer read; joint softmax over the cosine scores of one anchor against five candidate image sets (v5: real, in-domain wrong, out-of-domain wrong and two further negatives; v4.1: real, in-domain wrong, three in-batch real images), temperature 0.2, batch 16, lr $10^{-4}$, 300 fixed steps, fp32 master weights with a bf16 forward; pool 2{,}750 source groups / 2{,}747 anchors; E1 and E3 Stage B share one negative schedule \\
E2 Grounding then TIDE & the E1 projector, frozen, then a fresh TIDE adapter (no adapter is inherited) \\
E3 TIDE then Grounding & the TIDE adapter (Stage A), then Stage B grounding on the E1 pool and schedule \\
E5 evidence-advantage contrast (v4.1 diagnostic) & projector only, from the released weights; row-wise cross-entropy over $\log p(a\mid q,v_k)-\log p(a\mid q,\emptyset)$ across the five candidates, temperature 1.0, batch 16, 300 fixed steps, fp32 masters \\
generation (every cell) & greedy, \texttt{max\_new\_tokens} 3{,}072, one renderer (options visible once, every image at its marker), pinned fast processor, bf16, seed 0 \\
selection & fixed final step for every stage; no early stopping and no checkpoint selection; the 250-group held-match probe is a diagnostic only \\
\bottomrule
\end{tabular}

%% file: tables/arxiv/supp_judge_record.tex
\begin{tabular}{@{}lccccc@{}}
\toprule
rule-B judge model & agreement & $\kappa$ & precision & recall & $F_1$ \\
\midrule
gpt-5.5 (in sample) & 1.000 & 1.000 & 1.000 & 1.000 & 1.000 \\
gpt-4o-mini & 0.974 & 0.874 & 1.000 & 0.800 & 0.889 \\
\bottomrule
\end{tabular}

%% file: tables/arxiv/supp_decision.tex
\begin{tabular}{@{}l|cc|cccc|cccc|c|c|c@{}}
\toprule
cell & \multicolumn{2}{c|}{real rows (390)} & \multicolumn{4}{c|}{counterfactual rows (1{,}170)} & \multicolumn{4}{c|}{string rule} & exact & judge & agree \\
 & decline & refusal wd. & decline & near & refusal wd. & empty & $F$ & $M$ & $\tfrac14F+\tfrac34M$ & 95\% CI & only & proxy & real / cf. \\
\midrule
E0 Base & 0 & 58 & 0 & 0 & 491 & 0 & 0.000 & 1.000 & 0.750 & [0.750, 0.750] & 0.750 & 0.516 & 383/390, -- \\
TIDE only & 49 & 0 & 1{,}056 & 1 & 0 & 0 & 0.126 & 0.097 & 0.104 & [0.090, 0.118] & 0.104 & 0.104 & 390/390, 1{,}057/1{,}057 \\
E1 Grounding & 0 & 23 & 0 & 0 & 194 & 1 & 0.000 & 1.000 & 0.750 & [0.750, 0.750] & 0.750 & 0.389 & 302/390, -- \\
E2 Grounding then TIDE & 59 & 0 & 1{,}060 & 0 & 0 & 0 & 0.151 & 0.094 & 0.108 & [0.094, 0.122] & 0.108 & 0.108 & 390/390, 1{,}060/1{,}060 \\
E3 TIDE then Grounding & 5 & 1 & 398 & 4 & 1 & 0 & 0.013 & 0.656 & 0.495 & [0.490, 0.501] & 0.497 & 0.177 & 325/390, 402/402 \\
E5 contrast & 0 & 71 & 0 & 0 & 531 & 0 & 0.000 & 1.000 & 0.750 & [0.750, 0.750] & 0.750 & -- & -- \\
\bottomrule
\end{tabular}

%% file: tables/arxiv/supp_u.tex
\begin{tabular}{@{}lrrrrcc@{}}
\toprule
cell & $M_\mathrm{blank}$ & $M_\mathrm{id}$ & $M_\mathrm{od}$ & $M_\mathrm{inv}$ & I-MR & TC-LIA MR \\
\midrule
E0 Base & 273 & 317 & 298 & 888 & 0.7590 & 0.5692 \\
TIDE only (3-cond.) & 1 & 91 & 21 & 113 & 0.0966 & 0.0724 \\
E1 Grounding & 252 & 375 & 378 & 1005 & 0.8590 & 0.6442 \\
E2 Grounding then TIDE & 0 & 92 & 18 & 110 & 0.0940 & 0.0705 \\
E3 TIDE then Grounding & 0 & 365 & 354 & 719 & 0.6145 & 0.4609 \\
E5 evidence-advantage & 256 & 304 & 280 & 840 & 0.7179 & 0.5385 \\
\bottomrule
\end{tabular}

%% file: tables/arxiv/supp_ruleb.tex
\begin{tabular}{@{}llccc|cc|ccc|c@{}}
\toprule
judge model & cell & blank & in-dom. & out-dom. & IA-IMR & IA-FMR & acc. & answered & false abst. & unparseable \\
\midrule
gpt-5.5 (primary) & E0 Base & 0.718 & 0.331 & 0.244 & 0.431 & 0.323 & 0.577 & 0.982 & 0.018 & 0 \\
gpt-5.5 (primary) & TIDE only & 0.003 & 0.208 & 0.054 & 0.088 & 0.066 & 0.559 & 0.874 & 0.126 & 0 \\
gpt-5.5 (primary) & E1 Grounding & 0.631 & 0.869 & 0.897 & 0.799 & 0.599 & 0.377 & 0.774 & 0.226 & 0 \\
gpt-5.5 (primary) & E2 Grounding then TIDE & 0.000 & 0.215 & 0.044 & 0.086 & 0.065 & 0.526 & 0.849 & 0.151 & 0 \\
gpt-5.5 (primary) & E3 TIDE then Grounding & 0.000 & 0.726 & 0.744 & 0.490 & 0.367 & 0.423 & 0.821 & 0.179 & 0 \\
\midrule
gpt-4o-mini (secondary) & E0 Base & 0.580 & 0.526 & 0.374 & 0.493 & 0.370 & 0.528 & 0.941 & 0.059 & 0 \\
gpt-4o-mini (secondary) & TIDE only & 0.003 & 0.213 & 0.051 & 0.089 & 0.067 & 0.564 & 0.867 & 0.133 & 0 \\
gpt-4o-mini (secondary) & E1 Grounding & 0.510 & 0.446 & 0.297 & 0.418 & 0.314 & 0.331 & 0.746 & 0.254 & 0 \\
gpt-4o-mini (secondary) & E2 Grounding then TIDE & 0.000 & 0.218 & 0.046 & 0.088 & 0.066 & 0.526 & 0.841 & 0.159 & 0 \\
gpt-4o-mini (secondary) & E3 TIDE then Grounding & 0.000 & 0.346 & 0.210 & 0.185 & 0.139 & 0.415 & 0.813 & 0.187 & 0 \\
\bottomrule
\end{tabular}

%% file: tables/arxiv/supp_probes.tex
\begin{tabular}{@{}llcccc@{}}
\toprule
cell & role & P1 substantive answer & P2 acknowledges problem & P3 about shown image & P4 describes absent content \\
\midrule
E0 Base & real & 0.982 & 0.082 & 0.949 & 0.867 \\
 & blank & 0.756 & 0.349 & -- & 0.564 \\
 & in-domain & 0.695 & 0.487 & 0.749 & 0.272 \\
 & out-of-domain & 0.595 & 0.659 & 0.741 & 0.164 \\
TIDE only & real & 0.874 & 0.126 & 0.726 & 0.451 \\
 & blank & 0.003 & 0.997 & -- & 0.000 \\
 & in-domain & 0.233 & 0.767 & 0.028 & 0.121 \\
 & out-of-domain & 0.054 & 0.946 & 0.003 & 0.028 \\
E1 Grounding & real & 0.774 & 0.100 & 0.718 & 0.551 \\
 & blank & 0.656 & 0.426 & -- & 0.482 \\
 & in-domain & 0.426 & 0.087 & 0.097 & 0.210 \\
 & out-of-domain & 0.246 & 0.085 & 0.046 & 0.095 \\
E2 Grounding then TIDE & real & 0.849 & 0.151 & 0.695 & 0.431 \\
 & blank & 0.000 & 1.000 & -- & 0.000 \\
 & in-domain & 0.236 & 0.764 & 0.021 & 0.128 \\
 & out-of-domain & 0.046 & 0.954 & 0.005 & 0.010 \\
E3 TIDE then Grounding & real & 0.821 & 0.059 & 0.674 & 0.426 \\
 & blank & 0.000 & 1.000 & -- & 0.000 \\
 & in-domain & 0.328 & 0.254 & 0.026 & 0.156 \\
 & out-of-domain & 0.203 & 0.249 & 0.013 & 0.056 \\
\bottomrule
\end{tabular}

%% file: tables/arxiv/tab_contrasts.tex
\begin{tabular}{@{}lcc@{}}
\toprule
contrast (v4.1 final test, 1{,}170 invalid rows) & $\Delta$ I-MR, judge U [95\% CI] & $\Delta$ IA-IMR, rule B [95\% CI] \\
\midrule
TIDE only minus E0 Base & -0.662 [-0.692, -0.632] & -0.343 [-0.380, -0.308] \\
E2 minus E0 Base & -0.665 [-0.696, -0.636] & -0.344 [-0.381, -0.311] \\
E2 minus TIDE only & -0.003 [-0.018, 0.014] & -0.002 [-0.016, 0.014] \\
E1 Grounding minus E0 Base & +0.100 [0.073, 0.127] & +0.368 [0.338, 0.397] \\
E3 minus TIDE only & +0.518 [0.497, 0.538] & +0.402 [0.376, 0.425] \\
E5 minus E0 Base & -0.041 [-0.068, -0.014] & -- \\
\bottomrule
\end{tabular}

%% file: tables/arxiv/tab_domains.tex
\begin{tabular}{@{}lcccc|cccc|cc|cc@{}}
\toprule
& \multicolumn{4}{c|}{I-MR, judge U} & \multicolumn{4}{c|}{IA-IMR, rule B} & \multicolumn{2}{c|}{in-dom., U} & \multicolumn{2}{c}{false abst.} \\
domain (117 invalid, 39 real rows) & E0 & TIDE & E2 & E3 & E0 & TIDE & E2 & E3 & TIDE & E2 & TIDE & E2 \\
\midrule
Chest radiographs & 0.63 & 0.23 & 0.23 & 0.63 & 0.26 & 0.17 & 0.17 & 0.49 & 0.56 & 0.59 & 0.15 & 0.21 \\
DocVQA & 0.58 & 0.04 & 0.03 & 0.57 & 0.14 & 0.03 & 0.01 & 0.32 & 0.08 & 0.08 & 0.05 & 0.03 \\
InfoVQA & 0.81 & 0.00 & 0.01 & 0.64 & 0.28 & 0.00 & 0.01 & 0.50 & 0.00 & 0.03 & 0.03 & 0.08 \\
MedXpertQA-MM & 0.89 & 0.09 & 0.08 & 0.61 & 0.64 & 0.09 & 0.08 & 0.49 & 0.23 & 0.18 & 0.10 & 0.13 \\
MicroVQA & 0.92 & 0.10 & 0.09 & 0.56 & 0.83 & 0.10 & 0.09 & 0.49 & 0.31 & 0.28 & 0.08 & 0.08 \\
MMMU-Pro & 0.79 & 0.04 & 0.05 & 0.66 & 0.54 & 0.04 & 0.05 & 0.48 & 0.05 & 0.08 & 0.03 & 0.00 \\
PathVQA & 0.63 & 0.13 & 0.10 & 0.63 & 0.22 & 0.13 & 0.10 & 0.59 & 0.33 & 0.28 & 0.49 & 0.64 \\
ReXVQA & 0.76 & 0.21 & 0.21 & 0.63 & 0.63 & 0.21 & 0.21 & 0.59 & 0.59 & 0.62 & 0.33 & 0.28 \\
TextVQA & 0.78 & 0.03 & 0.00 & 0.64 & 0.23 & 0.01 & 0.00 & 0.48 & 0.05 & 0.00 & 0.00 & 0.08 \\
Video-MMMU & 0.79 & 0.09 & 0.14 & 0.56 & 0.54 & 0.09 & 0.14 & 0.48 & 0.13 & 0.23 & 0.00 & 0.00 \\
\bottomrule
\end{tabular}

%% file: tables/arxiv/supp_stress_full.tex
\begin{tabular}{@{}llcccc|cc|cccc|cc@{}}
\toprule
& & \multicolumn{6}{c|}{rule B, gpt-5.5 (primary)} & \multicolumn{6}{c}{rule B, gpt-4o-mini (secondary)} \\
cell & prompt & IA-IMR & blank & in-dom. & out-dom. & acc. & answered & IA-IMR & blank & in-dom. & out-dom. & acc. & answered \\
\midrule
E0 Base & P0 & 0.457 & 0.720 & 0.360 & 0.290 & 0.550 & 0.970 & 0.503 & 0.560 & 0.580 & 0.370 & 0.460 & 0.950 \\
 & P1 & 0.920 & 1.000 & 0.910 & 0.850 & 0.500 & 1.000 & 0.710 & 0.510 & 0.850 & 0.770 & 0.490 & 0.900 \\
 & P2 & 0.340 & 0.530 & 0.340 & 0.150 & 0.450 & 0.980 & 0.380 & 0.410 & 0.510 & 0.220 & 0.390 & 0.960 \\
E1 Grounding & P0 & 0.773 & 0.640 & 0.830 & 0.850 & 0.250 & 0.790 & 0.457 & 0.480 & 0.560 & 0.330 & 0.210 & 0.770 \\
 & P1 & 0.990 & 1.000 & 0.990 & 0.980 & 0.320 & 0.930 & 0.653 & 0.510 & 0.880 & 0.570 & 0.300 & 0.800 \\
 & P2 & 0.683 & 0.440 & 0.800 & 0.810 & 0.330 & 0.800 & 0.383 & 0.300 & 0.520 & 0.330 & 0.280 & 0.730 \\
E2 Grounding then TIDE & P0 & 0.093 & 0.000 & 0.240 & 0.040 & 0.410 & 0.820 & 0.093 & 0.000 & 0.240 & 0.040 & 0.400 & 0.820 \\
 & P1 & 0.127 & 0.100 & 0.240 & 0.040 & 0.400 & 0.770 & 0.117 & 0.070 & 0.240 & 0.040 & 0.400 & 0.770 \\
 & P2 & 0.087 & 0.000 & 0.230 & 0.030 & 0.380 & 0.800 & 0.087 & 0.000 & 0.230 & 0.030 & 0.370 & 0.800 \\
E3 TIDE then Grounding & P0 & 0.513 & 0.010 & 0.780 & 0.750 & 0.340 & 0.850 & 0.223 & 0.010 & 0.380 & 0.280 & 0.310 & 0.840 \\
 & P1 & 0.530 & 0.030 & 0.710 & 0.850 & 0.340 & 0.840 & 0.303 & 0.030 & 0.440 & 0.440 & 0.310 & 0.820 \\
 & P2 & 0.477 & 0.010 & 0.690 & 0.730 & 0.320 & 0.850 & 0.203 & 0.010 & 0.370 & 0.230 & 0.300 & 0.830 \\
\bottomrule
\end{tabular}

%% file: tables/arxiv/supp_diag.tex
\begin{tabular}{@{}p{0.72\linewidth}r@{}}
\toprule
diagnostic (train-owned keys; not a held-out result) & value \\
\midrule
E3 Stage B (grounding after TIDE): teacher-forced loss on the model's own decline targets, 10 train-owned probe keys, before / after & 0.021 / 1.675 \\
same, on the real-image gold targets, before / after & 0.332 / 2.611 \\
E1 Grounding: held-match accuracy on the 250 train-owned probe pairs at the fixed final step & 0.924 \\
E5: contrastive loss on 64 train-owned probe anchors, before / after 300 steps & 0.760 / 0.145 \\
E5: mean evidence advantage of the gold answer (nats), before / after & +4.96 / +12.05 \\
E5: candidate top-1 rate, before / after & 0.750 / 0.938 \\
E5: evidence-positive rate, before / after & 0.625 / 0.984 \\
E5: largest $|\partial\mathcal{L}/\partial\log p(a\mid q,\emptyset)|$ over 300 steps (autograd) & $1.4\times10^{-8}$ \\
\bottomrule
\end{tabular}

%% file: tables/arxiv/supp_qual.tex
\begin{tabular}{@{}p{0.12\linewidth}p{0.41\linewidth}p{0.41\linewidth}@{}}
\toprule
\multicolumn{3}{@{}l}{\emph{DEV group} chest\_vqa / 361, canonical prompt P0, stress-pilot generations} \\
\midrule
input & E0 Base & E2 Grounding then TIDE (v4.1, three-condition) \\
\midrule
real image & The hyperintensities surrounding the aorta in this CT image are likely to be fat deposits. In CT scans, fat appears as bright or hyperintense signals compared to other ... \newline {\scriptsize[answered]} & lymph nodes \newline {\scriptsize[answered]} \\
blank image & The hyperintensities surrounding the aorta in the image you provided are likely due to fat deposits. In MRI images, fat appears as bright signals (hyperintense) compared ... \newline {\scriptsize[mirage]} & I can't answer this question because no image was provided. I need to see the image to answer. \newline {\scriptsize[not mirage]} \\
in-domain wrong & The hyperintensities surrounding the aorta in this CT image are likely to be related to the presence of air or gas. In CT scans, air and gas appear as bright ... \newline {\scriptsize[mirage]} & I can't answer this question from this image - it does not show what the question is asking about. \newline {\scriptsize[not mirage]} \\
out-of-domain wrong & The image you provided appears to be a whiteboard with notes related to a database schema and some HTTP methods. It does not contain any medical images or information ... \newline {\scriptsize[not mirage]} & I can't answer this question from this image - it does not show what the question is asking about. \newline {\scriptsize[not mirage]} \\
\bottomrule
\end{tabular}

%% file: tables/arxiv/tab_cells.tex
\begin{tabular}{@{}p{0.07\textwidth}p{0.17\textwidth}p{0.05\textwidth}p{0.27\textwidth}p{0.2\textwidth}p{0.17\textwidth}@{}}
\toprule
paper id & name & internal tag & historical aliases (never used in this paper) & trainable parameters & measured / pending \\
\midrule
E0 & Base & \texttt{e0} & Stock; stock baseline; E0 (v4.1 draft) & none & v4.1: measured; v5: in progress \\
E1 & Grounding & \texttt{e2} & G; Grounding only; E2 (v4.1 draft); M2B merger; grounding merger & vision projector (merger) only, 44.57M parameters, fp32 master weights with bf16 forward & v4.1: measured; v5: in progress \\
E2 & Grounding then TIDE & \texttt{e3} & G+TIDE; Grounding then SC-SFT; E3 (v4.1 draft); TIDE (offline\_analyses.py) & stage 1: projector & v4.1: measured; v5: in progress \\
E3 & TIDE then Grounding & \texttt{e4} & TIDE+G; SC-SFT then Grounding; Reverse; E4 (v4.1 draft) & Stage A: LoRA rank 8 on q/k/v/o & v4.1: measured; v5: in progress \\
TIDE-only & TIDE only (appendix ablation; $=$ E3 Stage A) & \texttt{e1} & TIDE (2026-09-18 draft); SC-SFT (2026-10-09 draft, V2.1); E1 (v4.1 draft); C1 abstain SFT; support-completion SFT only & LoRA rank 8 / alpha 16 on q/k/v/o & v4.1: measured; v5: in progress \\
E5 (v4.1 diagnostic) & Evidence-advantage contrast (proposal Component 2, literal) & \texttt{e5} & TIDE Component 2; C2 (legacy) & projector only & v4.1: measured; v5: not in the v5 grid \\
\bottomrule
\end{tabular}

%% file: tables/arxiv/tab_experiments.tex
\begin{tabular}{@{}p{0.17\textwidth}p{0.05\textwidth}p{0.11\textwidth}p{0.6\textwidth}@{}}
\toprule
experiment & priority & status (2026-10-10) & scope \\
\midrule
EXP-GRID-V41 & P0 & MEASURED & three-condition predecessor of TIDE (out-of-domain not trained); superseded for the v5 method but measured and traceable; reported as the predecessor study \\
EXP-POC-TEXTSUFF & descriptive & MEASURED (API observation, pre-protocol pipeline; descriptive) & question-only correctness of the ten domains (B-Clean union count with row-level format, chance-level and leave-one-model-out references) and the exploratory association between substantive blind answering (denominator: all evaluated questions; empty responses a separate category) and the paired complete-case accuracy gain on the full benchmark; B-Clean subset as a selection-conditioned sensitivity; descriptive context for Sec. 4.1; audited in docs/POC\_V3\_COMPARISON.md \\
EXP-GRID-V5-QWEN & P0 & IN\_PROGRESS & full held-out benchmark (1,000 groups, 4,000 rows) + fresh sensitivity holdout (288) + v4.1 comparison subset (390) \\
EXP-GRID-V5-LLAVA & P0 & IN\_PROGRESS & benchmark + comparison \\
EXP-GRID-V5-INTERNVL & P0 & IN\_PROGRESS & benchmark + comparison \\
EXP-GRID-V5-QWEN3VL & P0 & PLANNED &  \\
EXP-GRID-V5-GEMMA3 & P0 & PLANNED &  \\
EXP-SEEDS & P0 & PLANNED & three seeds of every trainable principal condition; report mean, SD and group-level bootstrap intervals; distinguishes reproducibility from ablation \\
EXP-ABL-A0A7 & P0 & PLANNED & matched optimizer steps (800), LoRA budget and answer-token exposure; A4 = TIDE-only (E3 Stage A); A0 = E0 \\
EXP-BASELINES & P0 & PLANNED & competitive abstention baselines: A5 refusal SFT, a thresholded-confidence selective predictor on the stock model (risk-coverage curve), prompt-based refusal instruction; reported with accuracy-coverage curves \\
EXP-JUDGE-HUMAN & P0 & PLANNED & the existing 38-row in-sample validation is not independent evidence \\
EXP-UTILITY & P0 & MEASURED (v4.1) / IN\_PROGRESS (v5) & real-image accuracy, answered rate, false abstention, accuracy conditional on answering, per domain \\
EXP-CF-AUDIT & P0 & MEASURED (construction) / PLANNED (human audit) & content-level disjointness verified by the v5 preflight (0 cross-tier sources, 0 same-component sources); a human audit of 200 counterfactual pairs for genuine unrelatedness is planned \\
EXP-HARDNESS & P1 & PLANNED & bin in-domain wrong images by visual similarity to the valid image (frozen image encoder cosine) and report the TIDE residual per bin \\
EXP-NEGSAMPLING & P1 & PLANNED & A6-style random negatives versus in-domain hard negatives at matched counts \\
EXP-TARGETS & P1 & PLANNED & distinct (mapping B) versus shared decline strings (A7) \\
EXP-LODO & P1 & PLANNED & leave-one-domain-out: 10 trainings (9 domains each, leakage-safe by component), evaluated on the held-out domain; report real-image utility and invalid-evidence behaviour; distinct from out-of-domain wrong images on seen-domain questions \\
EXP-VISUAL-SWAP & P1 & PLANNED & swap the valid image for its in-domain counterfactual after decoding starts; sensitivity of the answer to the image \\
EXP-MARGIN & P1 & PLANNED & log p(a|q,v\_real) - log p(a|q,v\_wrong) before and after TIDE (correct-versus-wrong evidence margin) \\
EXP-SEPARABILITY & P1 & PLANNED & two probe targets on the same activations (projector outputs and residual stream), per cell: (a) input validity, real versus counterfactual, the target of Proposition 3; (b) the response-level mirage label of Mirage Probes. Decodability of (b) says nothing about (a) without this matched experiment; neither is identified with the other in the text. \\
EXP-ORDER-FORGET & P1 & MEASURED (v4.1 retention diagnostic) / IN\_PROGRESS (v5) & abstain-target loss before and after Stage B of E3 \\
EXP-STRESS-V5 & P1 & PLANNED (pilot v4.1 MEASURED) & P0/P1/P2 on the v5 eval pilot and, if approved, the benchmark; must include the TIDE-only cell \\
EXP-NOIMAGE & P1 & PLANNED & actual image removal versus the black image on the same groups (Phantom-0 comparability) \\
EXP-EXTERNAL-HB & P0 & PLANNED (smokes only) & HallusionBench zero-shot, stock versus TIDE-only versus E2/E3, official grader; distinguish hallucination metrics from mirage metrics and from task accuracy \\
EXP-SENS-LORA & P2 & PLANNED & LoRA rank 4 / 8 / 16 / 32 \\
EXP-SENS-LR & P2 & PLANNED & learning rate 3e-5 / 1e-4 / 3e-4 \\
EXP-SENS-STEPS & P2 & PLANNED & 200 / 400 / 800 / 1,600 steps \\
EXP-SENS-SIZE & P2 & PLANNED & 500 / 1,000 / 2,000 / 3,000 train groups \\
EXP-SENS-MODULES & P2 & PLANNED & q/k/v/o versus q/k/v/o plus MLP; projector unfrozen \\
EXP-SENS-TOKENLOSS & P2 & PLANNED & per-example token mean versus sequence sum \\
EXP-SENS-CORRUPT & P2 & PLANNED & image corruption levels (blur, noise, occlusion) on real images \\
EXP-SENS-PARAPHRASE & P2 & PLANNED & question paraphrases \\
EXP-SENS-RATIO & P2 & PLANNED & counterfactual-to-real ratios 0 / 1 / 2 / 3 per group \\
EXP-SENS-JOINT & P2 & PLANNED & joint Grounding plus TIDE in one stage \\
\bottomrule
\end{tabular}

%% file: tables/arxiv/tab_abl_design.tex
\begin{tabular}{@{}p{0.04\textwidth}p{0.2\textwidth}p{0.38\textwidth}p{0.3\textwidth}@{}}
\toprule
id & cell & composition (matched 800 steps, LoRA rank 8, same decline strings) & hypothesis \\
\midrule
A0 & Stock model & equals E0 & reference: the released weights \\
A1 & Conventional real-image-only SFT & 3,000 real examples only (sft\_mode real\_only), matched to A4 in optimizer steps (800) and LoRA budget & answer-only SFT leaves off-support behaviour unconstrained (Prop.~1) \\
A2 & Real plus blank-image abstention & sft\_mode real\_blank (6,000 examples), 800 steps & blank is trivially separable; nothing rewards declining on wrong images \\
A3 & Real plus in-domain wrong-image abstention & sft\_mode real\_id (6,000 examples), 800 steps & in-domain declining transfers partly to blank and out-of-domain inputs \\
A4 & Full support completion (TIDE, four conditions) & equals TIDE-only (internal e1) under v5 & the four-condition optimum declines on every counterfactual kind (Prop.~2) \\
A5 & Generic refusal-SFT baseline & 3,000 real examples plus 9,000 refusal examples whose inputs are text-only unanswerable or out-of-scope questions with the real image kept (no counterfactual images); same decline strings; 800 steps & a refusal style without counterfactual images is not an evidence check \\
A6 & Matched random-negative SFT & real plus blank plus two random-image negatives drawn uniformly from all train groups regardless of domain (no in-domain hard negative), 12,000 examples, 800 steps & random negatives are mostly out-of-domain, so the in-domain residual stays \\
A7 & Shared abstention target & A4 with one decline string for blank, in-domain and out-of-domain (mapping with two behaviour classes) & one decline string lowers $H(B)$ but leaves the optimum's decisions unchanged \\
\bottomrule
\end{tabular}